\documentclass[journal]{IEEEtran}
\usepackage{times}
\usepackage{epsfig}
\usepackage{graphicx}
\usepackage{amsmath}
\usepackage{amssymb}
\usepackage{cite}
\usepackage{amsmath,amssymb,amsfonts}
\usepackage{textcomp}
\usepackage{subfigure}
\usepackage[linesnumbered,boxed,ruled,commentsnumbered]{algorithm2e}
\usepackage{algpseudocode}  
\usepackage{amsmath}  
\UseRawInputEncoding

\newtheorem{proposition}{Proposition}
\newtheorem{definition}{Definition}

\newtheorem{problem}{Problem}
\newtheorem{proof}{Proof}
\ifCLASSINFOpdf
\else
\fi

\begin{document}

%
\title{ICOS: Efficient and Highly Robust Rotation Search and Point Cloud Registration with Correspondences \thanks{Lei Sun is with the School
of Mechanical and Power Engineering, East China University of Science and Technology, Shanghai 200237, P.R.China; e-mail: (leisunjames@126.com).}}

%

\author{Lei Sun}

\maketitle

\begin{abstract}
Rotation search and point cloud registration are two fundamental problems in robotics and computer vision, which aim to estimate the rotation and the transformation between the 3D vector sets and point clouds, respectively. Due to the presence of outliers, probably in very large numbers, among the putative vector or point correspondences in real-world applications, robust estimation is of great importance. In this paper, we present ICOS (Inlier searching using COmpatible Structures), a novel, efficient and highly robust solver for both the correspondence-based rotation search and point cloud registration problems. Specifically, we (i) propose and construct a series of \textit{compatible structures} for the two problems where various \textit{invariants} can be established, and (ii) design three time-efficient frameworks, the first for rotation search, the second for known-scale registration and the third for unknown-scale registration, to filter out outliers and seek inliers from the invariant-constrained random sampling based on the compatible structures proposed. In this manner, even with extreme outlier ratios, inliers can be sifted out and collected for solving the optimal rotation and transformation effectively, leading to our robust solver ICOS. Through plentiful experiments over standard datasets, we demonstrate that: (i) our solver ICOS is fast, accurate, robust against over 95\% outliers with nearly 100\% recall ratio of inliers for rotation search and both known-scale and unknown-scale registration, outperforming other state-of-the-art methods, and (ii) ICOS is practical for use in multiple real-world applications.
\end{abstract}


%
\IEEEpeerreviewmaketitle

\section{Introduction}

\IEEEPARstart{R}{otation} search and point cloud registration are two cornerstones in robotics and computer vision. The former has found broad applications in attitude estimation~\cite{wahba1965least}, image stitching~\cite{bazin2014globally}, etc, while the latter has been widely applied in 3D scene reconstruction~\cite{blais1995registering,henry2012rgb,choi2015robust}, object localization and recognition~\cite{drost2010model,papazov2012rigid,guo20143d}, medical imaging~\cite{audette2000algorithmic}, Simultaneous Localization and Mapping (SLAM)~\cite{zhang2014loam}, etc.

\subsection{Problems Formulation}

\textbf{Rotation search} aims to find the best rotation that aligns two 3D vector sets in different coordinate frames, which can be formulated as follows.

Given two 3D vector sets $\mathcal{U}=\{\mathbf{u}_i\}_{i=1}^N$ and $\mathcal{V}=\{\mathbf{v}_i\}_{i=1}^N$ where we consider tuple $(\mathbf{u}_i,\mathbf{v}_i)$ as a vector correspondence, the following relation is satisfied:
\begin{equation}\label{Definition-RS}
\boldsymbol{R}\mathbf{u}_i+\boldsymbol{\epsilon}_i=\mathbf{v}_i,
\end{equation}
where $\boldsymbol{R}\in SO(3)$~\cite{tron2015inclusion,briales2017convex} denotes the rotation matrix, and $\boldsymbol{\epsilon}_i\in\mathbb{R}^3$ represents the measurement of Gaussian noise with isotropic covariance $\sigma^2\mathbf{I}_3$ ($\sigma$ is the standard deviation and $\mathbf{I}_3\in\mathbb{R}^{3\times3}$ is the identity matrix) and mean $\mu=0$ w.r.t. the {$i_{th}$} vector correspondence. Thus, the rotation search problem can be mathematically defined as:

\begin{problem}[Rotation Search]\label{Prob-RS}
The rotation search problem under the perturbance of noise can be equivalently formulated as a minimization problem such that
\begin{equation}\label{Optimization}
\underset{\boldsymbol{R}\in\mathbf{SO}3}{\min} \sum_{i=1}^N||\boldsymbol{R}\mathbf{u}_i-\mathbf{v}_i||^2,
\end{equation}
where globally optimal solution $\boldsymbol{\hat{R}}$ is the best rotation desired.
\end{problem}

\begin{figure}[t]
\centering

\footnotesize{(a) Image Stitching using ICOS}

\subfigure{
\begin{minipage}[t]{1\linewidth}
\centering
\includegraphics[width=0.39\linewidth]{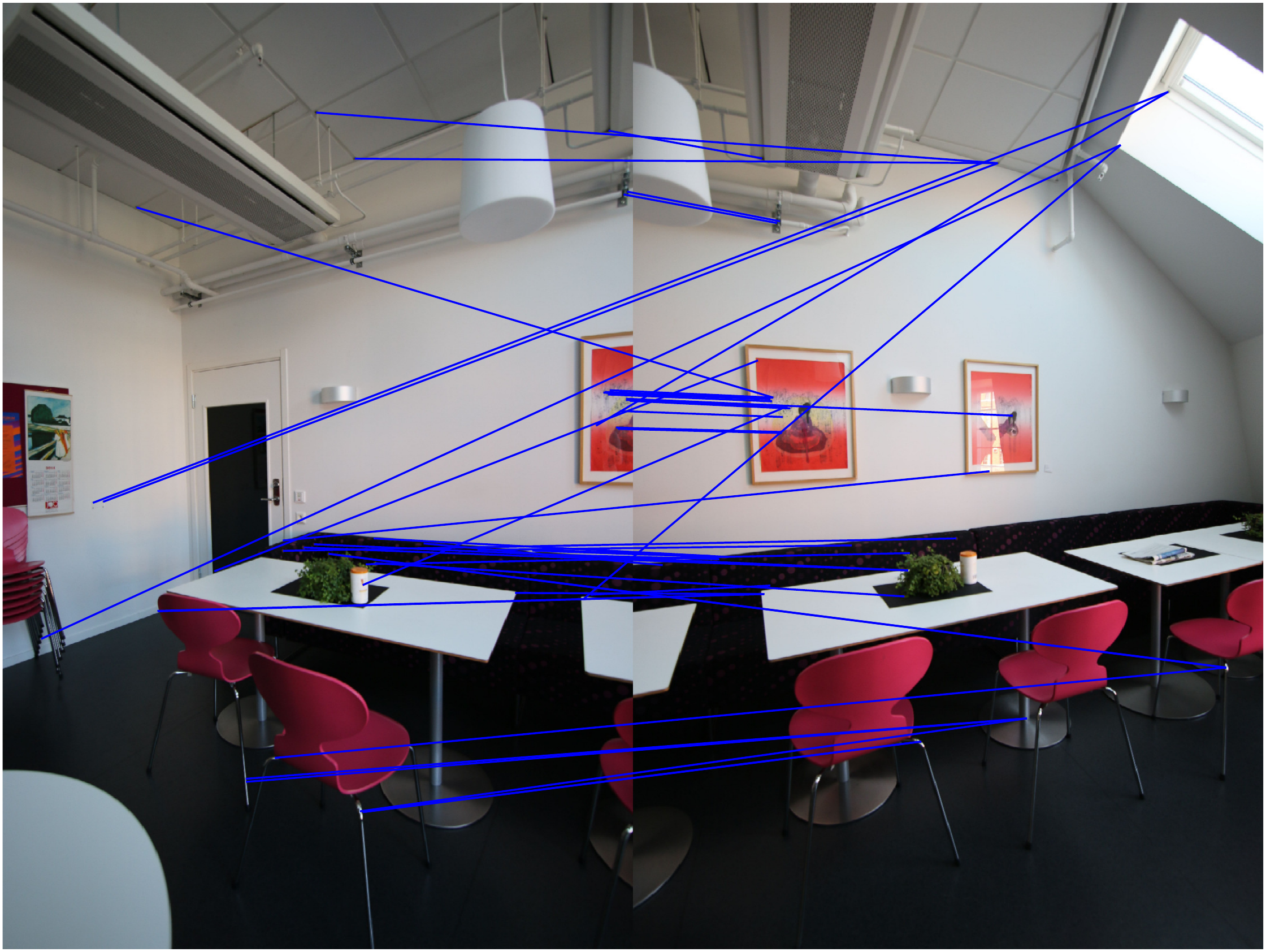}
\includegraphics[width=0.39\linewidth]{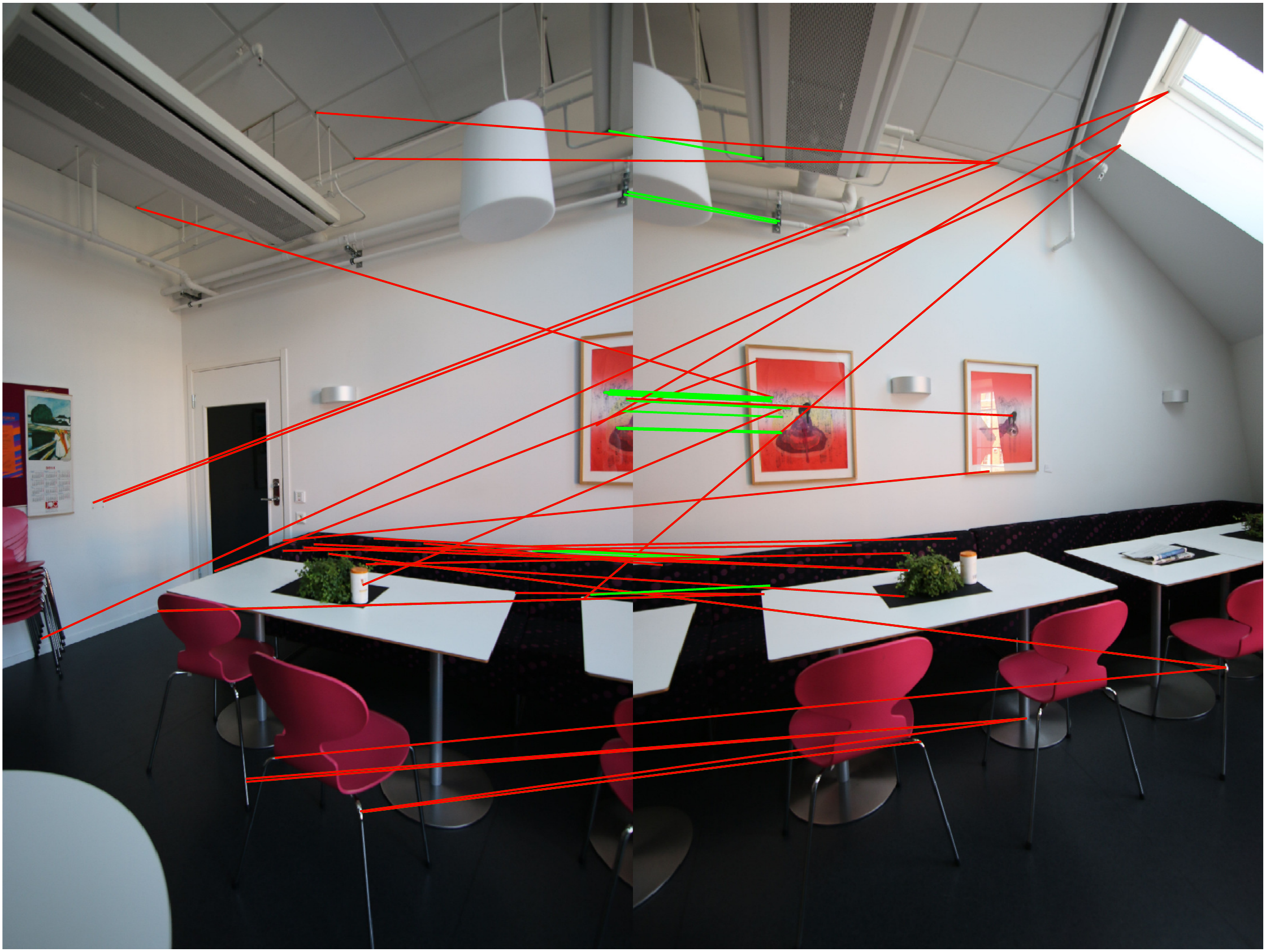}
\includegraphics[width=0.196\linewidth]{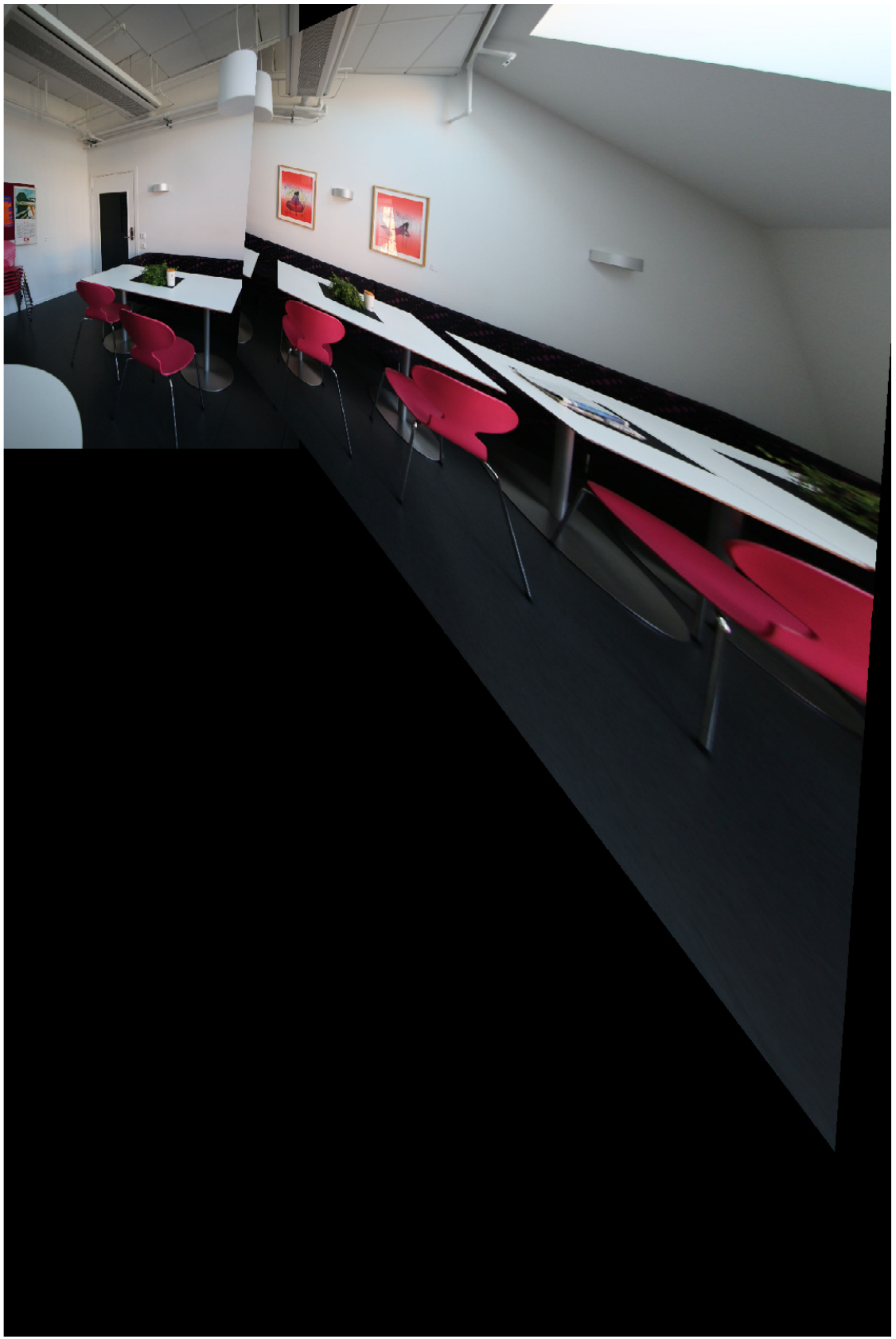}
\end{minipage}
}%

\footnotesize{(b) Unknown-scale Point Cloud Registration using ICOS}

\subfigure{
\begin{minipage}[t]{1\linewidth}
\centering
\includegraphics[width=0.49\linewidth]{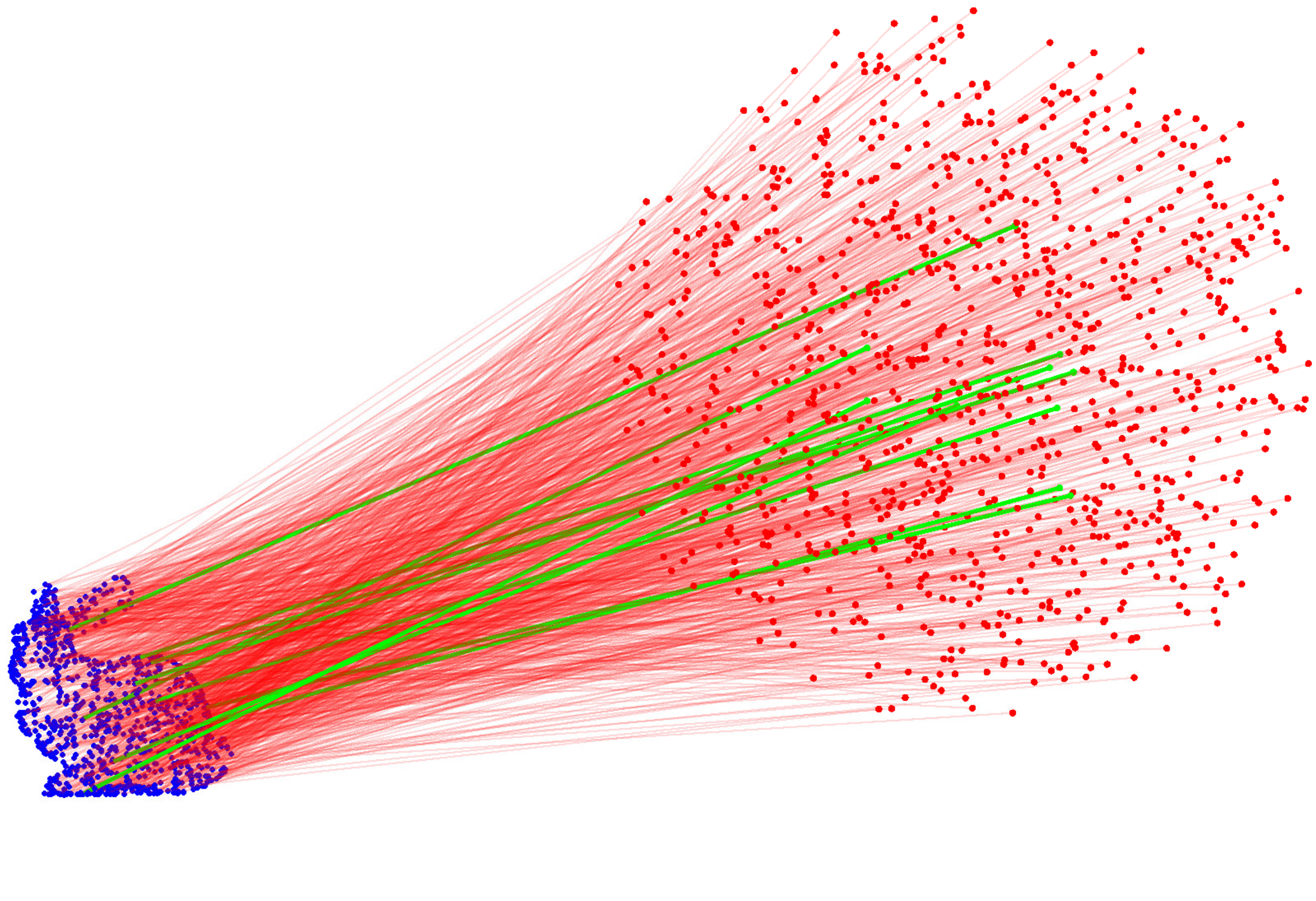}
\includegraphics[width=0.49\linewidth]{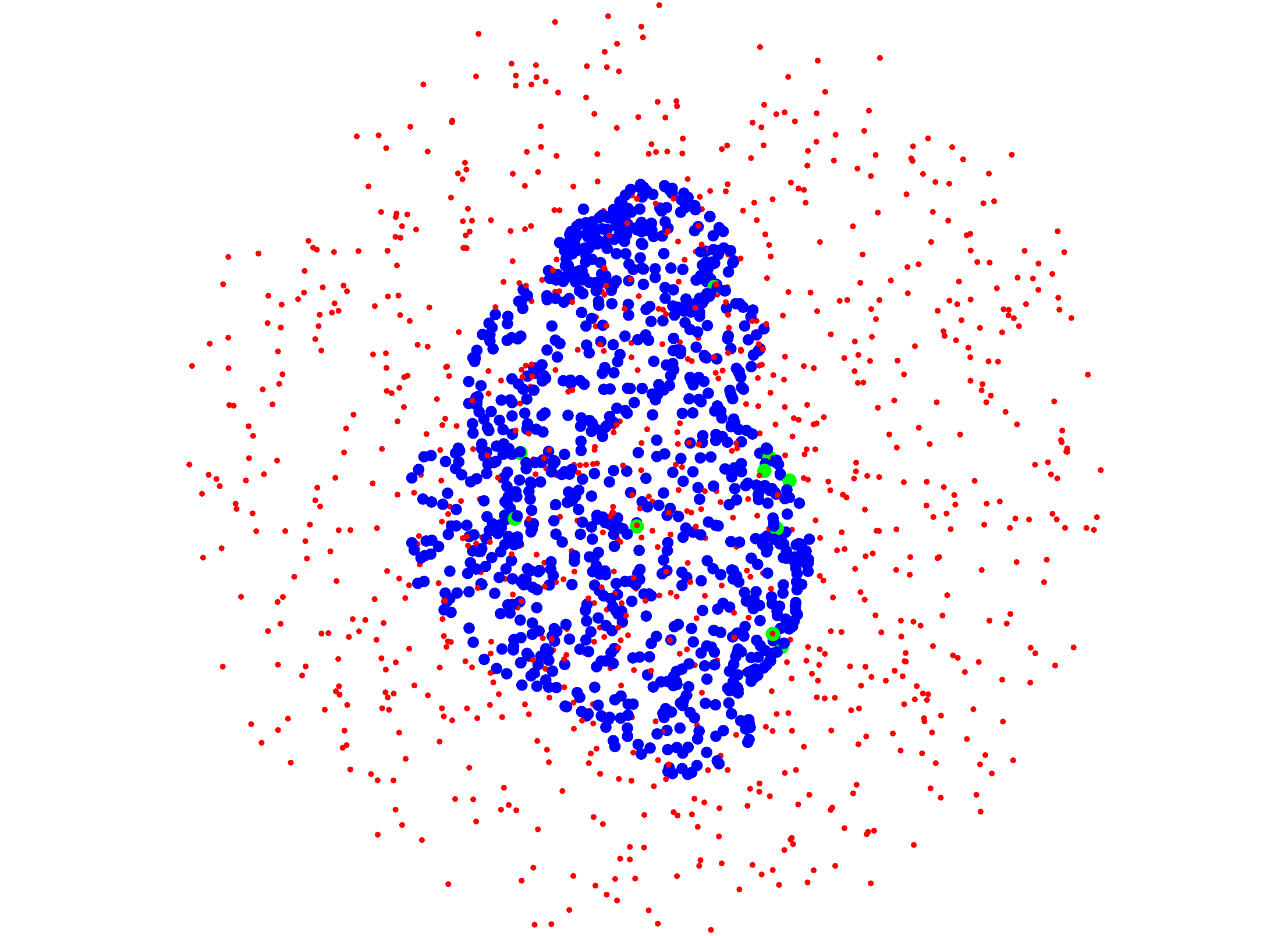}
\end{minipage}
}%

\footnotesize{(c) Object Localization over real RGB-D scene using ICOS}
\subfigure{
\begin{minipage}[t]{1\linewidth}
\centering
\includegraphics[width=0.49\linewidth]{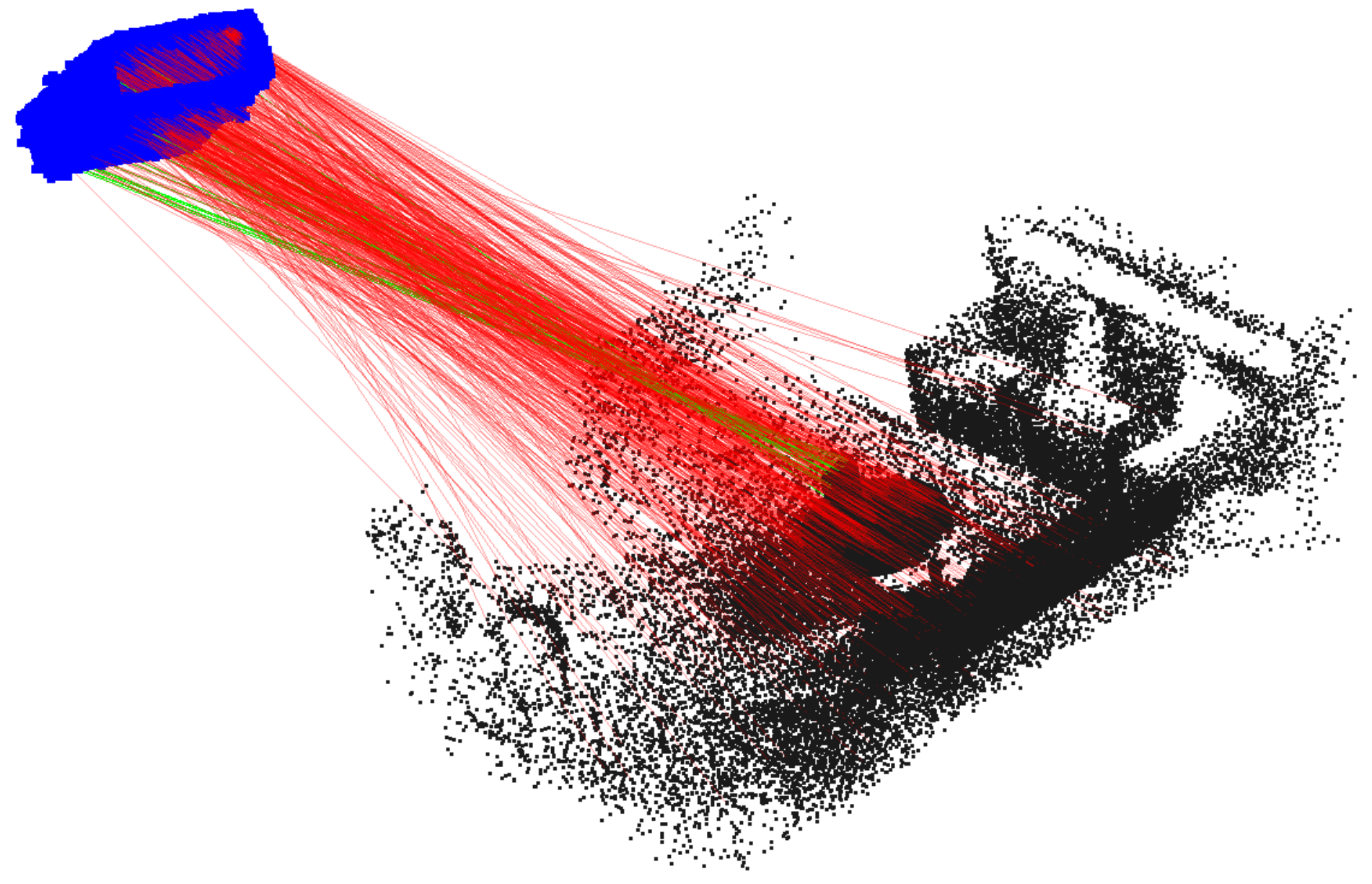}
\includegraphics[width=0.49\linewidth]{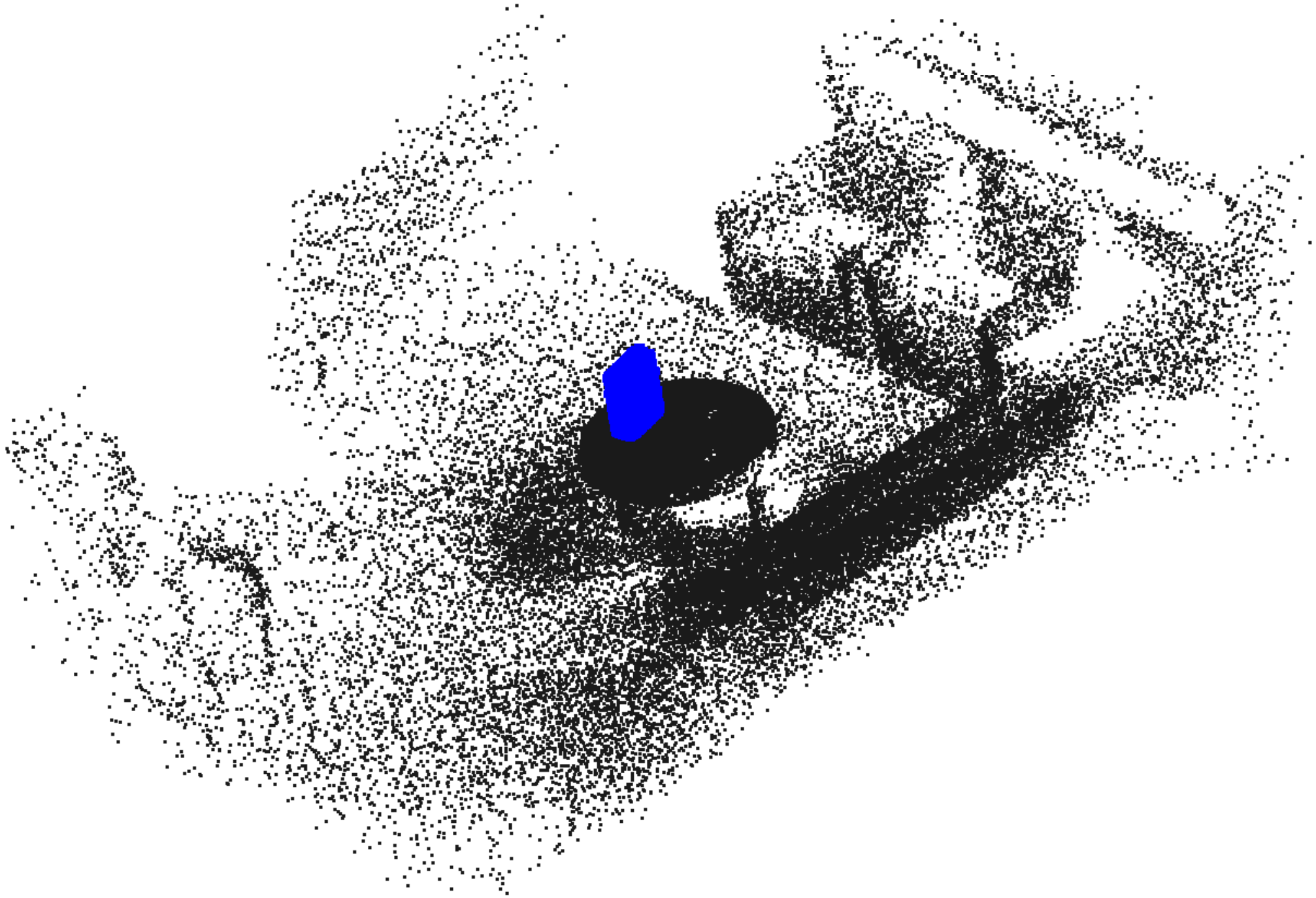}
\end{minipage}
}%

\vspace{-2mm}

\centering
\caption{Rotation Search and Point Cloud Registration using ICOS. (a) Image stitching. Left: Correspondences matched by SURF~\cite{bay2008speeded}. Middle: Inliers found by ICOS (green lines). Right: Stitched images with the rotation estimated by ICOS. The outlier ratio is over 74\% and ICOS can correctly find the inliers. (b) Unknown-scale point cloud registration. Left: Among the 1000 putative correspondences, there exist 990 outliers (red lines) and merely 10 inliers (green lines). Right: ICOS can still estimate the scale, rotation and translation robustly and precisely. (c) Object localization with unknown scale in a realistic RGB-D scene~\cite{zeisl2013automatic}. Left: Correspondences are built between the object and the scene with over 95\% outliers. Right: ICOS can accurately reproject the object back to the scene with the transformation it estimated.}
\label{demo}
\end{figure}

\textbf{Point cloud registration} consists in estimating the transformation (including scale, rotation and translation) between two given 3D point clouds. Now we introduce the correspondence-based point cloud registration problem.

With two 3D point clouds $\mathcal{P}=\{\mathbf{P}_i\}_{i=1}^N$ and $\mathcal{Q}=\{\mathbf{Q}_i\}_{i=1}^N$, where tuple $(\mathbf{P}_i,\mathbf{Q}_i)$ represents a point correspondence, we can have the following relation:
\begin{equation}\label{Definition}
\mathit{s}\boldsymbol{R}\mathbf{P}_i+\boldsymbol{t}+\boldsymbol{\varepsilon}_i=\mathbf{Q}_i,
\end{equation}
where $\mathit{s}>0$ is the scale, $\boldsymbol{R}\in SO(3)$ is the rotation matrix, $\boldsymbol{t}\in\mathbb{R}^3$ is the translation vector, and $\boldsymbol{\varepsilon}_i\in\mathbb{R}^3$ represents the measurement of Gaussian noise similar to that in Problem~\ref{Prob-RS}. The generalized point cloud registration problem can be defined as follows.

\begin{problem}[Generalized Point Cloud Registration]\label{Prob-1}
Solving the generalized point cloud registration problem under the perturbance of noise is equivalent to the following minimization problem:
\begin{equation}\label{Optimization}
\underset{\mathit{s}>0,\boldsymbol{R}\in\mathbf{SO}3, \boldsymbol{t}\in\mathbb{R}^3}{\min} \sum_{i=1}^N||\mathit{s}\boldsymbol{R}\mathbf{P}_i+\boldsymbol{t}-\mathbf{Q}_i||^2,
\end{equation}
where global minimizers $(\mathit{\hat{s}}, \boldsymbol{\hat{R}}, \boldsymbol{\hat{t}})$ are the optimal solutions to this registration problem.
\end{problem}

When the scale aligning the two point clouds is fixed, which means scale $\mathit{s}=1$, we can directly omit $\mathit{s}$ in Problem~\ref{Prob-1}, called the \textit{known-scale} registration problem. However, the scale could vary, which, for instance, exists in the cross-source point cloud registration problems. In this case, the scale needs to be estimated additionally, and these problems are called \textit{unknown-scale} registration.

Our work intends to solve both the known-scale and the unknown-scale registration in the presence of both noise and a great number of outliers among the putative correspondences, providing a highly robust backend for real applications.

\subsection{Motivation and Contributions}

\textbf{Motivation.} The correspondence-based rotation search and point cloud registration approaches are often confronted with one difficulty: in many real-world scenes, the input correspondences are very likely to be corrupted by outliers; for example, the feature matchers (e.g. SIFT~\cite{lowe1999object} and SURF~\cite{bay2008speeded} for 2D images and FPFH~\cite{rusu2008aligning} for 3D point clouds) may generate mismatched or spurious results owing to noise, repetitive patterns, low-textured scenes, etc. Therefore, in real applications, robust estimation solvers should be of paramount significance.

When it comes to robust estimation, RANSAC~\cite{fischler1981random} is a common robust heuristic~\cite{andrew2001multiple}, but it has a critical disadvantage: its runtime increases exponentially with the outlier ratio among the correspondences. Thus, RANSAC could be impractical for use when encountering the extreme outlier ratios (e.g. more than 90\%). Branch-and-Bound (BnB) is able to search in the rotation ($SO(3)$) or transformation ($SE(3)$) group to globally optimize the objective function, whereas its computational cost grows exponentially with the correspondence number; hence, it may be too slow for large-sized problems. 

In addition, the outlier removal methods GORE~\cite{bustos2017guaranteed,parra2015guaranteed} as well as the non-minimal solvers FGR \cite{zhou2016fast}, QUASAR~\cite{yang2019quaternion} and TEASER/TEASER++ \cite{yang2019polynomial,yang2020teaser} have been proposed to deal with the two problems with over 90\% outliers. However, for point cloud registration, these solvers can only succeed when the scale between the two point clouds is known and fixed. When the scale varies (e.g. point clouds formed with different sources), none of these methods could tolerate more than 90\% outliers, which indeed necessitates a general solver capable of solving not only the known-scale problems but the unknown-scale ones. Besides, for handling rotation search problems with high outlier ratios, though GORE~\cite{parra2015guaranteed} is efficient when the problem size is small, it requires more time to remove outliers as the correspondence number increases. Hence, proposing a high-outlier-ratio-tolerant solver that can be efficient regardless of the correspondence number seems rather imperative.

%
%
%
%

\textbf{Contributions.} Motivated by these facts, in this paper, we propose a novel highly robust paradigm, named ICOS (Inlier searching with COmpatible Structures) for addressing rotation search and both known-scale and unknown-scale point cloud registration with extreme outlier ratios efficiently.

To be specific, ICOS aims to seek the correct inliers through the random sample flow (rather than pruning outliers as done in the previous works~\cite{shi2020robin,bustos2017guaranteed}), but significantly different from the RANSAC-family (e.g.~\cite{fischler1981random,chum2003locally,lebeda2012fixing}), the sampling and filtering process of ICOS is built upon and also facilitated by the constraints derived from the \textit{compatible structures} that we define based on the theory of \textit{invariance}.

First, we introduce and construct a series of \textit{compatible structures} for the rotation search and generalized point cloud registration problems, with which we are able to `customize' special invariant-based constraints that can be used for effectively differentiating inliers from outliers with simple Boolean conditions (Section~\ref{COS}).

Moreover, we design three sampling-based frameworks that can fit with the compatible structures to seek and preserve eligible inliers out of the putative correspondences for the two problems, in which different sampling strategies are employed to make the proposed solvers more time-efficient (Section~\ref{ICOS}).

Finally, we validate ICOS in a variety of experiments over different datasets and demonstrate that: (i) ICOS is precise and robust even when over 95\% of the correspondences are outliers in the rotation search, and both the known-scale and unknown-scale registration problems (robust against up to 99\% outliers), (ii) ICOS runs fast, yields accurate estimates and recalls approximately 100\% of the inliers, and (iii) it is also applicable to practical application problems over real datasets (Section~\ref{Experiments}).

\section{Related Work}

In this section, we provide some reviews on: (i) methods for naive and robust rotation search, (ii) methods for naive and robust point cloud registration, (iii) classic robust heuristics for outlier rejection.

\textbf{Rotation Search.} The rotation search (`Wahba') problem was first proposed by Wahba in~\cite{wahba1965least}, in order to estimate the attitude of satellite from the observation of vectors. And it can be solved in closed form based on either rotation matrix~\cite{horn1988closed,markley1988attitude,arun1987least,forbes2015linear} or the quaternion-based parametrization~\cite{horn1987closed,markley2014fundamentals}, although they usually fail in the existence of outliers.

It is known to all that in the process of building the vector correspondences using the feature descriptors~\cite{rusu2008aligning,bay2008speeded} based on 2D or 3D keypoints, it is likely to obtain mismatches, that is, the so-called outliers. Thus, robust solvers are required. 

The BnB framework is effective in finding the optimal rotation from the outlier-corrupted corresponding vectors, and the typical examples include the first globally optimal method searching over the rotation space~\cite{hartley2009global}, the L2-distance-based approach~\cite{bazin2012globally}, and the speeded angular-distance-based solution~\cite{parra2014fast}. In addition to BnB, guaranteed outlier removal~\cite{parra2015guaranteed} is another way to estimate the rotation robustly, and is also much faster than the BnB solvers. Some other robust rotation search solvers are also the subroutines for point cloud registration problems (e.g. \cite{zhou2016fast,bustos2017guaranteed,yang2020teaser}), which will be later discussed in the review on point cloud registration.

\textbf{Point Cloud Registration.} To solve the point cloud registration problem, Iterative Closest Point (ICP)~\cite{besl1992method} is a common local-minimum approach, but it is too dependent on the initial guess of the transformation, thus prone to fail without a good initialization. Therefore, an initialization-free way to address the registration problem\raisebox{0.3mm}{---}to build correspondences between point clouds from 3D keypoints~\cite{tam2012registration} using feature descriptors (e.g. FPFH~\cite{rusu2008aligning}, ISS~\cite{zhong2009intrinsic} and FCGF~\cite{choy2019fully}), has grown increasingly popular. Many correspondence-based methods have been proposed to address the registration problem in the presence of noise. For example, it can be solved using closed-form solutions \cite{horn1987closed,arun1987least} by decoupling the estimation of rotation, scale and translation. Besides, Branch-and-Bound (BnB) has also been applied to solve the registration problem~\cite{olsson2008branch} in a globally optimal way. More recently, the Lagrangian Dual Relaxation is developed to produce a globally optimal generalized registration solver~\cite{briales2017convex} for handling point-to-point, point-to-line and point-to-plane registration with optimality guarantees. 

However, the 3D feature descriptors generally have lower accuracy and robustness compared with 2D ones~\cite{lowe1999object,bay2008speeded}, which could result in relatively high or even extreme outlier ratios in some real-world scenes, as has been discussed in~\cite{bustos2017guaranteed}. And the above-mentioned solutions are fairly sensitive to wrong correspondences and will be fragile once there exist outliers among the correspondences, which consequently demands the following robust registration methods.

The minimal closed-form solutions~\cite{horn1987closed,arun1987least} can be fit into the RANSAC-family frameworks~\cite{fischler1981random,chum2003locally,lebeda2012fixing} to deal with outliers, though may be slow when the outlier ratio is high. Another idea to become robust is that: if outliers can be guaranteed to be removed or eliminated as many as possible, then what still remain should be mostly inliers. GORE~\cite{bustos2017guaranteed} is such a preprocessing outlier removal approach that guarantees that the correspondences removed must not be inliers. In terms of the non-minimal robust solvers, FGR~\cite{zhou2016fast} is the first method using Graduated Non-Convexity (GNC) to reject outliers, which estimates the rigid transformation matrix with a Gauss-Newton method. Though not globally optimal, FGR is efficient for known-scale registration. Currently, TEASER/TEASER++~\cite{yang2019polynomial,yang2020teaser} are the more robust non-minimal certifiably globally optimal solvers that can dispose of 99\% outliers in the known-scale situation. Another recent solution~\cite{li2021point} utilizes RANSAC for scale and translation estimation, and also adopt the GNC framework to solve the rotation. Unfortunately, all these solvers are either incapable of estimating the scale or unable to tolerate over 90\% outliers for unknown-scale point cloud registration (as will be shown in Fig.~\ref{Scale} of Section~\ref{Experiments}).

\textbf{Robust Heuristics.} Apart from the classic hypothesize-and-test robust heuristic RANSAC that runs with minimal solvers to obtain the best consensus set through random sampling, M-estimation is useful for outlier rejection~\cite{huber2004robust,meer2004robust,mittal2012generalized,zhou2016fast}. When the naive solver is non-minimal, the special M-estimator GNC~\cite{yang2020graduated} is a general-purpose heuristic operated by alternating optimization and adopts robust cost functions like Geman-McClure (GM) or Truncated Least Squares (TLS), which can typically handle 70-80\% outliers in many robotic or computer vision problems. In addition, invariants or graph theory can be also applied for robust estimation. The methodologies mainly include: (i) to first prune outliers and then hand the rest of the correspondences over to other heuristics or solvers~(\cite{yang2019polynomial,yang2020teaser,shi2020robin}), and (ii) to first seek a portion of inliers, make an estimate, and then use the optimal solution to find the complete inlier set (our previous work~\cite{sun2021ransic}).

\section{Notation and Preliminaries}

\subsection{Invariants}

We first define the concept of \textit{invariants} in the two problems by adopting the notation and definition in~\cite{shi2020robin}.

\begin{definition}[Invariants]
In one geometric problem, assume we obtain $N$ putative correspondences whose indices are put in set $\mathcal{C}=\{1,2,\dots,N\}$. Let $\mathcal{S}\subset\mathcal{C}$ denote a subset of $\mathcal{C}$ with a fixed size and $\boldsymbol{x}$ denote the vector stacking all variables. Then let $\boldsymbol{y}_{\mathcal{S}}$ represent the measurements w.r.t. subset $\mathcal{S}$, $\boldsymbol{\eta}_{\mathcal{S}}$ represent the corresponding noise measurements which are independent of $\boldsymbol{x}$, and $\boldsymbol{h}_{\mathcal{S}}(\boldsymbol{x},\boldsymbol{\eta}_{\mathcal{S}})$ represent the corresponding model that is used to express the relation between $\boldsymbol{x}$ and $\boldsymbol{y}_{\mathcal{S}}$, given by $\boldsymbol{y}_{\mathcal{S}}=\boldsymbol{h}_{\mathcal{S}}(\boldsymbol{x},\boldsymbol{\eta}_{\mathcal{S}})$.

If we find a function $\boldsymbol{f}(\boldsymbol{y}_{\mathcal{S}})$ where $\boldsymbol{f}(\boldsymbol{y}_{\mathcal{S}})=\boldsymbol{f}\left(\boldsymbol{h}_{\mathcal{S}}(\boldsymbol{x},\boldsymbol{\eta}_{\mathcal{S}})\right)$ is always satisfied and also independent of $\boldsymbol{x}$ no matter which $\mathcal{S}\subset\mathcal{C}$ is selected, we can consider $\boldsymbol{f}$ as an \textit{invariant function}, also abbreviated as an \textit{invariant}.
\end{definition}

\begin{figure*}[t]
\centering
\subfigure{
\begin{minipage}[t]{1\linewidth}
\centering
\includegraphics[width=1\linewidth]{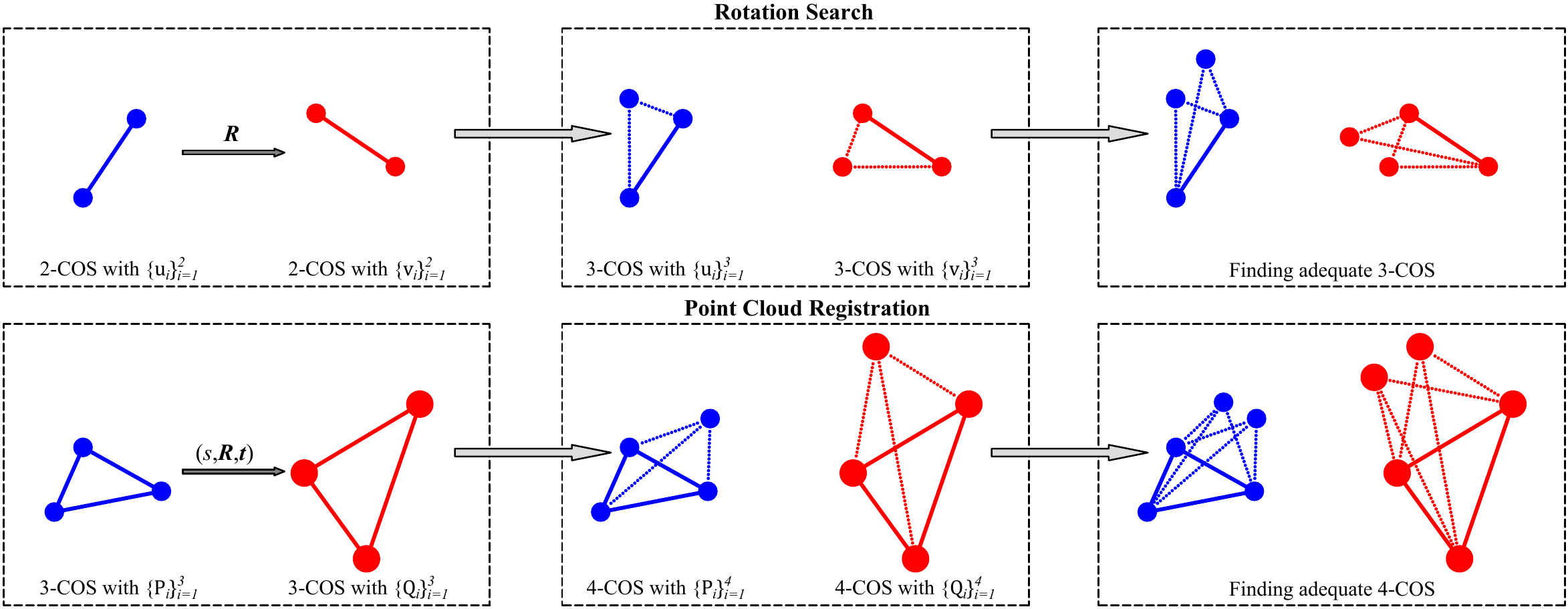}
\end{minipage}
}%

\centering
\caption{Overview of the inlier searching process of ICOS. We start by finding eligible $a$-COS across the two point sets, then seek new correspondences that can establish $b$-COS with the existing $a$-COS, and eventually stop sampling when we collect a sufficient number of eligible $b$-COS.}
\label{Overview}
\end{figure*}

\subsection{Naive Solvers}

In our solver ICOS, the computation of the scale, rotation and translation without outliers is realized by introducing the fast non-minimal naive solvers below.

Assume we have $M$ correspondences (inliers), we first need to acquire corresponding 3D vectors, which, in rotation search, are just the original vector correspondences themselves such that
\begin{equation}
\mathbf{\widetilde{m}}_i=\mathbf{u}_i,
\mathbf{\widetilde{n}}_i=\mathbf{v}_i,
\end{equation}
while in point cloud registration, can be represented with the centroid-based translation-free technique~\cite{arun1987least} such that
\begin{equation}
\mathbf{\widetilde{m}}_i=\mathbf{P}_i-\mathbf{\bar{P}},\,\,
\mathbf{\widetilde{n}}_i=\mathbf{Q}_i-\mathbf{\bar{Q}},
\end{equation}
where $\mathbf{\bar{P}}$ and $\mathbf{\bar{Q}}$ are the centroids in the two frames given by
\begin{equation}
\mathbf{\bar{P}}=\frac{\sum_{i=1}^{M} \mathbf{P}_i}{M},\,\,
\mathbf{\bar{Q}}=\frac{\sum_{i=1}^{M} \mathbf{Q}_i}{M}.
\end{equation}

The optimal rotation can be computed in closed-form using Singular Value Decomposition (SVD) such that

\begin{equation}\label{rot-compute}
\begin{gathered}
\mathbf{{U}}\mathbf{\Sigma}\mathbf{V}^{\top} ={SVD} \left( \sum_{i=1}^{M} \mathbf{\widetilde{m}}_i\mathbf{\widetilde{n}}_i^{\top}\right),\\
\boldsymbol{\hat{R}}=\mathbf{V}\mathbf{{U}}^{\top}.
\end{gathered}
\end{equation}
And in point cloud registration, the best scale can be computed by the weighted representation~\cite{yang2020teaser} such that
\begin{equation}\label{sca-compute}
\mathit{\hat{s}}=\frac{1}{\sum_{i=1}^M \omega_i}\cdot \sum_{i=1}^M \omega_i  \frac{\|\mathbf{\widetilde{n}}_i\|}{\|\mathbf{\widetilde{m}}_i\|},
\end{equation}
where $\omega_i=\frac{||\mathbf{\widetilde{m}}_{i}||^2}{A^2}$. Besides, with the rotation and the scale obtained, the translation can be recovered as:
\begin{equation}\label{tran-compute}
\boldsymbol{\hat{t}}=\mathbf{\bar{Q}}-\mathit{\hat{s}}\boldsymbol{\hat{R}}\mathbf{\bar{P}}.
\end{equation}

\section{Compatible Structures in the Point Cloud Registration Problem}\label{COS}

In this section, we define the \textit{compatible structures} (abbreviated as `COS' for the rest of this paper) for the two problems, and then fully explore the potential invariants among them to establish the constraints for the sampling process.

\subsection{COS in Rotation Search}

The basic structure in rotation search for constrained random sampling is the COS that consists of a pair of correspondences as introduced below.

\begin{definition}[2-COS and Length-based Invariants] Assume we have 2 vector correspondences such that {$(\mathbf{u}_i, \mathbf{v}_i), i=1,2$}, defined as one \textit{2-COS}. We first normalize them all to unit length such that
\begin{equation}
\mathbf{u}^{*}_1=\frac{\mathbf{u}_1}{||\mathbf{u}_1||}, \,
\mathbf{v}^{*}_1=\frac{\mathbf{v}_1}{||\mathbf{v}_1||},\,
\mathbf{u}^{*}_2=\frac{\mathbf{u}_2}{||\mathbf{u}_2||}, \,
\mathbf{v}^{*}_2=\frac{\mathbf{v}_2}{||\mathbf{v}_2||}.
\end{equation}
Then we derive the following pairwise length-based invariants:
\begin{equation}\label{l-in}
{I}^{l}_{12}=\left|\,||\mathbf{{u}}^{*}_1-\mathbf{{u}}^{*}_2||-||\mathbf{{v}}^{*}_1-\mathbf{{v}}^{*}_2||\,\right|=\epsilon^{\mathit{l}}_{12},
\end{equation}
where ${I}^{\mathit{l}}_{12}$ represents the length-based invariant w.r.t. this pair of correspondences, and $\epsilon^{\mathit{l}}_{12}$ denotes the noise measurement (perturbance) on this invariant.
\end{definition}

Intuitively, invariant ${I}^{l}_{12}$ represents the length difference between the third sides of the two triangles formed by the two vector correspondence pairs over the two frames. Without noise perturbance, ${I}^{l}_{12}$ should be strictly equal to $0$ as long as this pair of correspondences are both inliers. But with noise, we can also impose a constraint on ${I}^{l}_{12}$ to examine whether this correspondence pair contain outliers.

\begin{proposition}[Constraints on ${I}^{\mathit{l}}_{12}$]
For the pairwise invariant ${I}^{\mathit{l}}_{ij}$ in the 2-COS defined above, we can derive that
\begin{equation}\label{Constr-l}
{I}^{l}_{12}\leq L,
\end{equation}
where $L$ denotes the bound w.r.t. noise satisfying the inequality such that $ L^*+ ||\boldsymbol{\varepsilon}||\cdot \left(\left\|\frac{1}{\mathbf{v}_1}\right\|+\left\|\frac{1}{\mathbf{v}_2}\right\|\right) \leq L$, in which $L^*=\left\| \frac{\mathbf{u}_1}{||\mathbf{v}_1||}-\frac{\mathbf{u}_2}{||\mathbf{v}_2||}-\frac{\mathbf{u}_1}{||\mathbf{u}_1||}+\frac{\mathbf{u}_2}{||\mathbf{u}_2||}\right\|$.
\end{proposition}
\begin{proof}
Based on triangular inequality and the properties of $\boldsymbol{R}\in SO(3)$, we can always have:
\begin{equation}
\begin{gathered}
\left|\,||\mathbf{{u}}^{*}_1-\mathbf{{u}}^{*}_2||-||\mathbf{{v}}^{*}_1-\mathbf{{v}}^{*}_2||\,\right| \\
=\left| \|\boldsymbol{R}\| \cdot \left\| \frac{\mathbf{u}_1+\boldsymbol{R}^{\top}\boldsymbol{\epsilon}_1}{||\mathbf{v}_1||}-\frac{\mathbf{u}_2+\boldsymbol{R}^{\top}\boldsymbol{\epsilon}_2}{||\mathbf{v}_2||}\right\| -\left\| \frac{\mathbf{u}_1}{\|\mathbf{u}_1\|}-\frac{\mathbf{u}_2}{\|\mathbf{u}_2\|}\right\|  \right|,\\
\leq \, L^{*} +\|\boldsymbol{R}^{\top} \|  \cdot\left\| \frac{\boldsymbol{\epsilon}_1}{||\mathbf{v}_1||}- \frac{\boldsymbol{\epsilon}_2}{||\mathbf{v}_2||} \right\| \\ 
\leq L^{*} + \left\|\frac{\boldsymbol{\epsilon}_1}{||\mathbf{v}_1||}\right\| + \left\| \frac{\boldsymbol{\epsilon}_2}{||\mathbf{v}_2||}\right\| \leq L,
\end{gathered}
\end{equation}
where the inequality~\eqref{Constr-l} can be proved.
\end{proof}

If we have obtained an eligible 2-COS, we can compute the rotation w.r.t. this pair of correspondences using the Horn's triad-based closed-form transformation~\cite{horn1987closed} such that
\begin{equation}\label{horn}
\begin{gathered}
\mathbf{O}_{1}=\left[\begin{array}{ccc} \mathbf{u}_1 & (\mathbf{u}_1 \times \mathbf{v}_1) & \mathbf{u}_1 \times (\mathbf{u}_1 \times \mathbf{v}_1) \end{array}\right], \\
\mathbf{O}_{2}=\left[\begin{array}{ccc} \mathbf{u}_2 & (\mathbf{u}_2 \times \mathbf{v}_2) & \mathbf{u}_2 \times (\mathbf{u}_2 \times \mathbf{v}_2) \end{array}\right], \\
\boldsymbol{R}_{12}=\mathbf{O}_{2}\mathbf{O}_{1}^{\top}=\boldsymbol{R}_{gt}{Exp}(\boldsymbol{\rho}_{12}),
\end{gathered}
\end{equation}
where `$Exp$' is the exponential map~\cite{barfoot2017state} to define the effect of noise on this rotation, which can be represented as $Exp(\boldsymbol{\rho}_{12})$ and $\boldsymbol{\rho}_{12}$ denotes the noise measurement, and $\boldsymbol{R}_{gt}$ means the ground-truth rotation.
Then, we can further try to add a new correspondence to this 2-COS for building a 3-COS, based on which new constraints can be imposed.

\begin{proposition}[3-COS, Invariants and Constraints]\label{PropRS2} Assume that now we obtain a new correspondence, say the $3_{rd}$ correspondence, in addition to the 2-COS we have, as shown in Fig.~\ref{Overview}, we are able to construct a 3-COS with the $3_{rd}$ invariant such that
\begin{equation}\label{l-in-new}
{I}^{\mathit{l}}_{i3}=\left|\,||\mathbf{{u}}^{*}_i-\mathbf{{u}}^{*}_3||-||\mathbf{{v}}^{*}_i-\mathbf{{v}}^{*}_3||\,\right|=\epsilon^{\mathit{l}}_{i3},
\end{equation}
and two additional length-based constraints are satisfied:
\begin{equation}\label{Constr-R.1}
{I}^{\mathit{l}}_{i3} \leq L,
\end{equation}
where $i\in\{1,2\}$. Apart from simply constraining the length-based rigidity, we also supplement a residual-based constraint on $(\mathbf{{u}}_3,\mathbf{{v}}_3)$ such that
\begin{equation}\label{Constr-R.2}
\left\|\boldsymbol{R}_{12}\mathbf{{u}}_3-\mathbf{{v}}_3\right\| \leq G,
\end{equation}
where $G$ is the bound for the residual error w.r.t. the $3_{rd}$ correspondence. Besides, for every single 2-COS in this 3-COS, we can compute a raw rotation, represented as $\boldsymbol{R}_{ij}$ with the correspondence tuple $(i,j)\in\{(1,2),(2,3),(1,3)\}$. In this fashion, we are able to impose new constraints on the mutual compatibility between the raw rotations w.r.t. different correspondence tuples selected from this 3-COS such that
\begin{equation}\label{Constr-R.3}
\angle\left(\boldsymbol{R}_{ij}, \boldsymbol{R}_{ab}\right)\leq F,
\end{equation}
where `$\angle$' is the geodesic distance (error)~\cite{hartley2013rotation} between two rotations, written as
\begin{equation}
\angle\left(\boldsymbol{R}_{ij}, \boldsymbol{R}_{ab}\right)=\left| arccos\left(\frac{tr\left(\boldsymbol{R}_{ij}^{\top}\boldsymbol{R}_{ab}\right)-1}{2}\right)\right|,
\end{equation}
where $F$ is the bound of noise for the geodesic error with $2 \angle \left(Exp(\boldsymbol{\rho}),\mathbf{I}_3 \right)\leq F$, and $\mathbf{I}_3$ means the $3\times3$ identity matrix.
\end{proposition}
\begin{proof}
The proof of constraints~(\ref{Constr-R.1}-\ref{Constr-R.2}) is fairly apparent, thus omitted here. For constraints~\eqref{Constr-R.3}, on the basis of the triangular inequality of the geodesic distance in~\cite{hartley2013rotation}, similar to~\cite{shi2020robin}, we can derive:
\begin{equation}\label{geo-error}
\begin{gathered}
\angle(\boldsymbol{R}_{ij}, \boldsymbol{R}_{ab})=\angle\left(\boldsymbol{{R}}{Exp}(\boldsymbol{\rho}_{ij}), \boldsymbol{{R}} {Exp}(\boldsymbol{\rho}_{ab})\right)\\
=\left| arccos\left(\frac{tr\left({Exp}(\boldsymbol{\rho}_{ij})^{\top}\boldsymbol{R}^{\top}\boldsymbol{R}{Exp}(\boldsymbol{\rho}_{ab})\right)-1}{2}\right)\right|\\
=\angle \left({Exp}(\boldsymbol{\rho}_{ij}), {Exp}(\boldsymbol{\rho}_{ab}) \right) \\ 
\leq \angle \left({Exp}(\boldsymbol{\rho}_{ij}), \mathbf{I}_3 \right) + 
\angle \left({Exp}(\boldsymbol{\rho}_{ab}), \mathbf{I}_3 \right) \leq F.
\end{gathered}
\end{equation}
\end{proof}

As a result, in rotation search, the constraints over the 2-COS and 3-COS have been established, which will be later transformed as a series of Boolean conditions for the sampling and filtering frameworks of ICOS in Section~\ref{ICOS}.

\subsection{COS in Point Cloud Registration}

The basic structure in point cloud registration starts from the COS with 3 correspondences.

\begin{definition}[3-COS and Pairwise Scale-based Invariants] Assume now we obtain 3 non-colinear point correspondences {$(\mathbf{P}_i, \mathbf{Q}_i), i=1,2,3$}, defined as one \textit{3-COS}. We can derive the following pairwise scale-based invariants:
\begin{equation}\label{s-in}
{I}^{\mathit{s}}_{ij}=\frac{||\mathbf{{Q}}_i-\mathbf{{Q}}_j||}{||\mathbf{{P}}_i-\mathbf{{P}}_j||}=\mathit{s}_{gt}+\epsilon_{ij},
\end{equation}
where $(i,j)\in\{(1,2),(1,3),(2,3)\}$, ${I}^{\mathit{s}}_{ij}$ represents the scale-based invariant w.r.t. the $i_{th}$ and the $j_{th}$ correspondences, and $\mathit{s}_{gt}$ denotes the ground-truth scale, and $\epsilon_{ij}$ denotes the noise measurement on the scale.
\end{definition}

Then we can further endow any pair of scale-based invariants ${I}^{\mathit{s}}_{ij}$ with the constraint given as follows.

\begin{proposition}[Constraints on ${I}^{\mathit{s}}_{ij}$]\label{Prop1} For any two invariants ${I}^{\mathit{s}}_{ij}$ in the 3-COS above, we can derive that
\begin{equation}\label{Constr-1.1}
|{I}^{\mathit{s}}_{ij}-{I}^{\mathit{s}}_{jk}| \leq A\left(\frac{1}{||\mathbf{{P}}_i-\mathbf{{P}}_j||}+\frac{1}{||\mathbf{{P}}_j-\mathbf{{P}}_k||}\right)
\end{equation}
where $(i,j,k)\in\{(1,2,3),(2,3,1),(3,1,2)\}$ (here ${I}^{\mathit{s}}_{ij}$ and ${I}^{\mathit{s}}_{ji}$ denote the same invariant), and $A$ is the noise bound, satisfying that $2||\boldsymbol{\varepsilon}||\leq A$. In addition, when the scale is known ($\mathit{s}_{gt}=1$), we can further have that
\begin{equation}\label{Constr-1.2}
|{I}^{\mathit{s}}_{ij}-1| \leq A\left(\frac{1}{||\mathbf{{P}}_i-\mathbf{{P}}_j||}\right).
\end{equation}
\end{proposition}
\begin{proof}
It is easy to have that
\begin{equation}
\mathit{s}_{gt}\boldsymbol{R}_{gt}\mathbf{{P}}_{ij}+\boldsymbol{\varepsilon}_{ij}=\mathbf{{Q}}_{ij}, \Rightarrow \, ||\mathit{s}_{gt}\boldsymbol{R}_{gt}\mathbf{{P}}_{ij}||=||\mathbf{{Q}}_{ij}-\boldsymbol{\varepsilon}_{ij}||,
\end{equation}
where $\mathbf{{P}}_{ij}=\mathbf{{P}}_i-\mathbf{P}_j$, $\mathbf{{Q}}_{ij}=\mathbf{{Q}}_i-\mathbf{Q}_j$, and $\boldsymbol{\varepsilon}_{ij}=\boldsymbol{\varepsilon}_i-\boldsymbol{\varepsilon}_j$.
And using the triangular inequality, we can derive that
\begin{equation}
\begin{gathered}
||\mathbf{{Q}}_{ij}||-||\boldsymbol{\varepsilon}_{ij}|| \leq ||\mathit{s}_{gt}\boldsymbol{R}_{gt}\mathbf{{P}}_{ij}||= \mathit{s}_{gt}||\mathbf{{P}}_{ij}|| \leq ||\mathbf{{Q}}_{ij}||+||\boldsymbol{\varepsilon}_{ij}||,
\end{gathered}
\end{equation}
so that
\begin{equation}
\begin{gathered}
\mathit{s}_{gt}-\frac{||\boldsymbol{\varepsilon}_{ij}||}{||\mathbf{{P}}_{ij}||} \leq \frac{||\mathbf{{Q}}_{ij}||}{||\mathbf{{P}}_{ij}||} \leq \mathit{s}_{gt}+\frac{||\boldsymbol{\varepsilon}_{ij}||}{||\mathbf{{P}}_{ij}||}.
\end{gathered}
\end{equation}
Thus, inequalities~\eqref{Constr-1.1} and \eqref{Constr-1.2} can be derived as long as we have: $||\boldsymbol{\varepsilon}_{ij}||\leq||\boldsymbol{\varepsilon}_{i}||+||\boldsymbol{\varepsilon}_{j}||=2||\boldsymbol{\varepsilon}||\leq A$.
\end{proof}

Subsequently, we roughly calculate the raw scale $\mathit{s}_{123}$ as
\begin{equation}\label{mean-s}
\mathit{{s}}_{123}=\frac{1}{\sum_{1}^3 \omega_{ij} }\cdot \sum_{1}^3 \omega_{ij}  \frac{\|\mathbf{{Q}}_{ij}\|}{\|\mathbf{{P}}_{ij}\|},
\end{equation}
and rotation $\boldsymbol{R}_{123}$ that aligns the 3-COS w.r.t. the two point sets using Horn's 3-point minimal solver~\cite{horn1987closed} similar to~\eqref{horn}. Then we define new invariants based on the translation.

\begin{definition}[Translation-based Invariants] In addition to~\eqref{s-in}, we can derive the translation-based invariants:
\begin{equation}\label{t-in}
{I}^{\boldsymbol{t}}_{i}=\mathbf{Q}_i-\mathit{s}_{123}\boldsymbol{R}_{123}\mathbf{{P}}_i=\boldsymbol{t}_{gt}+\boldsymbol{\zeta}_i,
\end{equation}
where $i\in\{1,2,3\}$, ${I}^{\boldsymbol{t}}_{i}$ represents the translation-based invariant w.r.t. the $i_{th}$ correspondences, and $\boldsymbol{\zeta}_i$ describes its noise measurement. In known-scale cases, we directly set $\mathit{s}_{123}=1$. 
\end{definition}

\begin{proposition}[Constraints on ${I}^{\boldsymbol{t}}_{i}$]\label{Prop2} For invariants ${I}^{\boldsymbol{t}}_{i}$ in the 3-COS above, we can derive that
\begin{equation}\label{Constr-2.1}
||{I}^{\boldsymbol{t}}_{i}-{I}^{\boldsymbol{t}}_{j}|| \leq B
\end{equation}
where $(i,j)\in\{(1,2),(2,3),(3,1)\}$, and $B$ is the noise bound, satisfying that $2||\boldsymbol{\zeta}||\leq B$.
\end{proposition}
\begin{proof}
The proof here is trivial.
\end{proof}

Up to now, we have explored all the possible invariants as well as their constraints according to the rigid compatibility of the 3-COS. On the basis of this 3-COS, whenever we obtain a new correspondence as the $4_{th}$ node, defined as a 4-COS, additional constraints can be further imposed as follows.

\begin{algorithm}[t]
\caption{ICOS for Rotation Search}
\label{Algo0-ICOS}
\SetKwInOut{Input}{\textbf{Input}}
\SetKwInOut{Output}{\textbf{Output}}
\Input{correspondences $\mathcal{N}=\{(\mathbf{u}_i,\mathbf{v}_i)\}_{i=1}^N$; \\ minimum number of the inlier set size $\mathit{X}$\;}
\Output{optimal $\boldsymbol{\hat{R}}$; inlier set $\mathcal{N}^{In}$\;}
\BlankLine
\While {true}{
$\mathcal{N}^{In}\leftarrow \emptyset$, $pass\leftarrow0$\;
\For {$itr_1=1:maxItr_1$}{
{Randomly select $\mathcal{R}=\{(\mathbf{u}_i,\mathbf{v}_i)\}_{i=1}^2\subset\mathcal{N}$}\;
\If{$\mathcal{R}$ satisfies constraint~\eqref{Constr-l}}{
Solve $\boldsymbol{R}_{12}$ with $\mathcal{R}$ using~\eqref{horn}\;
$pass\leftarrow1$\;
\textbf{break}
}
}
\If{$pass=1$} {
\For{$itr_3=1:maxItr_3$}{
\textit{checkSampling}\;
Randomly select $(\mathbf{u}_k,\mathbf{v}_k)\in\mathcal{N}$ and construct a 3-COS with $\mathcal{R}$\;
\If {this 3-COS satisfies~(\ref{Constr-R.1}-\ref{Constr-R.3})}{
$count\leftarrow count+1$\;
$\mathcal{N}^{In}\leftarrow\mathcal{N}^{In}\cup\{(\mathbf{u}_k,\mathbf{v}_k)\}$\;}
\If{$count\geq X$}{
$\mathcal{N}^{In}\leftarrow\mathcal{N}^{In}\cup\mathcal{R}$\;
\textbf{break}}
}
}
}
Solve $\boldsymbol{\hat{R}}$ with set $\mathcal{N}^{In}$ using~\eqref{rot-compute}\;
Compute residuals $r_i$ ($i=1,2,\dots,N$) for all the correspondneces in $\mathcal{N}$; if $r_i\leq 5.2\sigma$, add it to $\mathcal{N}^{In}$\;
Solve optimal $\boldsymbol{\hat{R}}$ with $\mathcal{N}^{In}$ using~\eqref{rot-compute}\;
\Return optimal $\boldsymbol{\hat{R}}$ and the inlier set $\mathcal{N}^{In}$\;
\end{algorithm}

\begin{proposition}[4-COS, Invariants and Constraints] If we have the $4_{th}$ correspondence in addition to the 3-COS we have, we can build 3 more scale-based invariants such that
\begin{equation}\label{s-in-new}
{I}^{\mathit{s}}_{i4}=\frac{||\mathbf{{Q}}_i-\mathbf{{Q}}_4||}{||\mathbf{{P}}_i-\mathbf{{P}}_4||}=\mathit{s}_{gt}+\epsilon_{i4},
\end{equation}
where $i\in\{1,2,3\}$, with the following constraints:
\begin{equation}\label{Constr-3.1}
|{I}^{\mathit{s}}_{i4}-{I}^{\mathit{s}}_{j4}| \leq A\left(\frac{1}{||\mathbf{{P}}_i-\mathbf{{P}}_4||}+\frac{1}{||\mathbf{{P}}_j-\mathbf{{P}}_4||}\right),
\end{equation}
where $i,j\in\{1,2,3\},i\neq j$. Similar to Proposition~\ref{PropRS2}, we also supplement a residual-based constraint on $(\mathbf{{P}}_4,\mathbf{{Q}}_4)$ such that
\begin{equation}\label{Constr-3.2}
\left\|\mathit{s}_{123}\boldsymbol{R}_{123}\mathbf{{P}}_4+{\boldsymbol{t}}_{123}-\mathbf{{Q}}_4\right\| \leq C,
\end{equation}
where ${\boldsymbol{t}}_{123}$ can be computed from the 3-COS as
\begin{equation}\label{mean-t}
{\boldsymbol{t}}_{123}={\bar{I}}^{\boldsymbol{t}}_{123}=\frac{{{I}}^{\boldsymbol{t}}_{1}+{{I}}^{\boldsymbol{t}}_{2}+{{I}}^{\mathit{s}}_{3}}{3},
\end{equation}
and $C$ is the bound for this residual error. For the known-scale cases, ${\bar{I}}^{\mathit{s}}_{123}=1$ is adopted. Moreover, for every single 3-COS in this 4-COS, we can compute a raw rotation, represented as $\boldsymbol{R}_{ijk}$ with the correspondence tuple $(i,j,k)$. Thus, we can establish new constraints on the mutual compatiblity of the different rotations as in Proposition~\ref{PropRS2} such that
\begin{equation}\label{Constr-3.3}
\angle\left(\boldsymbol{R}_{ijk}, \boldsymbol{R}_{abc}\right)\leq D,
\end{equation}
where `$\angle$' is also the geodesic error between two rotations, expressed as
\begin{equation}
\angle\left(\boldsymbol{R}_{ijk}, \boldsymbol{R}_{abc}\right)=\left| arccos\left(\frac{tr\left(\boldsymbol{R}_{ijk}^{\top}\boldsymbol{R}_{abc}\right)-1}{2}\right)\right|,
\end{equation}
$(i,j,k),(a,b,c)\in\{(1,2,3),(1,2,4),(1,3,4),(2,3,4)\}$, and $D$ is the noise bound satisfying $2 \angle \left(Exp(\boldsymbol{\rho}),\mathbf{I}_3 \right)\leq D$.
\end{proposition}
\begin{proof}
The proof is omitted here, because the derivations of constraints~(\ref{Constr-3.1}-\ref{Constr-3.2}) is fairly apparent based on the that of Proposition~\ref{Prop1} and~\ref{Prop2}, while for the derivation of constraint~\eqref{Constr-3.3}, readers can refer to Proposition~\ref{PropRS2}.
\end{proof}

So far, we have derived the invariant-based constraints from both the 3-COS and 4-COS for the point cloud registration problem, getting us prepared for the resulting robust solver.

\begin{algorithm}[t]
\caption{ICOS for Known-scale Registration}
\label{Algo1-ICOS}
\SetKwInOut{Input}{\textbf{Input}}
\SetKwInOut{Output}{\textbf{Output}}
\Input{correspondences $\mathcal{N}=\{(\mathbf{P}_i,\mathbf{Q}_i)\}_{i=1}^N$; \\ minimum number of the inlier set size $\mathit{X}$\;}
\Output{optimal $(\boldsymbol{\hat{R}},\boldsymbol{\hat{t}})$; inlier set $\mathcal{N}^{In}$\;}
\BlankLine
\While {true}{
$\mathcal{N}^{In}\leftarrow \emptyset$\;
\For {$itr_1=1:maxItr_1$}{
{Randomly select $\mathcal{R}=\{(\mathbf{P}_i,\mathbf{Q}_i)\}_{i=1}^2\subset\mathcal{N}$}\;
\If {$\mathcal{R}$ satisfies constraint~\eqref{Constr-1.2}}{
\For {$itr_2=1:maxItr_2$} {
Randomly select $(\mathbf{P}_3,\mathbf{Q}_3)\in\mathcal{N}$ and construct a 3-COS set $\mathcal{O}$ with $\mathcal{R}$\;
\If {$\mathcal{O}$ satisfies constraints~\eqref{Constr-1.2}}{
\textbf{break}
}
}
\textbf{break}
}
}
\If{$\mathcal{O}$ satisfies constraints~(\ref{Constr-1.1}-\ref{Constr-1.2})} {
Solve $\mathit{s}_{123},\boldsymbol{R}_{123}$ with $\mathcal{O}$ using~\eqref{mean-s} and~\eqref{horn}\;
\If{$\mathcal{O}$ satisfies constraints~\eqref{Constr-2.1}}{
\For{$itr_3=1:maxItr_3$}{
\textit{checkSampling}\;
Randomly select $(\mathbf{P}_k,\mathbf{Q}_k)\in\mathcal{N}$ and construct a 4-COS with $\mathcal{O}$\;
\If {this 4-COS satisfies~(\ref{Constr-3.1}-\ref{Constr-3.3})} {
$count\leftarrow count+1$\;
$\mathcal{N}^{In}\leftarrow\mathcal{N}^{In}\cup\{(\mathbf{P}_k,\mathbf{Q}_k)\}$\;}
\If{$count\geq X$}{
$\mathcal{N}^{In}\leftarrow\mathcal{N}^{In}\cup\mathcal{O}$\;
\textbf{break}}
}
}
}
}
Solve $(\boldsymbol{\hat{R}},\boldsymbol{\hat{t}})$ with set $\mathcal{N}^{In}$ using~\eqref{rot-compute} and~\eqref{tran-compute}\;
Compute residuals $r_i$ ($i=1,2,\dots,N$) for all the correspondneces in $\mathcal{N}$; if $r_i\leq 5.2\sigma$, add it to $\mathcal{N}^{In}$\;
Solve optimal $(\boldsymbol{\hat{R}},\boldsymbol{\hat{t}})$ with $\mathcal{N}^{In}$ using~\eqref{rot-compute} and~\eqref{tran-compute}\;
\Return optimal $(\boldsymbol{\hat{R}},\boldsymbol{\hat{t}})$ and the inlier set $\mathcal{N}^{In}$\;
\end{algorithm}

\section{Our Solver: ICOS}\label{ICOS}

In this section, we present our solver ICOS for rotation search and both the known-scale and the unknown-scale point cloud registration problems.

We first provide the insights and the complete overview of the proposed solver. As intuitively illustrated in Fig.~\ref{Overview}, we begin with finding eligible $a$-COS across the two point sets by roughly checking the first layer of invariant-based constraints, where $a=2$ for rotation search and $a=3$ for point cloud registration, as discussed in Section~\ref{COS}.

Once a $a$-COS satisfies the required constraints, we solve the raw solutions (rotation, scale or translation) with it, which is for further examining whether the second layer of constraints can be satisfied. Holding one eligible $a$-COS, we then pick random correspondences one after another with testing both their residual-based and invariant-based constraints simultaneously, in order to construct an eligible $b$-COS, where $b=a+1$, on the basis of this $a$-COS. 

The process of finding $b$-COS continues until we are able to obtain a sufficient number of such $b$-COS within reasonable times of the single-correspondence random sampling, which also serves as a hidden condition aiming to determine whether the $a$-COS and its corresponding $b$-COS ca be all true inliers.

The following subsections provide more details on formulating the pseudocode of ICOS for the specific problems.

\begin{algorithm}[t]
\caption{ICOS for Unknown-scale Registration}
\label{Algo2-ICOS}
\SetKwInOut{Input}{\textbf{Input}}
\SetKwInOut{Output}{\textbf{Output}}
\Input{correspondences $\mathcal{N}=\{(\mathbf{P}_i,\mathbf{Q}_i)\}_{i=1}^N$; \\ minimum number of the inlier set size $\mathit{X}$\;}
\Output{optimal $(\mathit{\hat{s}},\boldsymbol{\hat{R}},\boldsymbol{\hat{t}})$; inlier set $\mathcal{N}^{In}$\;}
\BlankLine
\While {true}{
$\mathcal{N}^{In}\leftarrow \emptyset$\;
Randomly select $\mathcal{O}=\{(\mathbf{P}_i,\mathbf{Q}_i)\}_{i=1}^3\subset\mathcal{N}$ as the 3-COS set\;
\If{$\mathcal{O}$ satisfies constraints~\eqref{Constr-1.1}} {
Solve $\mathit{s}_{123},\boldsymbol{R}_{123}$ with $\mathcal{O}$ using~\eqref{mean-s} and~\eqref{horn}\;
\If{$\mathcal{O}$ satisfies constraints~\eqref{Constr-2.1}}{
\For{$itr_3=1:maxItr_3$}{
\textit{checkSampling}\;
Randomly select $(\mathbf{P}_k,\mathbf{Q}_k)\in\mathcal{N}$ and construct a 4-COS with $\mathcal{O}$\;
\If {this 4-COS satisfies~(\ref{Constr-3.1}-\ref{Constr-3.3})} {
$count\leftarrow count+1$\;
$\mathcal{N}^{In}\leftarrow\mathcal{N}^{In}\cup\{(\mathbf{P}_k,\mathbf{Q}_k)\}$\;}
\If{$count\geq X$}{
$\mathcal{N}^{In}\leftarrow\mathcal{N}^{In}\cup\mathcal{O}$\;
\textbf{break}}
}
}
}
}
Solve $(\mathit{\hat{s}},\boldsymbol{\hat{R}},\boldsymbol{\hat{t}})$ with set $\mathcal{N}^{In}$ using~(\ref{rot-compute}-\ref{tran-compute})\;
Compute residuals $r_i$ ($i=1,2,\dots,N$) for all the correspondneces in $\mathcal{N}$; if $r_i\leq 5.2\sigma$, add it to $\mathcal{N}^{In}$\;
Solve optimal $(\mathit{\hat{s}},\boldsymbol{\hat{R}},\boldsymbol{\hat{t}})$ with $\mathcal{N}^{In}$ using~(\ref{rot-compute}-\ref{tran-compute})\;
\Return optimal $(\mathit{\hat{s}},\boldsymbol{\hat{R}},\boldsymbol{\hat{t}})$ and the inlier set $\mathcal{N}^{In}$\;
\end{algorithm}

\subsection{ICOS for Rotation Search}

For addressing the rotation search problem, we render the pseudocode of ICOS in Algorithm~\ref{Algo0-ICOS} as well as the description for each part as follows.

\textbf{Lines 2-10}: ICOS initiates with random sampling. But different from ordinary random sampling procedures, we impose constraints on each random sample selected. When picking one random subset $\mathcal{R}$ of 2 correspondences, we immediately check the length-based invariant~\eqref{Constr-l}. If it can be satisfied, we solve the raw rotation with this 2-COS $\mathcal{R}$ and then allow it to proceed to the next step, which is operated by setting $pass=1$; if not, we start over from the sampling process.

\textbf{Lines 11-25}: If we manage to obtain an eligible 2-COS, say $\mathcal{R}$, we can then seek new correspondences to build a 3-COS, also by constrained sampling, where we select and test the correspondence samples one at a time. If one correspondence can satisfy constraints~(\ref{Constr-R.1}-\ref{Constr-R.3}), which indicates it can form an eligible 3-COS with $\mathcal{R}$, we add it to the inlier set and make $count$ increase by 1. Once we collect enough correspondences ($count\geq X$), we stop sampling and return the $X$+2 inliers. 

During this process, we supplement a subroutine: \textit{checkSampling}, in order to boost time-efficiency, which is designed to timely break from the sampling if the obtained 2-COS $\mathcal{R}$ does not seem to be inliers, and this operation could be realized by regularly checking the number of 4-COS formed according to the iteration number, as demonstrated in Algorithm~\ref{Algo3-cS}. To be specific, as long as the correspondences in 2-COS $\mathcal{R}$ are both inliers, at least one inlier correspondence ought to be sampled after a sufficient number of iterations, which can be computed from~\eqref{maxitr} by setting a high confidence; conversely, if the latter condition cannot be satisfied, we have every reason to assume that set $\mathcal{R}$ contains at least one outlier.

\textbf{Lines 26-28}: We first compute the rotation with the $X$+2 inliers and consider all correspondences whose residual errors satisfy $r_i\leq 5.2\sigma$ as the ultimate inliers. Finally, we are able to solve the optimal rotation with all the inliers found among the putative correspondences using the non-minimal estimator~\eqref{rot-compute}.

\subsection{ICOS for Known-scale Registration}

The pseudocode of ICOS for point cloud registration with known scale ($\mathit{s}=1$) is given in Algorithm~\ref{Algo1-ICOS}.

\textbf{Lines 2-14}: We design a \textit{decoupled sampling method} for the known-scale registration. First, we select a random subset $\mathcal{R}$ of 2 correspondences, rather than 3, and check if their scale-based invariant $I^\mathit{s}_{12}$ is close enough to 1 using constraint~\eqref{Constr-1.2}. If it is, we can then proceed to select random samples for the remaining one correspondence, with which we try to form an eligible 3-COS that can satisfy constraints~\eqref{Constr-1.2} based on $\mathcal{R}$; if not, we go back to the sampling process for seeking $\mathcal{R}$.

\textbf{Lines 15-32}: Once a qualified 3-COS (set $\mathcal{O}$) is obtained, we check their constraints~\eqref{Constr-1.1}. If satisfied, we then solve the raw rotation with $\mathcal{O}$, and determine whether constraints~\eqref{Constr-2.1} are fulfilled in~$\mathcal{O}$. If yes, we then move on to sampling once again, for building one 4-COS based on set $\mathcal{O}$. We take random correspondences one at a time, and feed them sequentially to $\mathcal{O}$ to check constraints~(\ref{Constr-3.1}-\ref{Constr-3.3}). After collecting enough correspondences, say $X$, which can construct qualified 4-COS with $\mathcal{O}$, we jump out and return the inliers.

\textbf{Lines 33-35}: Similar to Algorithm~\ref{Algo0-ICOS}, we check the residual error w.r.t. each correspondence to gather the entire inlier set, which is then used to compute optimal solutions with the non-minimal solvers~\eqref{sca-compute}, \eqref{rot-compute} and~\eqref{tran-compute}.

\begin{algorithm}[t]\label{Algo3}
\caption{\textit{checkSampling} (Subroutine of ICOS)}
\label{Algo3-cS}
\SetKwInOut{Input}{\textbf{Input}}
\SetKwInOut{Output}{\textbf{Output}}
\Input{iteration number $itr_3$, and $count$\;}
\Output{\textbf{break} or \textit{do nothing}\;}
\BlankLine
\If {$itr_3\geq maxItr_4$ and $count<1$} { 
\textbf{break}
}
\If {$itr_3\geq 2maxItr_4$ and $count<2$} { 
\textbf{break}
}
\If {$itr_3\geq 3maxItr_4$ and $count<3$} { 
\textbf{break}
}
\end{algorithm}
\vspace{-3mm}

\subsection{ICOS for Unknown-scale Registration}

When the scale is unknown, we design a different framework for ICOS, and its pseudocode is shown in Algorithm~\ref{Algo2-ICOS}.

Algorithm~\ref{Algo2-ICOS} employs the methodology similar to \textbf{lines 15-35} of Algorithm~\ref{Algo1-ICOS}. The differences lie in that: (i) we are unable to impose constraints~\eqref{Constr-1.2} on the 3-COS using the decoupled sampling method as operated in \textbf{lines 3-14}, since the scale is not fixed, which may result in more runtime than the known-scale problems generally, and (ii) the raw scale and the scale-based invariants are not set to 1, but need to be computed as in~\eqref{s-in} and~\eqref{s-in-new} instead.

\begin{table}[h]
\caption{Parameter Setup of ICOS for the Rotation Search and Point Cloud Registration Problems}\label{parameter}
\setlength{\tabcolsep}{0.5mm}
\centering
\vspace{-2mm}
\begin{tabular}{c|c|c}

Algorithm~\ref{Algo0-ICOS} & Algorithm~\ref{Algo1-ICOS} & Algorithm~\ref{Algo2-ICOS} \\
\hline 

$X=2\,(N=100)$ & $X=4$  &  $X=4$ \\

$X=4\,(N=500)$ & $A=4.5\sigma$  & $A=4.5\sigma$ \\

$X=5\,(N=1000)$ & $B=5.0\sigma$  & $B=5.0\sigma$ \\

$L=2.5\sigma$ & $C=6.0\sigma$  & $C=6.0\sigma$ \\

$G=4\sigma$ & $D=10.5\sigma$  & $D=10.5\sigma$ \\

$F=10.5\sigma$ & $maxItr_1=40000$  & $maxItr_3=1600$ \\

$maxItr_1=40000$ & $maxItr_2=400$  & $maxItr_4=400$ \\

$maxItr_3=2000$ & $maxItr_3=1600$  &  \\

$maxItr_4=400$ & $maxItr_4=400$  &   \\

\hline

\end{tabular}

\centering
\vspace{-1mm}

\end{table}

\subsection{Discussion on the Maximum Iteration Numbers}

Note that in Algorithm~(\ref{Algo0-ICOS}-\ref{Algo3}), a series of maximum iteration numbers are involved during the sampling process, so now we discuss their proper choice of values.

Generally, the maximum iteration we set can be given by the following formulation:

\begin{equation}\label{maxitr}
maxItr \geq x\cdot\frac{\log(1-p)}{\log\left(1-(1-outlier\,\,ratio)^n\right)},
\end{equation}
where $x$ is the number of eligible inlier subsets desired, $p$ is the confidence for the fact that all the correspondences selected are inliers (typically we choose $p=0.98$-$0.99$), usually the outlier ratio should be within 0-99\% (typically, 99\% is chosen unless we can know in advance that the outlier ratio is below a certain level), and $n$ is the number of correspondences we select for each random sample (subset). 

Therefore, for $maxItr_1$, we adopt $n=2$, choosing two correspondences at one time, while for $maxItr_2$, $maxItr_3$ and $maxItr_4$, we all set $n=1$ for single-correspondence sampling.

\subsection{Discussion on the Performance of ICOS}

Since the main body of our solver ICOS consists in simply judging a series of Boolean conditions (\textit{true} or \textit{false}) according to the constraints over the invariants w.r.t. different variables, it can be fast to implement. Besides, ICOS abandons the strategy of making an estimate for every random sample selected (as operated in RANSAC), and smartly solves the transformation only when the scale-based constraints~(\ref{Constr-1.1}-\ref{Constr-1.2}) allow it to. Also, the residual errors w.r.t. all the correspondences only require to be solved once (just for finding the whole inlier set), saving much time for realistic implementation, especially when the point cloud correspondences are in huge numbers, which is not uncommon in practice.

\begin{figure*}[htp]
\centering

\footnotesize{(a) $N=100$}

\subfigure{
\begin{minipage}[t]{1\linewidth}
\centering
\includegraphics[width=0.246\linewidth]{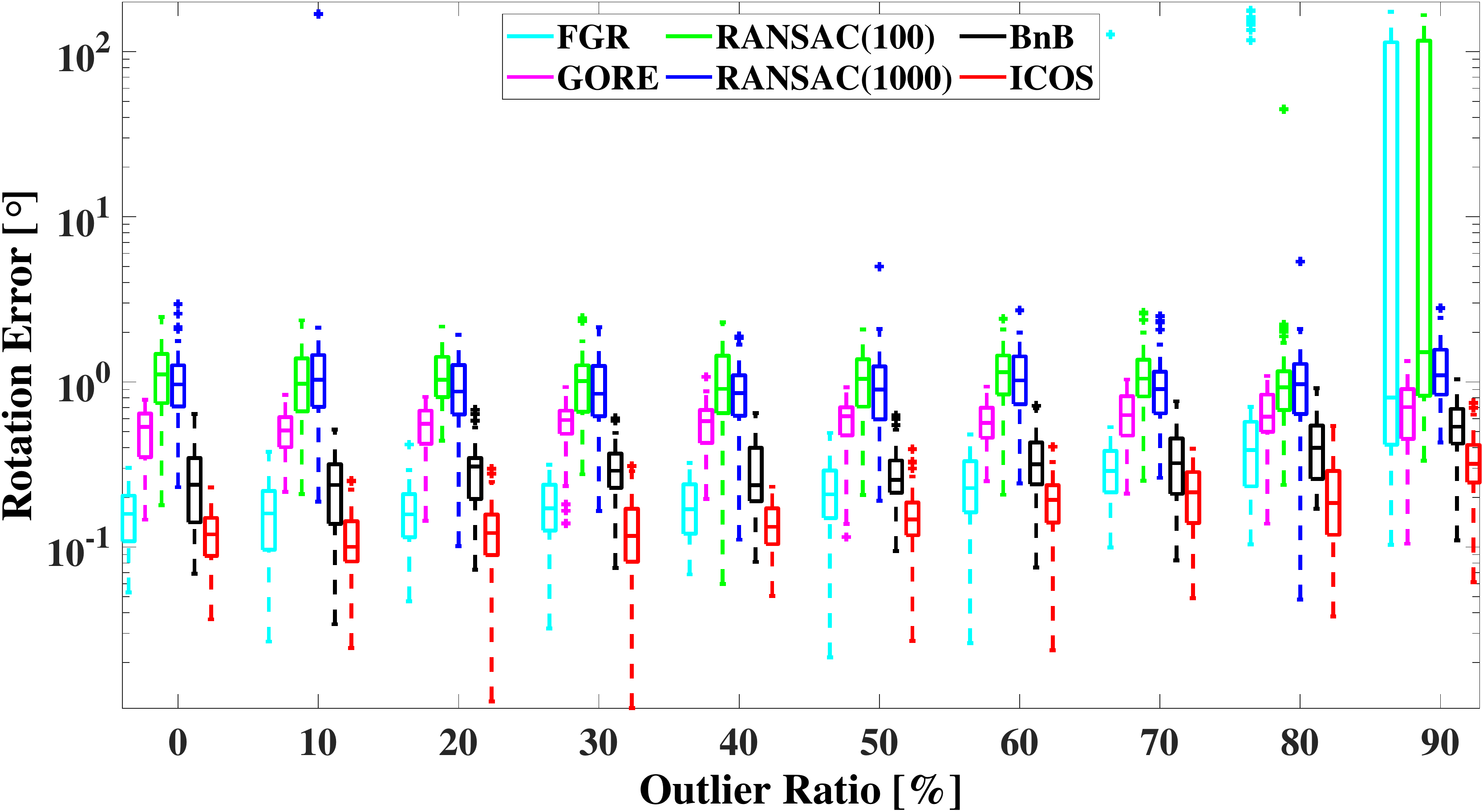}
\includegraphics[width=0.246\linewidth]{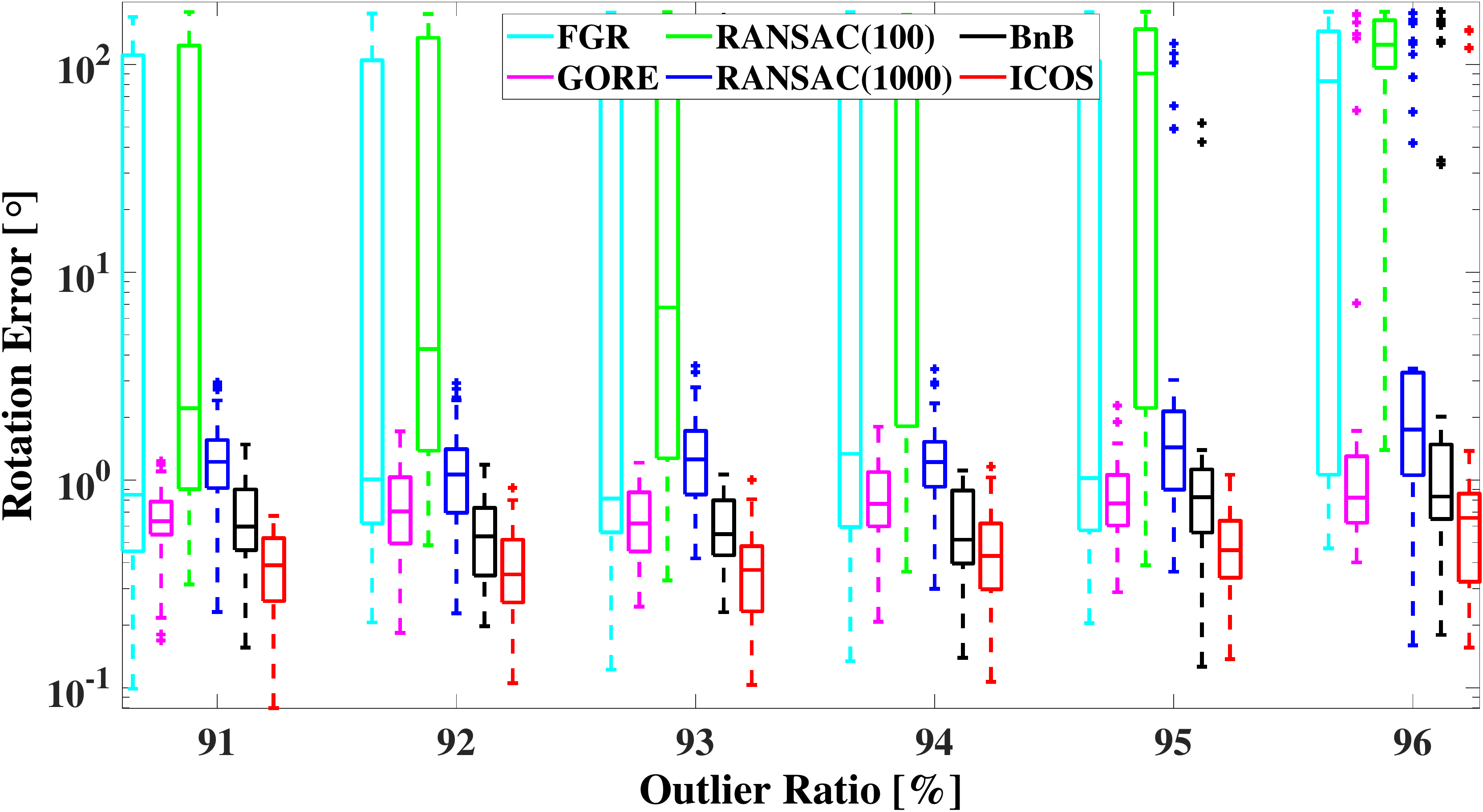}
\includegraphics[width=0.246\linewidth]{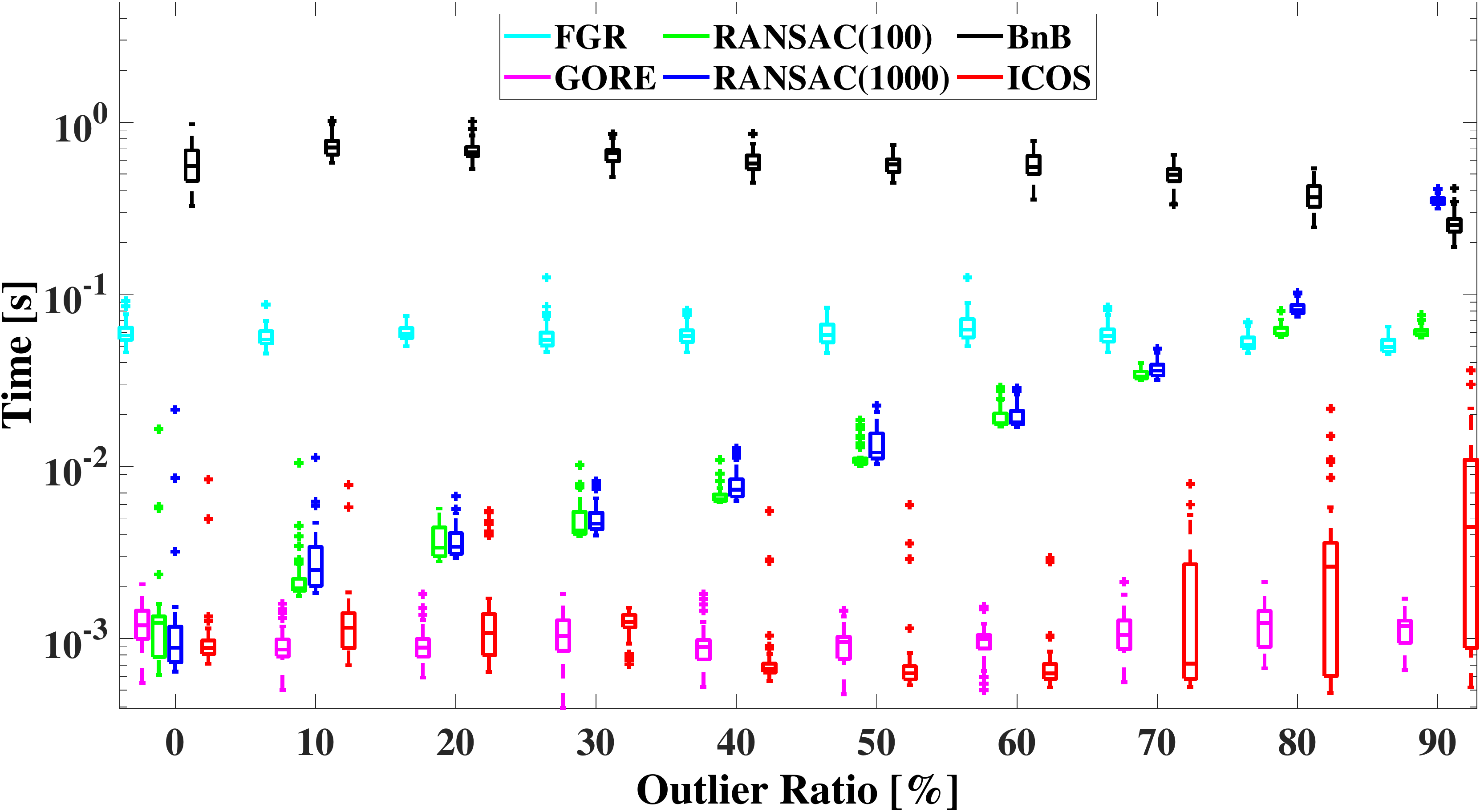}
\includegraphics[width=0.246\linewidth]{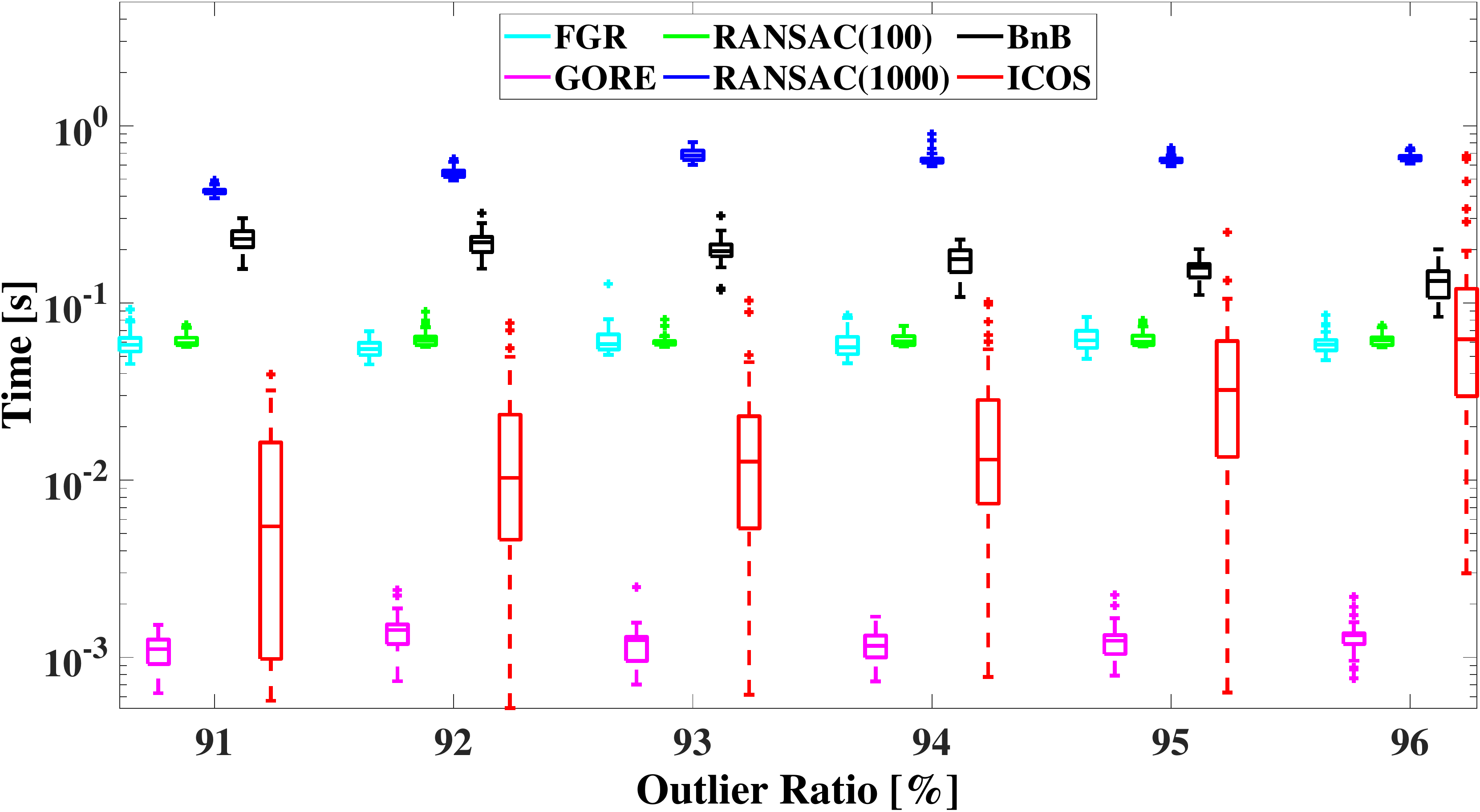}
\end{minipage}
}%

\footnotesize{(b) $N=500$}

\subfigure{
\begin{minipage}[t]{1\linewidth}
\centering
\includegraphics[width=0.246\linewidth]{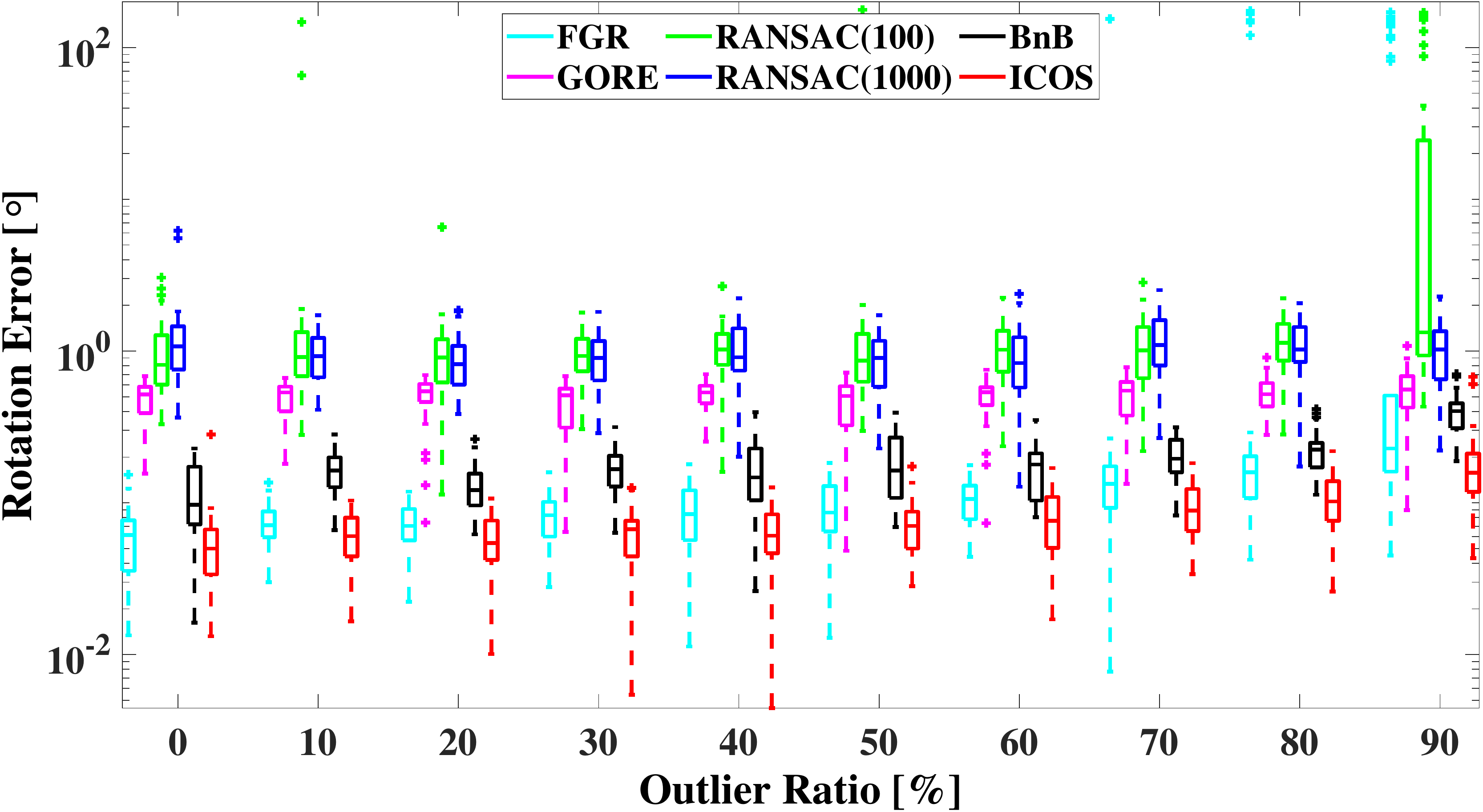}
\includegraphics[width=0.246\linewidth]{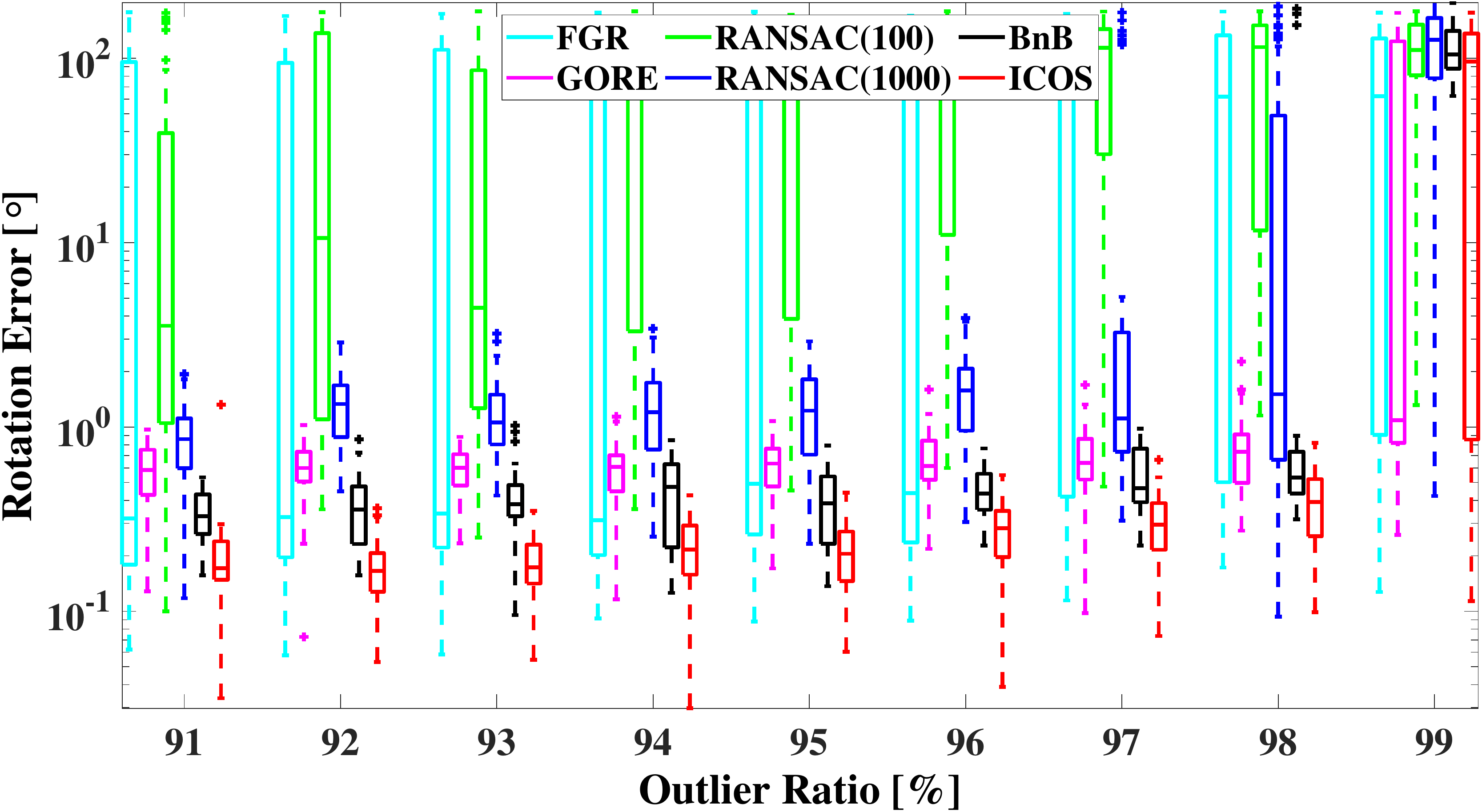}
\includegraphics[width=0.246\linewidth]{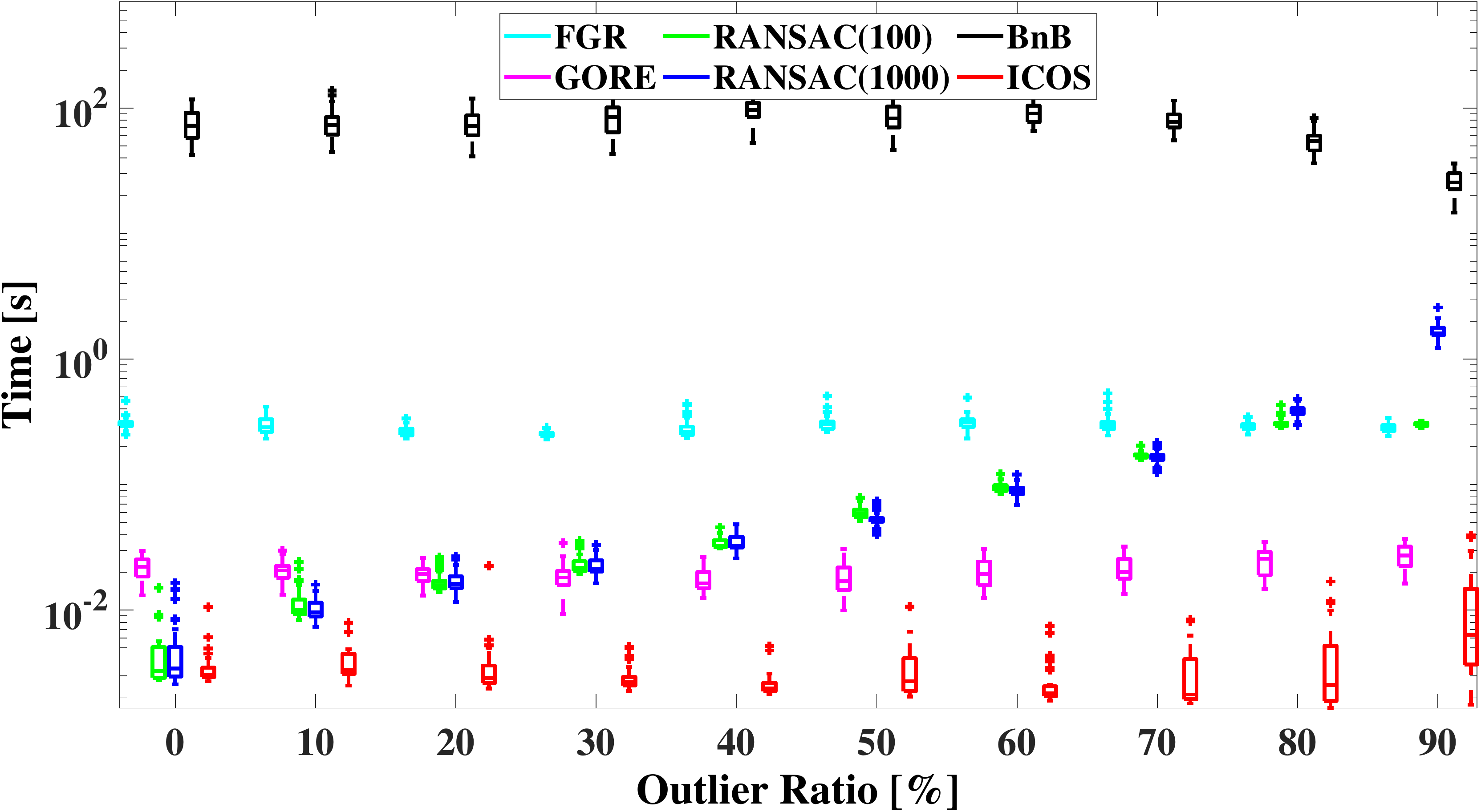}
\includegraphics[width=0.246\linewidth]{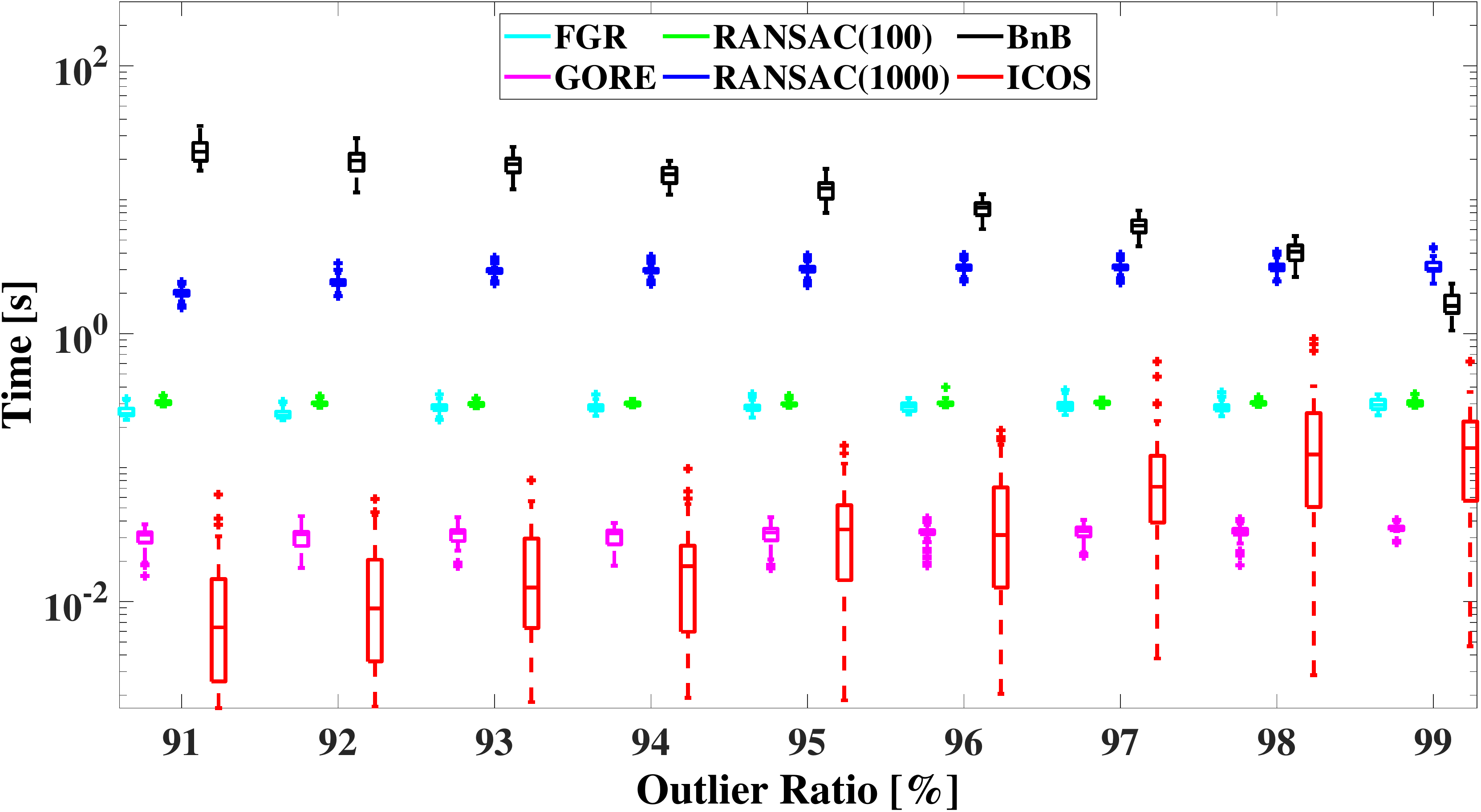}
\end{minipage}
}%

\footnotesize{(c) $N=1000$}

\subfigure{
\begin{minipage}[t]{1\linewidth}
\centering
\includegraphics[width=0.246\linewidth]{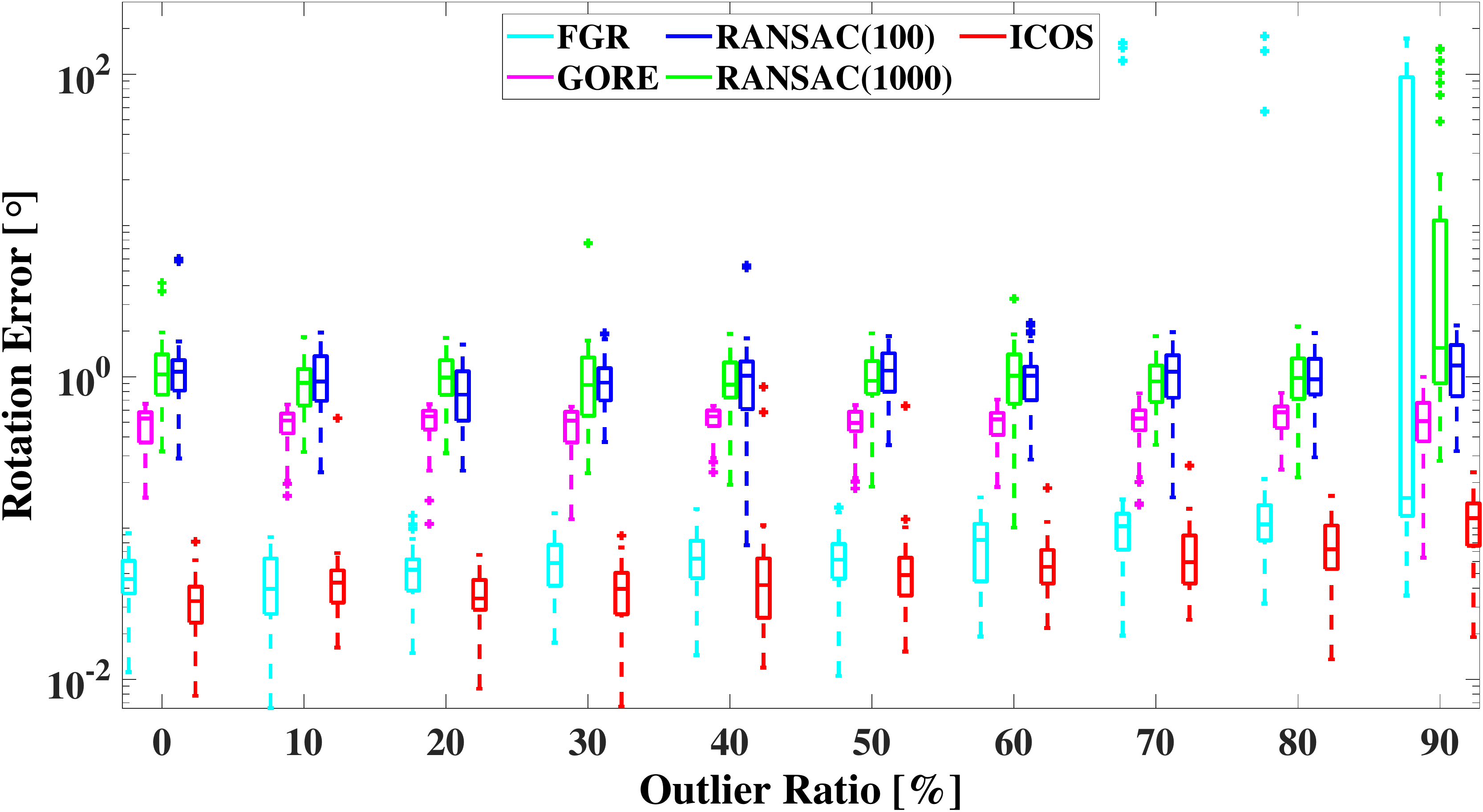}
\includegraphics[width=0.246\linewidth]{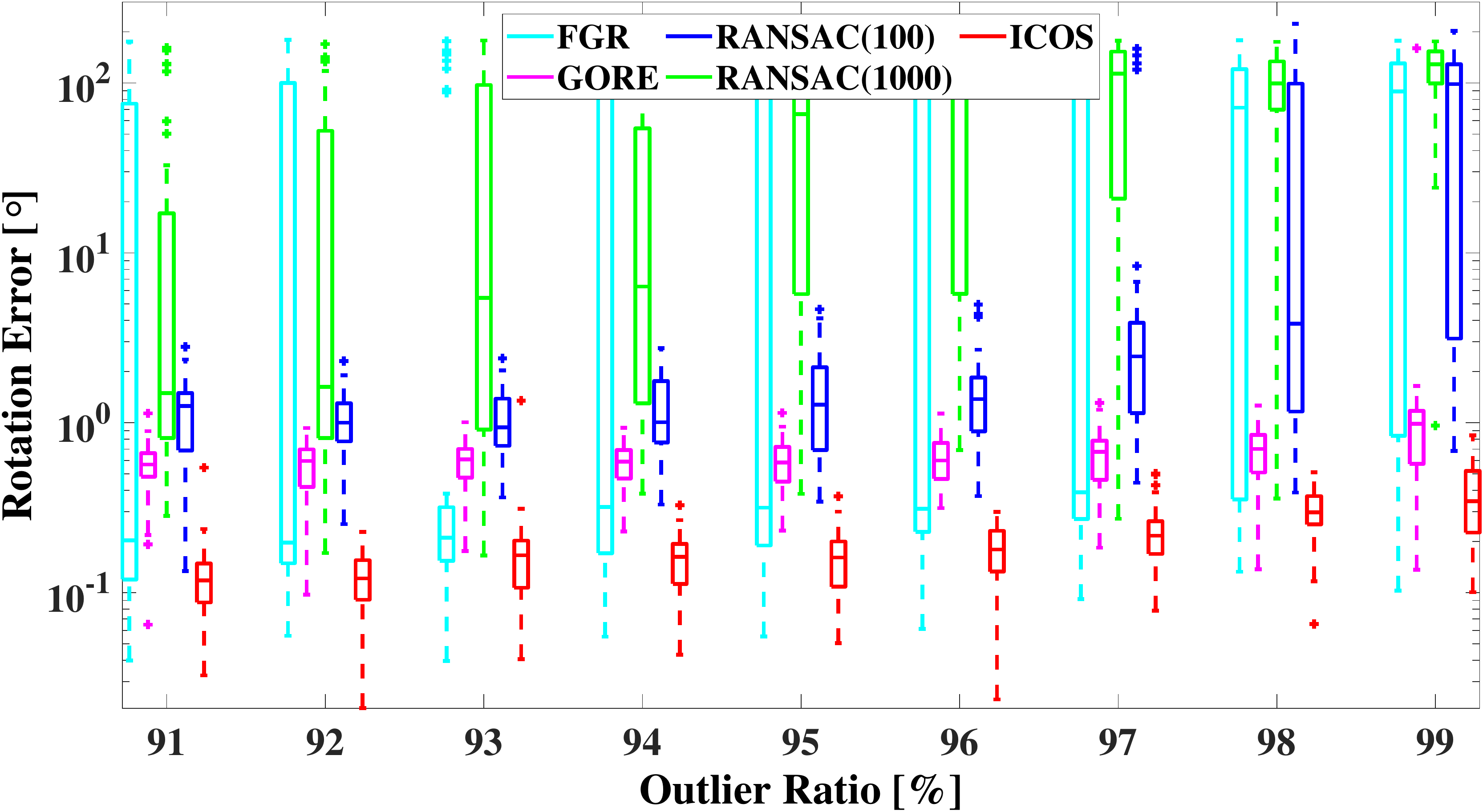}
\includegraphics[width=0.246\linewidth]{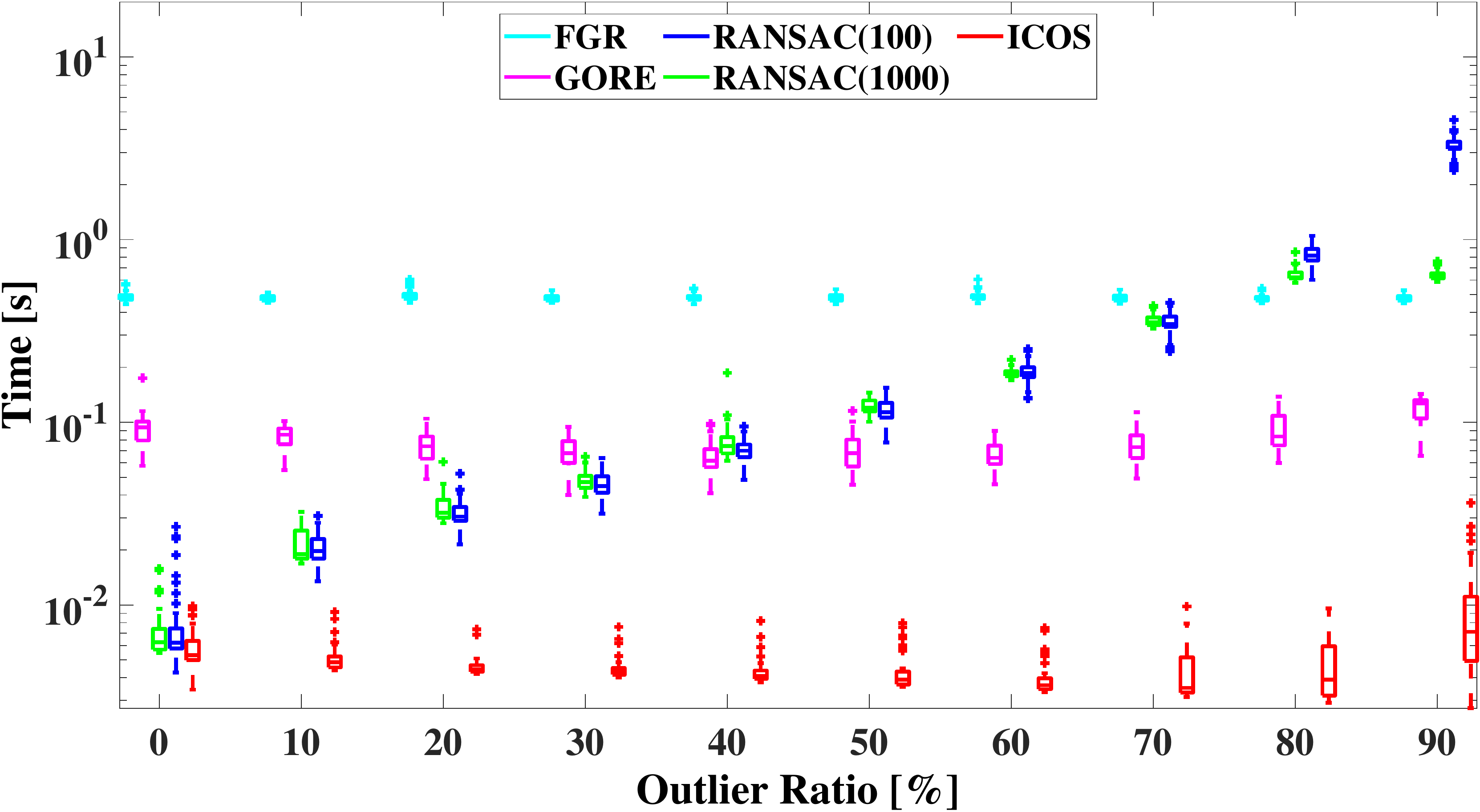}
\includegraphics[width=0.246\linewidth]{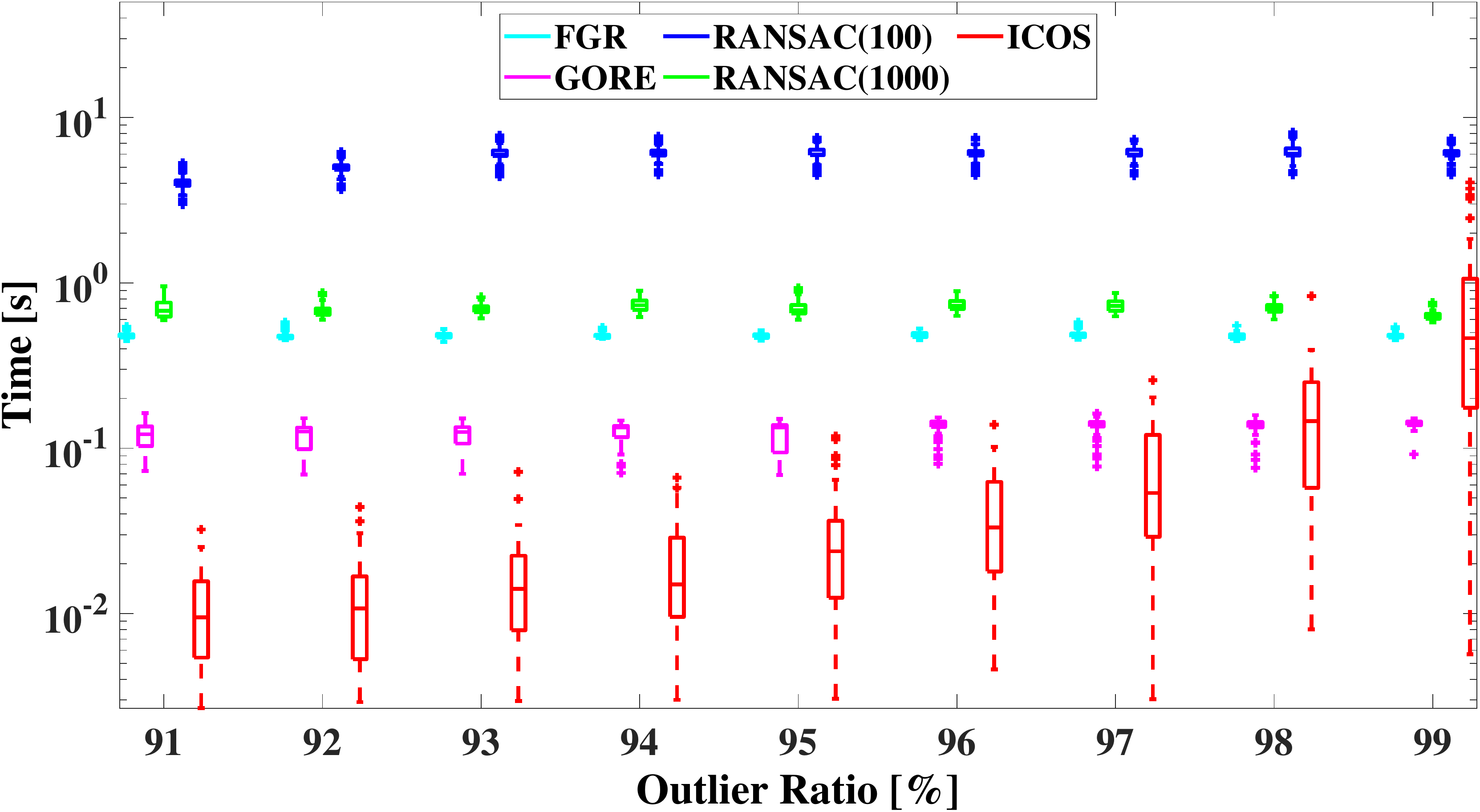}
\end{minipage}
}%
\vspace{-3mm}
\centering
\caption{Benchmarking of rotation search over the synthetic data. (a) Results with $N=100$ correspondences w.r.t. increasing outlier ratio (0-99\%). (b) Results with $N=500$ correspondences w.r.t. increasing outlier ratio (0-99\%). (c) Results with $N=1000$ correspondences w.r.t. increasing outlier ratio (0-99\%).}
\label{Syn-RS}
\end{figure*}

\begin{figure*}[t]
\centering

\footnotesize{(a)Point Cloud Registration with Known Scale: $\mathit{s}=1$}

\subfigure{
\begin{minipage}[t]{1\linewidth}
\centering
\includegraphics[width=0.245\linewidth]{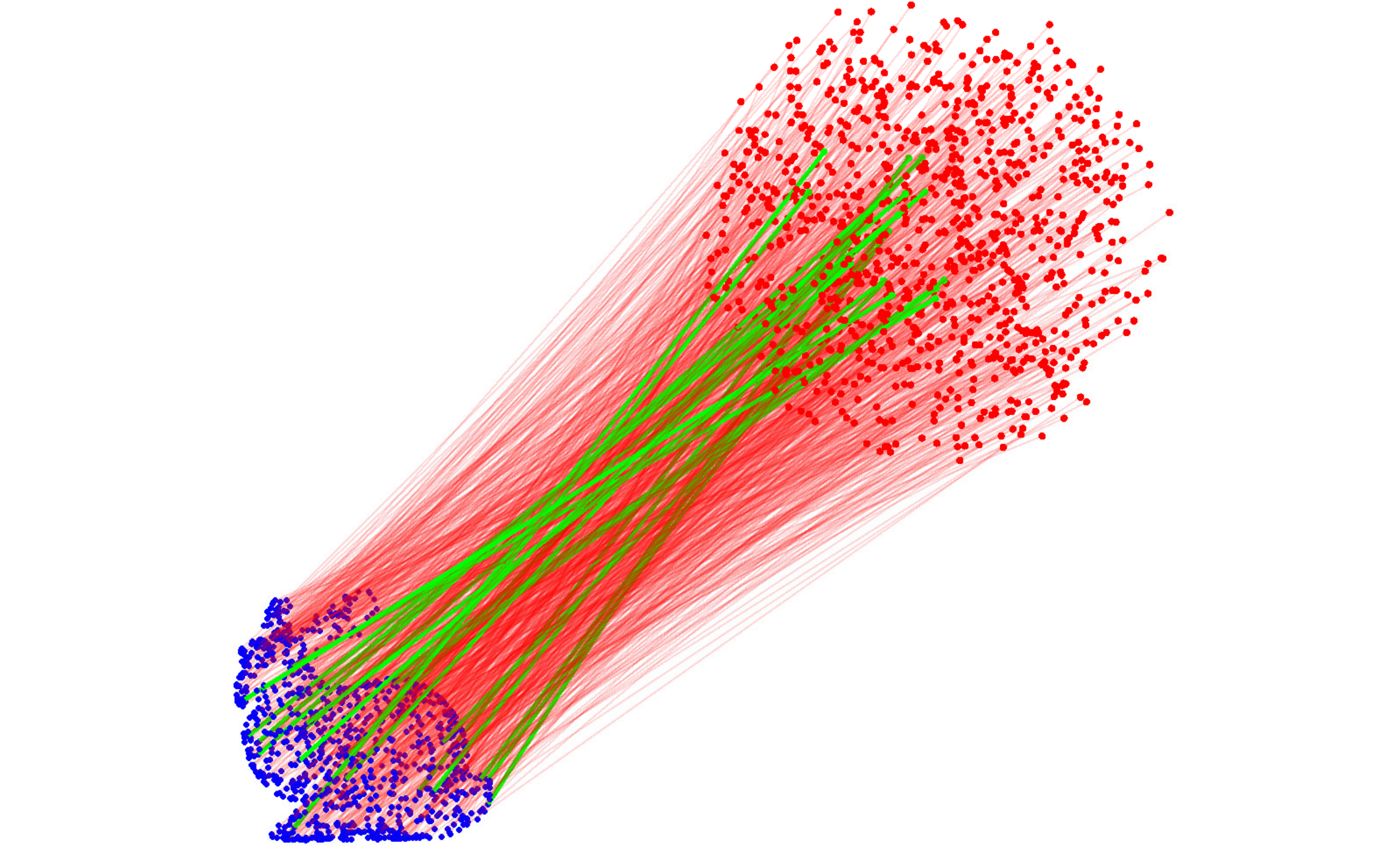}
\includegraphics[width=0.245\linewidth]{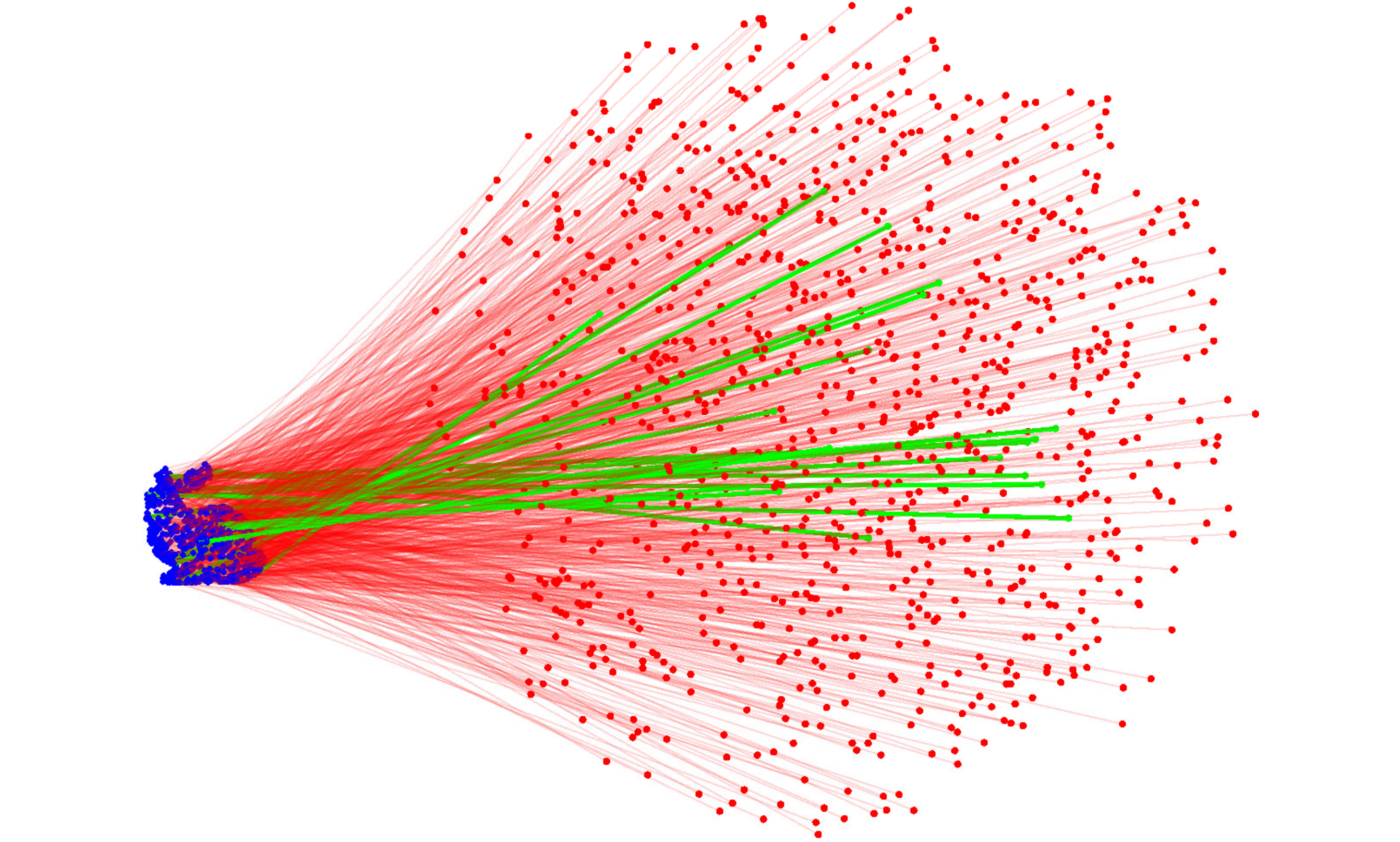}
\includegraphics[width=0.245\linewidth]{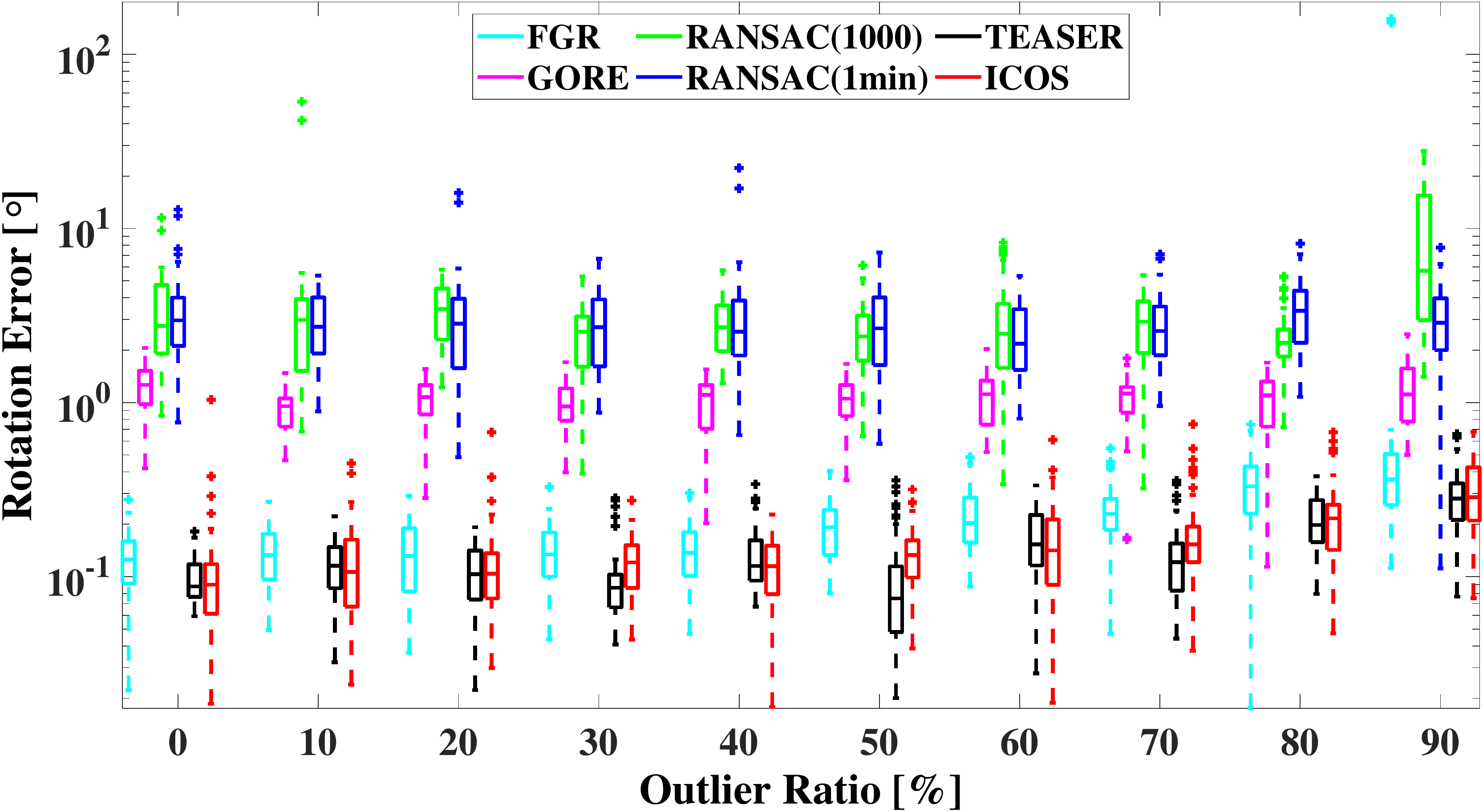}
\includegraphics[width=0.245\linewidth]{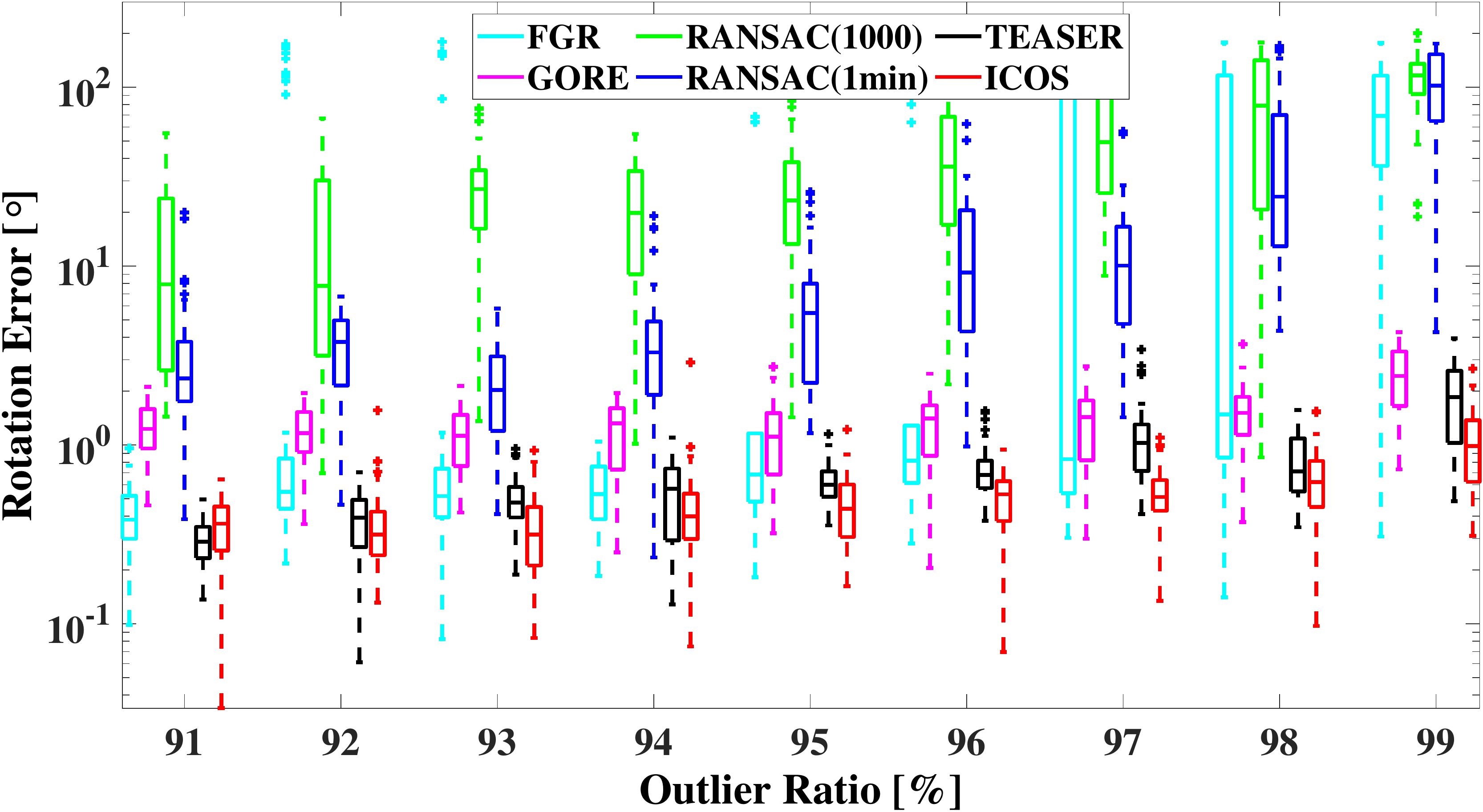}
\end{minipage}
}%

\subfigure{
\begin{minipage}[t]{1\linewidth}
\centering
\includegraphics[width=0.245\linewidth]{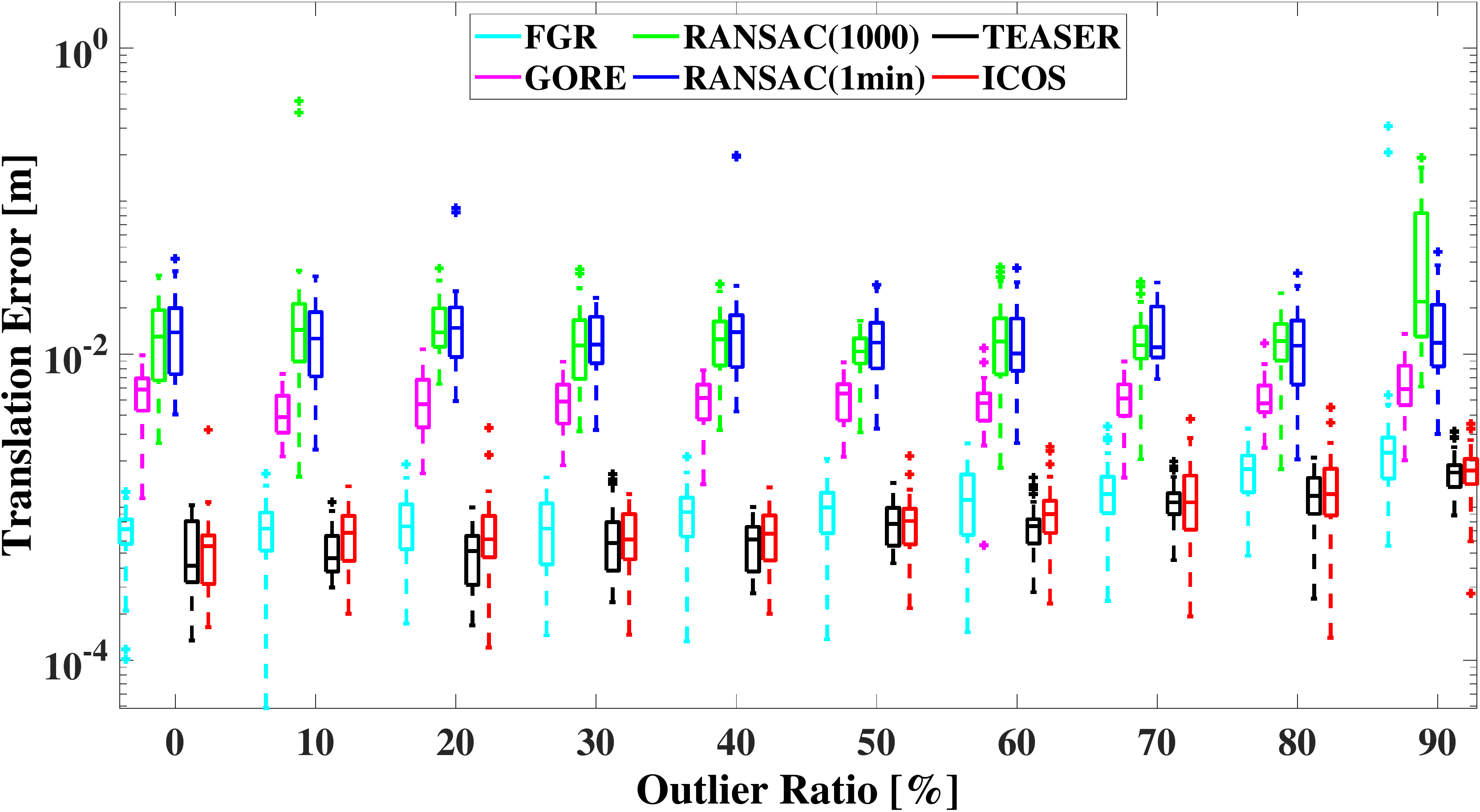}
\includegraphics[width=0.245\linewidth]{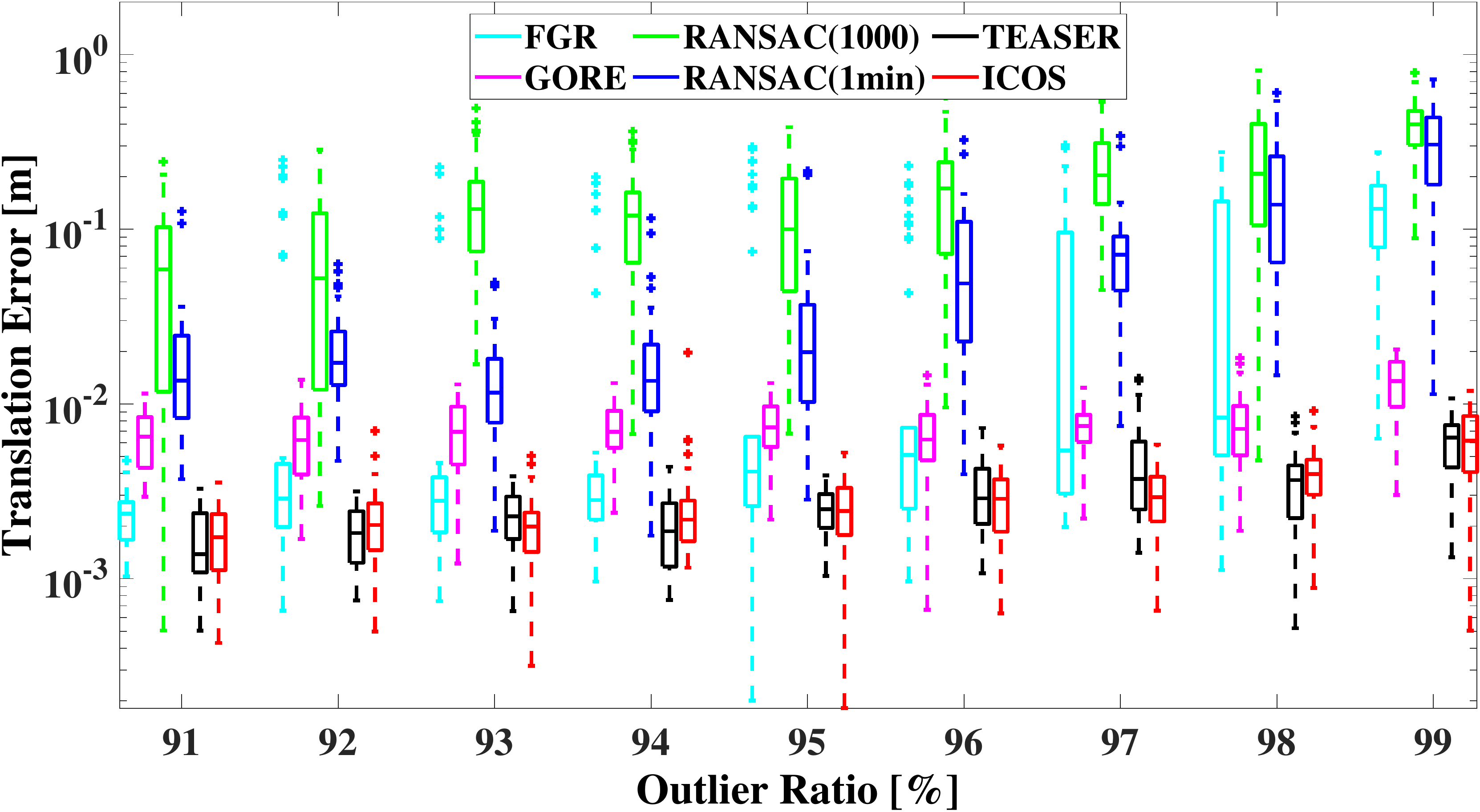}
\includegraphics[width=0.245\linewidth]{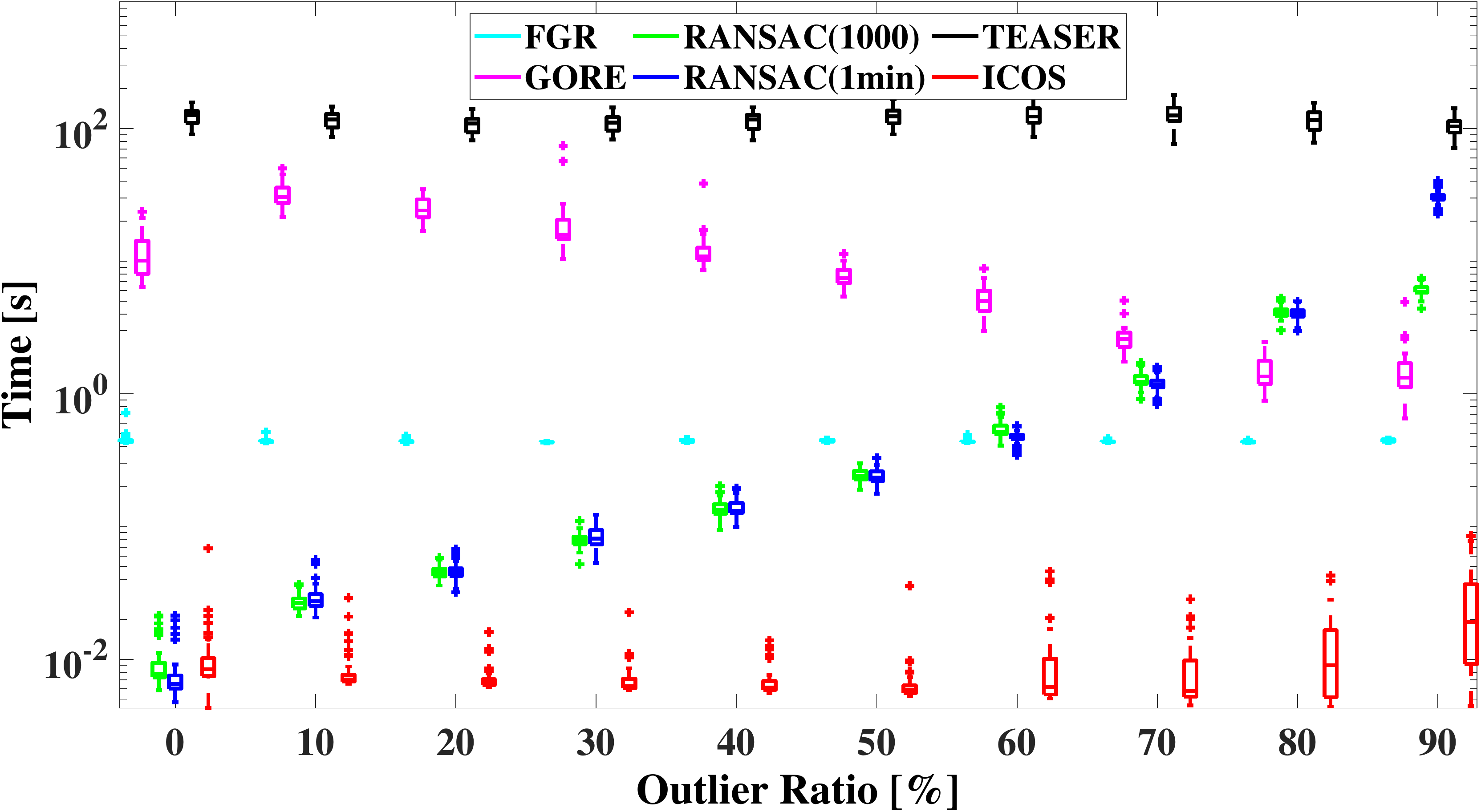}
\includegraphics[width=0.245\linewidth]{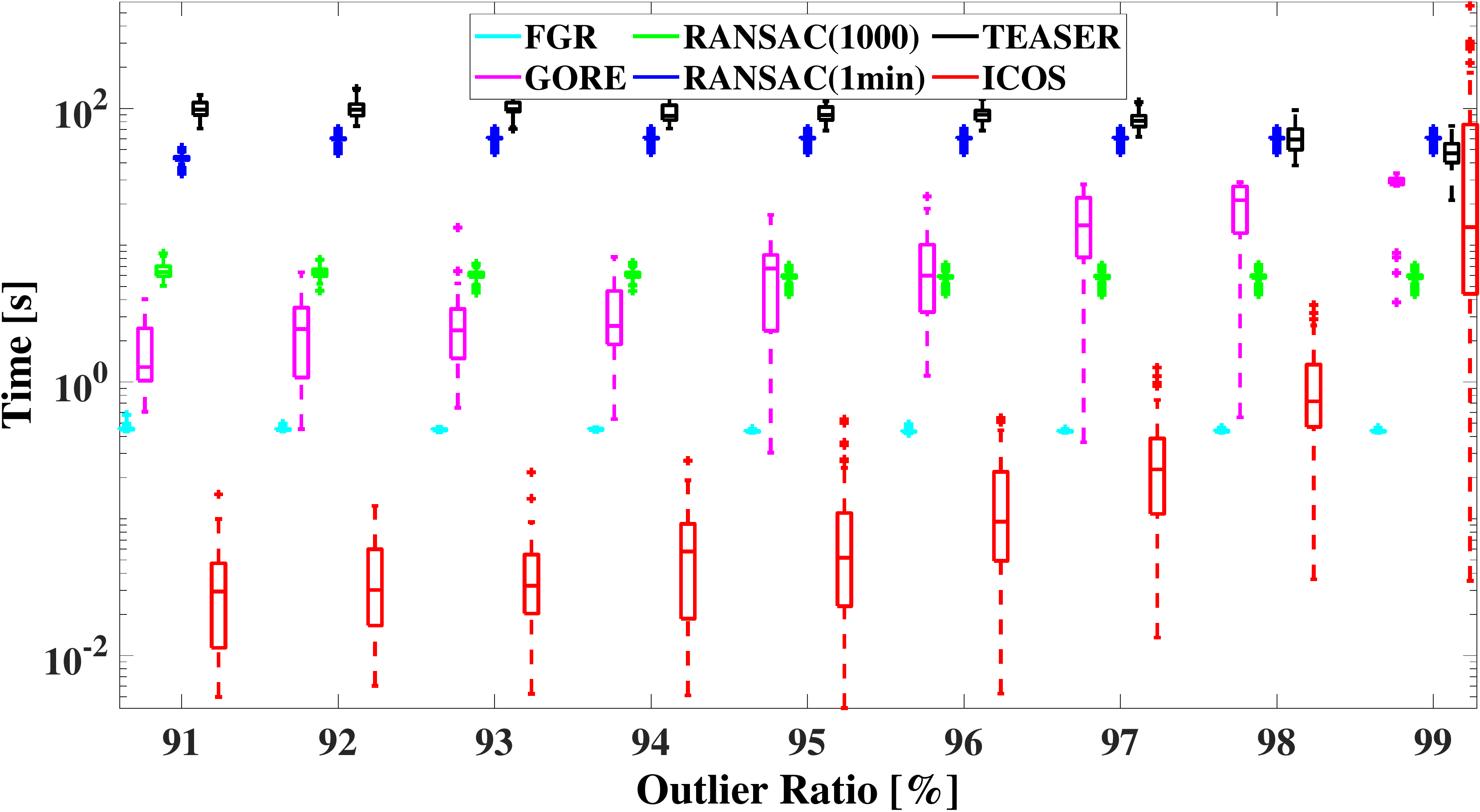}
\end{minipage}
}%

\footnotesize{(b)Point Cloud Registration with Unknown Scale: $\mathit{s}\in(1,5)$}

\subfigure{
\begin{minipage}[t]{1\linewidth}
\centering
\includegraphics[width=0.245\linewidth]{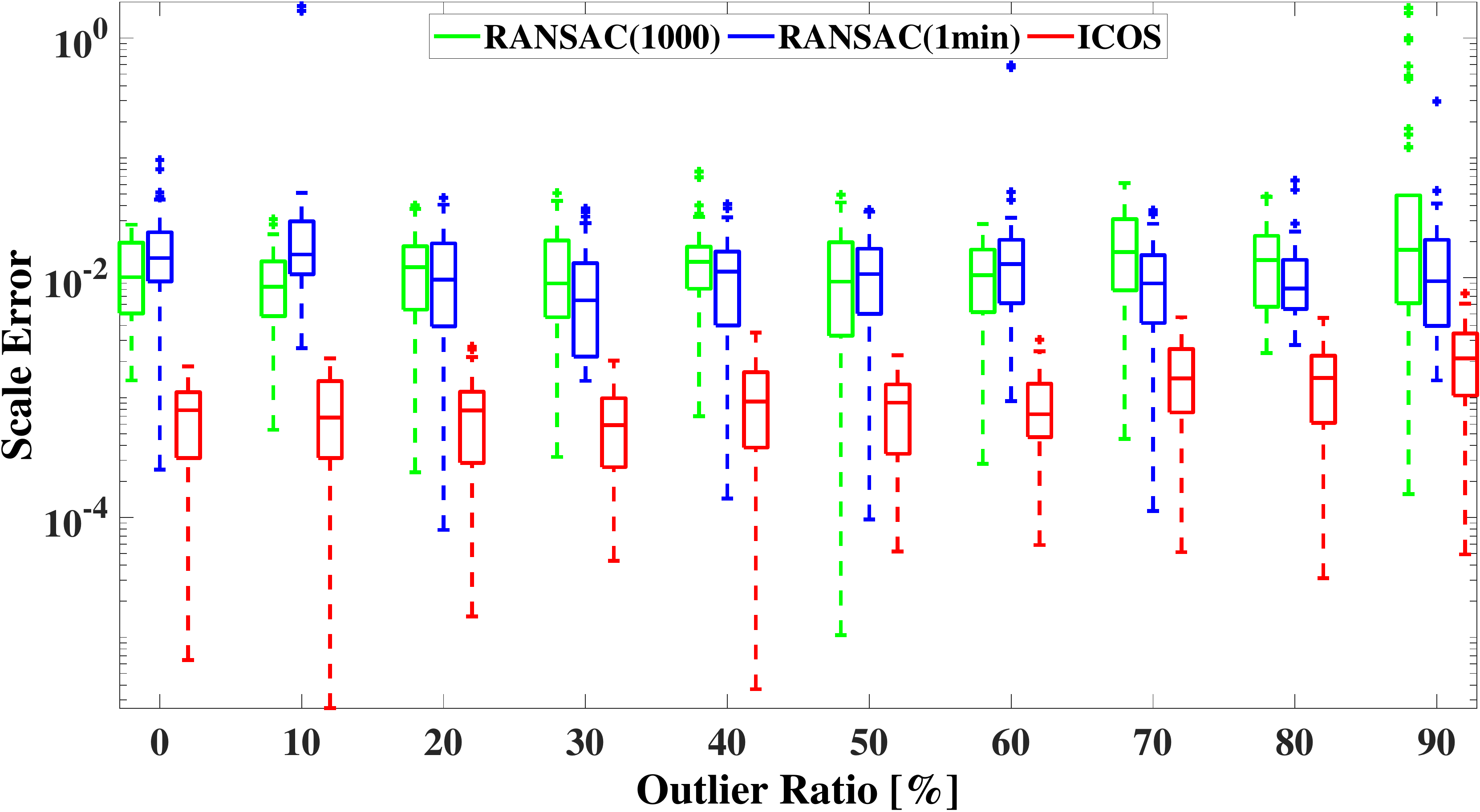}
\includegraphics[width=0.245\linewidth]{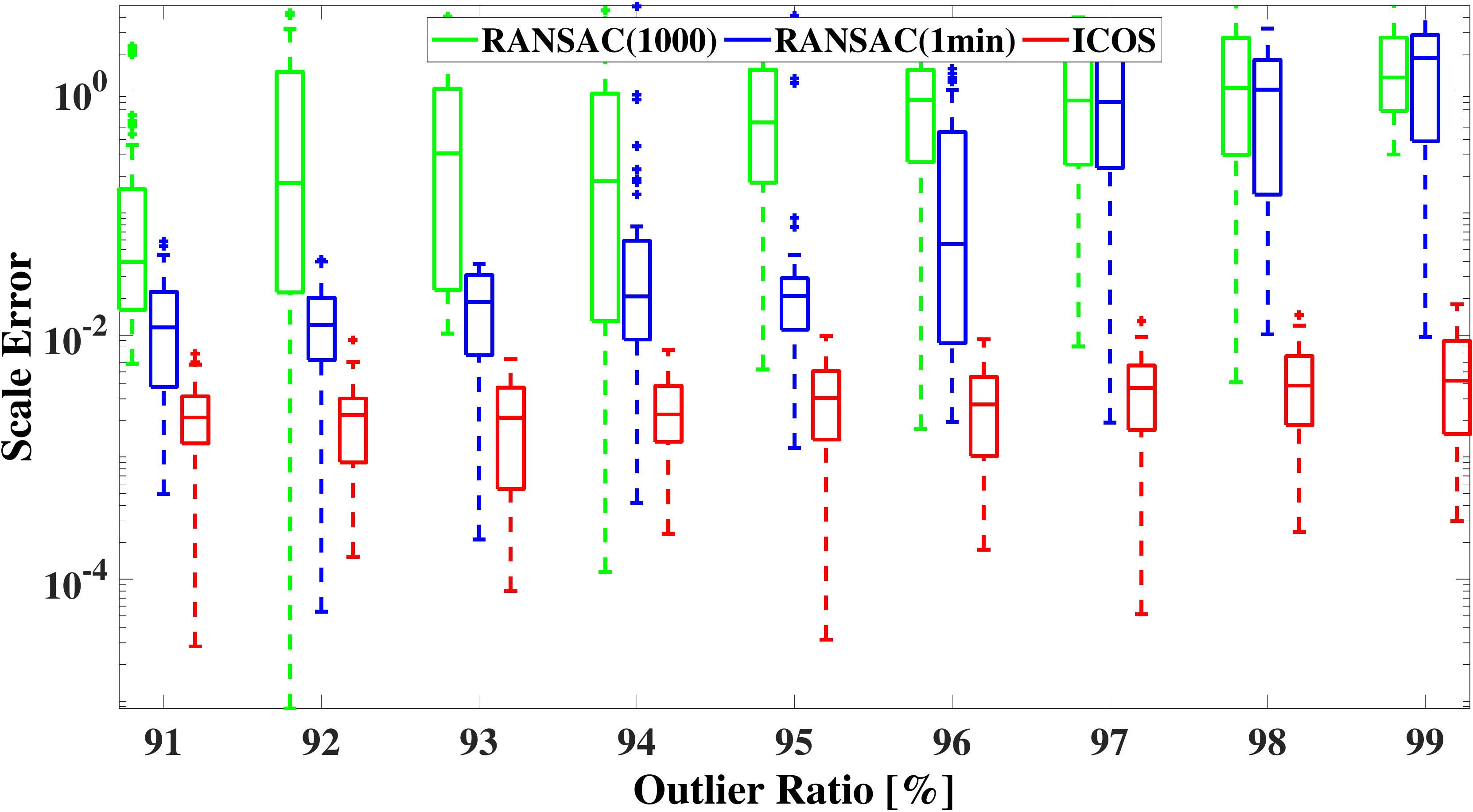}
\includegraphics[width=0.245\linewidth]{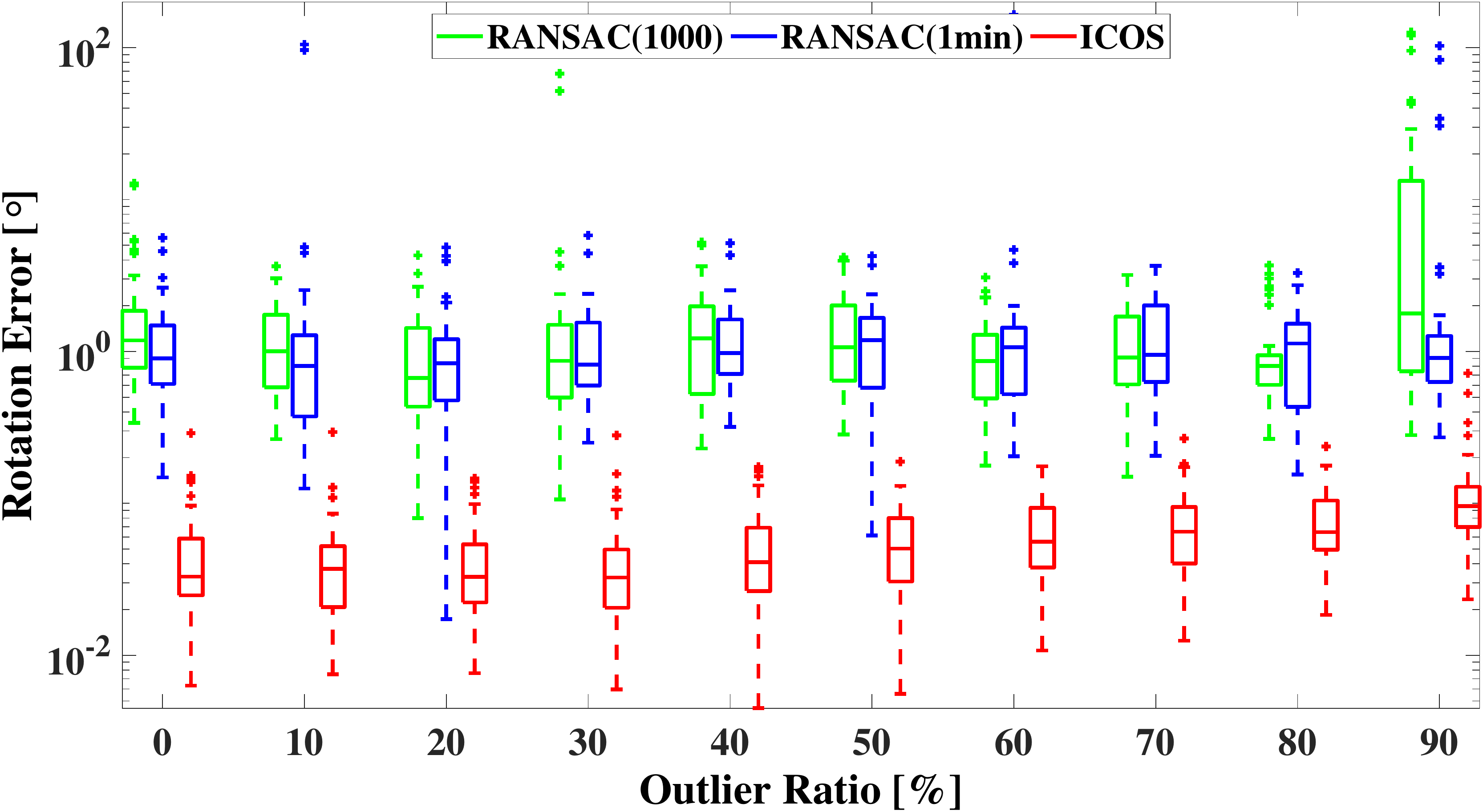}
\includegraphics[width=0.245\linewidth]{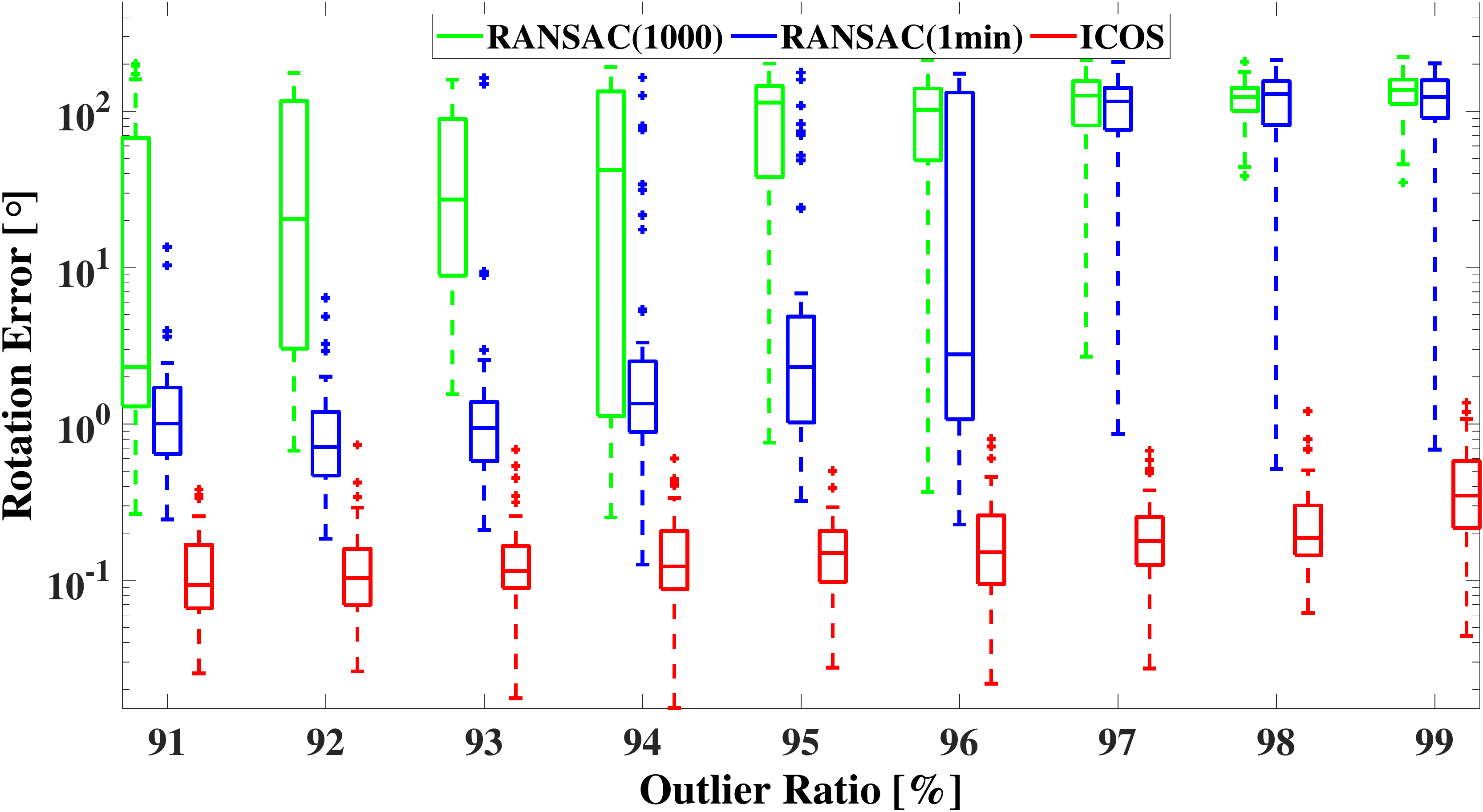}
\end{minipage}
}%

\subfigure{
\begin{minipage}[t]{1\linewidth}
\centering
\includegraphics[width=0.245\linewidth]{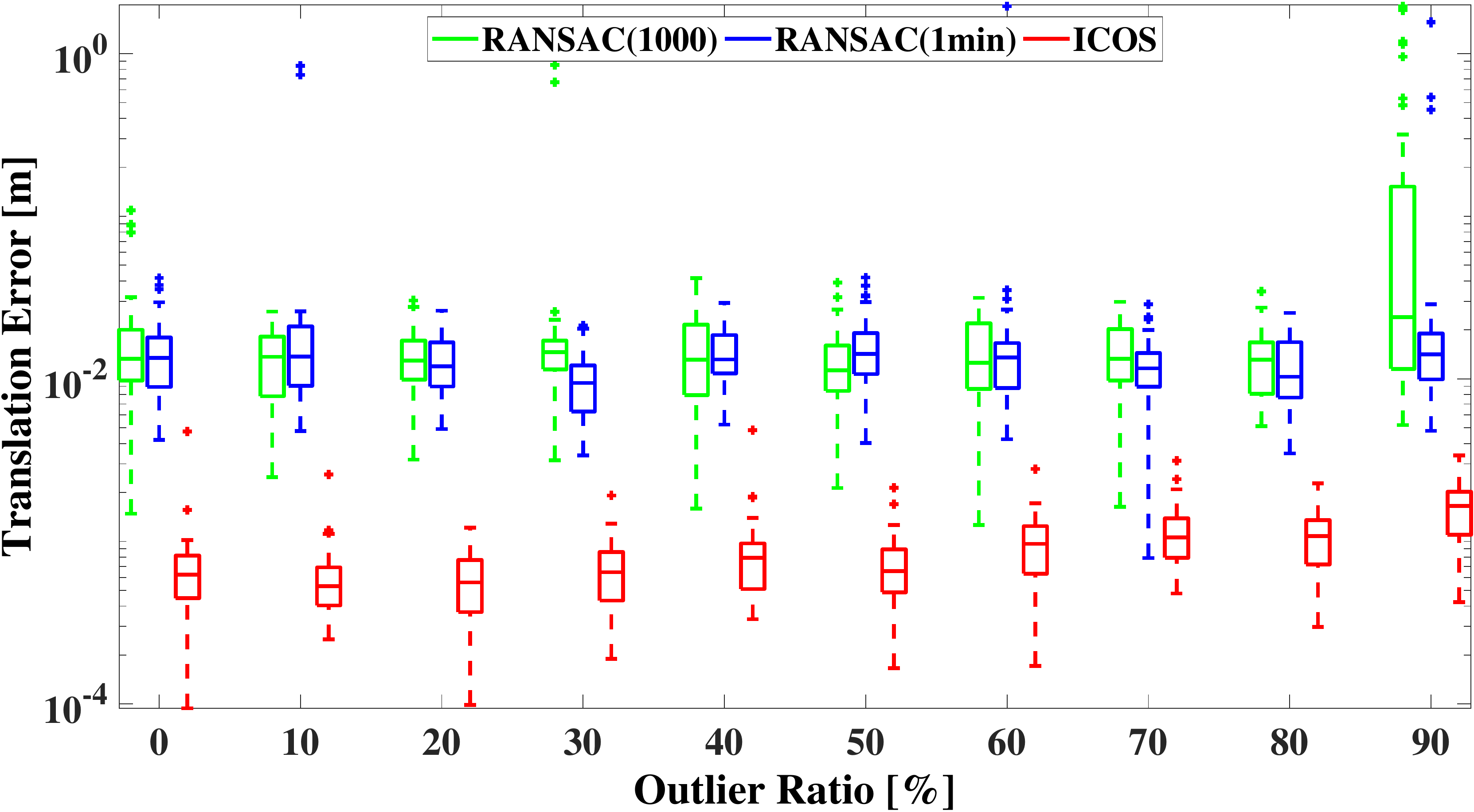}
\includegraphics[width=0.245\linewidth]{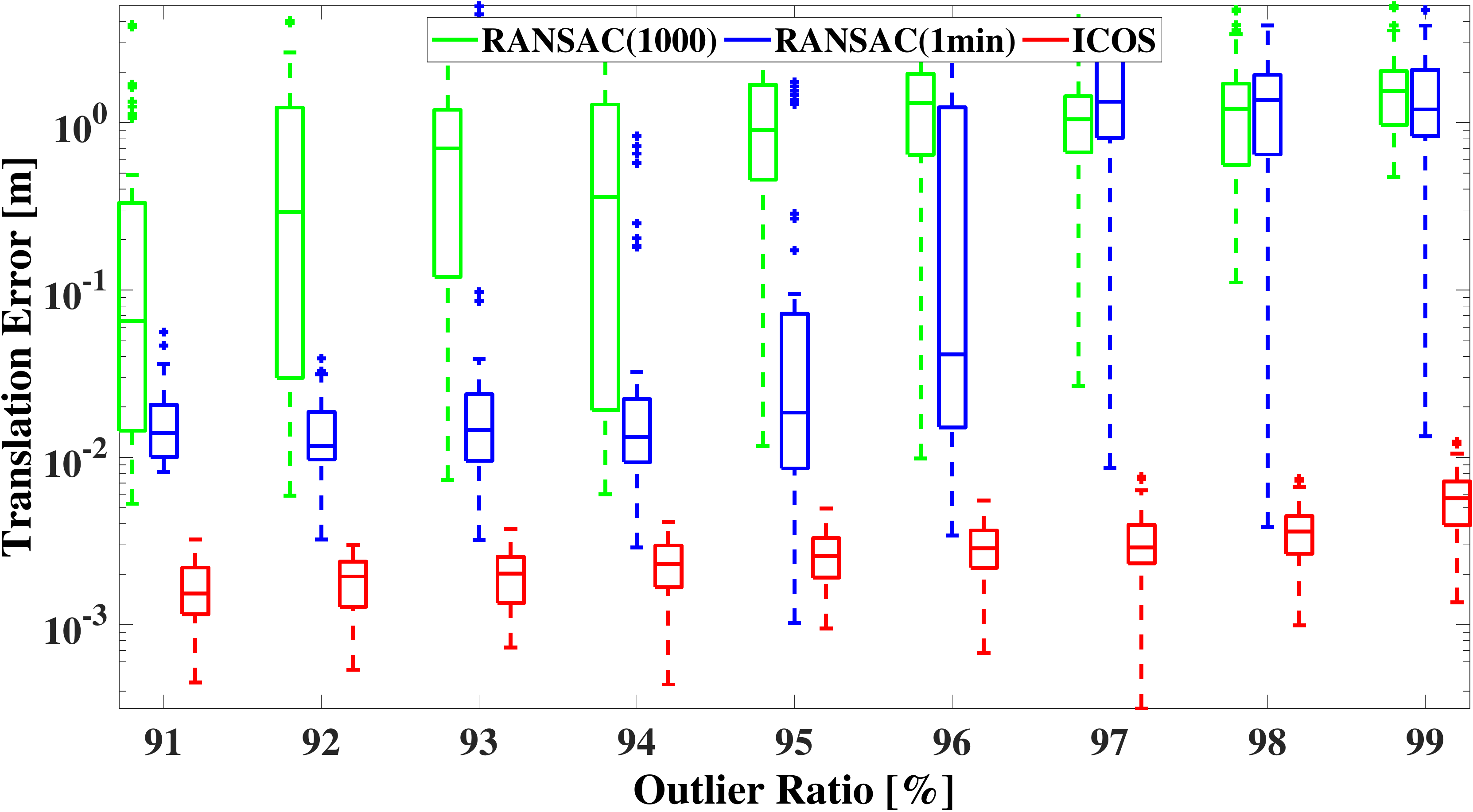}
\includegraphics[width=0.245\linewidth]{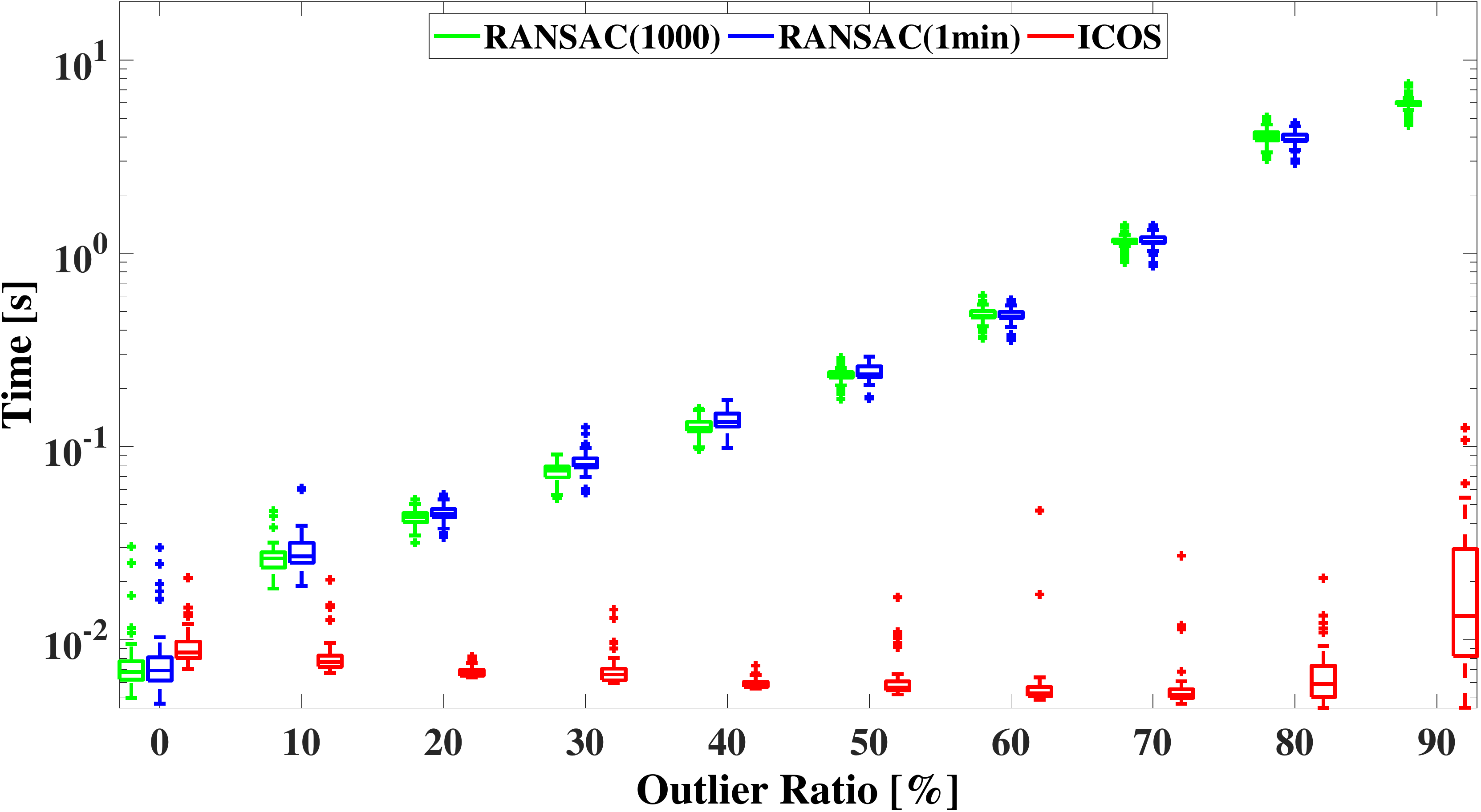}
\includegraphics[width=0.245\linewidth]{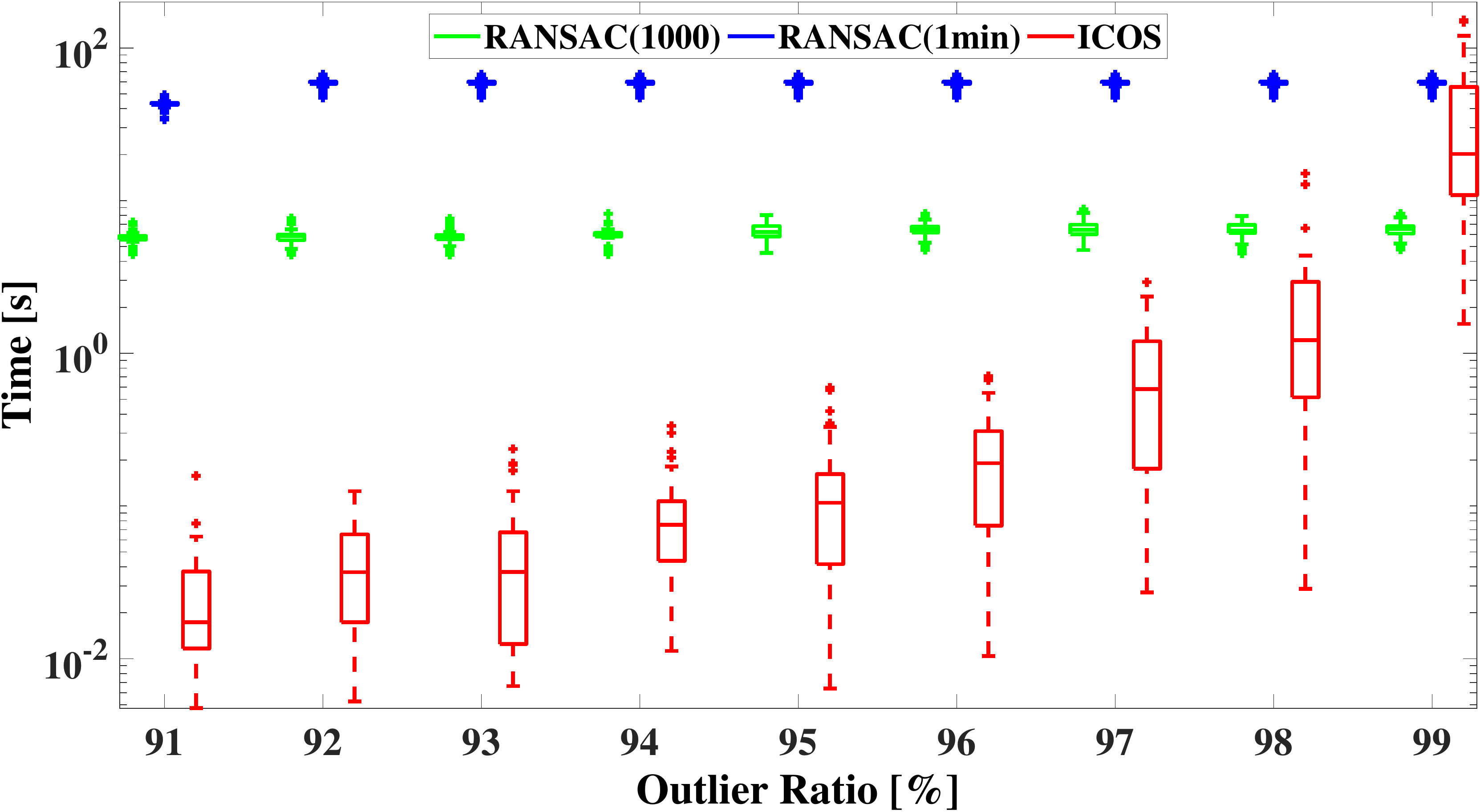}
\end{minipage}
}%

\centering
\caption{Benchmarking of point cloud registration over the `bunny'~\cite{curless1996volumetric}. Left-top: Examples of a known-scale and an unknown-scale point cloud registration problem with 98\% outliers. (a) Known-scale registration results w.r.t. increasing outlier ratio (0-99\%). (b) Unknown-scale registration results w.r.t. increasing outlier ratio (0-99\%).}
\label{bunny}
\vspace{-9pt}
\end{figure*}

\begin{figure}[t]
\centering
\subfigure{
\begin{minipage}[t]{0.95\linewidth}
\centering
\includegraphics[width=1\linewidth]{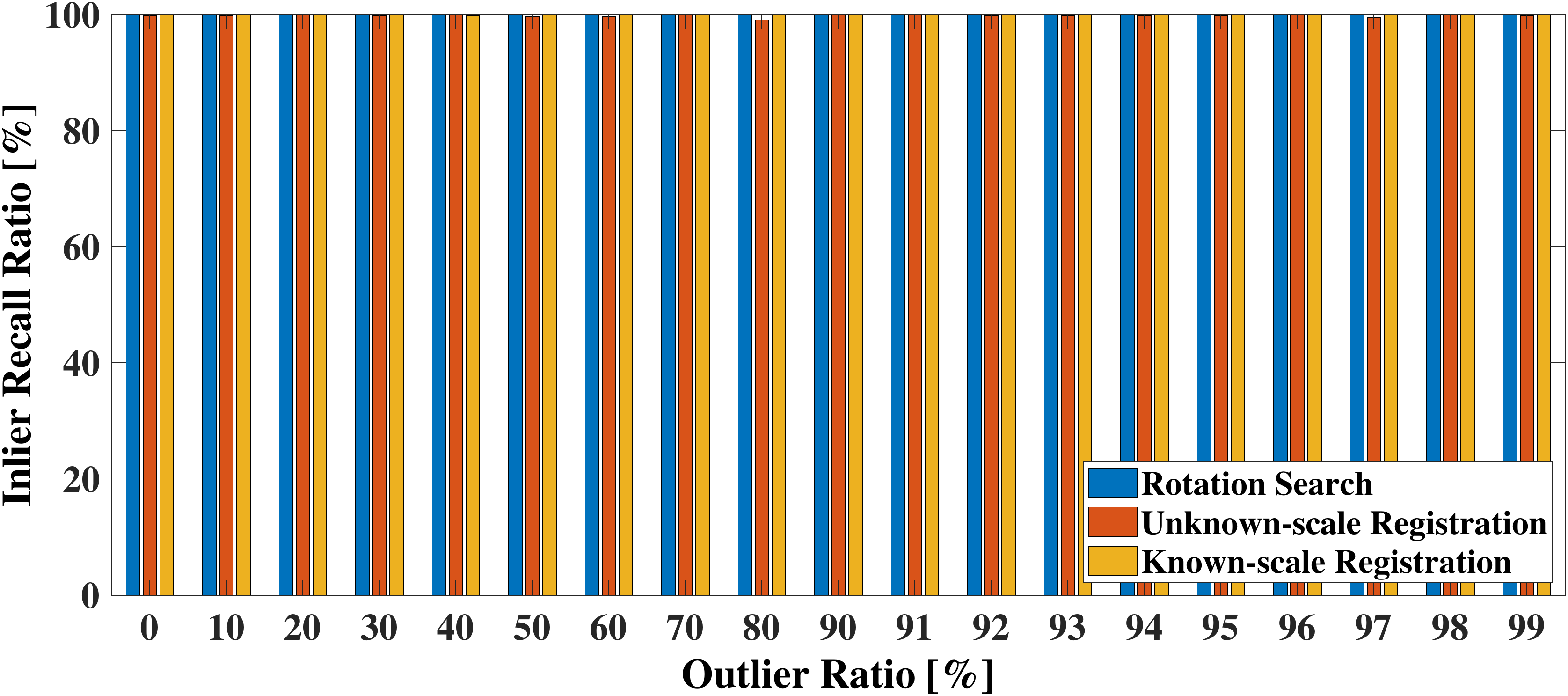}
\end{minipage}
}%
\vspace{-3mm}

\centering
\caption{Inlier recall ratio of ICOS in rotation search, and both known-scale and unknown-scale point cloud registration w.r.t. increasing outlier ratio (0-99\%).}
\label{Recall}
\end{figure}

\begin{figure}[t]
\centering
\subfigure{
\begin{minipage}[t]{1\linewidth}
\centering
\includegraphics[width=0.49\linewidth]{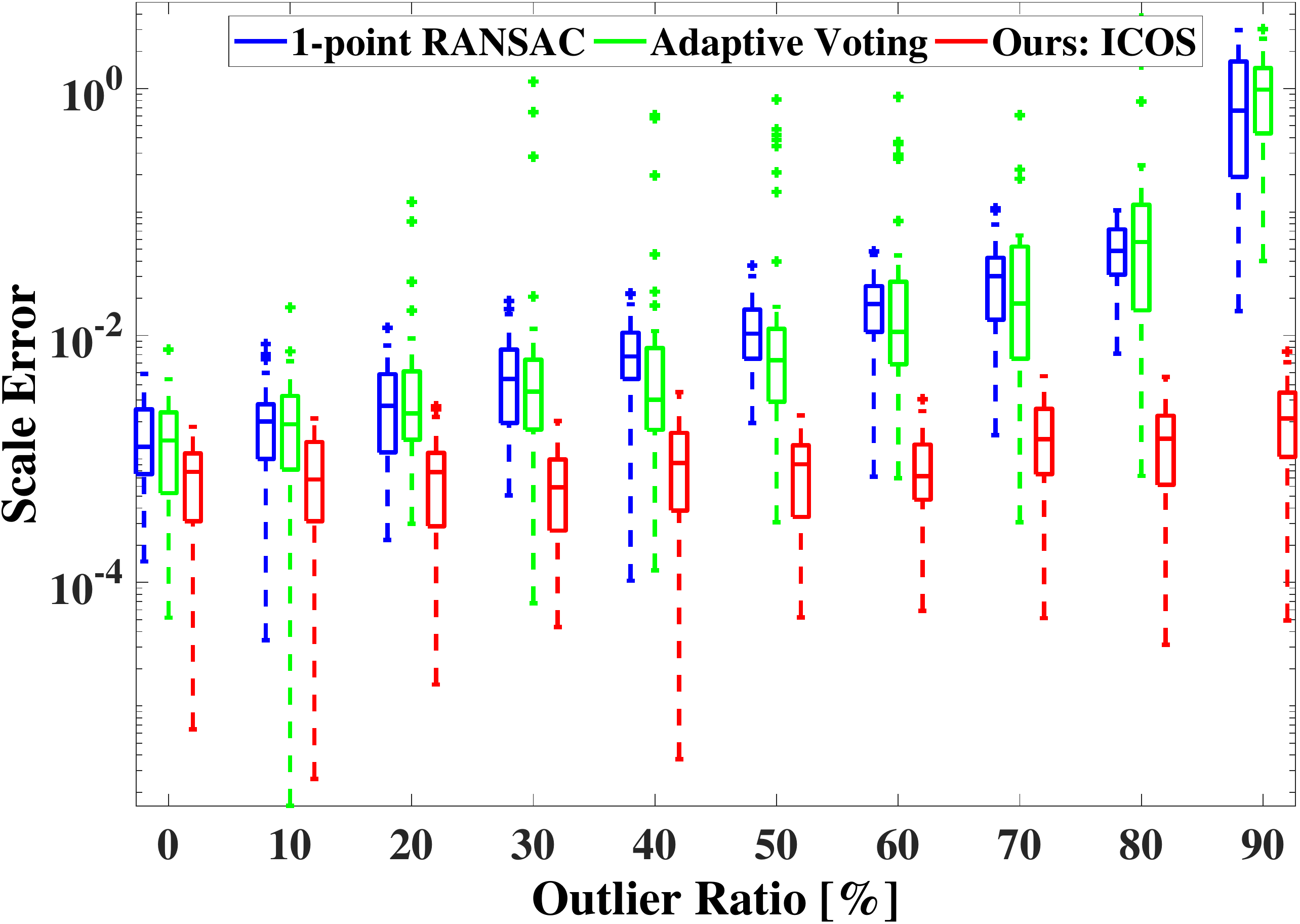}
\includegraphics[width=0.49\linewidth]{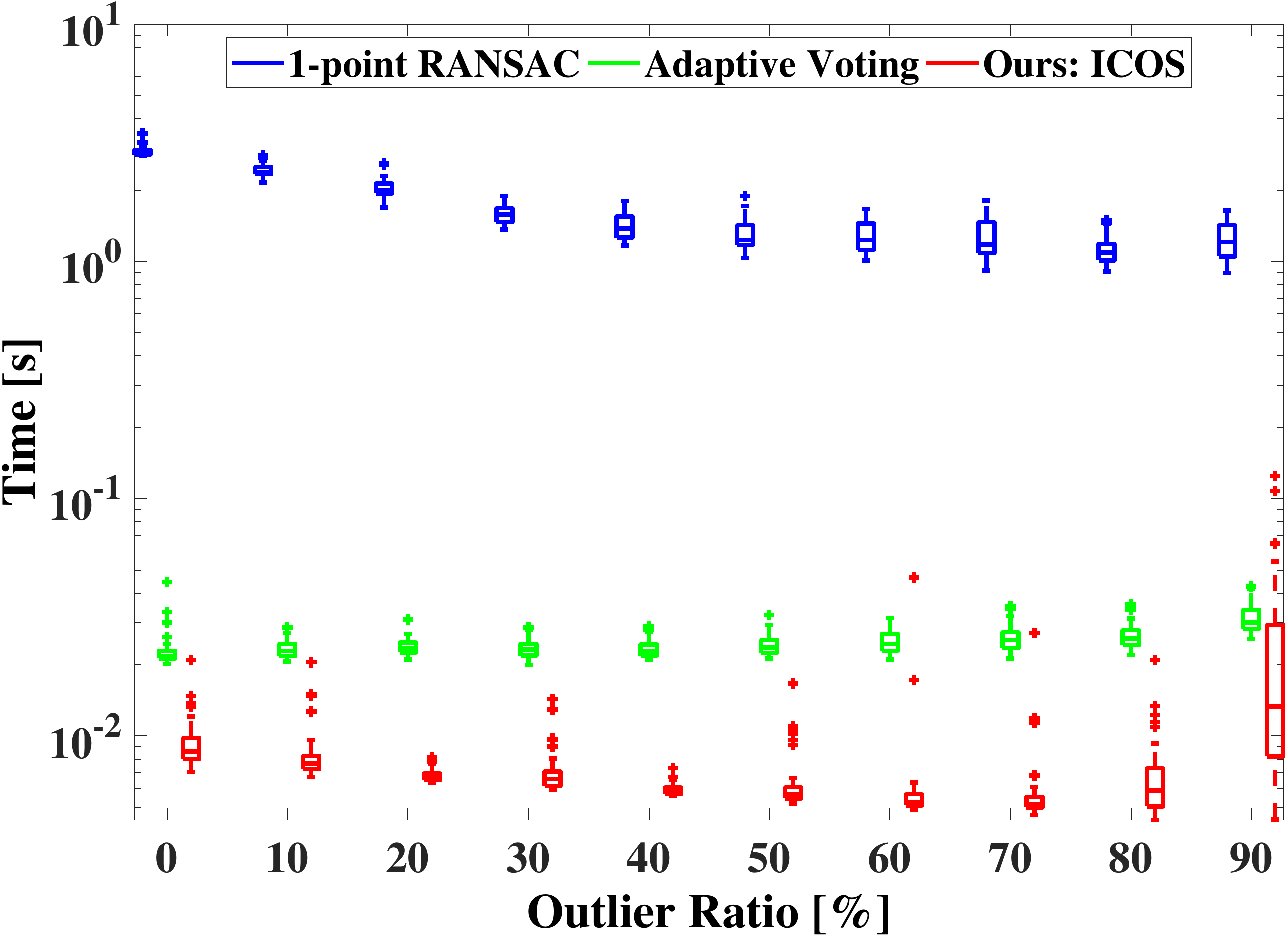}
\end{minipage}
}%
\vspace{-1mm}

\centering
\caption{Benchmarking of point cloud registration on scale estimation. Left: Scale errors w.r.t. increasing outlier ratio (0-99\%). Right: Runtime w.r.t. increasing outlier ratio (0-99\%).}
\label{Scale}
\end{figure}

\section{Experiments}\label{Experiments}

In this section, we conduct a series of experiments based on multiple datasets to evaluate the proposed solver ICOS for point cloud registration, also in comparison with the existing state-of-the-art solvers. All the experiments are implemented in Matlab on a laptop with an i7-7700HQ CPU and 16GB RAM, and no parallelism programming is ever used. And the explicit parameter setup for ICOS is given in Table~\ref{parameter}.

\subsection{Benchmarking of Rotation Search on Synthetic Data}\label{RS-overall}

\textbf{Experimental Setup.} We test our solver ICOS for rotation search over synthetic experiments. We first randomly generate $N=\{100,500,1000\}$ vectors $\mathcal{U}=\{\mathbf{u}_i\}_{i=1}^{N}$ with unit-norm, and rotate $\mathcal{U}$ with a random rotation $\boldsymbol{R}\in SO(3)$ to obtain a new vector set $\mathcal{V}=\{\mathbf{v}_i\}_{i=1}^{N}$ in another coordinate frame. We then add random zero-mean Gaussian noise with $\sigma=0.01$ and $\mu=0$ to all vectors in $\mathcal{V}$. After that, a portion of the vectors in $\mathcal{V}$ (from 0\% up to 99\%) are substituted by randomly-generated unit-norm vectors to generate outliers among correspondences. All the experimental results are obtained over 50 Monte Carlo runs in the same environment.

To quantitatively represent the errors of rotation, we adopt the geodesic errors such that
\begin{equation}\label{Rot-error}
E_{\boldsymbol{R}}=\angle (\boldsymbol{\hat{R}},\boldsymbol{{R}}_{gt} )\cdot \frac{180}{\pi}^{\circ},
\end{equation}
where`$\angle$' is the geodesic error~\eqref{geo-error}~\cite{hartley2013rotation} and the subscript `$gt$' denotes the ground-truth value.

For benchmarking, we test our solver ICOS (Algorithm~\ref{Algo0-ICOS}) against FGR~\cite{zhou2016fast}, GORE~\cite{parra2015guaranteed}, RANSAC and BnB~\cite{parra2014fast}. In terms of FGR, we merely solve the rotation rather than the whole transformation using the Gauss-Newton method. As for RANSAC, we adopt Horn's minimal method~\cite{horn1987closed} to solve the rotation with the confidence set to 0.995, where we design two stop conditions: (i) RANSAC with at most 100 iterations, named RANSAC(100), and (ii) RANSAC with at most 1000 iterations, named RANSAC(1000). When $N=1000$, we exclude BnB as it takes tens of minutes per run. QUASAR~\cite{yang2019quaternion} is also not adopted due to its fairly long runtime.

\textbf{Results.} Fig.~\ref{Syn-RS} displays the benchmarking results on both accuracy and runtime. It is not hard to observe that: (i) both FGR and RANSAC(100) fail at 90\% and RANSAC(1000) breaks at 98\%, while ICOS is robust against 95-96\% with $N=100$, 98\% with $N=500$, and 99\% with $N=1000$, (iii) ICOS is the most accurate solver  for all time since it constantly demonstrates the lowest rotation errors, and (iv) with $N=\{500,1000\}$, ICOS is the fastest solver when the outlier ratio is below 95\%, and with $N=100$, it has similar speed as GORE when the outlier ratio is no greater than 60\%.

\subsection{Benchmarking of Point Cloud Registration on Real Data}\label{overall}

\textbf{Experimental Setup.} We then evaluate ICOS against other state-of-the-art methods over the `bunny' and the `Armadillo' point clouds from the Stanford 3D Scanning Repository~\cite{curless1996volumetric}. We first randomly downsample the point cloud to 1000 points and resize it to be placed in a cube of $[-0.5,0.5]^3$ meters, which can be regarded as the initial point set $\mathcal{P}=\{\mathbf{P}_i\}_{i=1}^{1000}$. Then, we generate a random rigid transformation: ($\mathit{s}, \boldsymbol{R},\boldsymbol{t}$), in which scale $\mathit{s}\in[1,5)$ meters, rotation $\boldsymbol{R}\in SO(3)$ and translation $||\boldsymbol{t}||\leq 3$ meters, to transform point set $\mathcal{P}$. We add random Gaussian noise with $\sigma=0.01$ and $\mu=0$ to the transformed point set, as the point set $\mathcal{Q}=\{\mathbf{Q}_i\}_{i=1}^{1000}$. To create outliers in clutter, a portion of the points in $\mathcal{Q}$ are substituted by random points inside a sphere of diameter $\mathit{s}\sqrt{3}$ meters according to the outlier ratio (from 0\% up to 99\%), as exemplified in Fig.~\ref{bunny}. All the experimental results are obtained over 50 Monte Carlo runs in the same environment.

\begin{figure*}[t]
\centering

\footnotesize{(a)Point Cloud Registration with Known Scale: $\mathit{s}=1$}

\subfigure{
\begin{minipage}[t]{1\linewidth}
\centering
\includegraphics[width=0.245\linewidth]{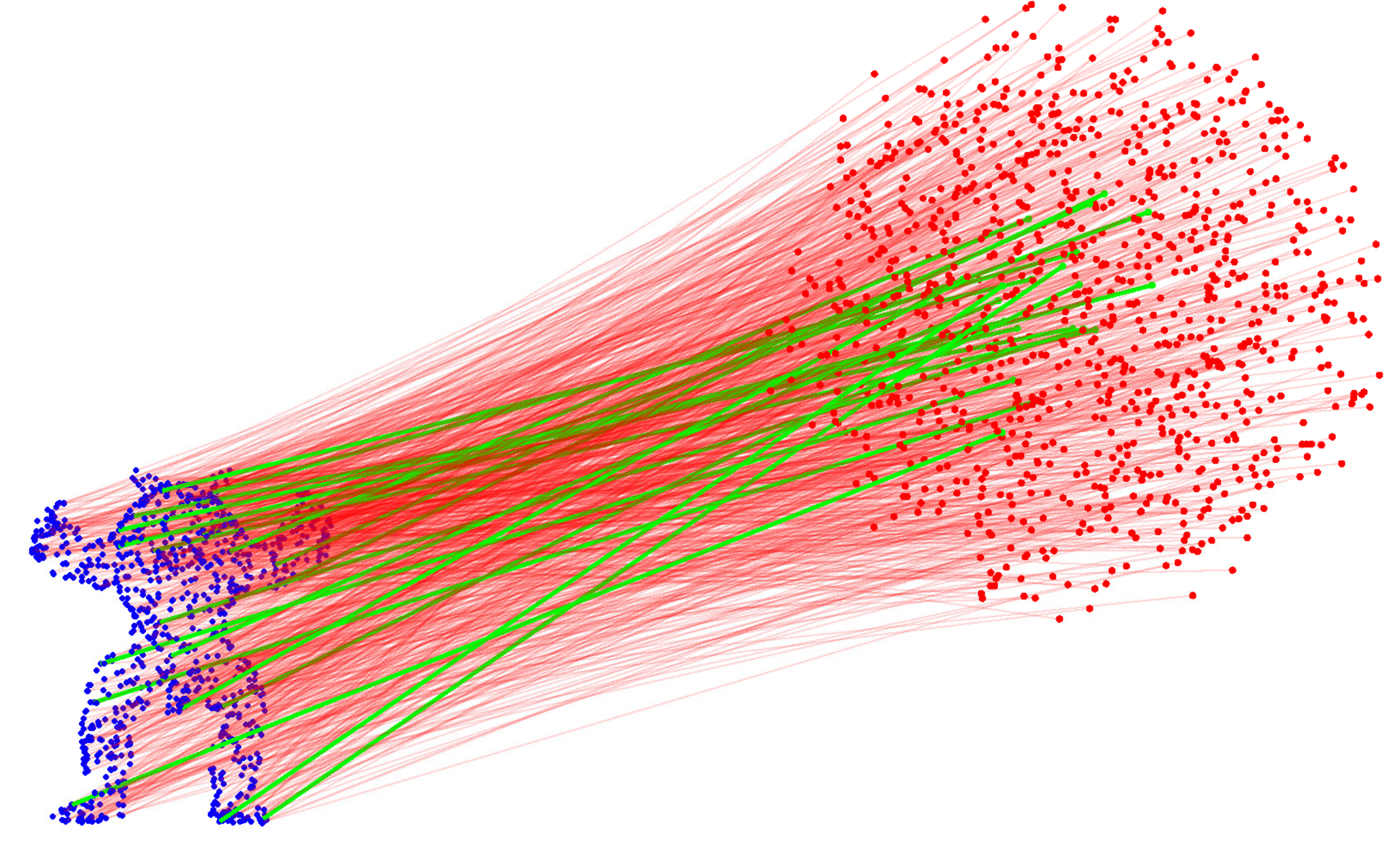}
\includegraphics[width=0.245\linewidth]{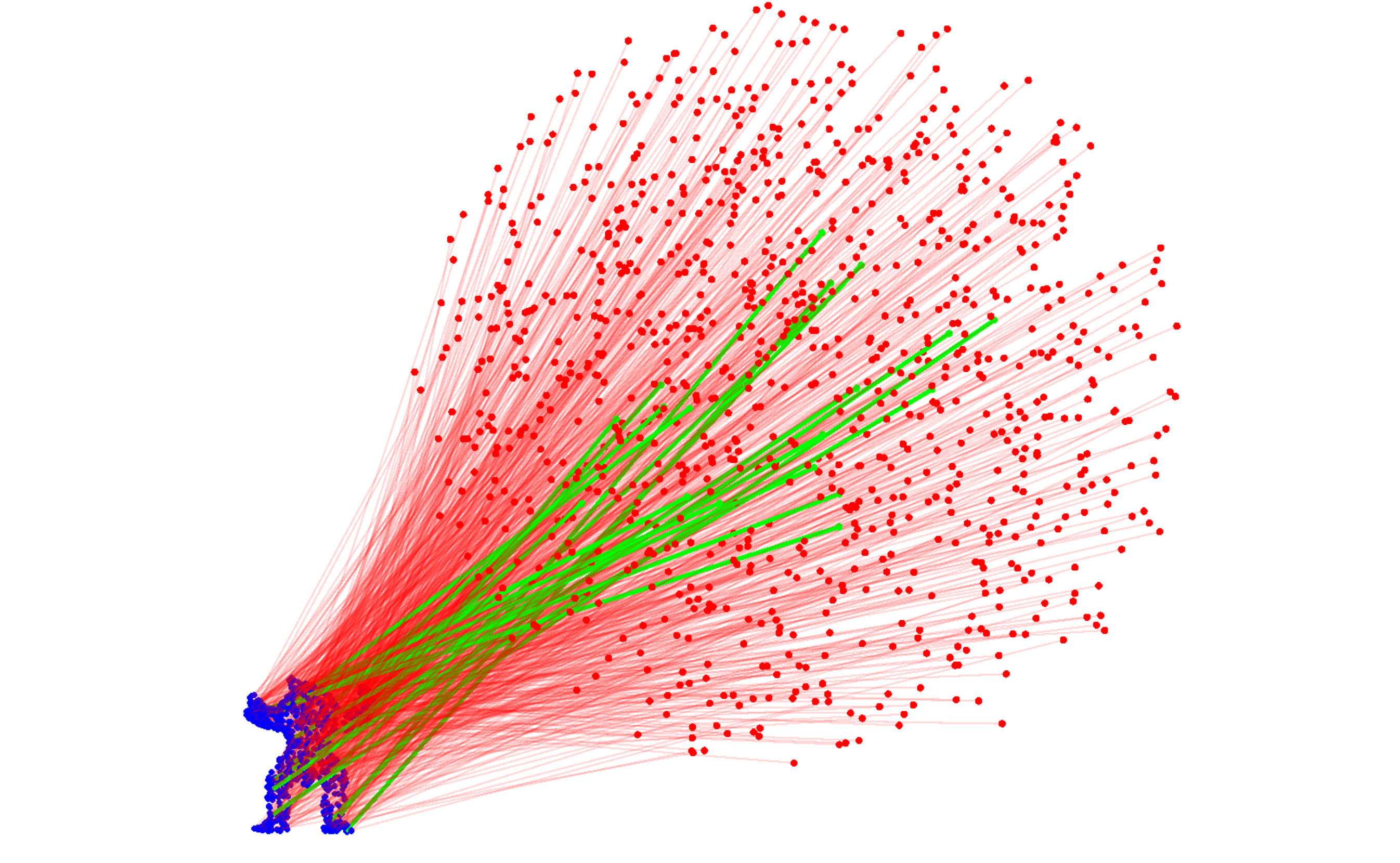}
\includegraphics[width=0.245\linewidth]{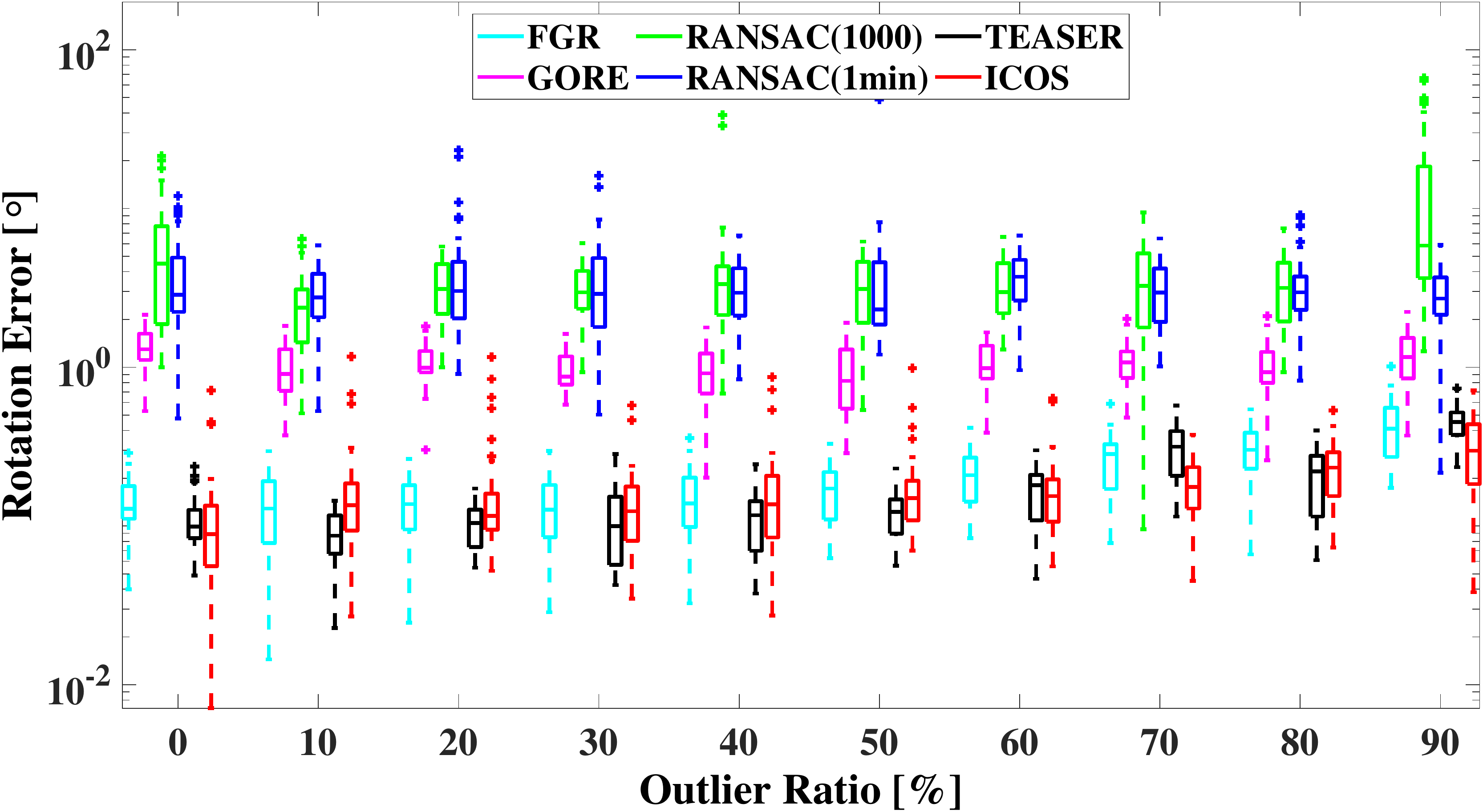}
\includegraphics[width=0.245\linewidth]{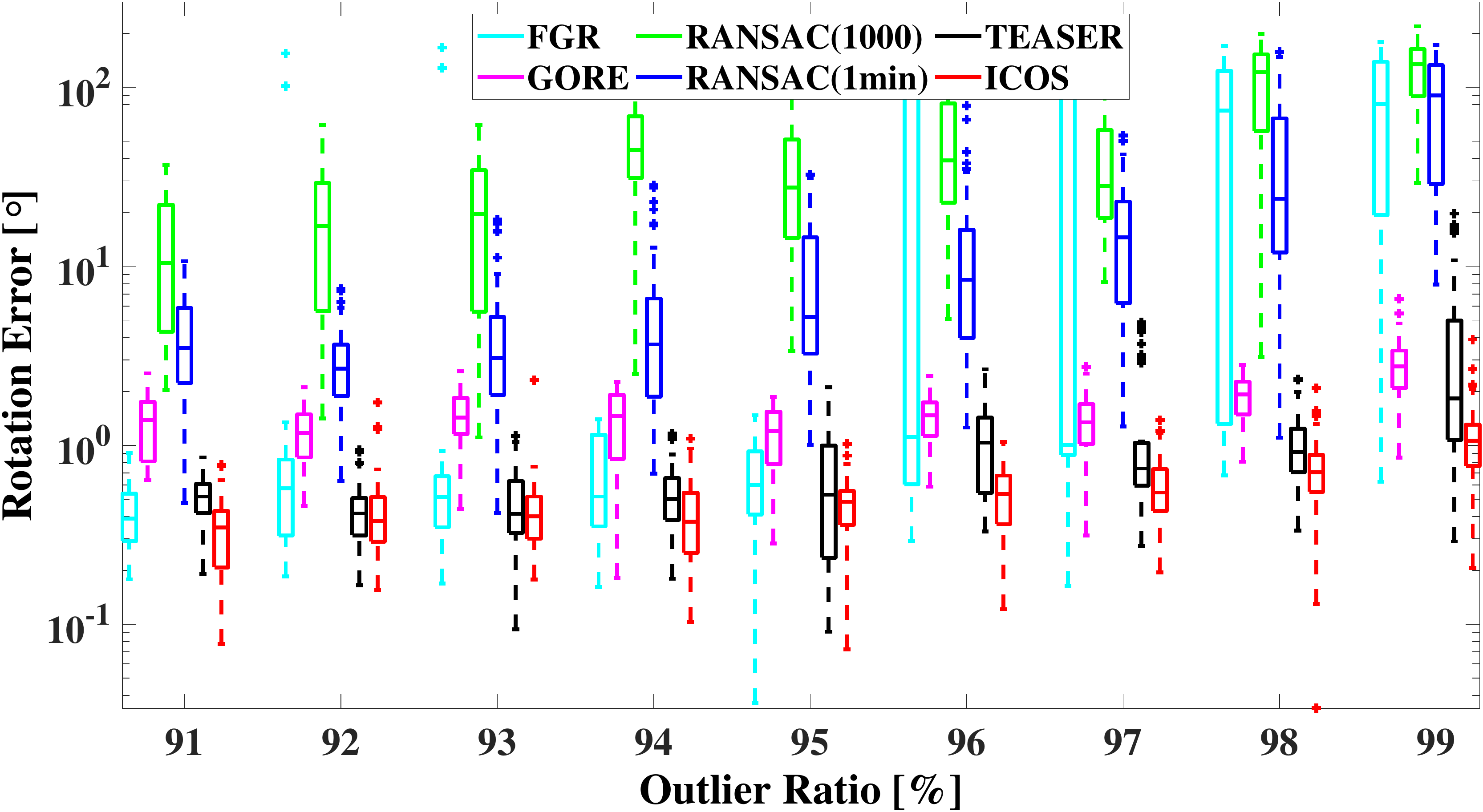}
\end{minipage}
}%

\subfigure{
\begin{minipage}[t]{1\linewidth}
\centering
\includegraphics[width=0.245\linewidth]{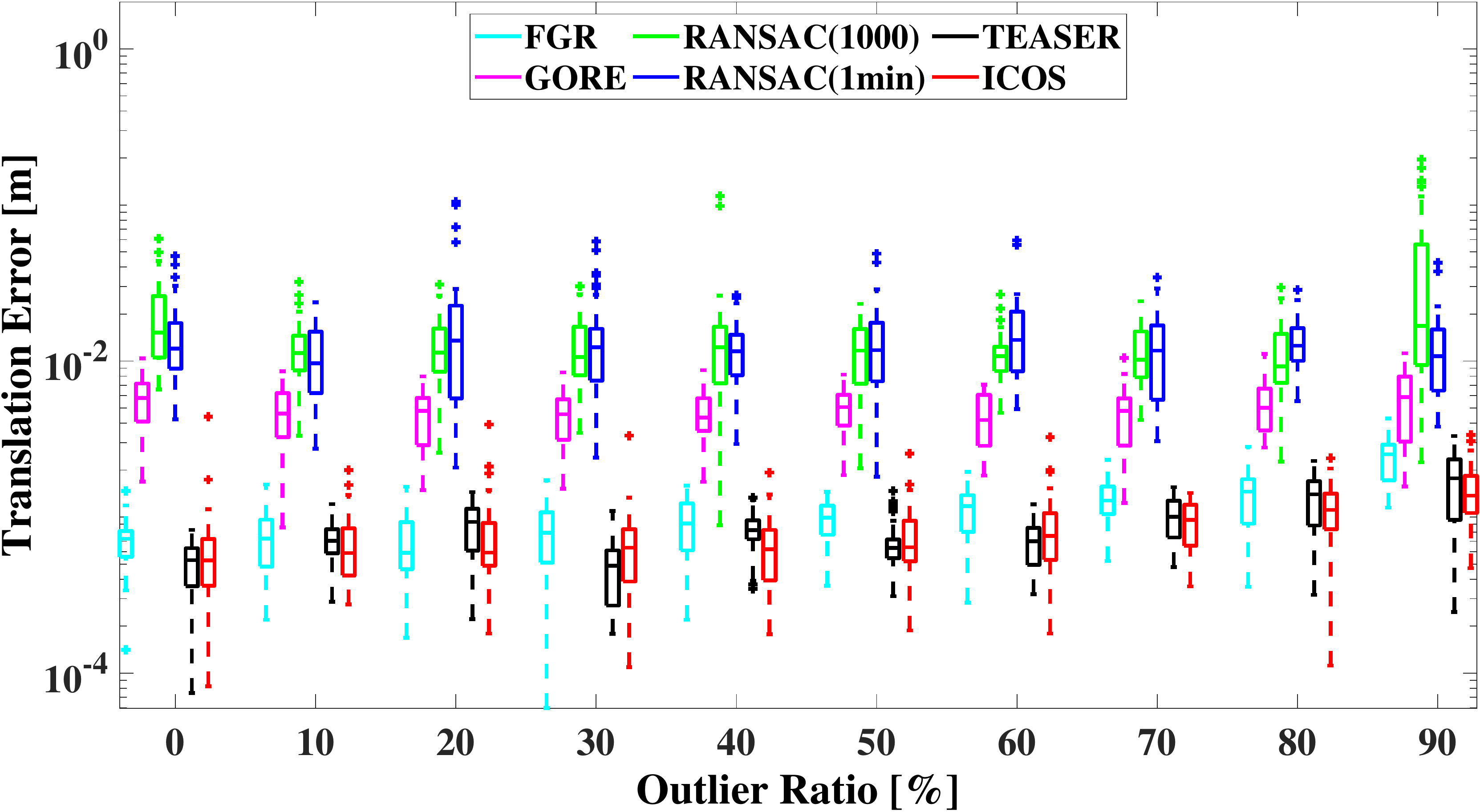}
\includegraphics[width=0.245\linewidth]{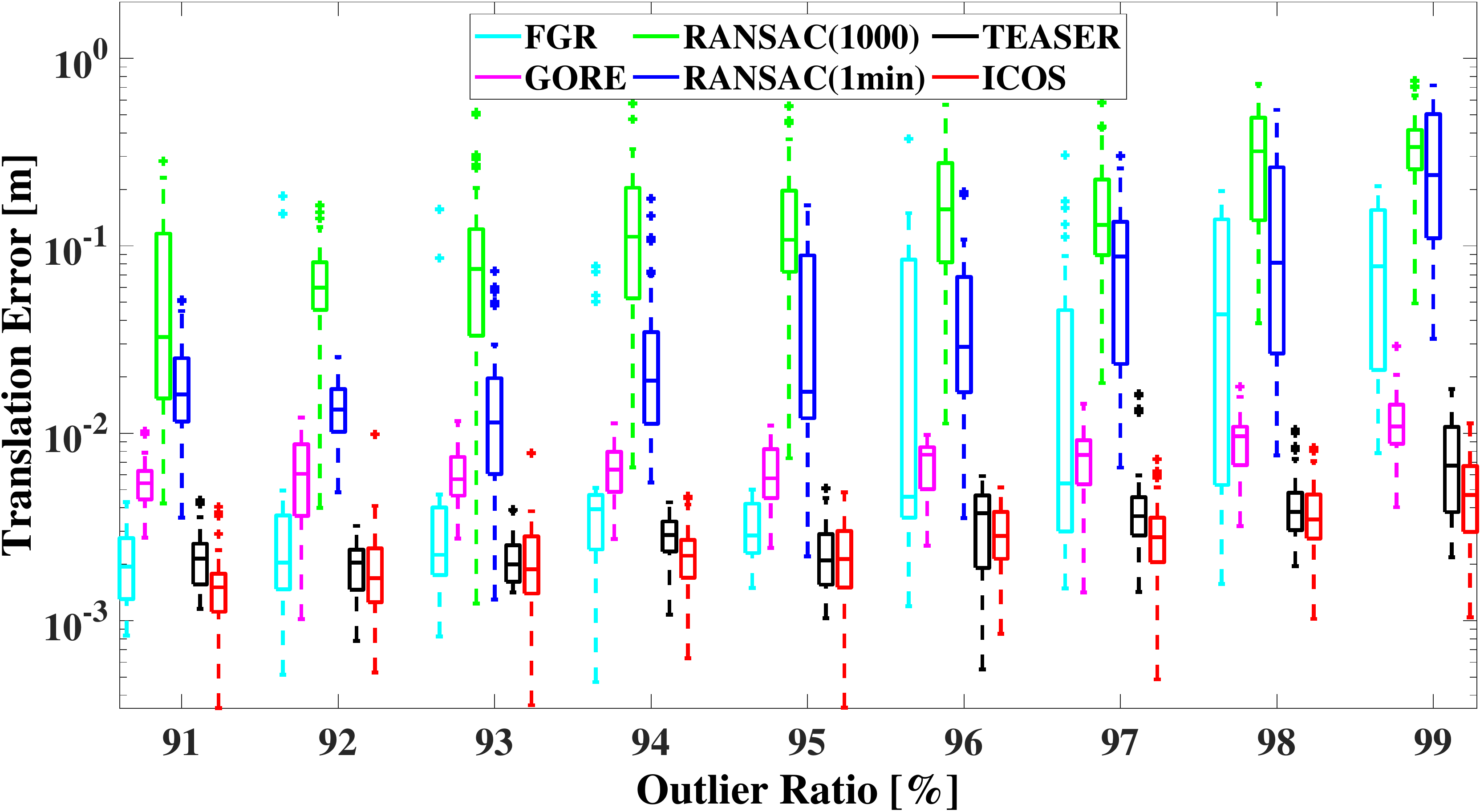}
\includegraphics[width=0.245\linewidth]{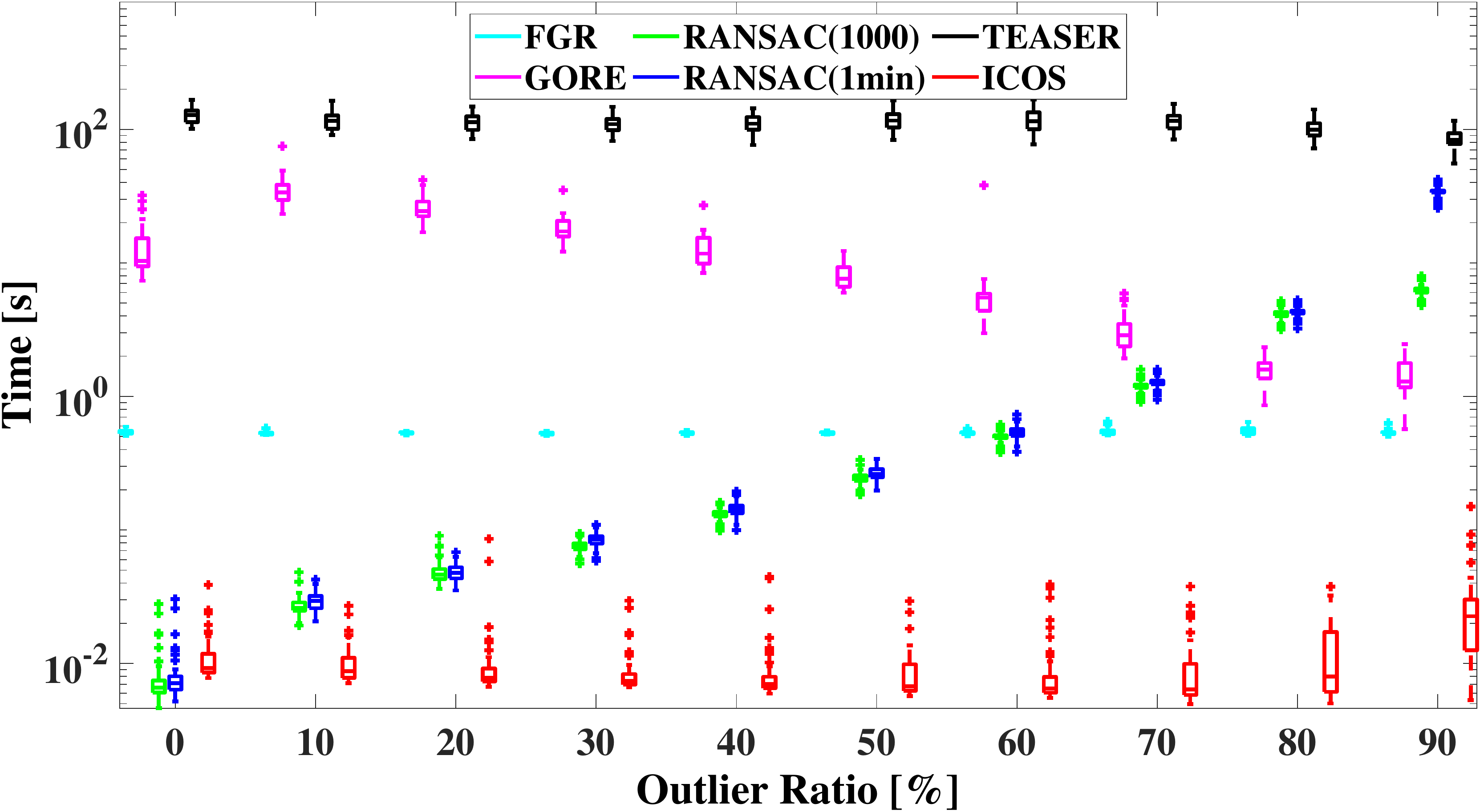}
\includegraphics[width=0.245\linewidth]{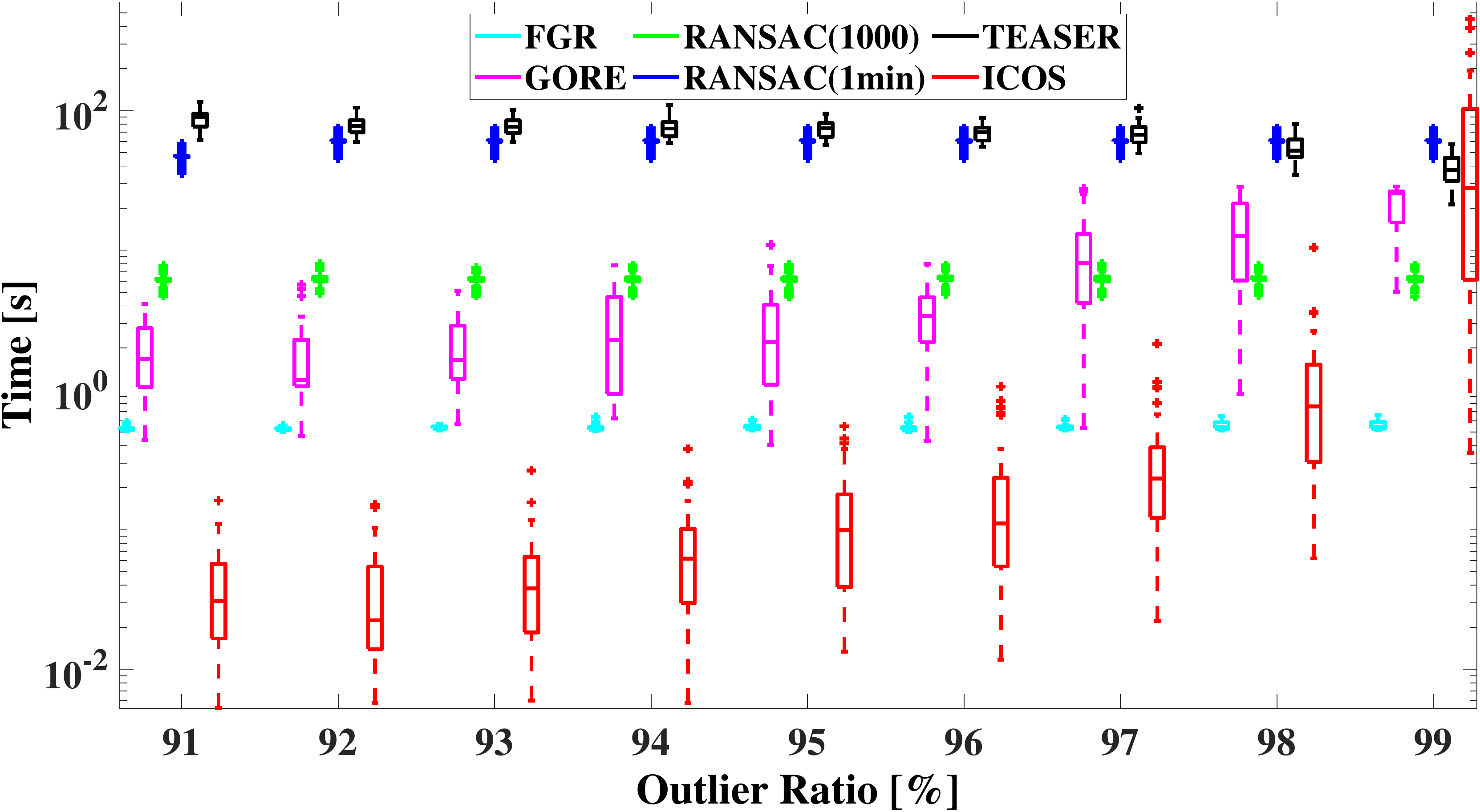}
\end{minipage}
}%

\footnotesize{(b)Point Cloud Registration with Unknown Scale: $\mathit{s}\in(1,5)$}

\subfigure{
\begin{minipage}[t]{1\linewidth}
\centering
\includegraphics[width=0.245\linewidth]{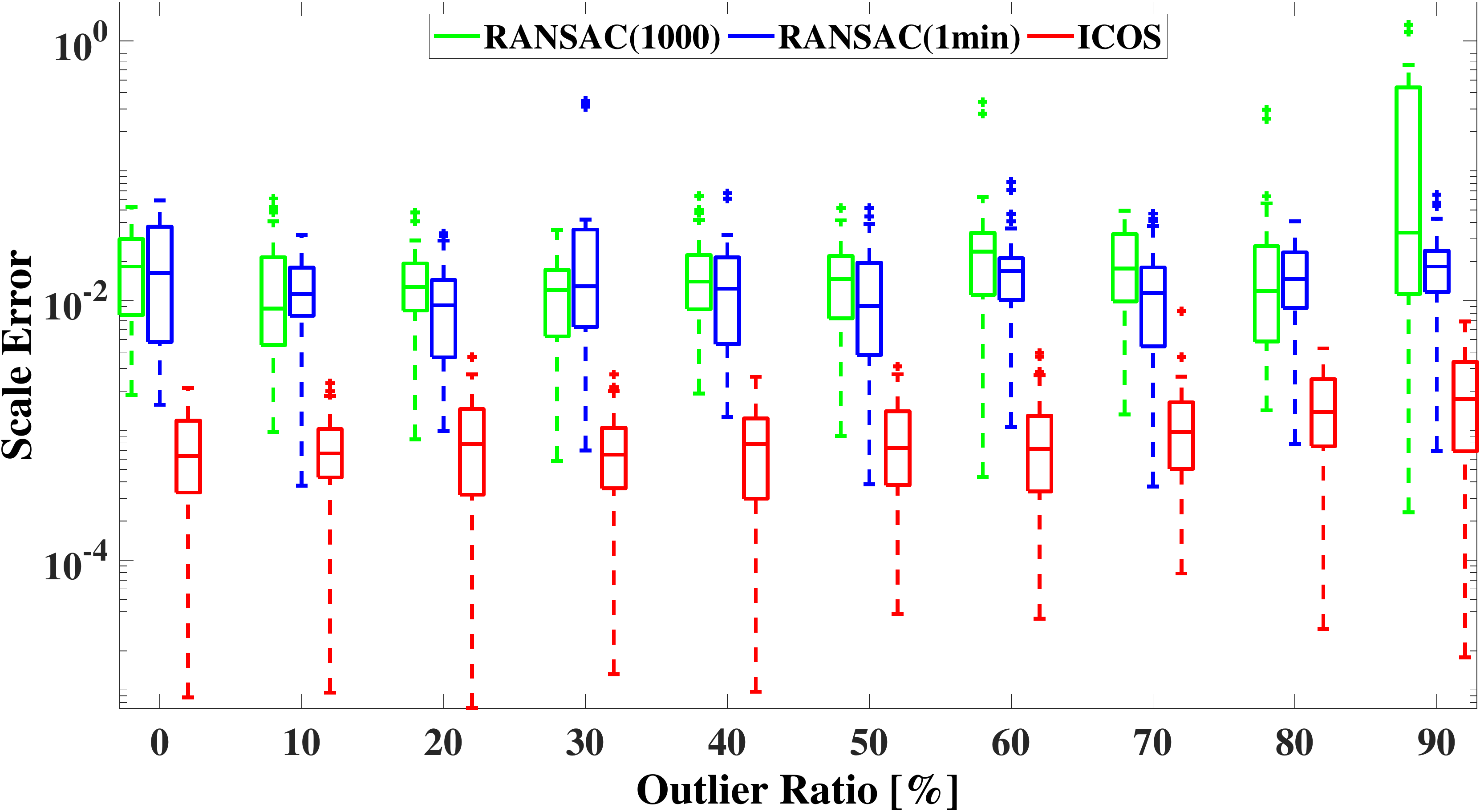}
\includegraphics[width=0.245\linewidth]{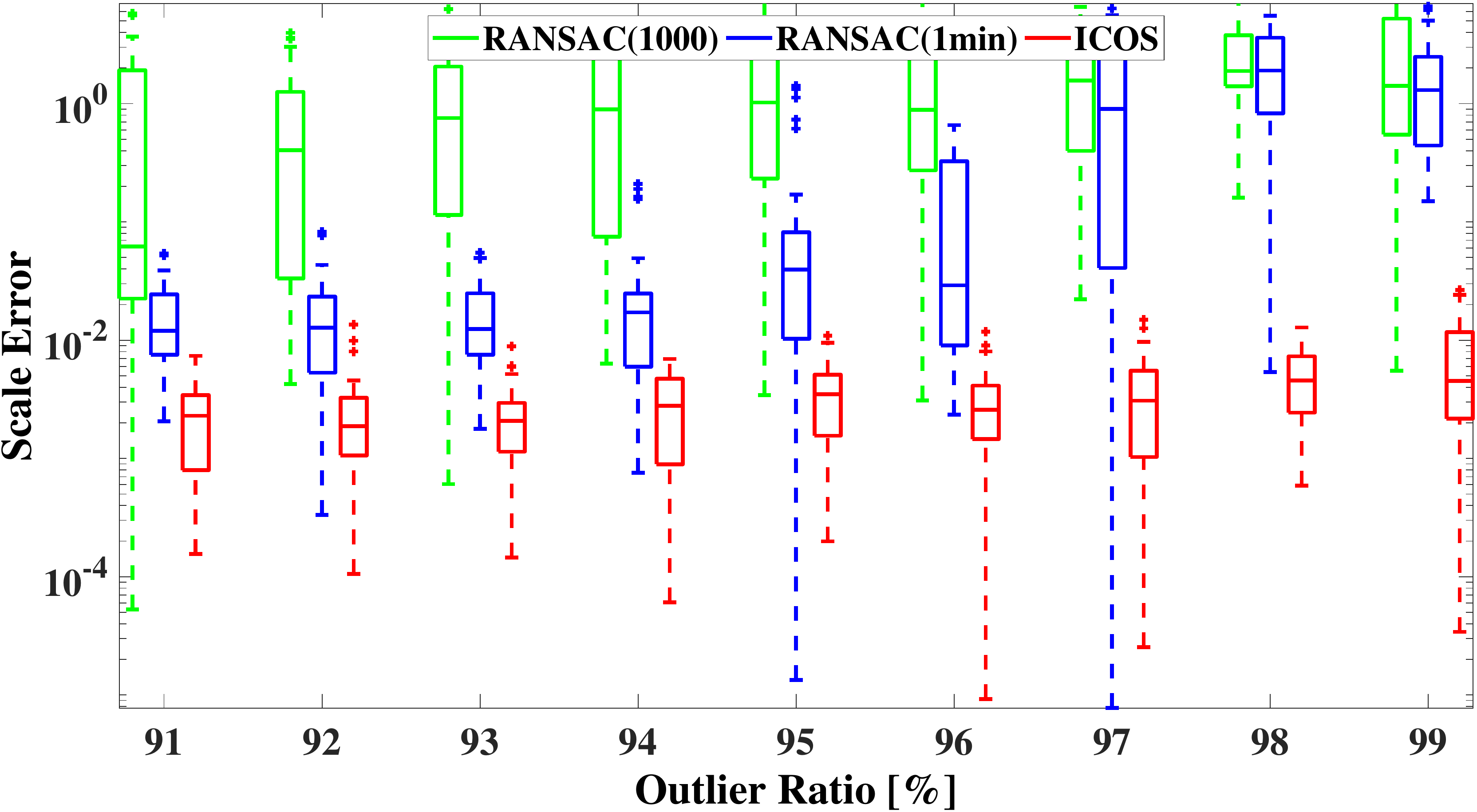}
\includegraphics[width=0.245\linewidth]{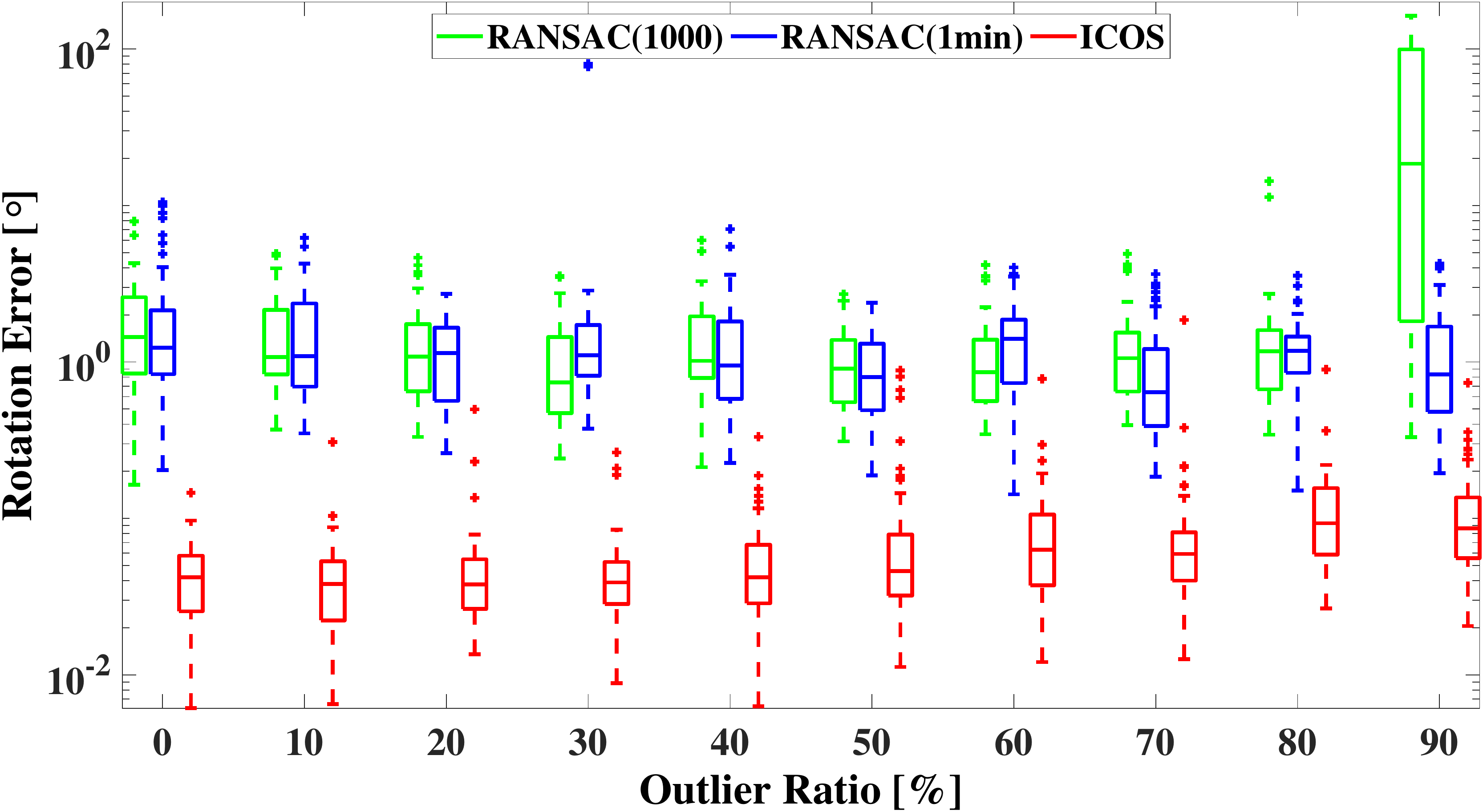}
\includegraphics[width=0.245\linewidth]{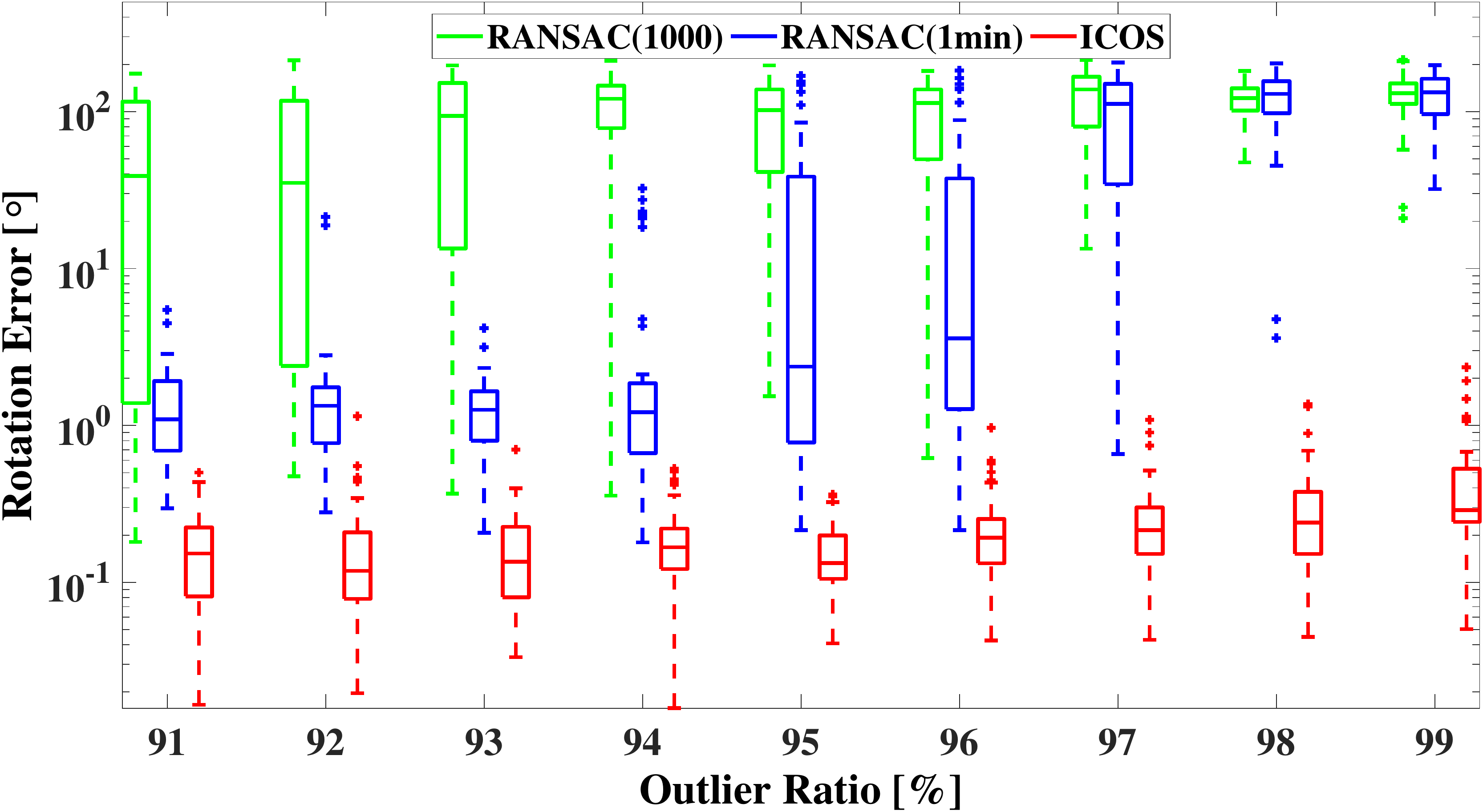}
\end{minipage}
}%

\subfigure{
\begin{minipage}[t]{1\linewidth}
\centering
\includegraphics[width=0.245\linewidth]{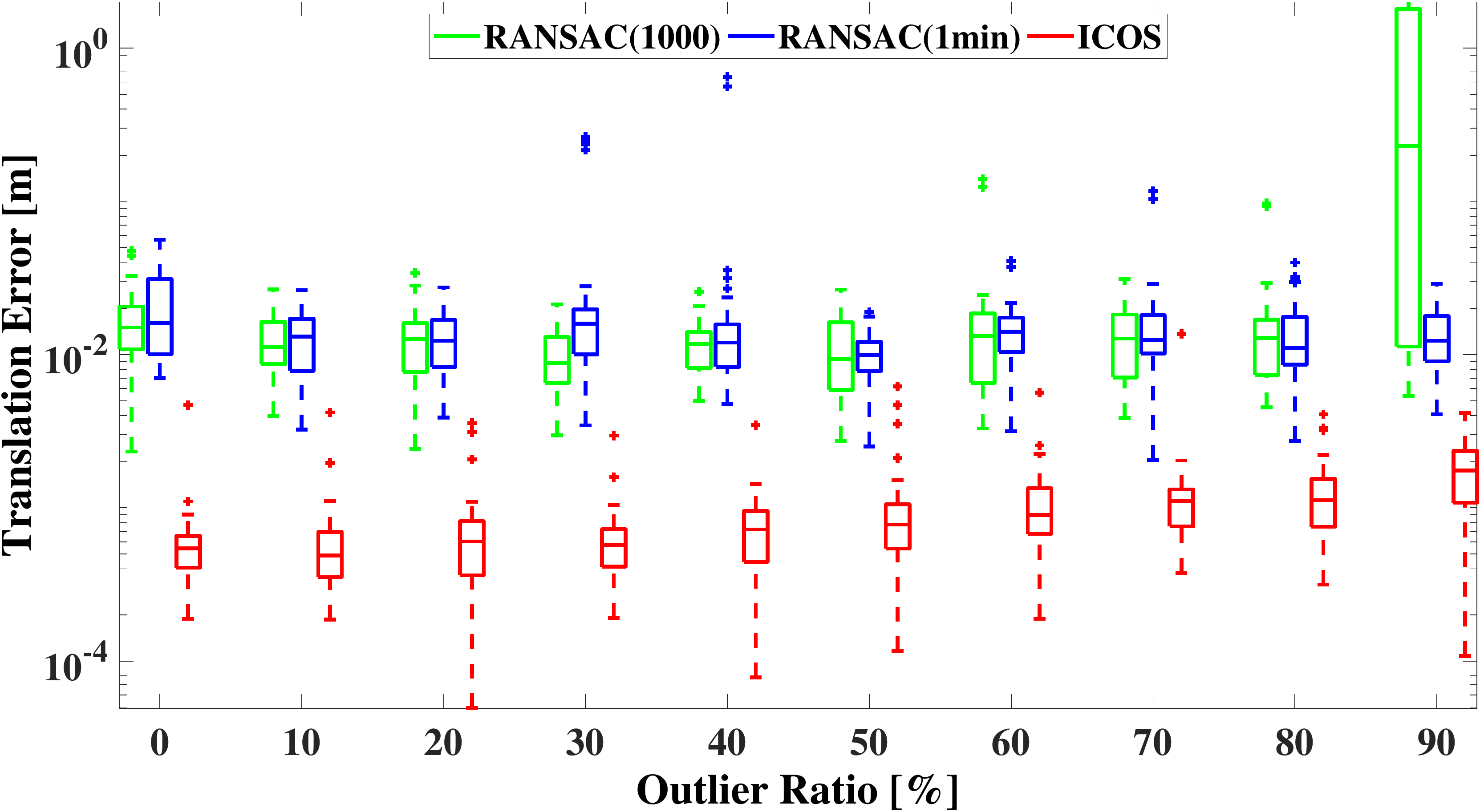}
\includegraphics[width=0.245\linewidth]{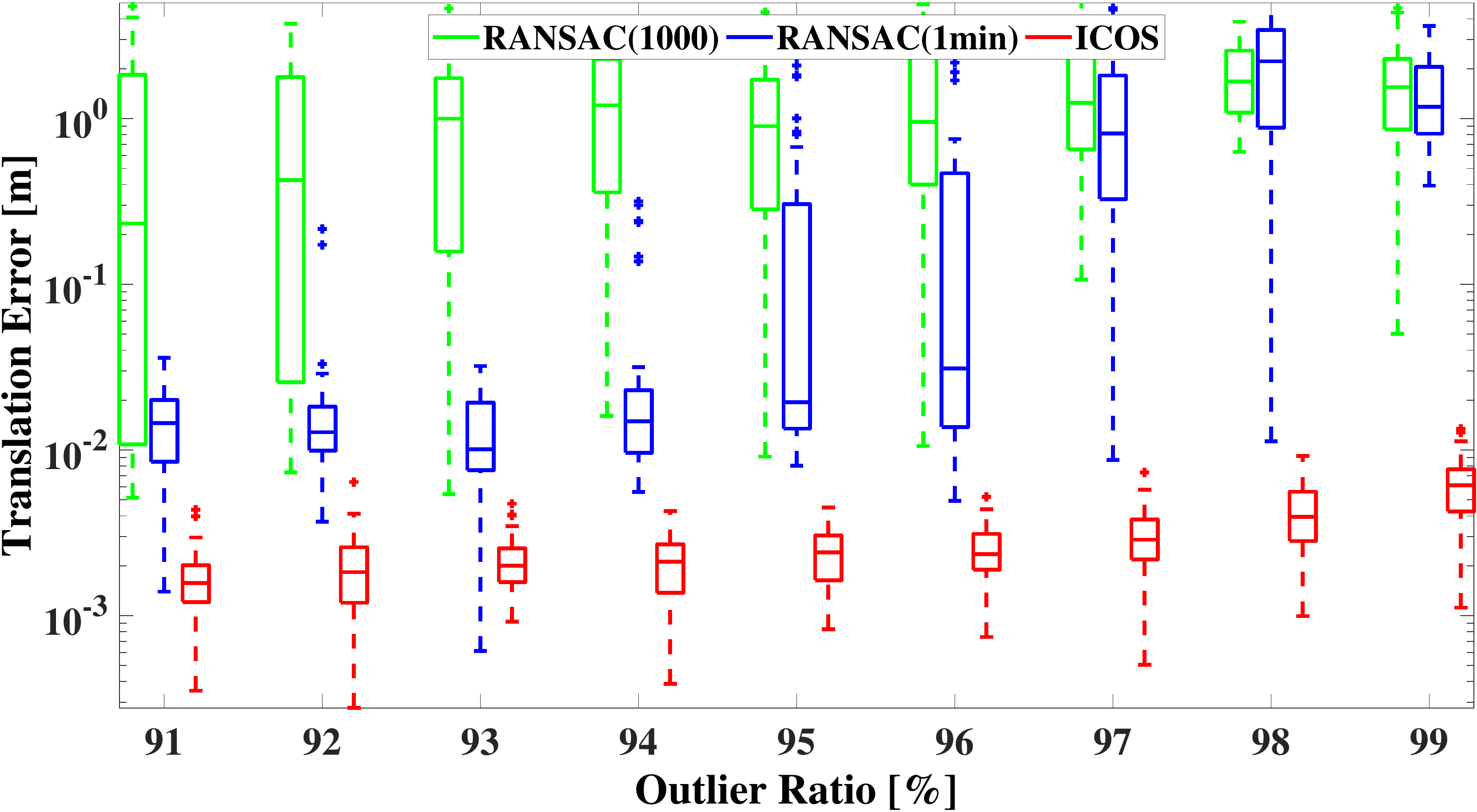}
\includegraphics[width=0.245\linewidth]{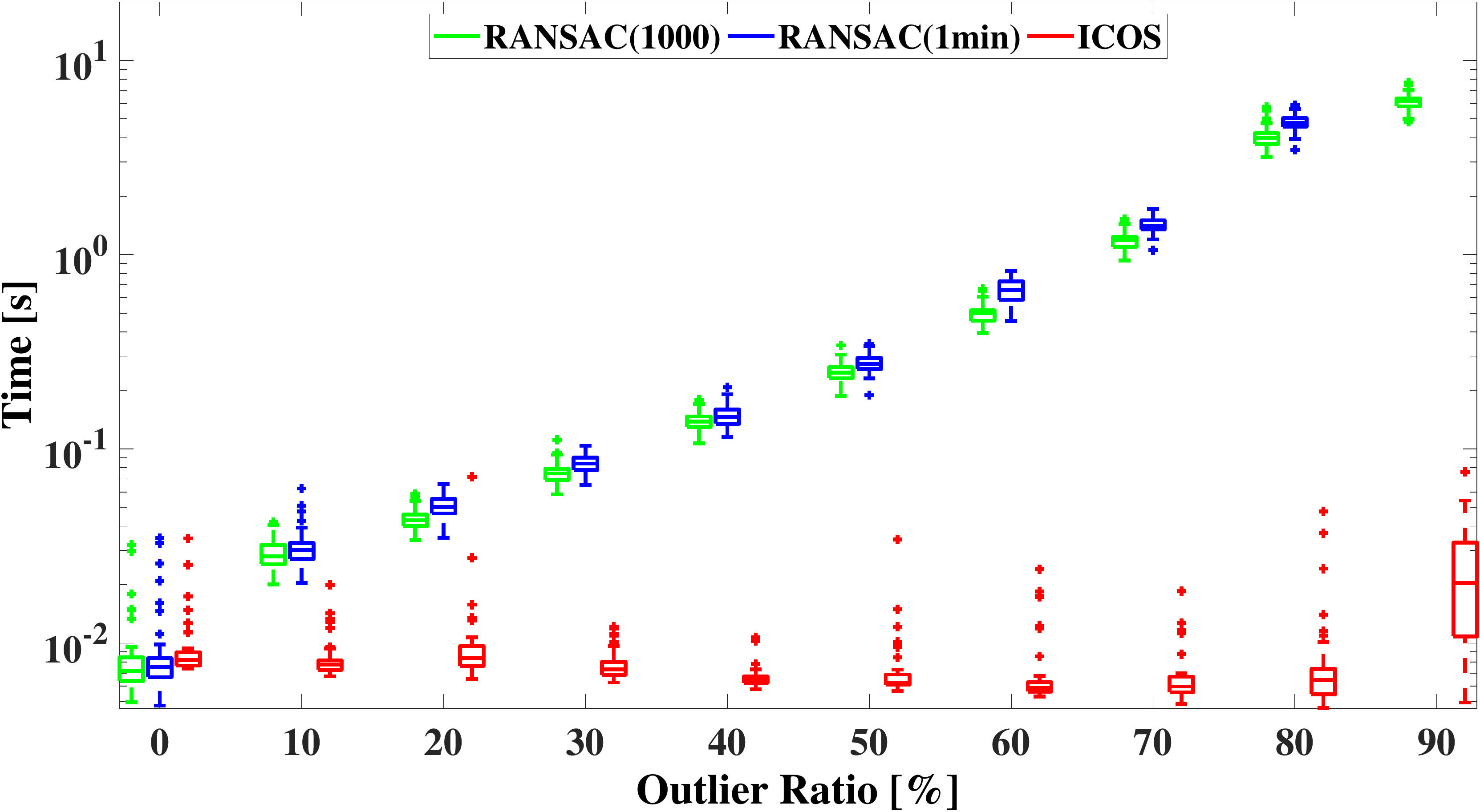}
\includegraphics[width=0.245\linewidth]{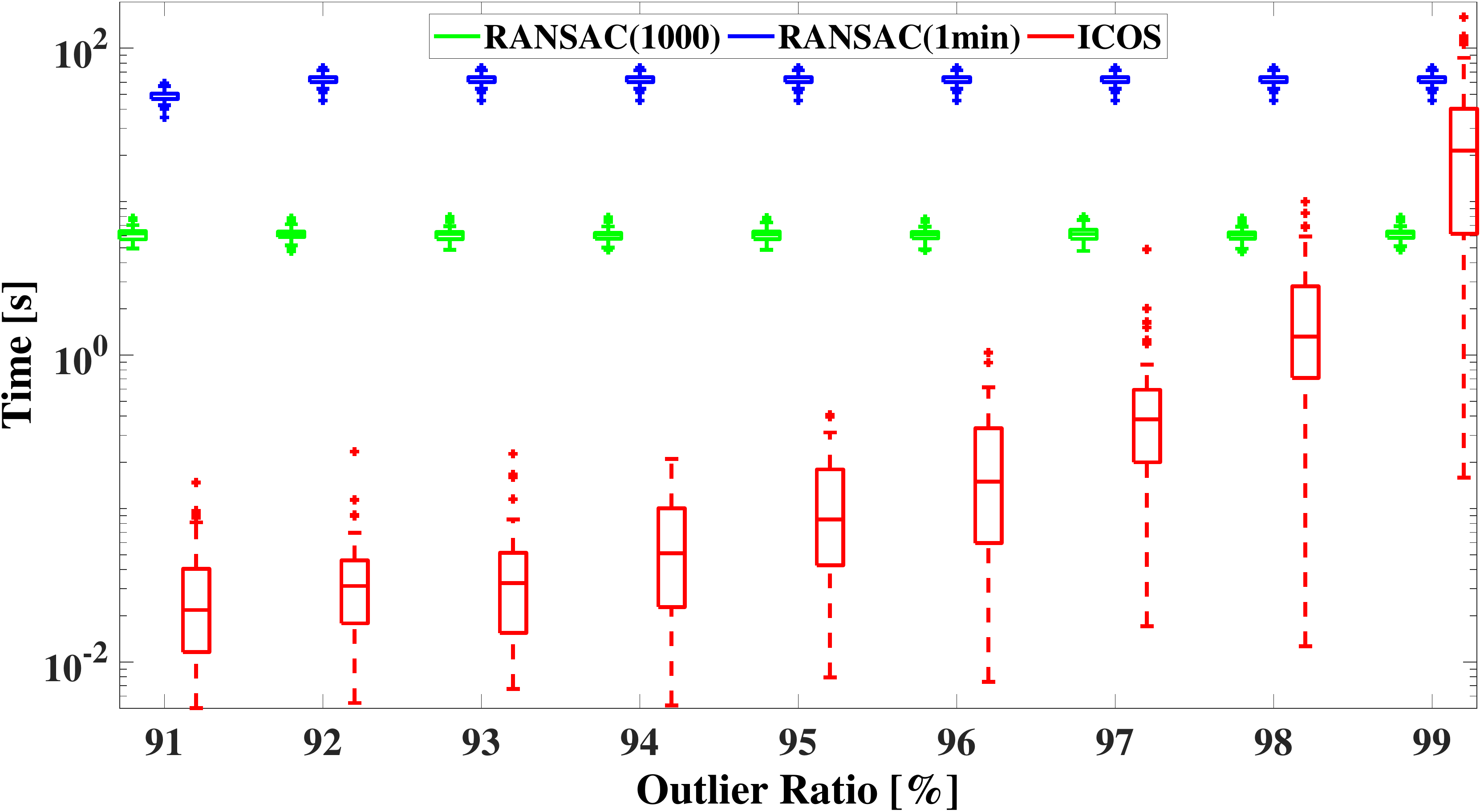}
\end{minipage}
}%

\centering
\caption{Benchmarking of point cloud registration over the `Armadillo'~\cite{curless1996volumetric}. Left-top: Examples of a known-scale and an unknown-scale point cloud registration problem with 98\% outliers. (a) Known-scale registration results w.r.t. increasing outlier ratio (0-99\%). (b) Unknown-scale registration results w.r.t. increasing outlier ratio (0-99\%).}
\label{Armadillo}
\vspace{-9pt}
\end{figure*}

The estimation errors in all experiments are computed as:

\begin{equation}\label{errors}
\begin{gathered}
E_{\mathit{s}}=\left| \mathit{\hat{s}}-\mathit{{s}}_{gt}\right|,\\
E_{\boldsymbol{t}}=\| \boldsymbol{\hat{t}}-\boldsymbol{{t}}_{gt}\|,
\end{gathered}
\end{equation}
and $E_{\boldsymbol{R}}$ is the same as~\eqref{Rot-error}, where subscripts `$gt$' all indicate the ground-truth.

\textbf{Known-scale problems:} We adopt FGR, GORE, RANSAC, TEASER and ICOS for the known-scale point cloud registration problems. As for RANSAC, we use Horn's minimal (3-point) method \cite{horn1987closed} to solve the rigid transformation and the confidence is set to 0.995. There are two stop conditions for RANSAC: (i) RANSAC with at most 1000 iterations, named RANSAC(1000), and (ii) RANSAC with at most 1 minute of runtime, named RANSAC(1min). In terms of TEASER, we adopt the GNC heuristic version as in~\cite{yang2020teaser}, which uses GNC-TLS to solve the large robust SDP problem, rather than using convex relaxation, but differently: (i) we do not apply parallelism programming in order to ensure the fairness in runtime evaluation for all the tested solvers, and (ii) TEASER is completely implemented in Matlab, not in C++, using the fast maximum clique solver~\cite{eppstein2010listing}.

\textbf{Unknown-scale problems:} For unknown-scale registration, we only compare ICOS with RANSAC because: (i) FGR and GORE are not designed for scale estimation originally, and (ii) the scale estimator Adaptive Voting in TEASER is too slow for handling 1000 correspondences (requiring tens of minutes per run) and is also unable to tolerate over 90\% outliers (as will be shown in Fig.~\ref{Scale}), thus no need to be compared.

\textbf{Results.} The main benchmarking results are shown in Fig.~\ref{bunny} (over the `bunny' point cloud) and Fig.~\ref{Armadillo} (over the `Armadillo' point cloud). We can observe that: (i) in the known-scale registration problem, though GORE, TEASER and our ICOS are all robust against 99\% outliers, TEASER and ICOS are generally more accurate, (ii) RANSAC(1000) fails at 90\%, while FGR and RANSAC(1min) fail at over 95\%, not as robust as ICOS, (iii) in the unknown-scale registration problems, ICOS can still address 99\% outliers highly robustly, and (iv) ICOS is the fastest solver when the outlier ratio is below 98\% (Note that though FGR seems faster at over 96\%, it already becomes hardly robust).

\subsection{Evaluation on the Inlier Recall Ratio}

We test the capability of ICOS for recalling correct inliers from the outlier-corrupted putative correspondence set in rotation search and known-scale and unknown-scale registration with the same experimental setup as in Section~\ref{RS-overall} and~\ref{overall} setting $N=1000$ for all the problems, and the results are shown in Fig.~\ref{Recall}. The recall ratio is defined as the ratio of the number of true inliers found by ICOS to the ground-truth number of inliers among the original correspondences. We can observe that ICOS is able to recall approximately 100\% (no less than 99\%) of the inliers from the correspondences with outlier ratio ranging from 0\% to 99\%, which underlies its high estimation accuracy against noise in the benchmarking above.

\subsection{Benchmarking on Scale Estimation in Registration}

We benchmark ICOS for scale estimation alone, compared with the Adaptive Voting in TEASER/TEASER++ and the 1-point RANSAC in~\cite{li2021point}. Following the setup in Section~\ref{overall}, we further downsample the `bunny' to only 100 points because solver Adaptive Voting will run in tens of minutes with 1000 correspondences. In each run, the scale is randomly generated within $(1,5)$ meters. Fig.~\ref{Scale} shows the scale errors and the runtime w.r.t. increasing outlier ratio. We can see that both Adaptive Voting and 1-point RANSAC break at 90\% outliers, while ICOS can stably tolerate 90\% outliers and is faster than the other two competitors. In fact, ICOS can be robust against up to 99\% outliers, as has been shown in Section~\ref{overall}.

\subsection{Qualitative Registration Results over Real Datasets}

In addition to the quantitative point cloud registration  results rendered above, we provide qualitative results of the registration problems over multiple real point clouds, including the `bunny' and `Armadillo' from~\cite{curless1996volumetric}, the `Mario' and `Squirrel' from the SHOT dataset\footnote[2]{http://vision.deis.unibo.it/list-all-categories/78-cvlab/80-shot}, and the `city' and `castle' from the RGB-D Scans dataset~\cite{zeisl2013automatic}, for concrete evaluation.

These experiments are all set up as follows: (i) we first transform a certain point cloud with a random transformation (with both known scale and unknown scale), (ii) we then adopt the FPFH 3D feature descriptor~\cite{rusu2009fast} to match and build putative correspondences between the transformed point cloud (colored in magenta) and the initial one (colored in blue), in which outliers are bound to exist, and (iii) we apply ICOS to solve the transformation and then reproject the transformed point cloud back to its initial place with the transformation solved. The qualitative results are shown in Fig.~\ref{FPFH}. Note that after reprojection, the better the magenta point cloud overlaps with the blue one, the smaller the registration errors can be.

From Fig.~\ref{FPFH}, we can see that ICOS yields rather satisfactory robust registration results.

\begin{figure}[t]
\centering

\setlength{\tabcolsep}{0mm}

\begin{tabular}{cccc}

$\mathit{s}=1$ & Result by ICOS & $\mathit{s}\in (1,5)$ & Result by ICOS \\
\hline
&&& \\
\includegraphics[width=0.25\linewidth]{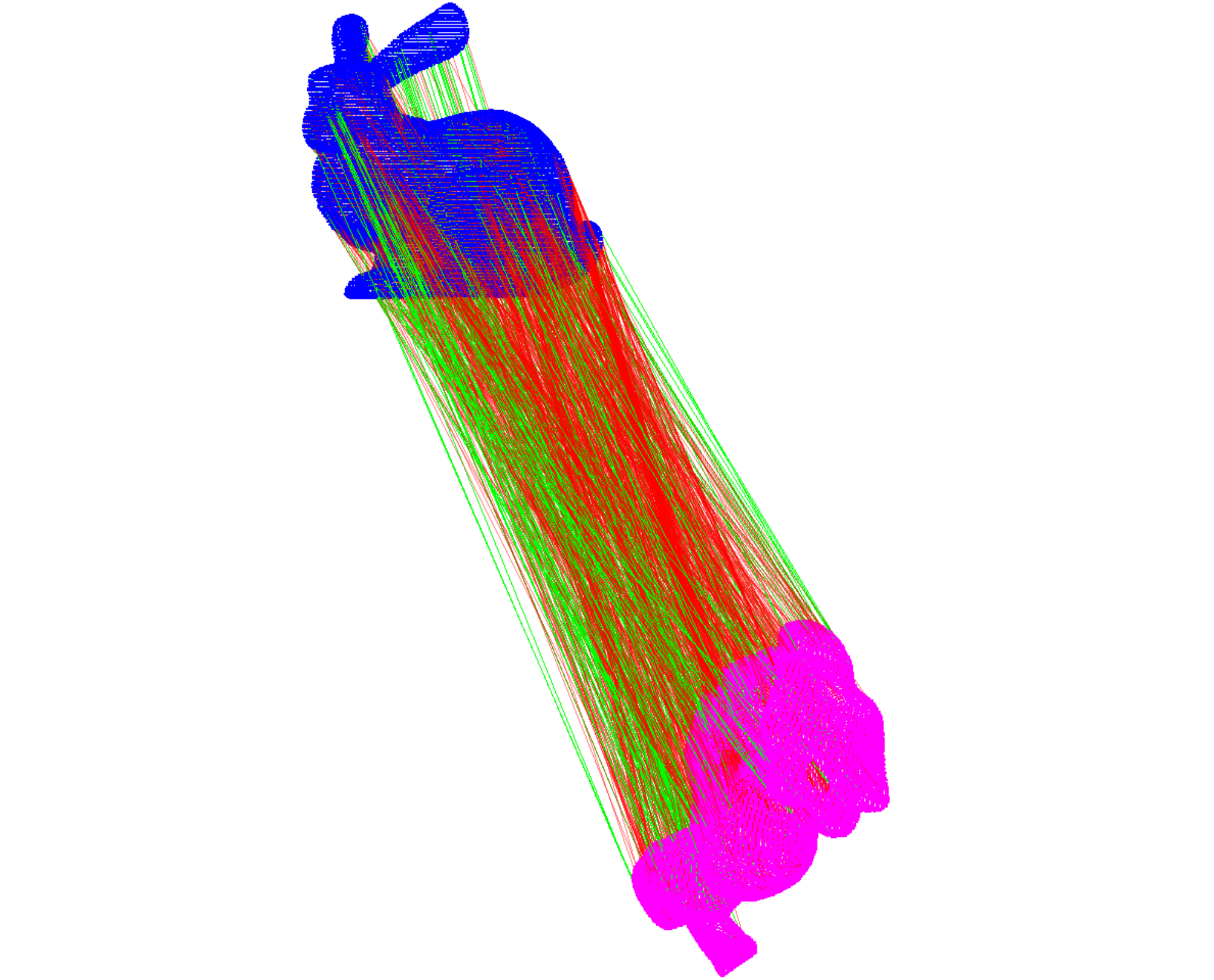}
&\includegraphics[width=0.15\linewidth]{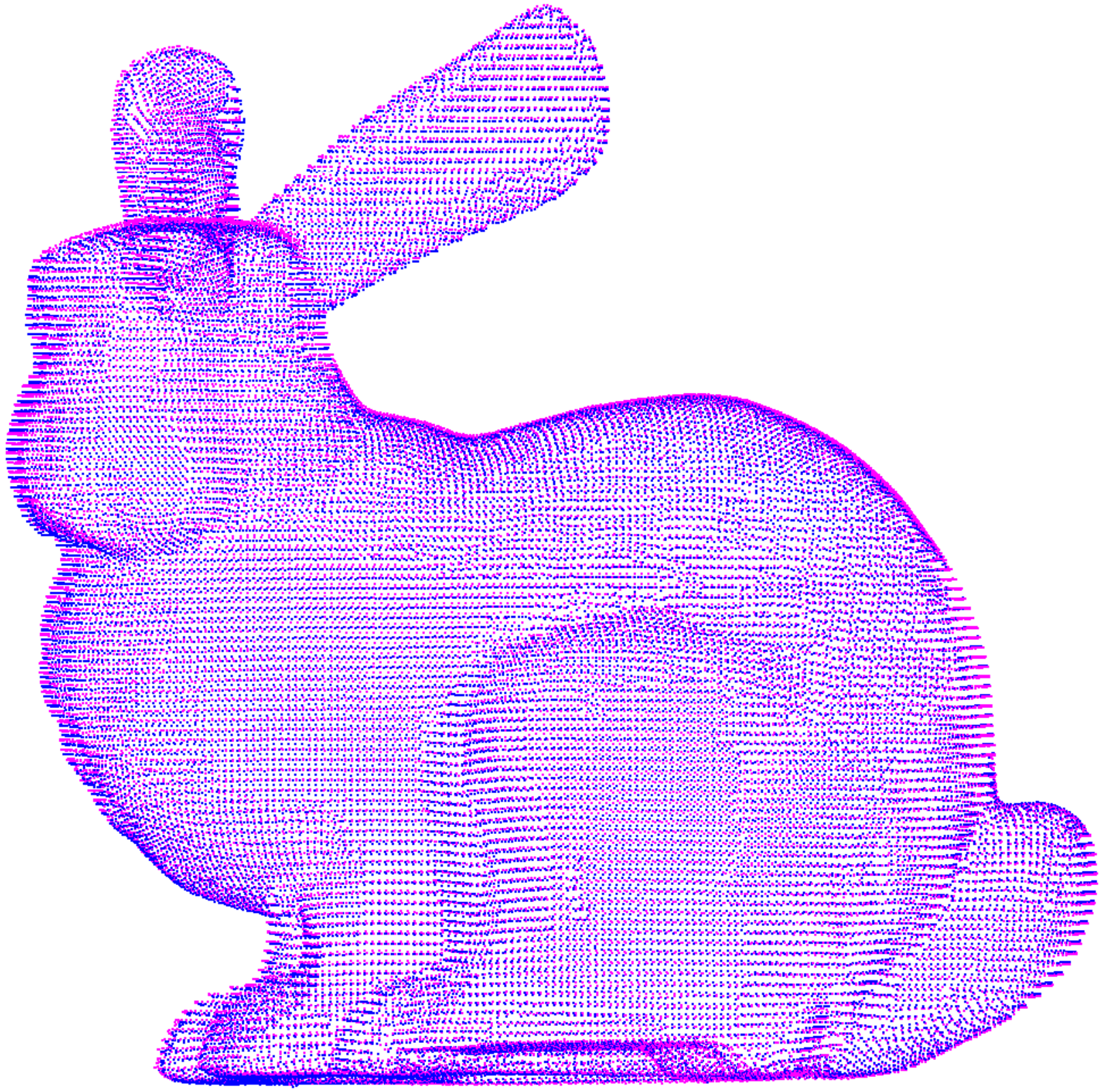}
&\includegraphics[width=0.25\linewidth]{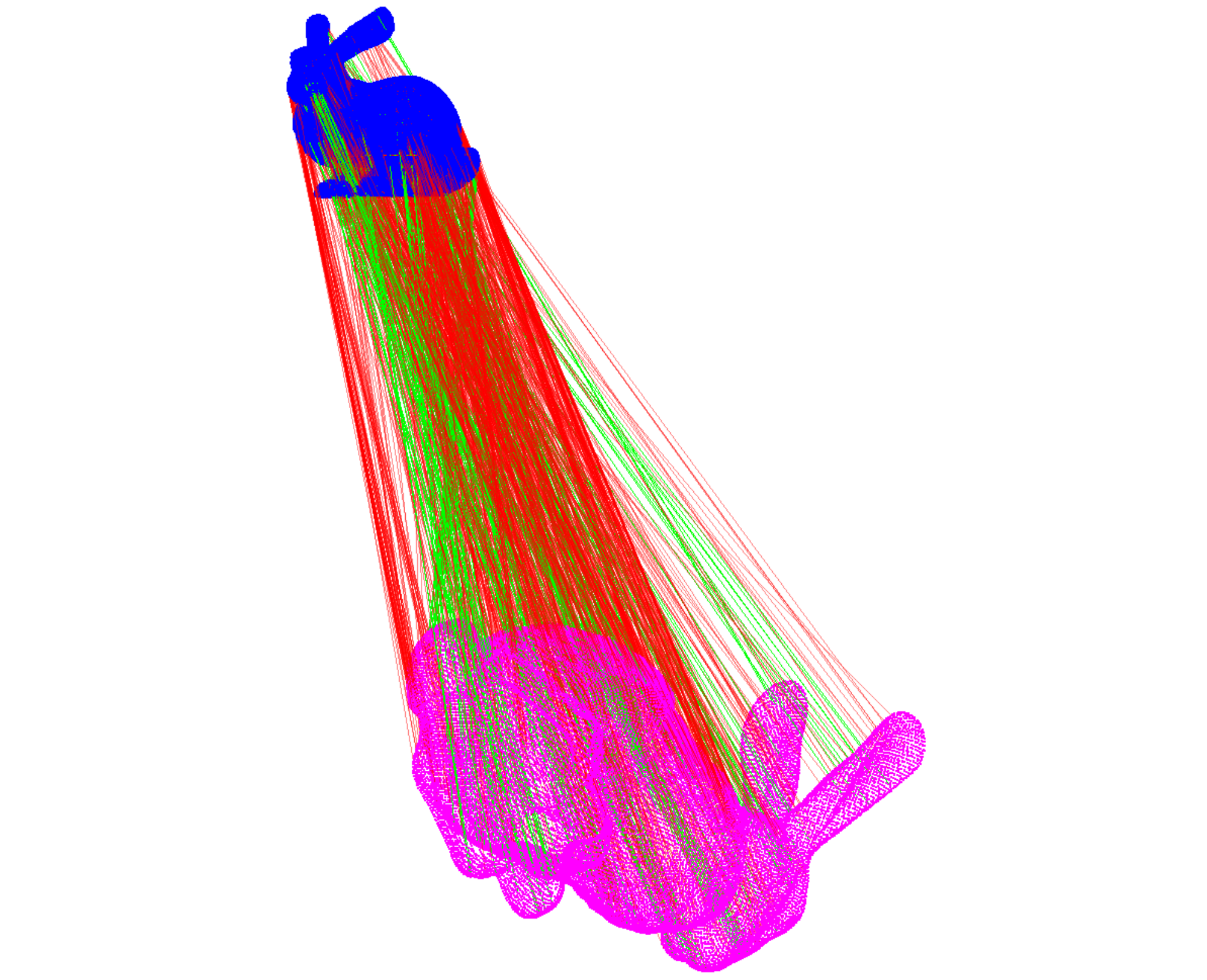}
&\includegraphics[width=0.15\linewidth]{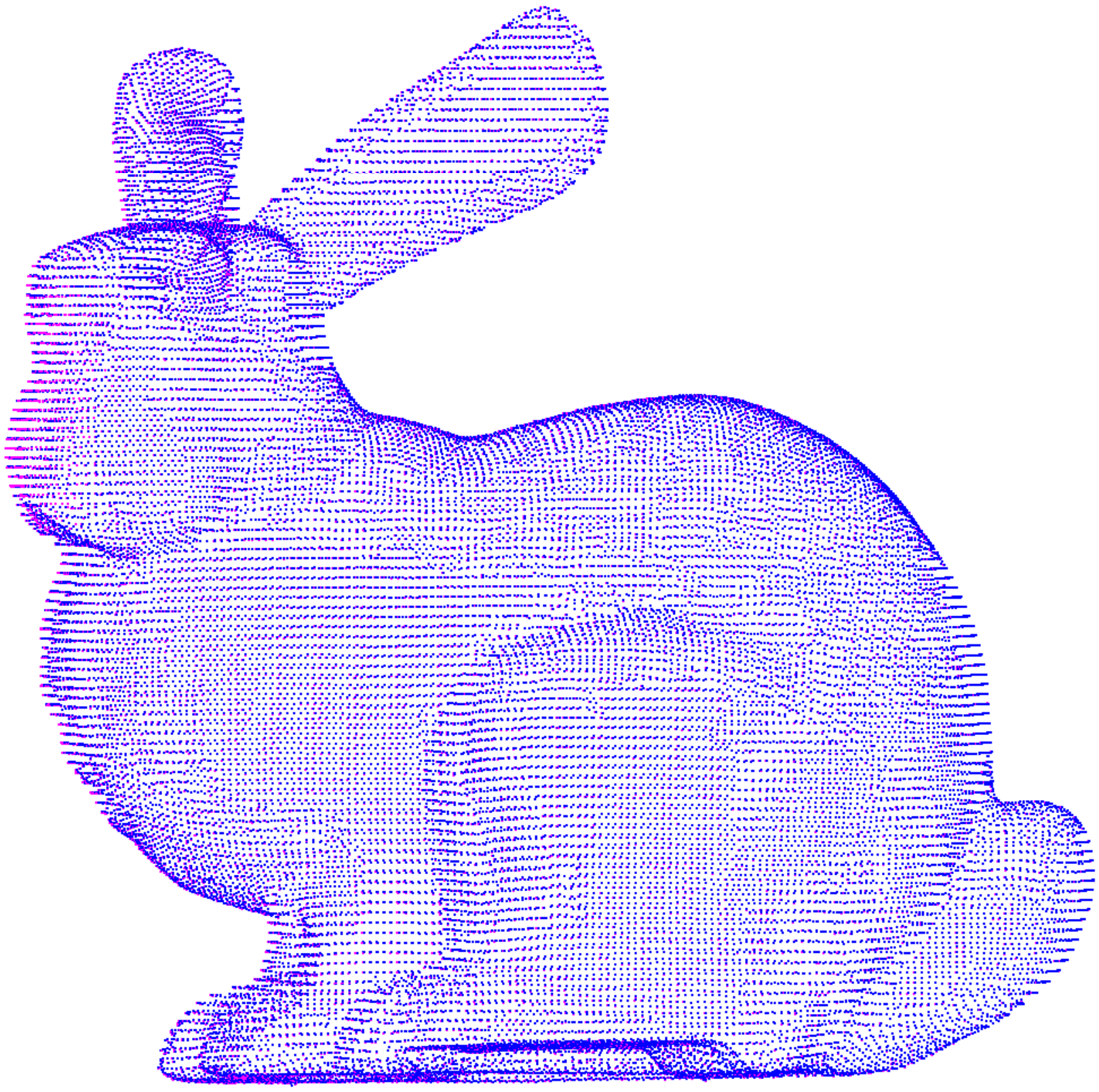} \\
\includegraphics[width=0.25\linewidth]{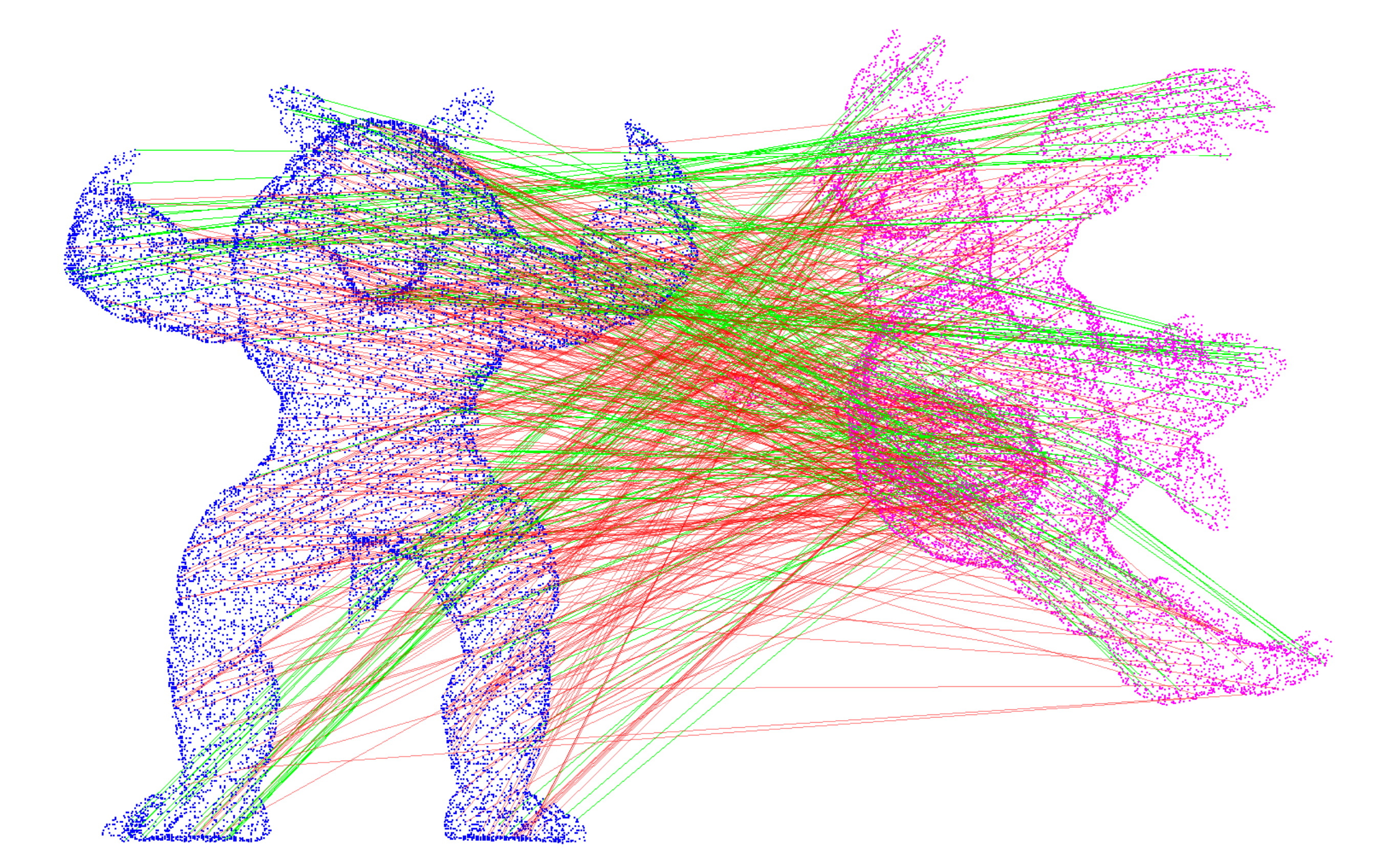} 
&\includegraphics[width=0.15\linewidth]{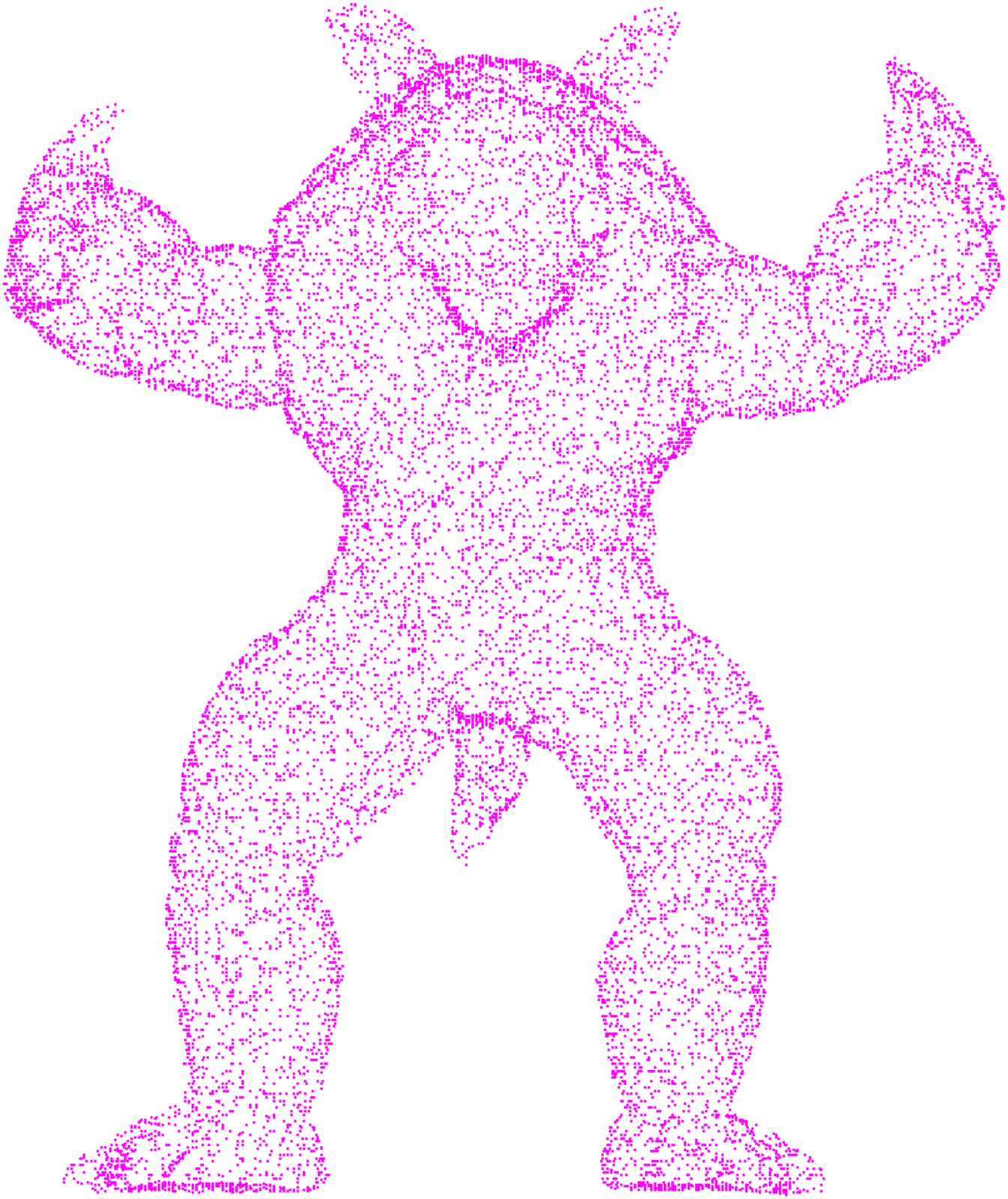}  &\includegraphics[width=0.25\linewidth]{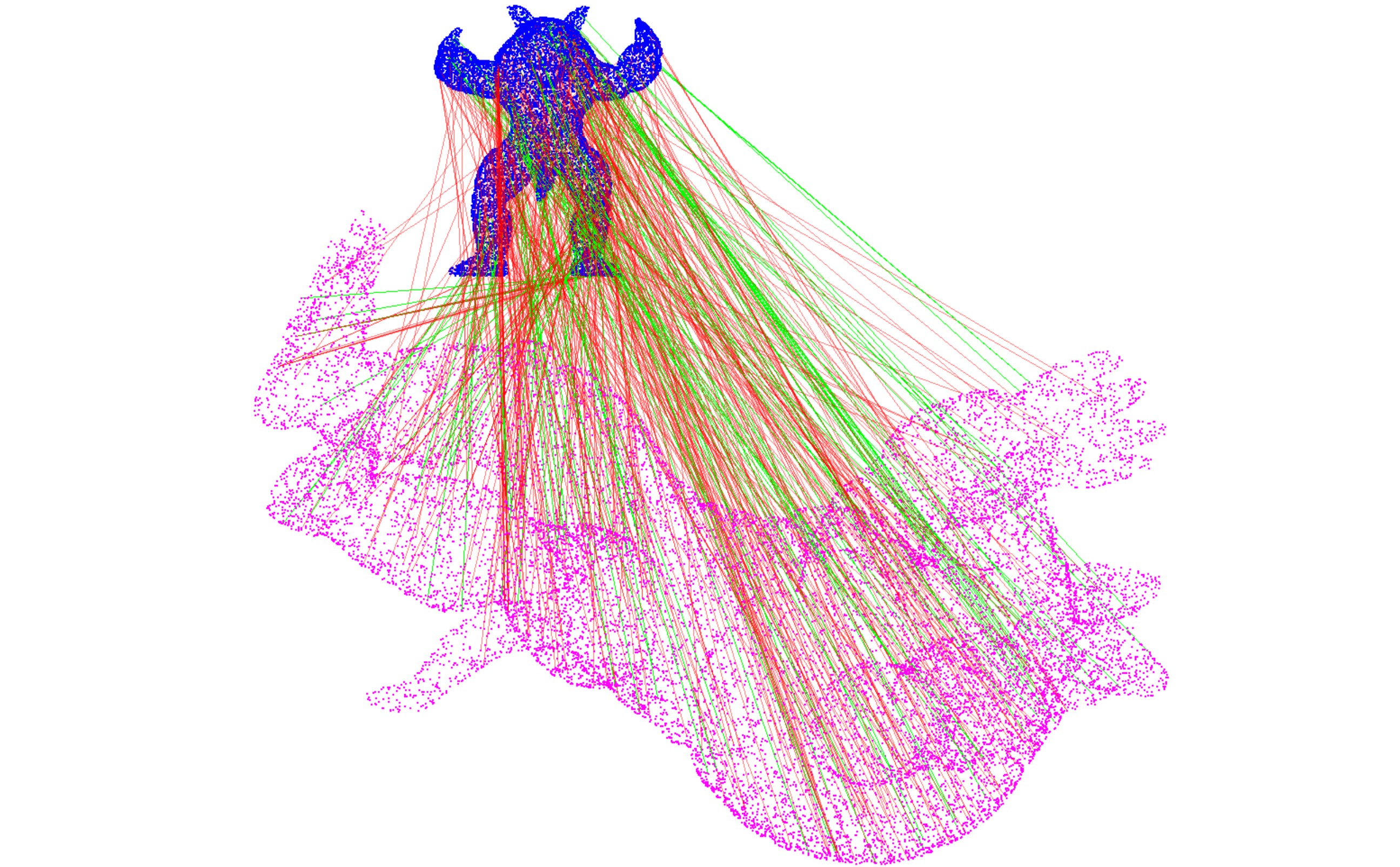}
&\includegraphics[width=0.15\linewidth]{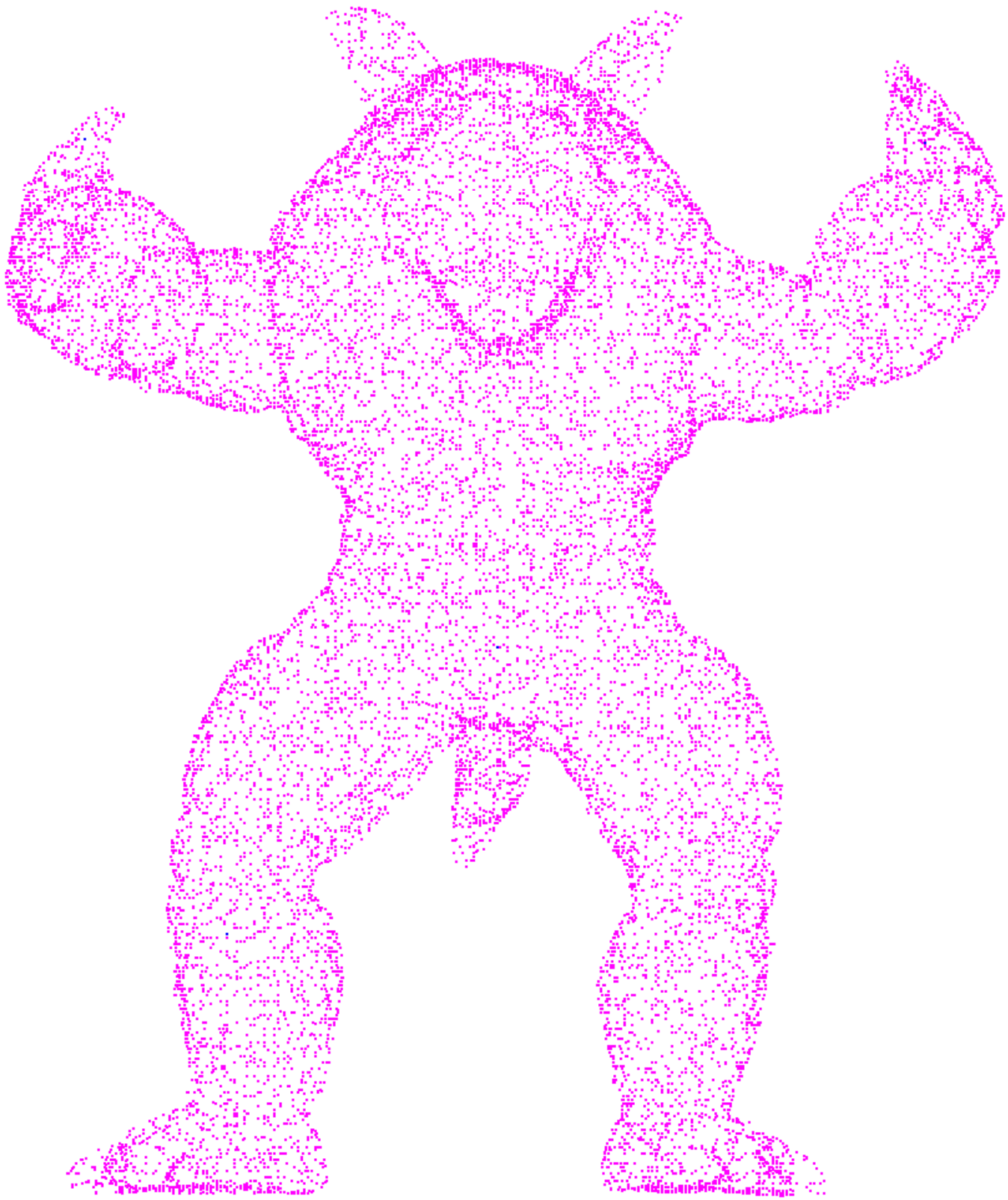} \\

\includegraphics[width=0.25\linewidth]{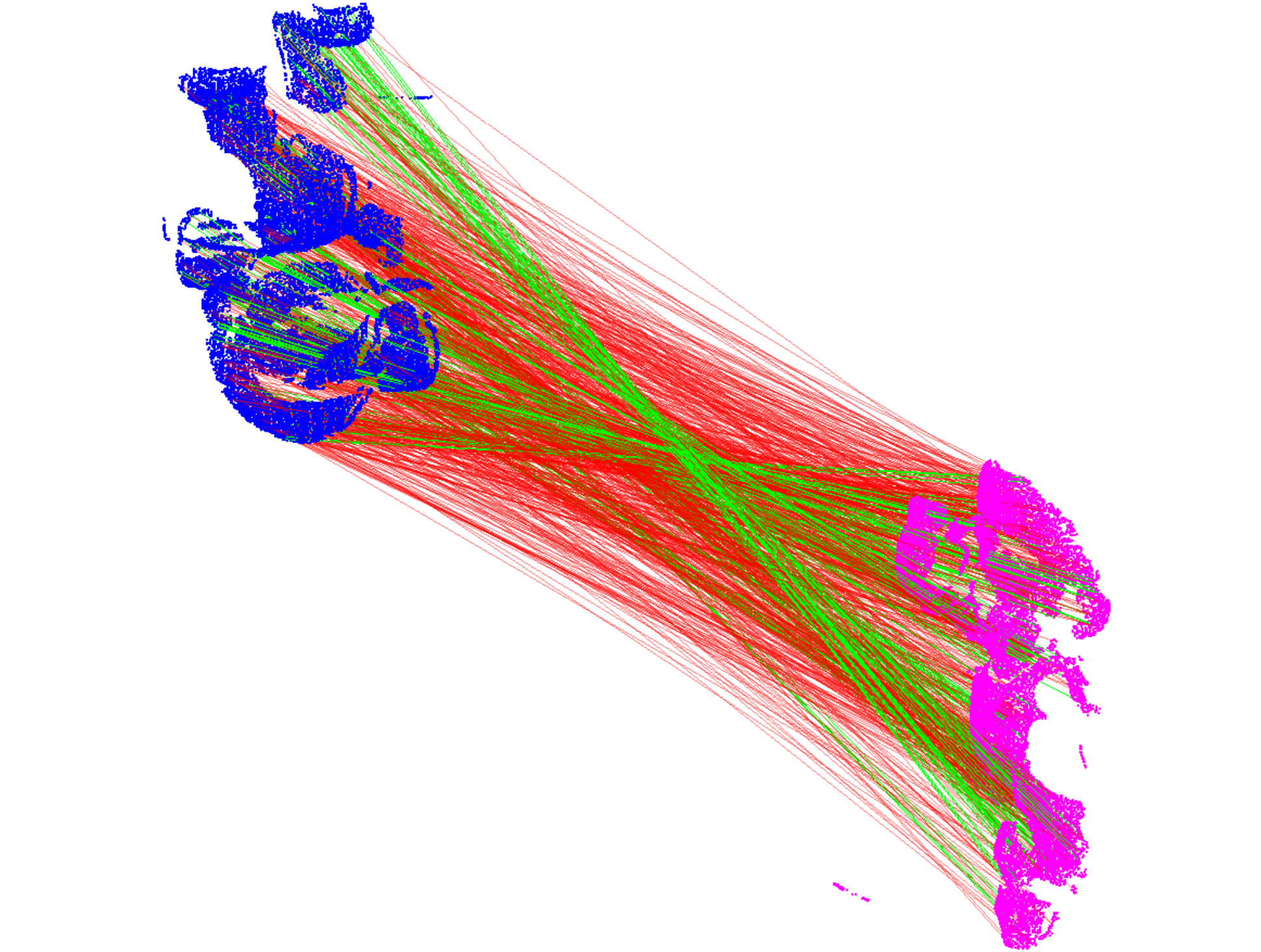}
&\includegraphics[width=0.15\linewidth]{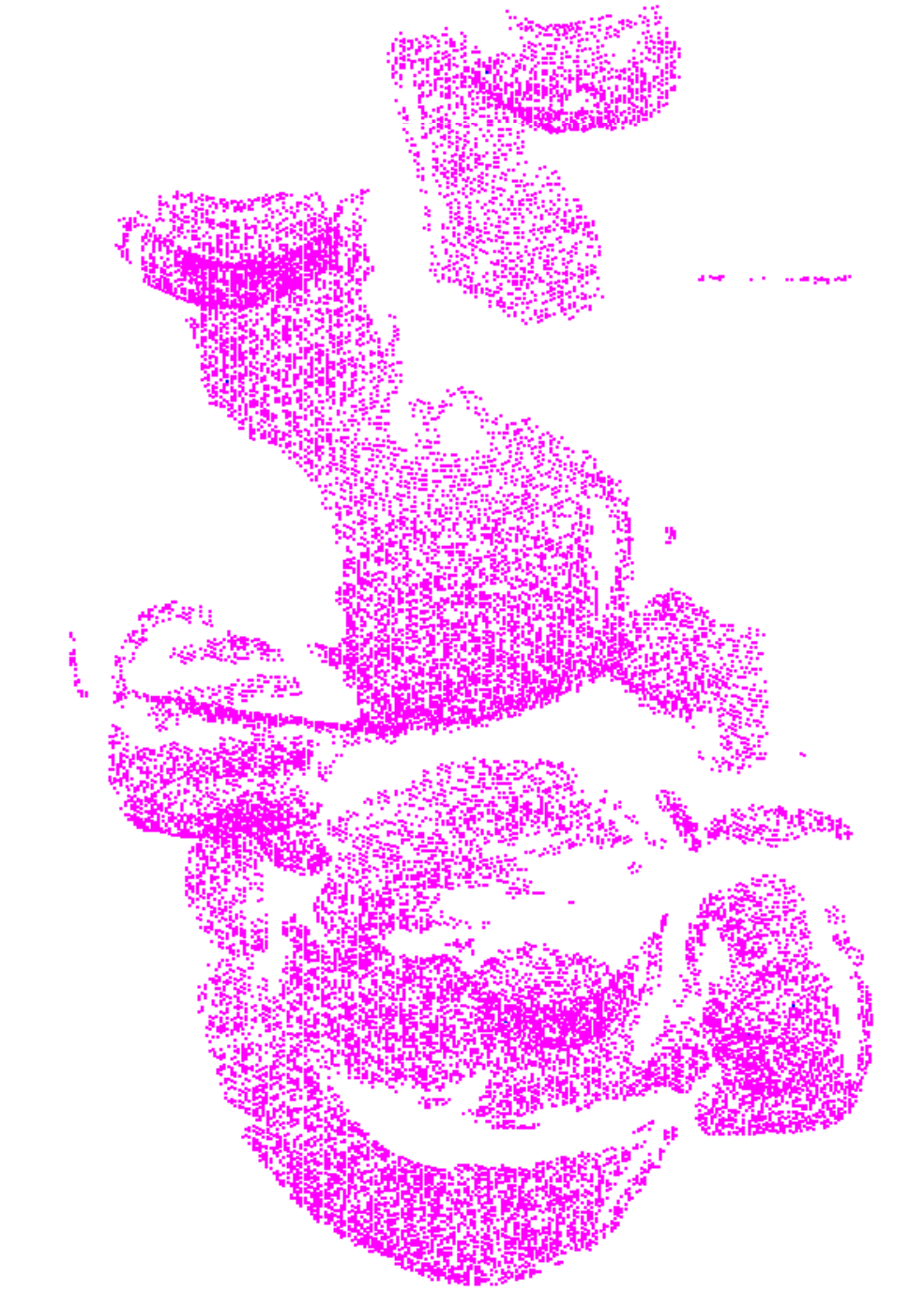}
&\includegraphics[width=0.25\linewidth]{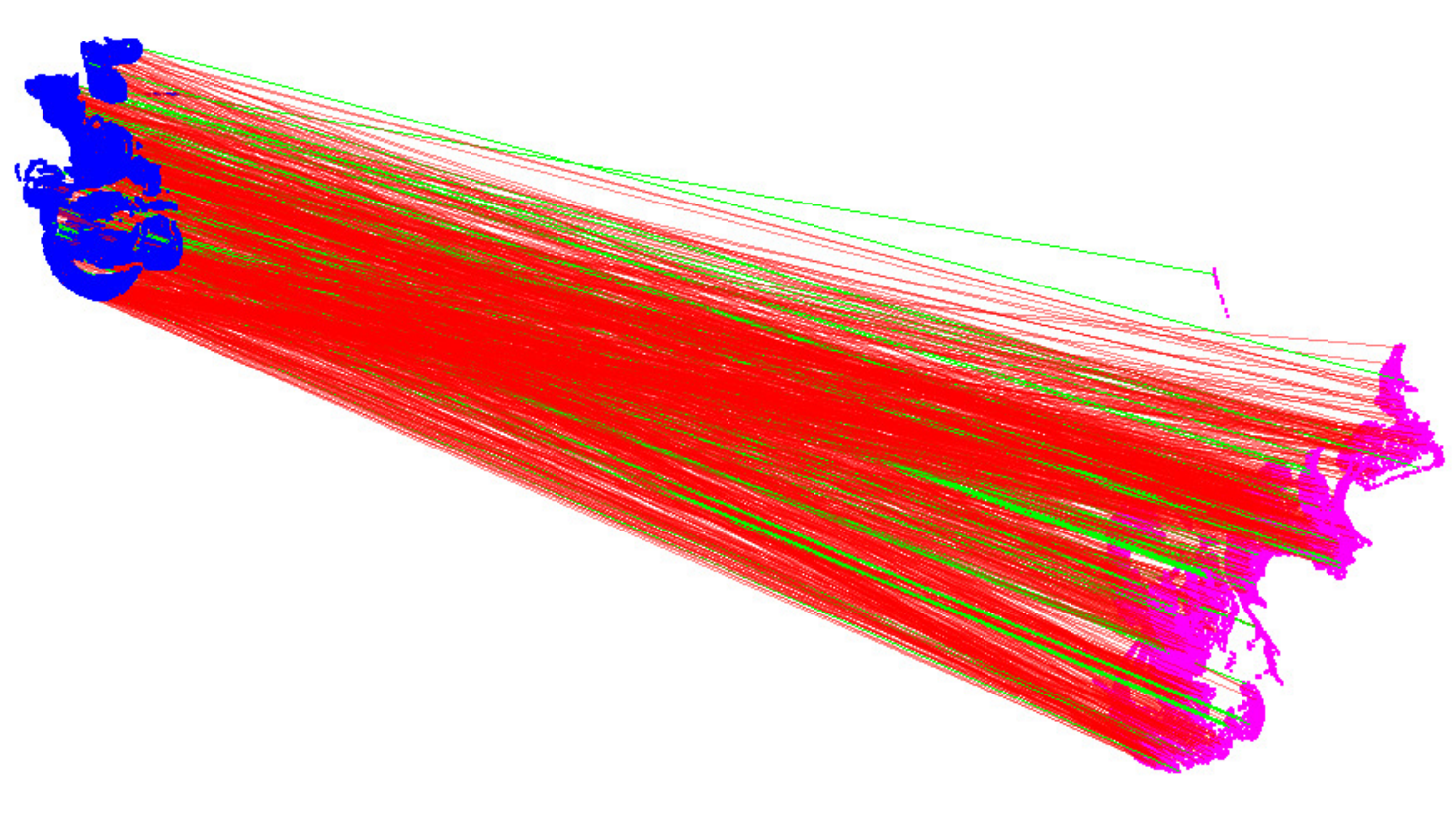}
&\includegraphics[width=0.15\linewidth]{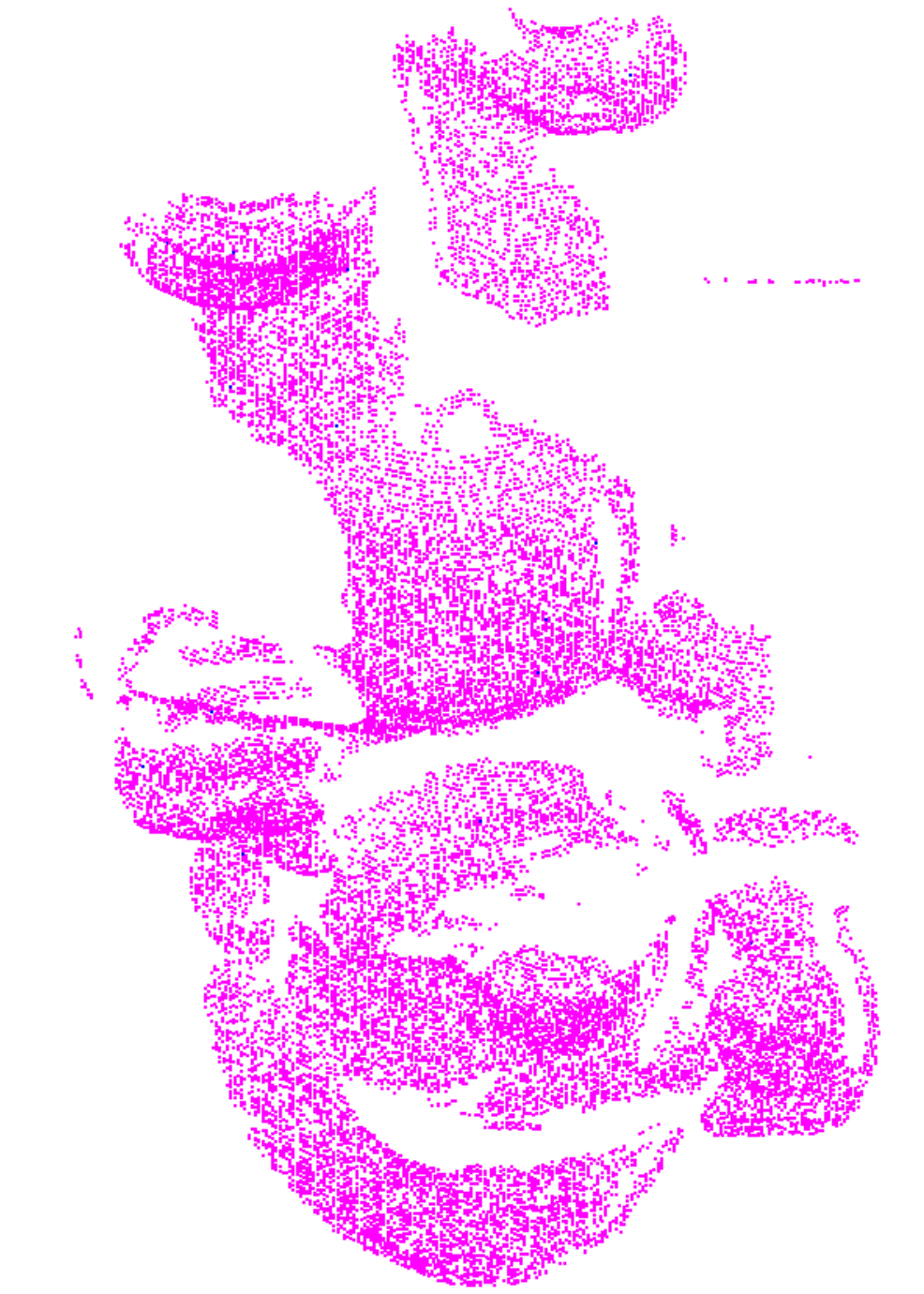}\\
\includegraphics[width=0.25\linewidth]{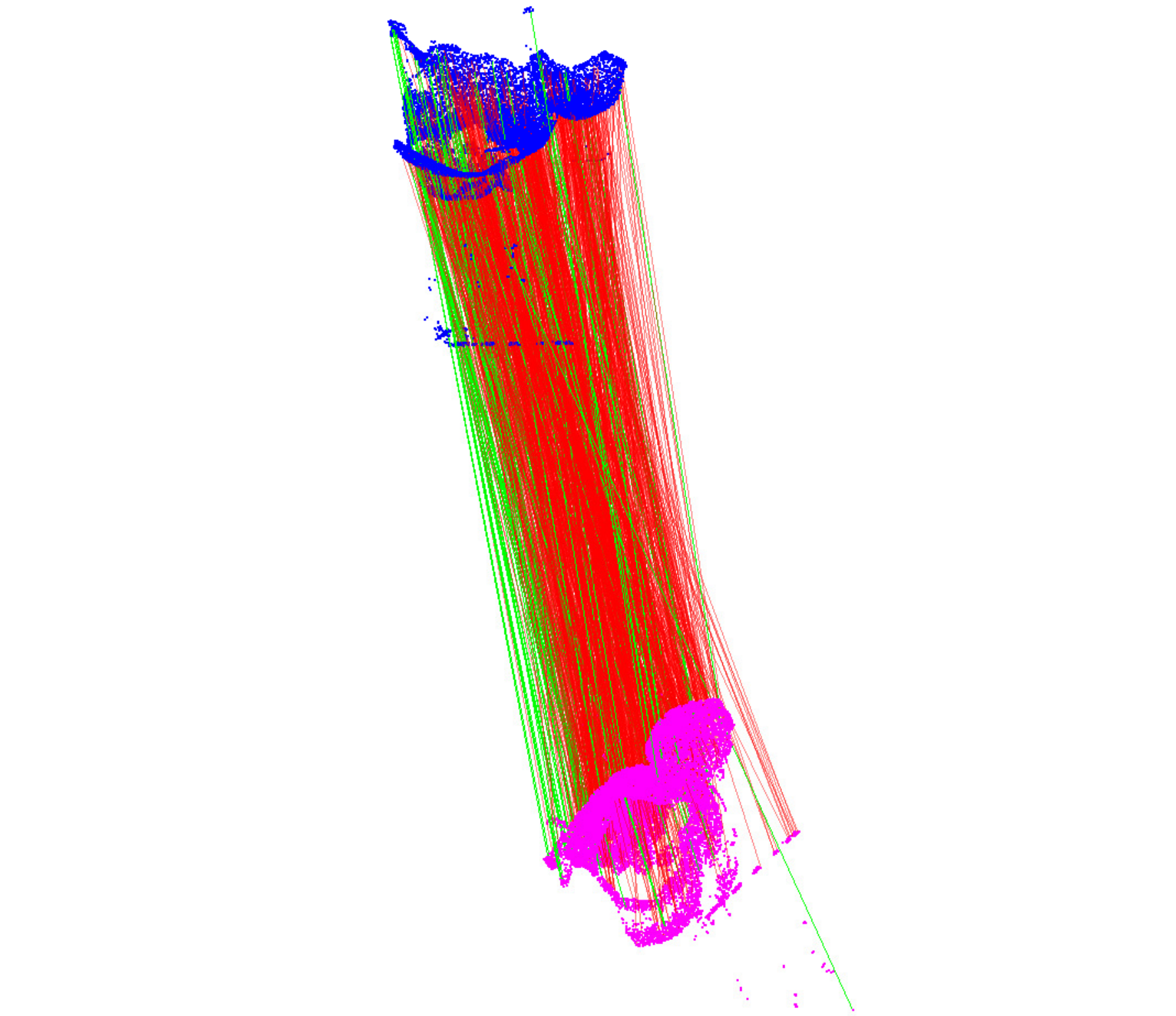}
&\includegraphics[width=0.15\linewidth]{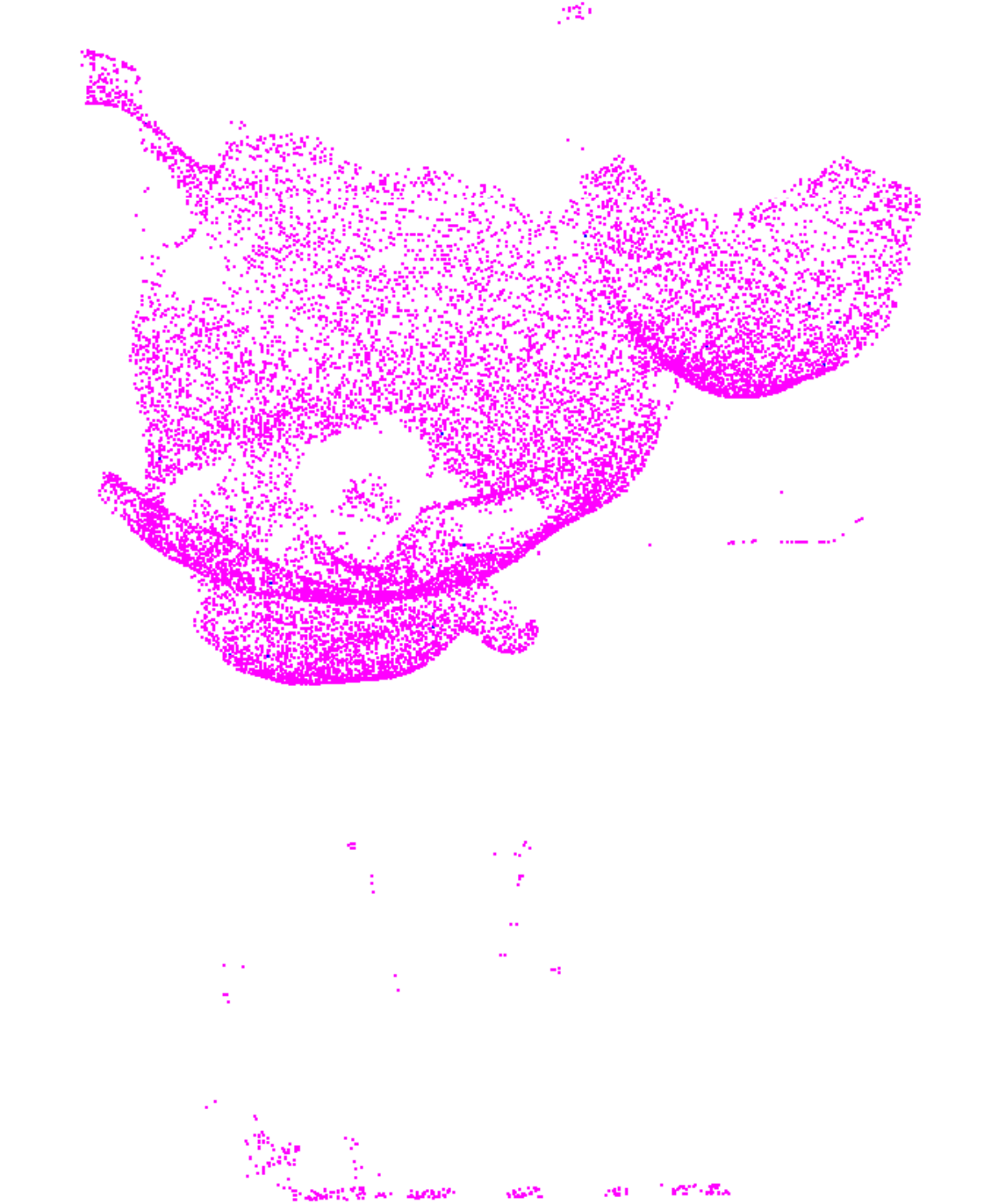}
&\includegraphics[width=0.25\linewidth]{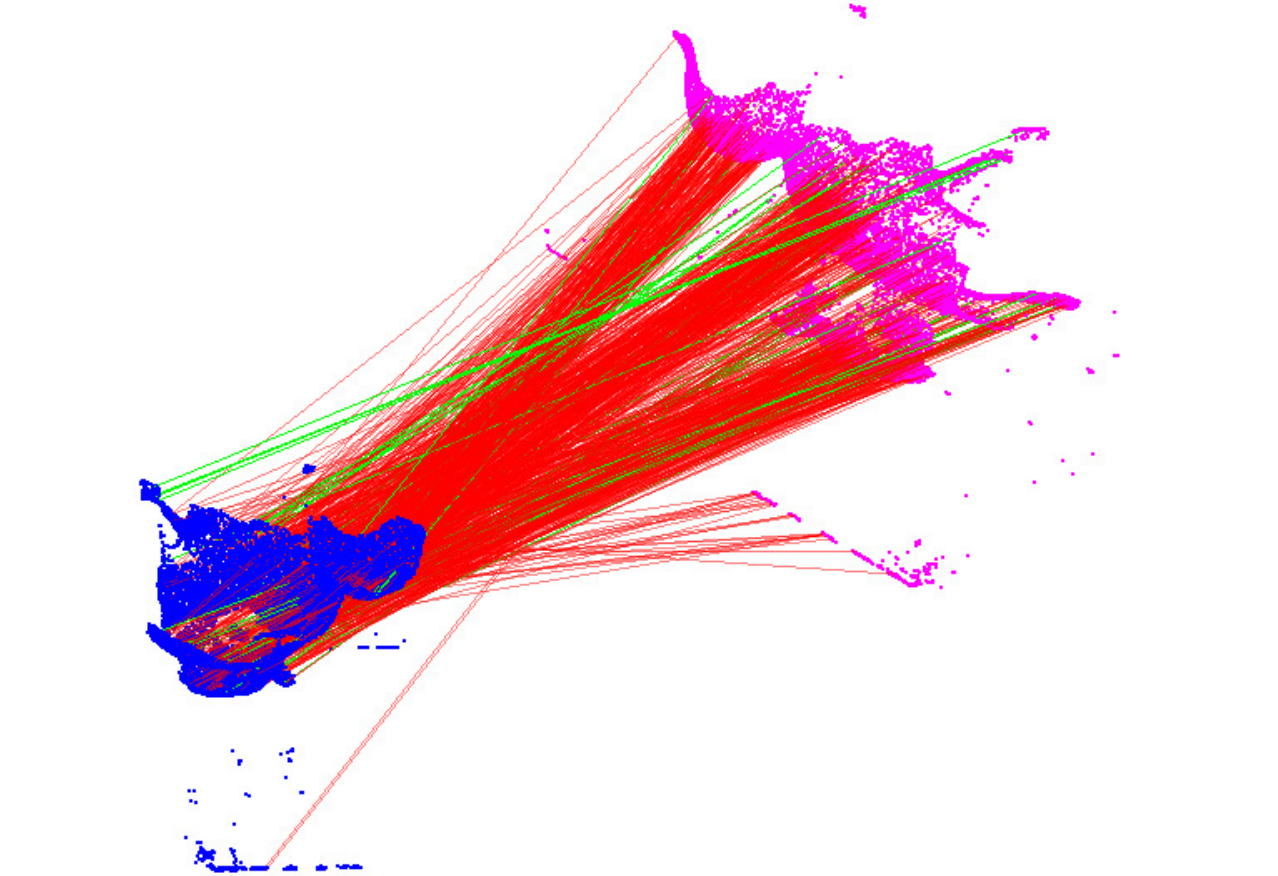}
&\includegraphics[width=0.15\linewidth]{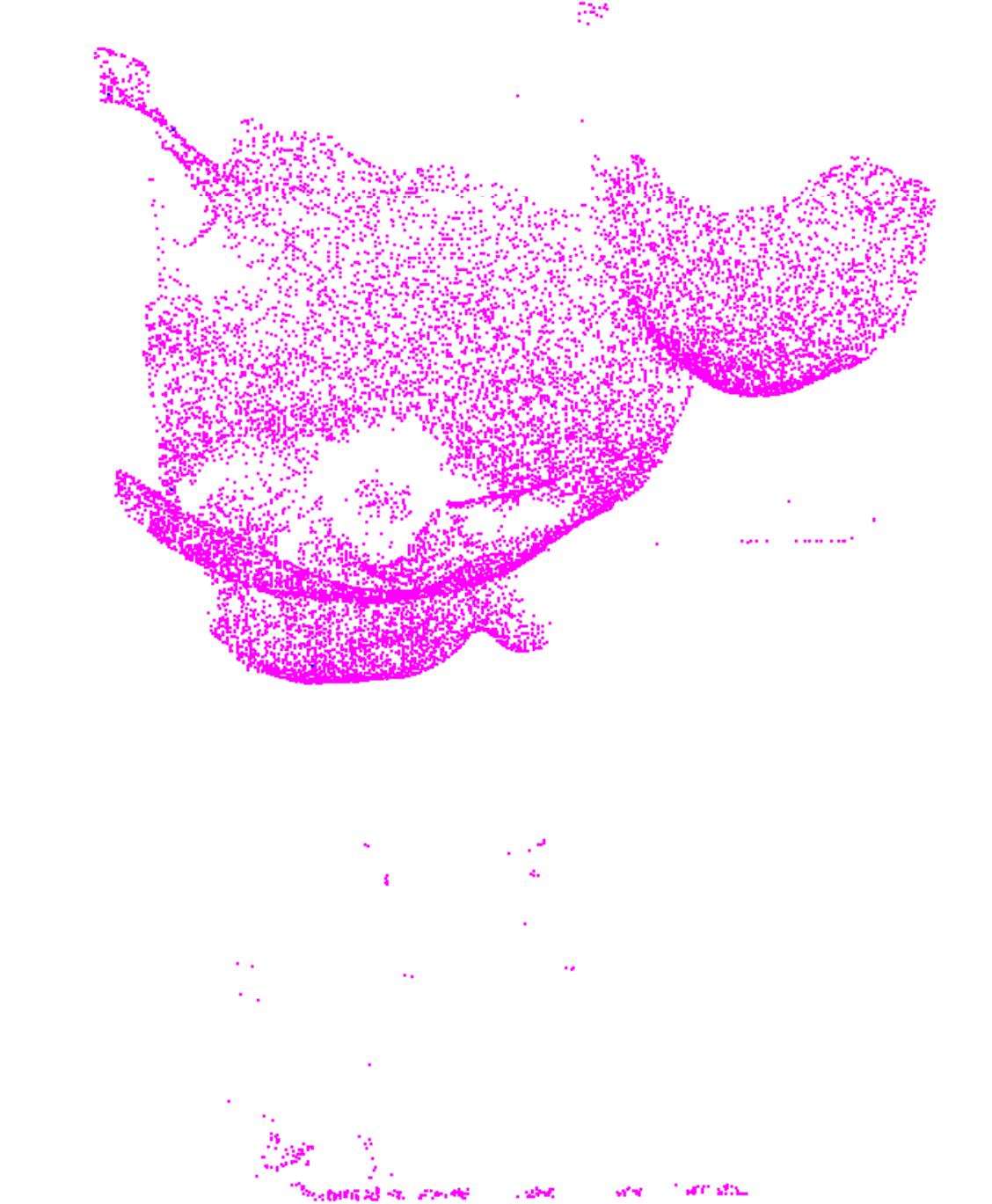} \\
\includegraphics[width=0.25\linewidth]{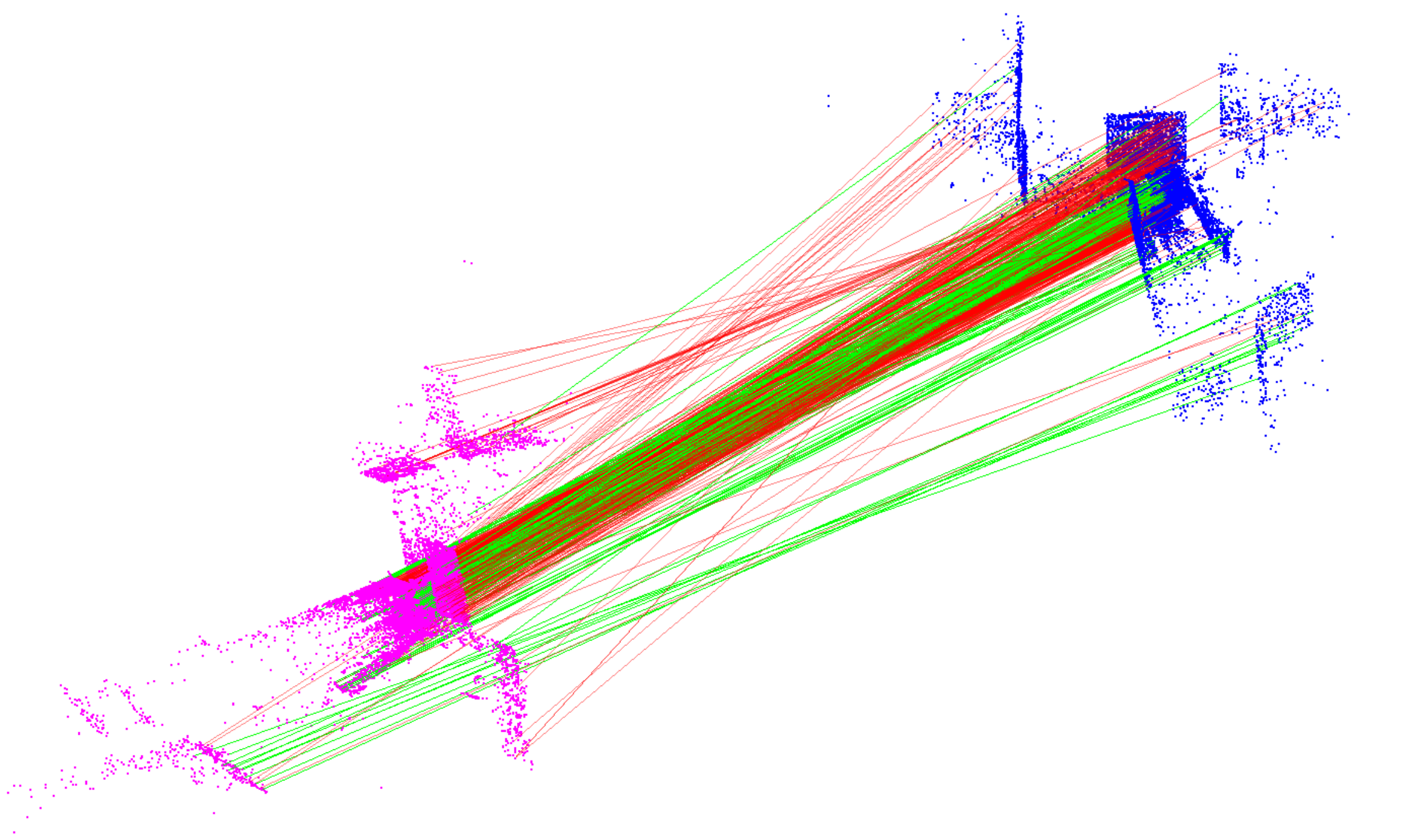}
&\includegraphics[width=0.15\linewidth]{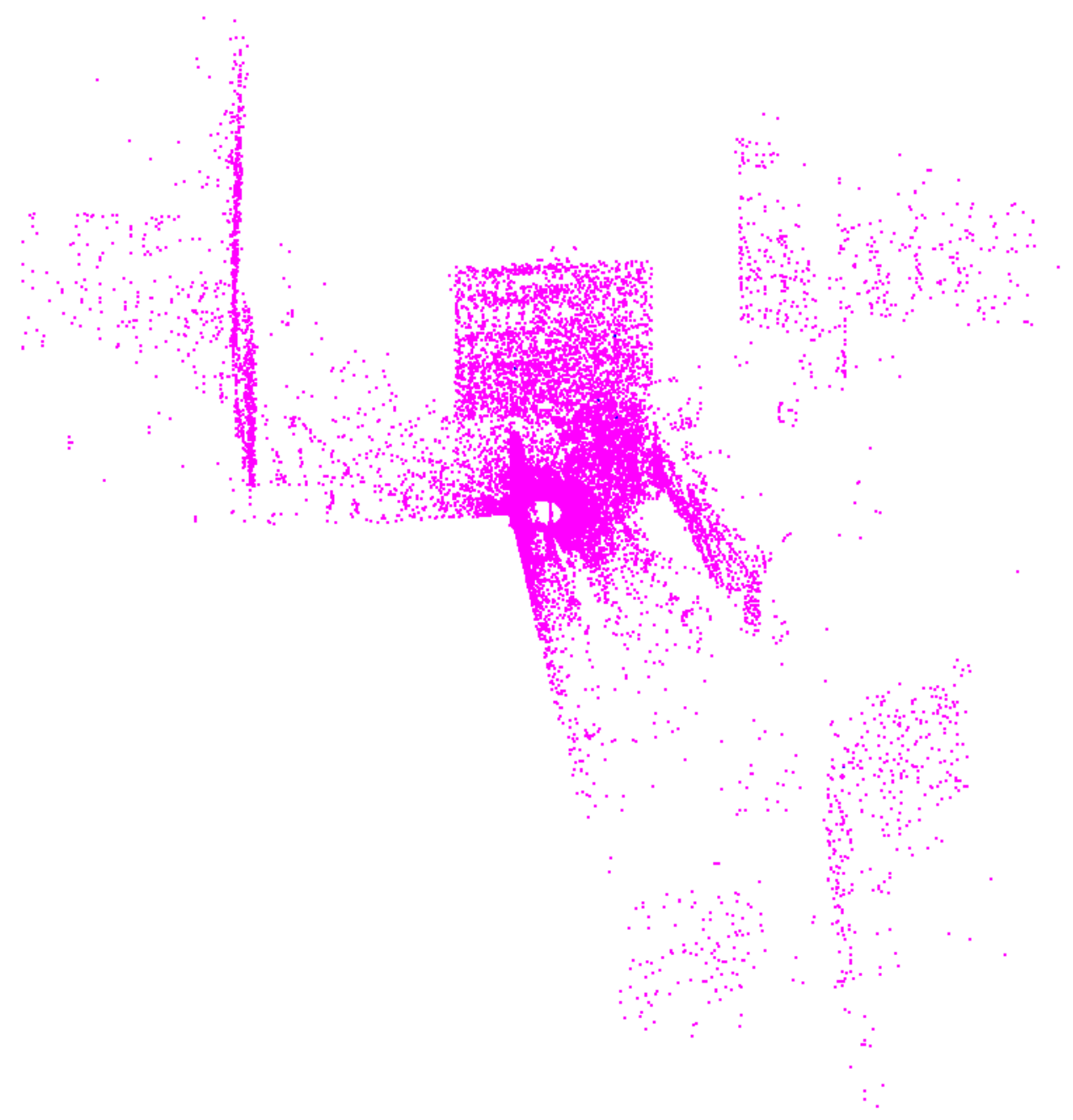}
&\includegraphics[width=0.25\linewidth]{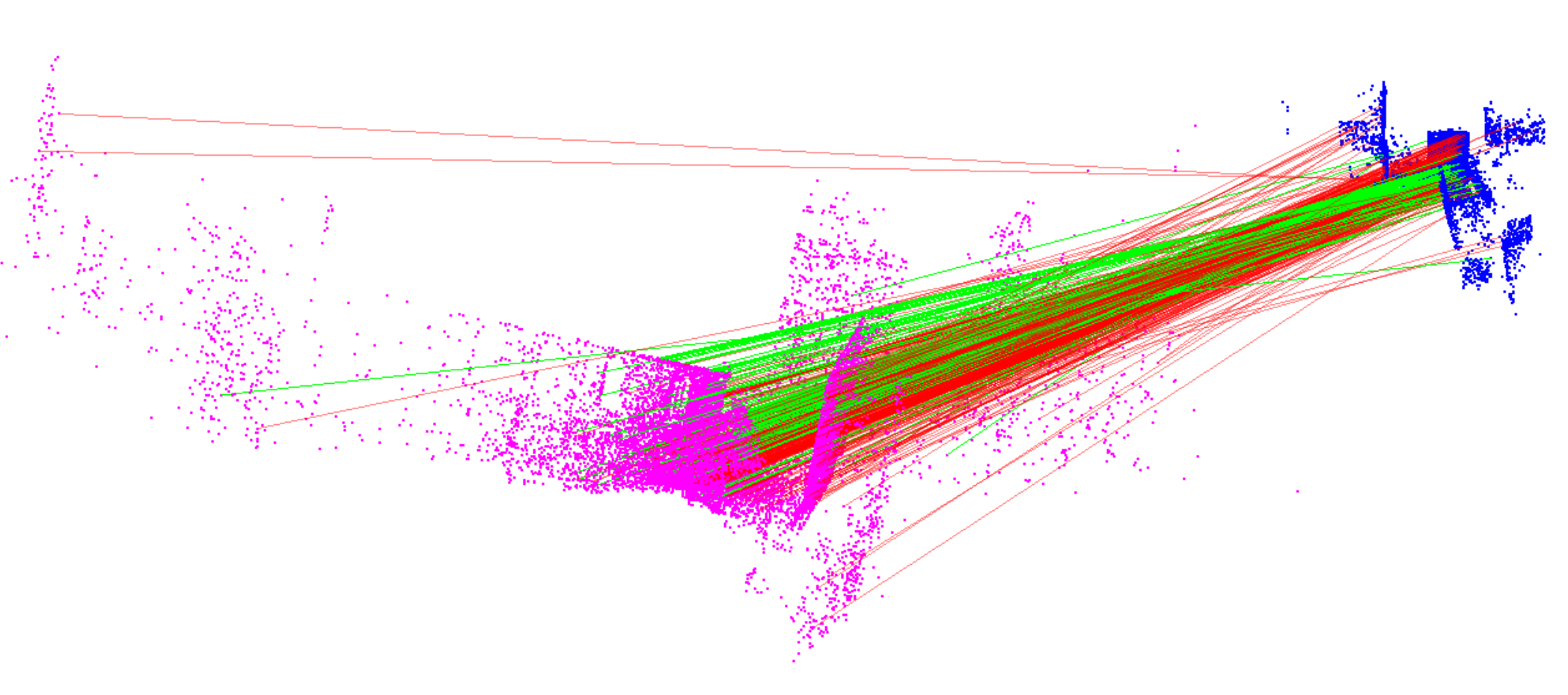}
&\includegraphics[width=0.15\linewidth]{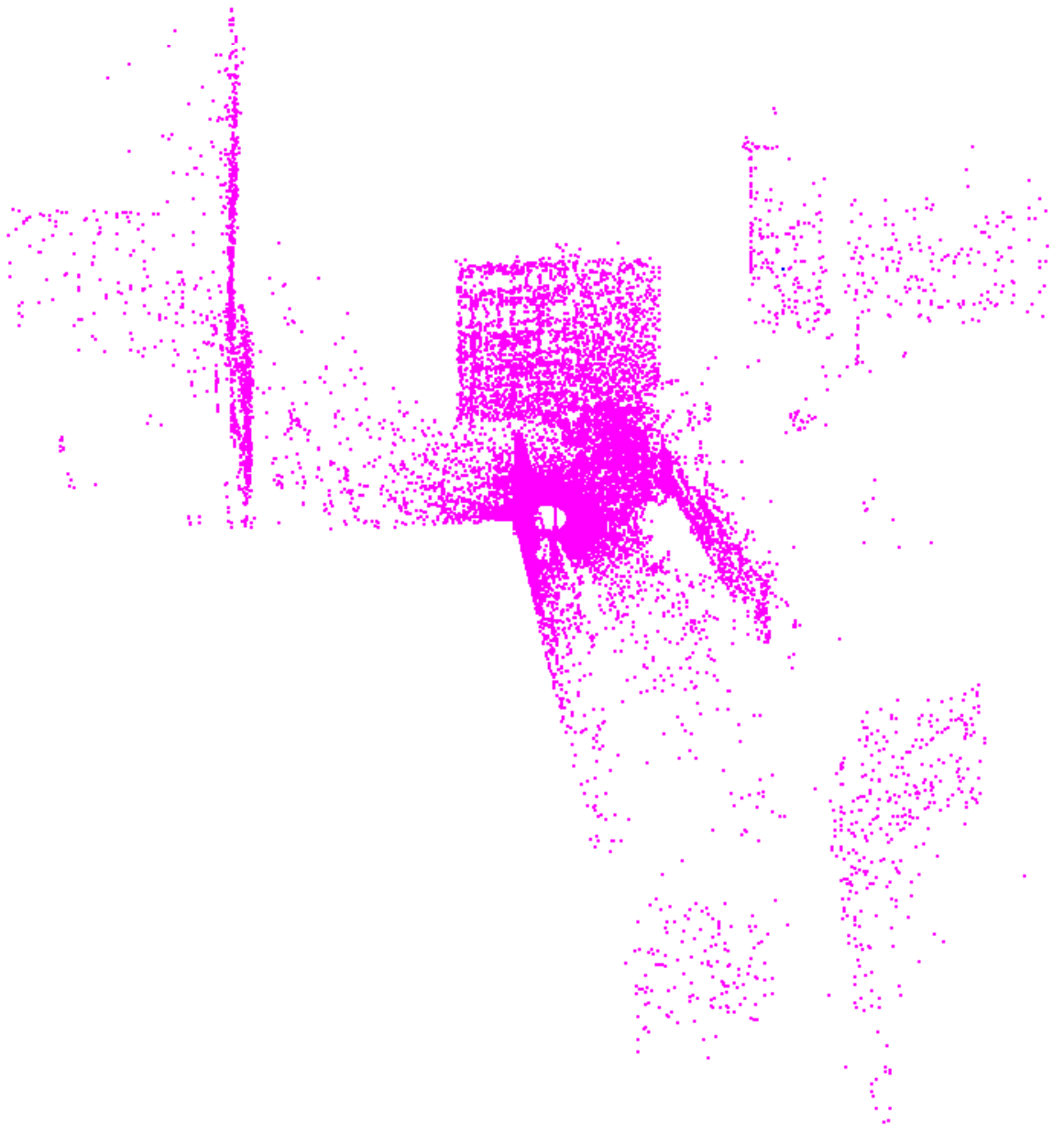} \\
\includegraphics[width=0.25\linewidth]{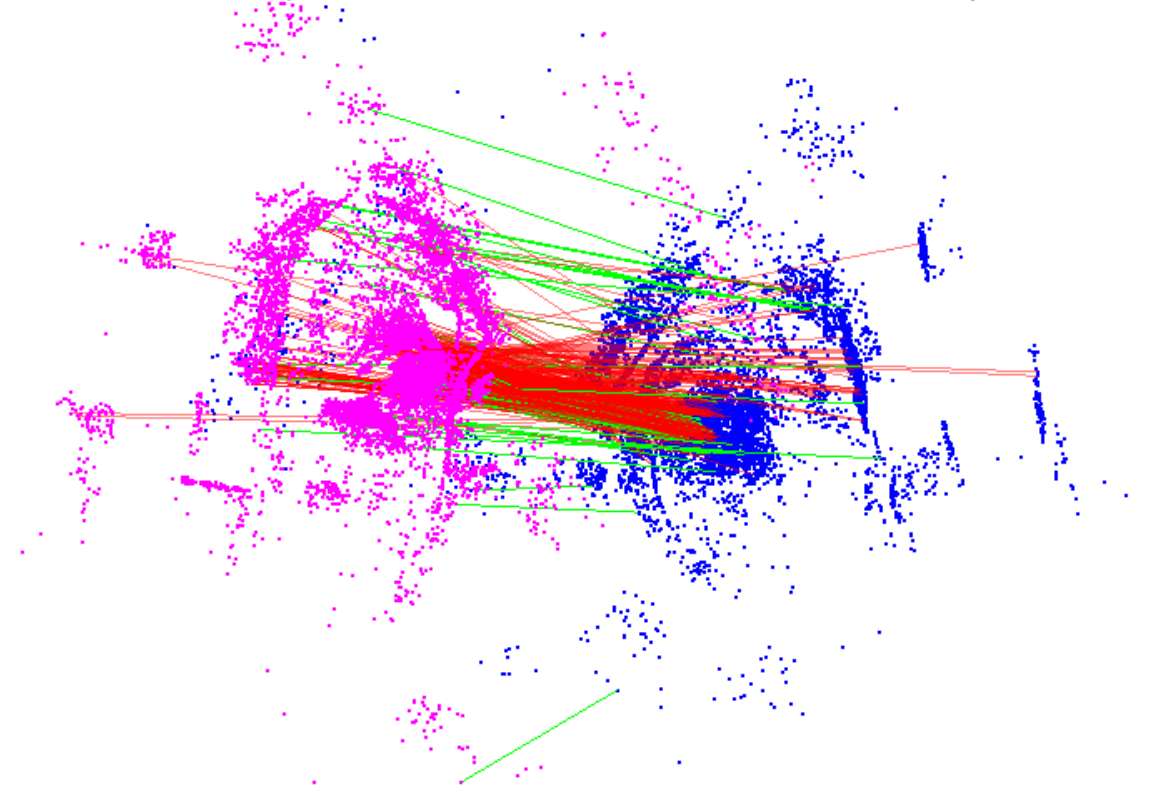}
&\includegraphics[width=0.15\linewidth]{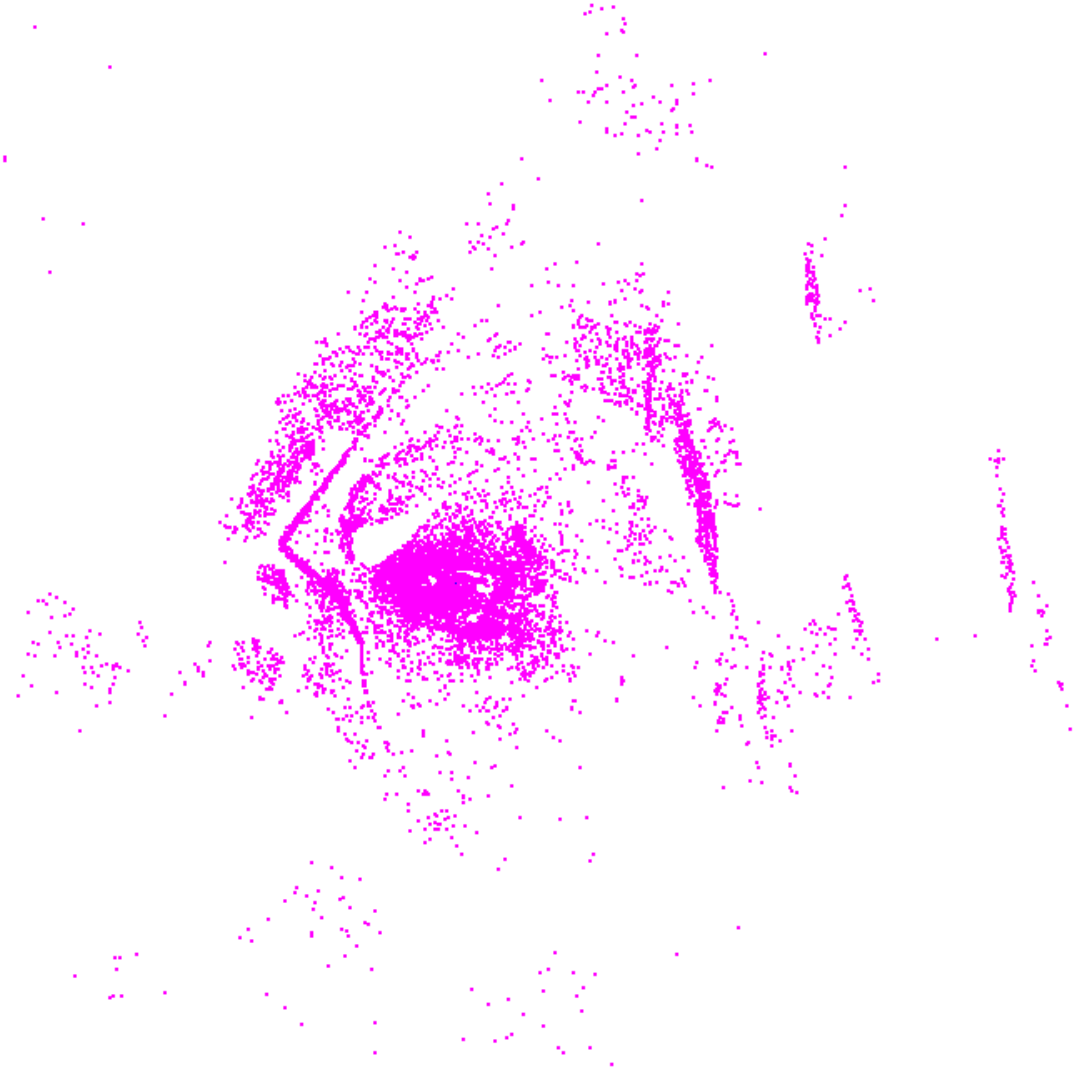}
&\includegraphics[width=0.25\linewidth]{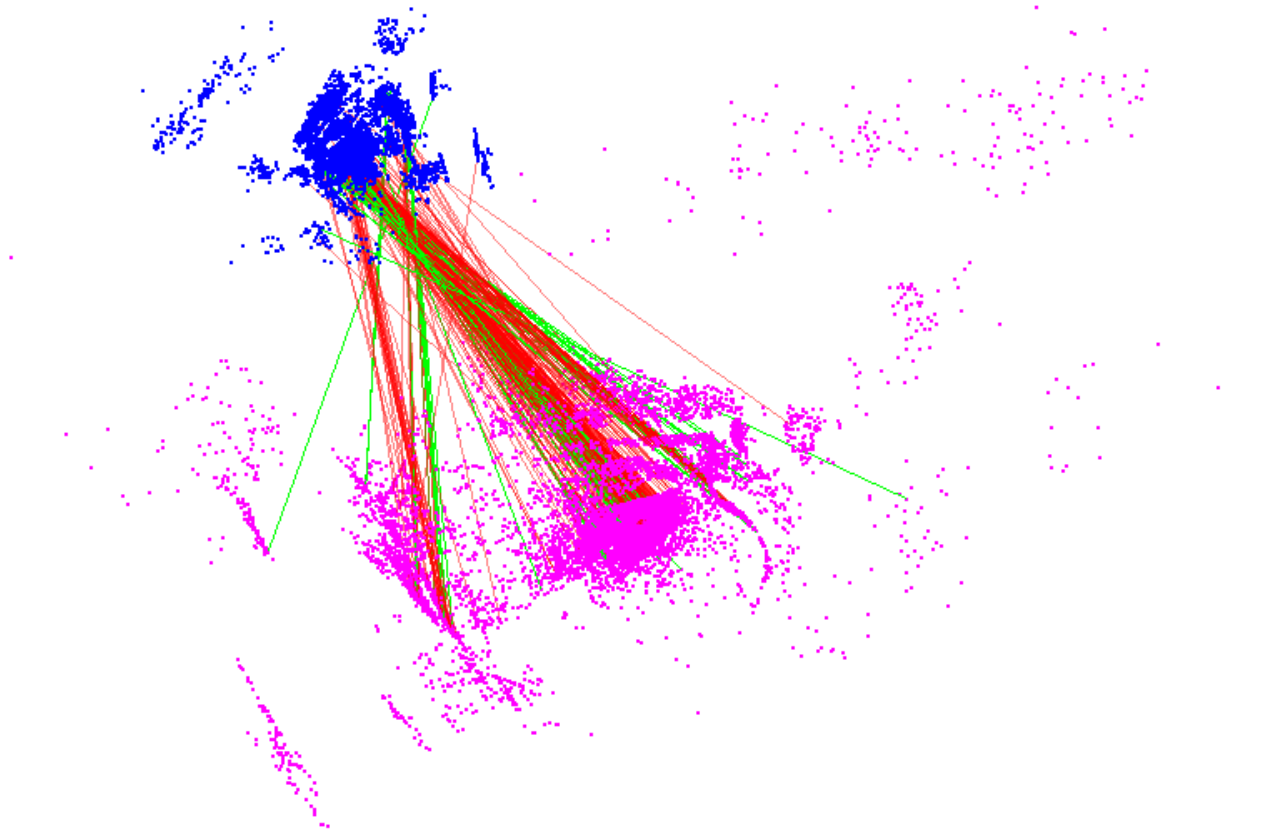}
&\includegraphics[width=0.15\linewidth]{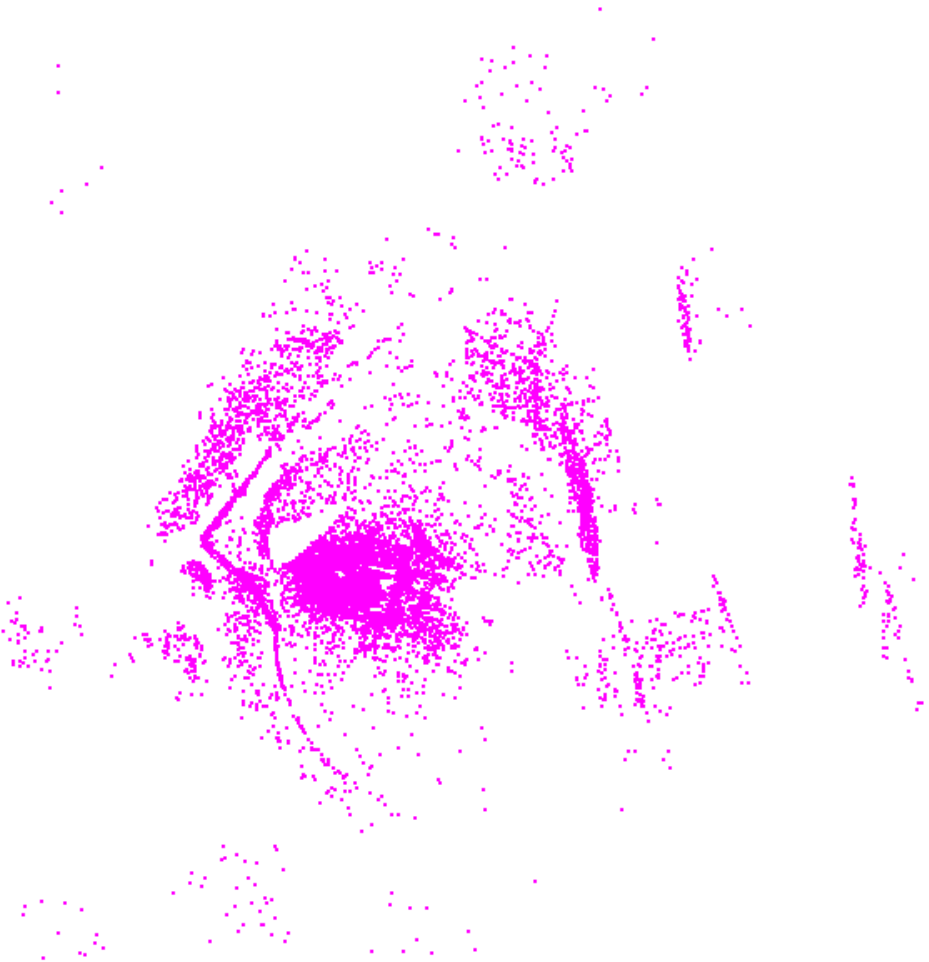}

\end{tabular}

\vspace{-1mm}

\centering
\caption{Qualitative registration results over multiple real datasets. From row 1-6, the point clouds adopted include `bunny', `Armadillo', `Mario', `Squirrel', `city', and `castle', respectively. Left-Half: Known-scale registration. Right-Half: Unknown-scale registration. In each half, from column 1-2, we display the correspondences (inliers are in green lines and outliers are in red lines), and the registration (reprojection) result using ICOS, respectively.}
\label{FPFH}
\end{figure}

\subsection{Additional Benchmarking on High Noise}

We supplement more experiments on rotation search and unknown-scale point cloud registration with high noise ($\sigma=0.1$). Since the `bunny' point cloud is placed in a $[-0.5,0.5]^3$ box, adding a noise with $\sigma=0.1$ can be extreme, as demonstrated in Fig.~\ref{High-noise}, where the shape of the `bunny' is hardly seen. We test ICOS against other solvers over the synthetic data with $N=100$ and the `bunny' with $N=1000$. The results are shown in Fig.~\ref{High-noise} (in unknown-scale registration, the translation errors are similar to rotation errors, thus omitted).

We can observe that ICOS is still robust against at least 80\% outliers in both of the tested problems with such high noise while keeping time-efficient, outperforming other competitors.

\begin{figure}[t]
\centering

\footnotesize{(a) Rotation Search with High Noise}

\subfigure{
\begin{minipage}[t]{1\linewidth}
\centering
\includegraphics[width=0.493\linewidth]{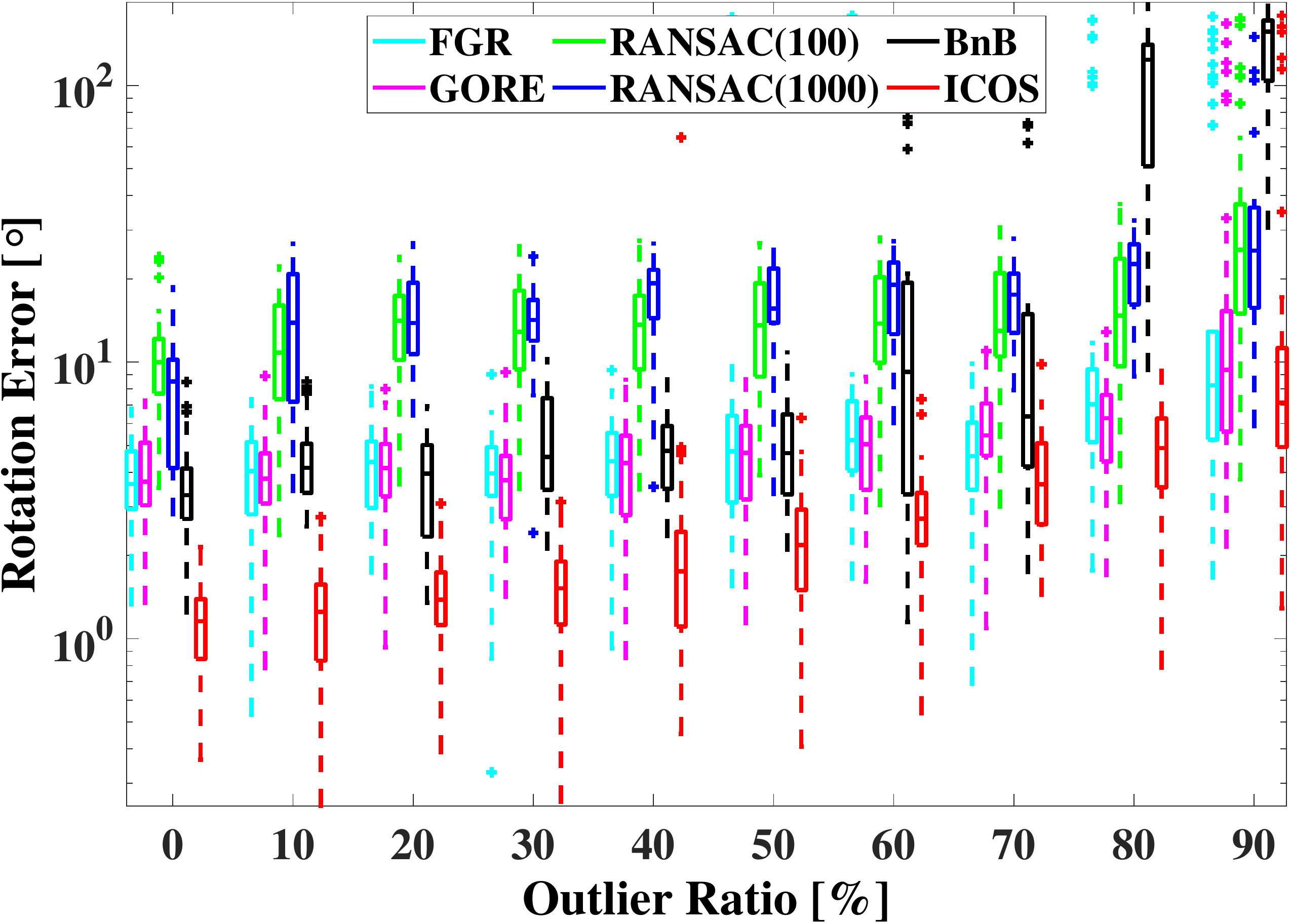}
\includegraphics[width=0.493\linewidth]{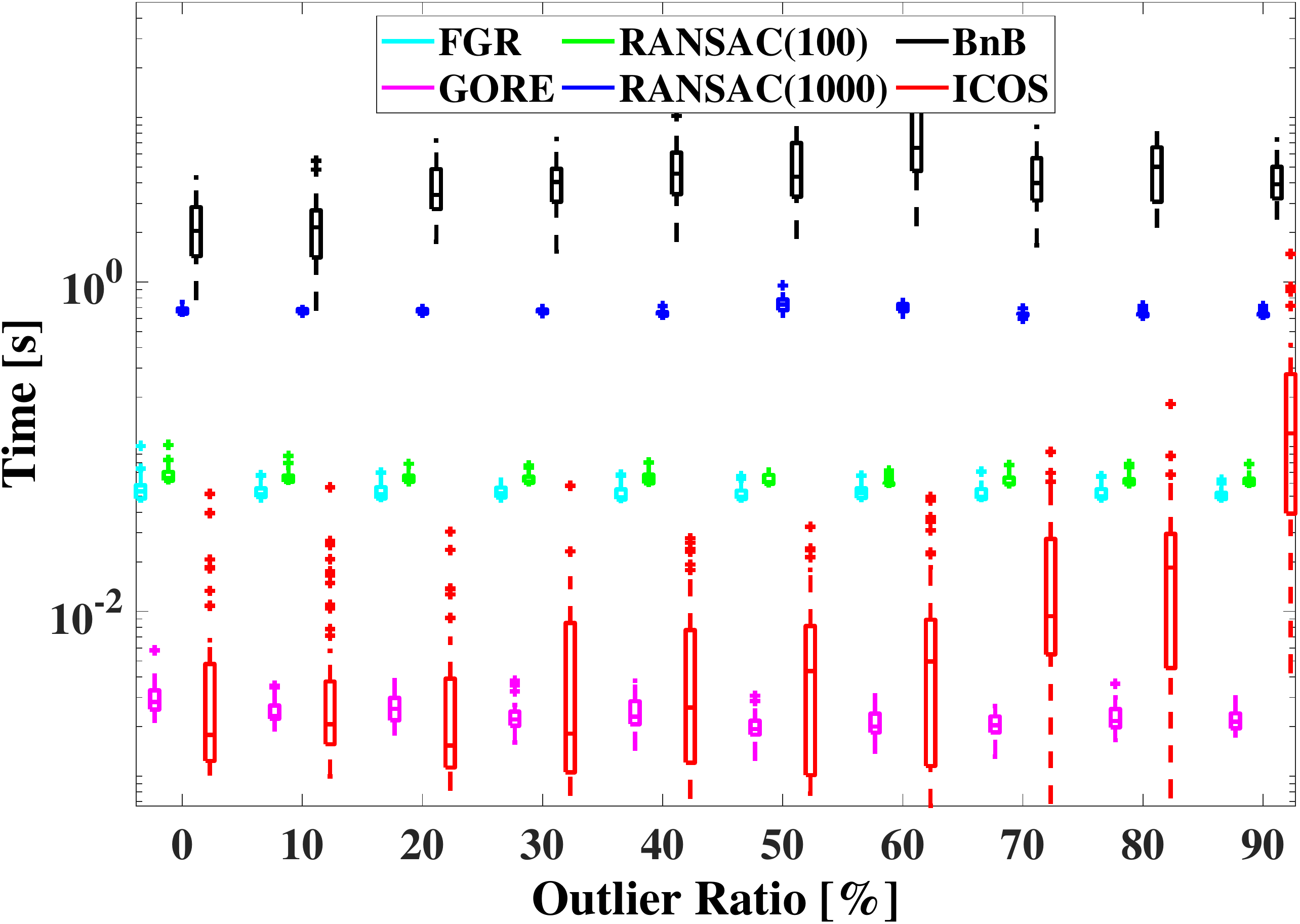}
\end{minipage}
}%

\footnotesize{(b) Point Cloud Registration with High Noise}

\subfigure{
\begin{minipage}[t]{1\linewidth}
\centering
\includegraphics[width=0.493\linewidth]{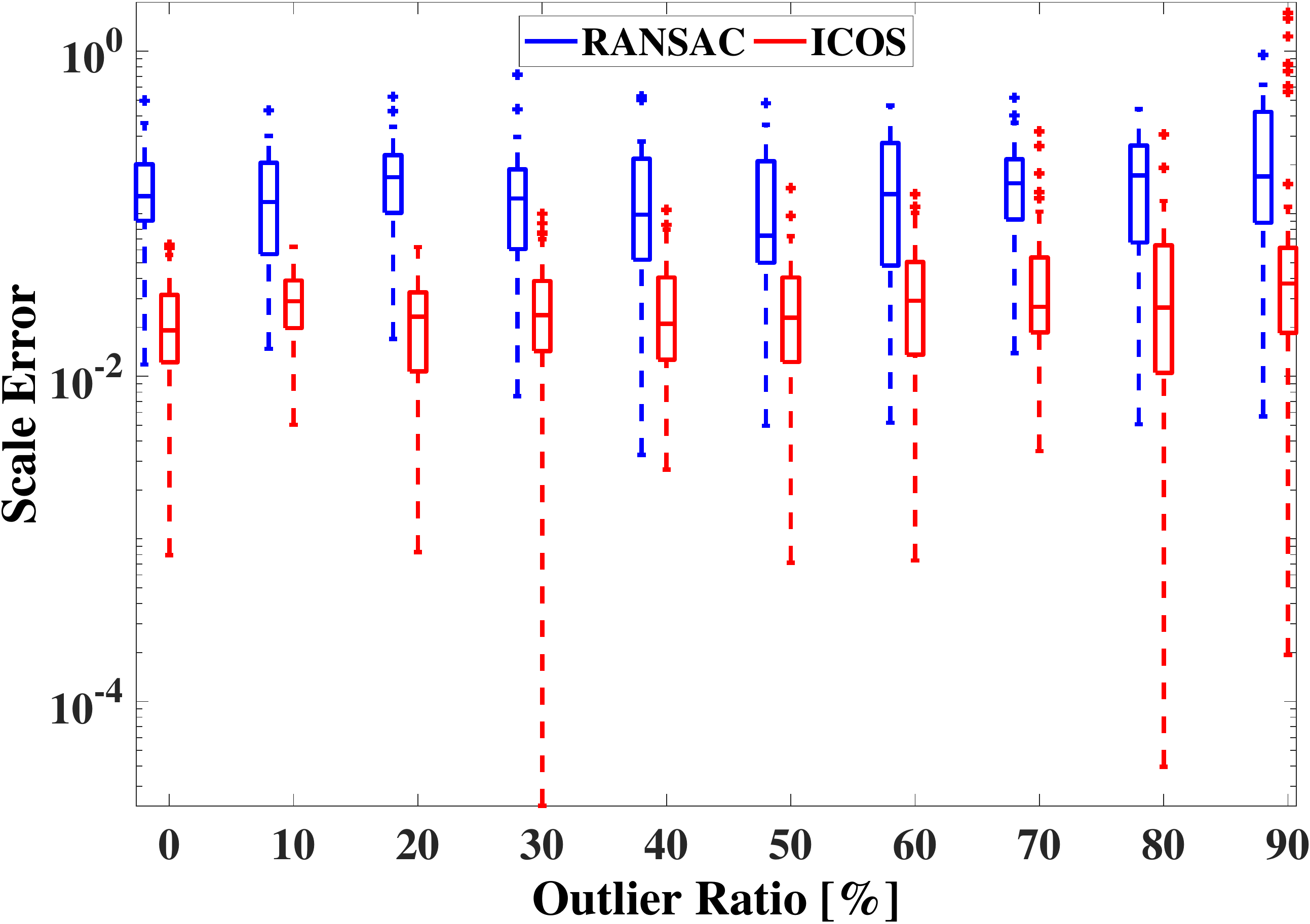}
\includegraphics[width=0.493\linewidth]{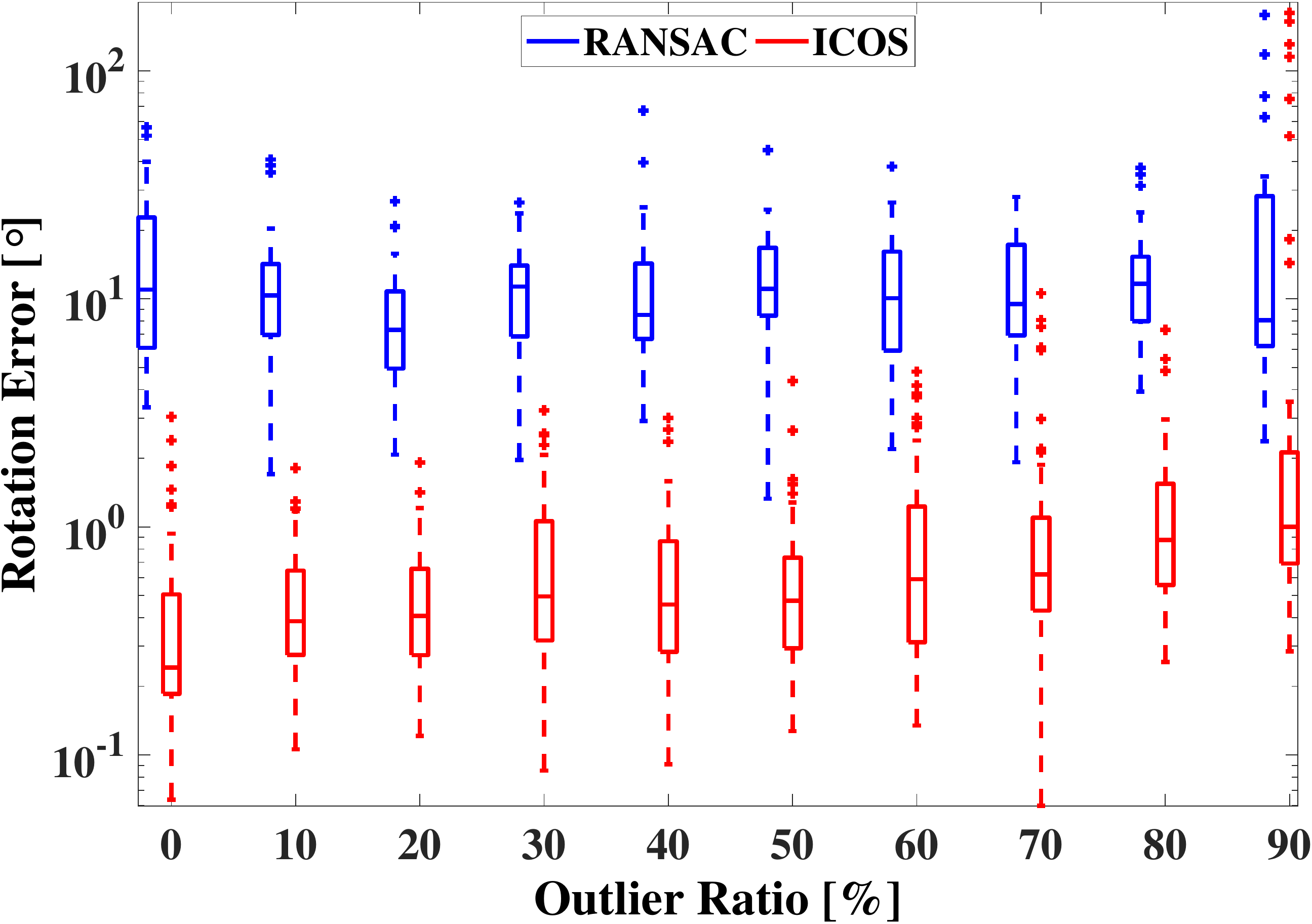}
\end{minipage}
}%

\subfigure{
\begin{minipage}[t]{1\linewidth}
\centering
\includegraphics[width=0.493\linewidth]{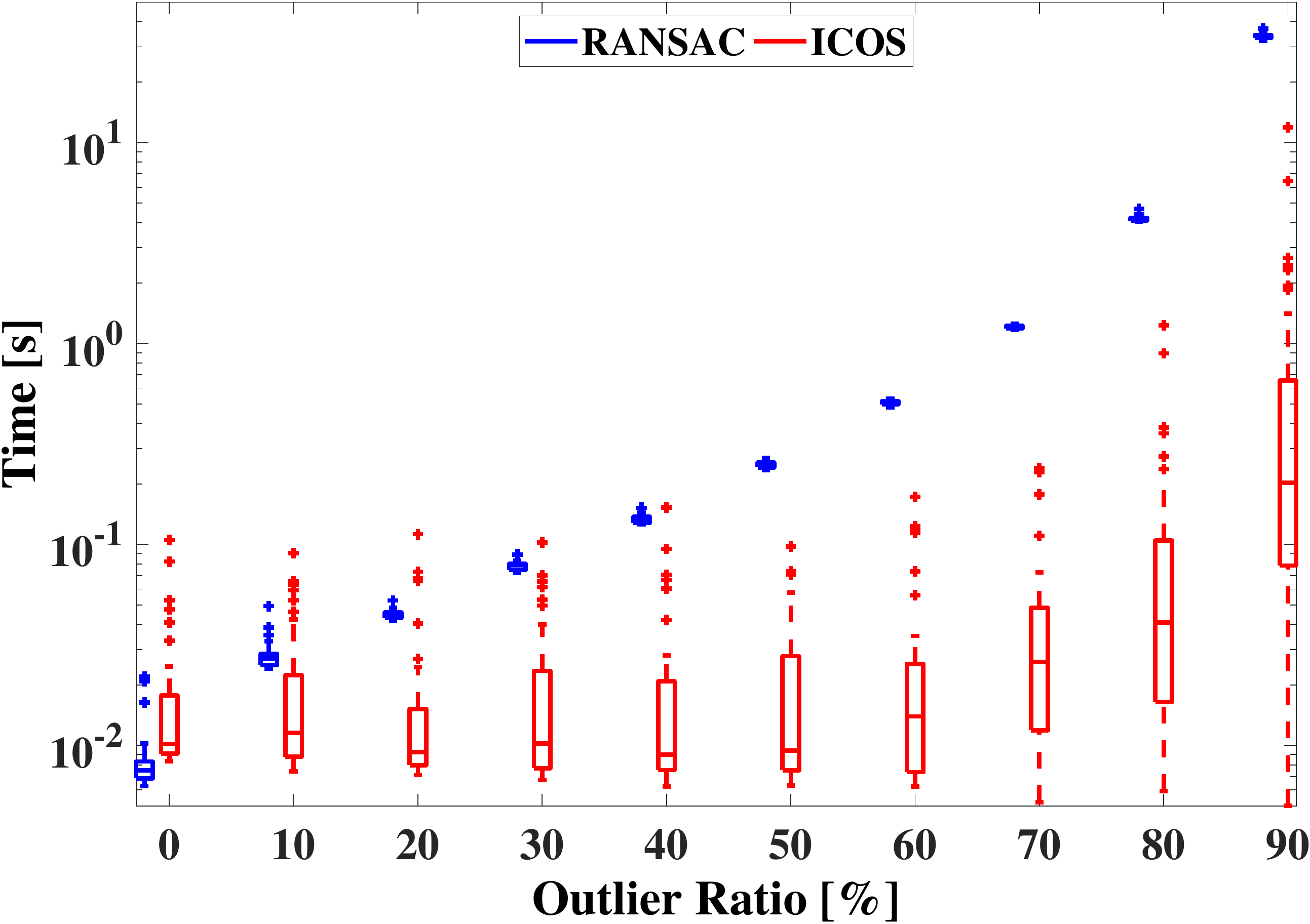}
\includegraphics[width=0.24\linewidth]{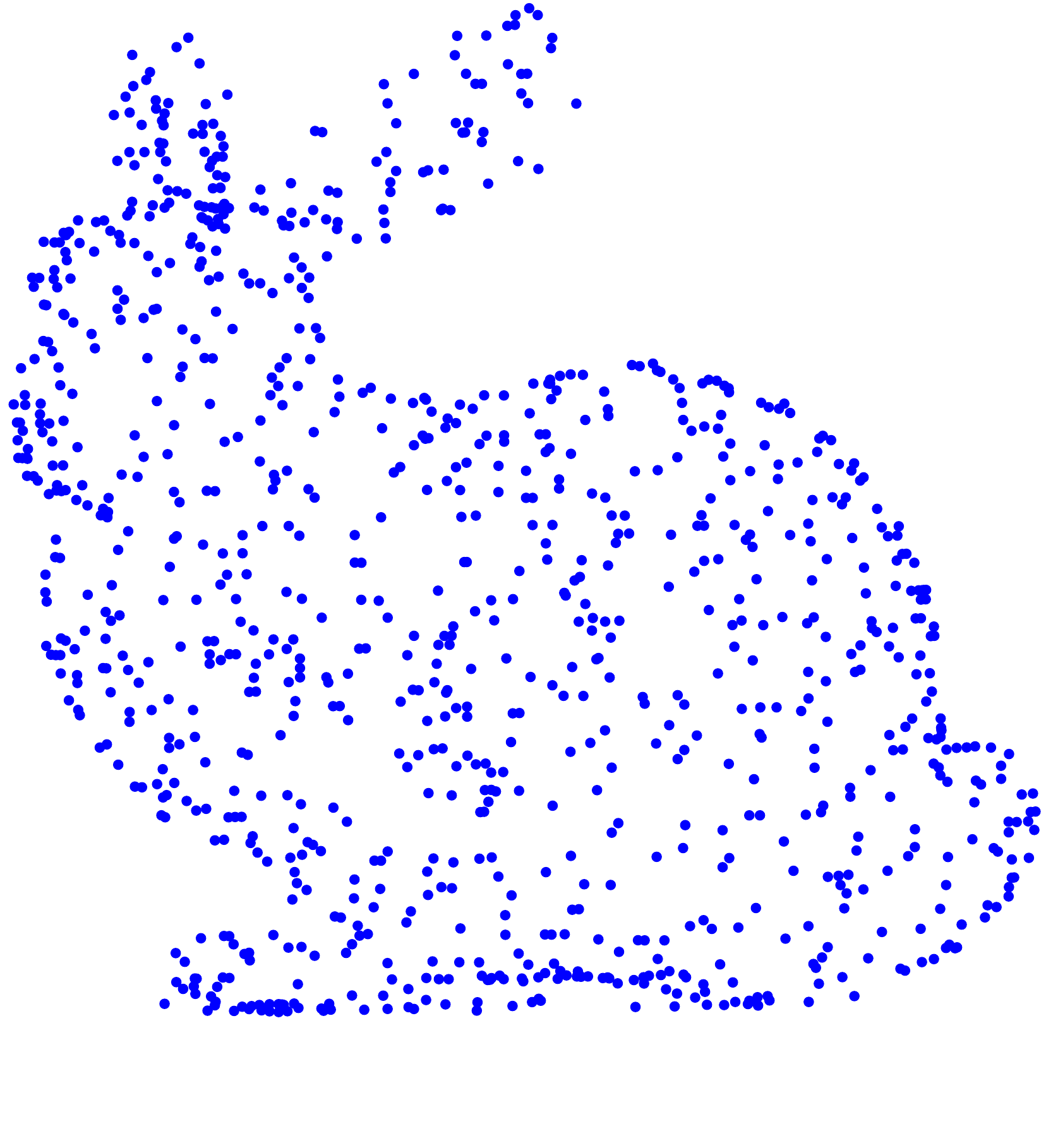}
\includegraphics[width=0.24\linewidth]{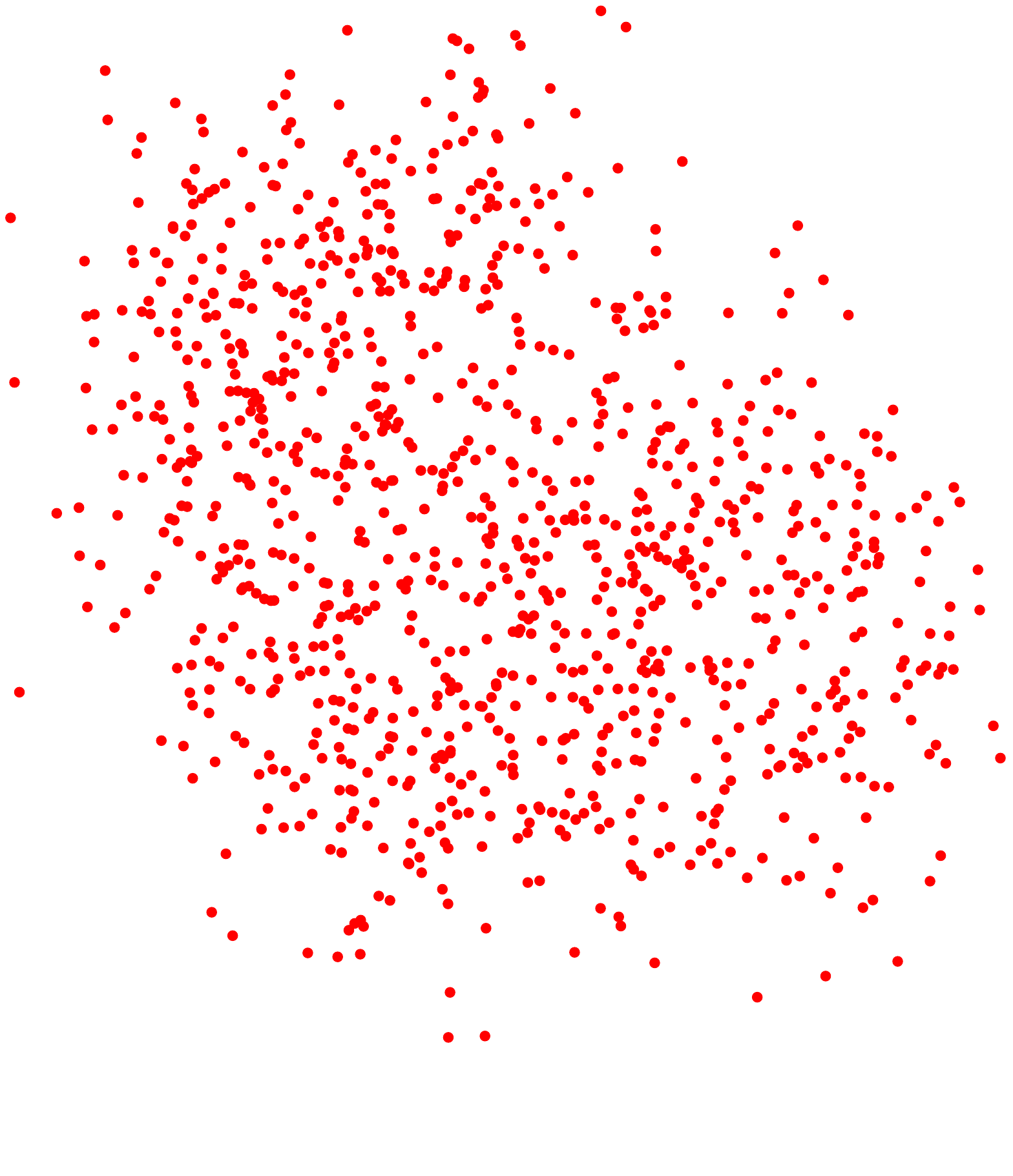}
\end{minipage}
}%

\vspace{-1mm}

\centering
\caption{Results with high noise. (a) Rotation search with high noise. (a) Unknown-scale registration with high noise. Right-bottom: The left image is the `bunny' without noise and the right one is that with noise $\sigma=0.1$.}
\label{High-noise}
\end{figure}

\subsection{Application 1: Image Stitching}

To validate ICOS in the real-world application w.r.t. rotation search, we adopt ICOS to solve the image stitching problem over the \textit{LunchRoom} image sequence from the PASSTA dataset~\cite{meneghetti2015image}. We use 2D feature descriptor SURF~\cite{bay2008speeded} to detect and match the 2D keypoints across a certain pair of images with overlapping scenes in them. Subsequently, we apply the inverse of the intrinsic matrix $\mathbf{K}^{-1}$ of the camera to create the vector correspondences $\mathcal{U}=\{\mathbf{u}_i\}_{i=1}^N$ and $\mathcal{V}=\{\mathbf{v}_i\}_{i=1}^N$ over the two different frames. Eventually, we employ Algorithm~\ref{Algo0-ICOS} to solve the rotation $\boldsymbol{\hat{R}}$ aligning the two vector sets (coordinate frames) and then obtain the homography matrix using matrix transformation $\mathbf{H}=\mathbf{K}\boldsymbol{\hat{R}}\mathbf{K}^{-1}$ to stitch the image pair together. The correspondences and experimental results are shown in Fig.~\ref{ImgStit}. We find that ICOS always finds the correct inliers among correspondences and stitches the images successfully.

\subsection{Application 2: 3D Object Localization}

We further test ICOS over the RGB-D scenes datasets~\cite{lai2011large} for dealing with the 3D object localization (pose estimation) problems. We first extract the 3D point cloud of the targeted object (\textit{cereal box, cap,} and \textit{table}) from the scene according to the labels provided and impose a random rigid transformation (with both known scale and unknown scale) as well as random noise with $\sigma=0.01$ and $\mu=0$ on the target object. Then, we adopt the FPFH 3D feature descriptor~\cite{rusu2009fast} to build putative correspondences between the scene and the transformed object. We find that the correspondences constructed by FPFH contain over 80\% outliers in all the scenes tested, and what is worse, the outlier ratios even exceed 90\% in most scenes. 

For each object-scene registration problem, we apply ICOS, RANSAC(1000) and RANSAC(1min) to estimate the transformation (relative pose) of the object w.r.t. the scene, and then reproject the object back to the scene with the transformation estimated by each solver, respectively. The qualitative results are demonstrated in Fig.~\ref{Object-Lo} and~\ref{Object-Lo2}, and the quantitative results are specified in Table.~\ref{quantitative} and~\ref{quantitative2}, in which the unit for $E_{\boldsymbol{R}}$ is degree ($^\circ$), and the unit for $E_{\boldsymbol{t}}$ is meter, computed as in~\eqref{errors}. In each table, the best results are labeled in the \textbf{bold} font.

\clearpage

\begin{figure*}[t]
\centering
\setlength{\tabcolsep}{1mm}

\begin{tabular}{ccc|ccc}

\textit{Image 01\&07}, 12.88\% & Inliers Found by ICOS & Stitching&
\textit{Image 07\&14}, 56.41\% & Inliers Found by ICOS & Stitching \\

\includegraphics[width=0.19\linewidth]{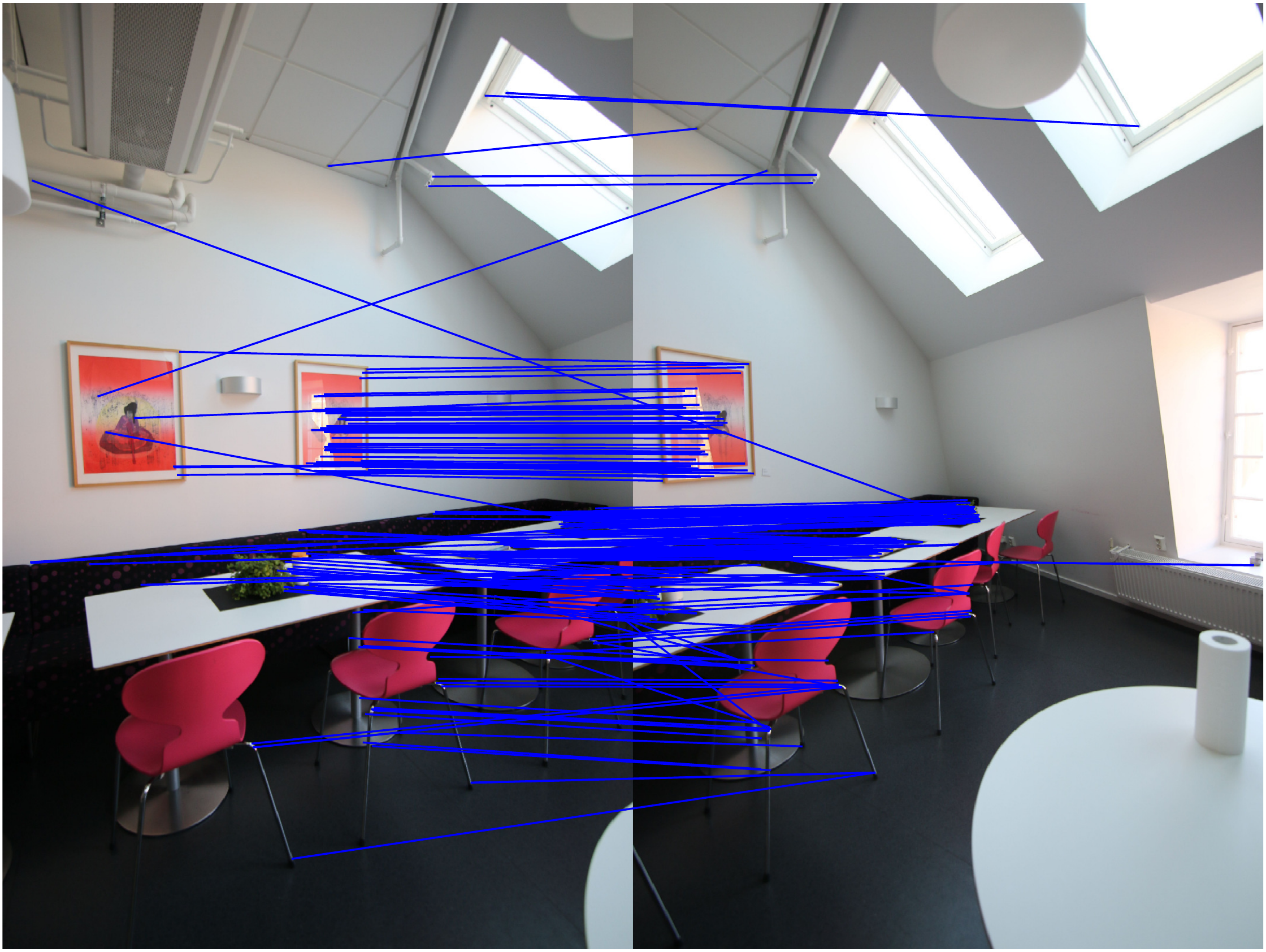}
&\includegraphics[width=0.19\linewidth]{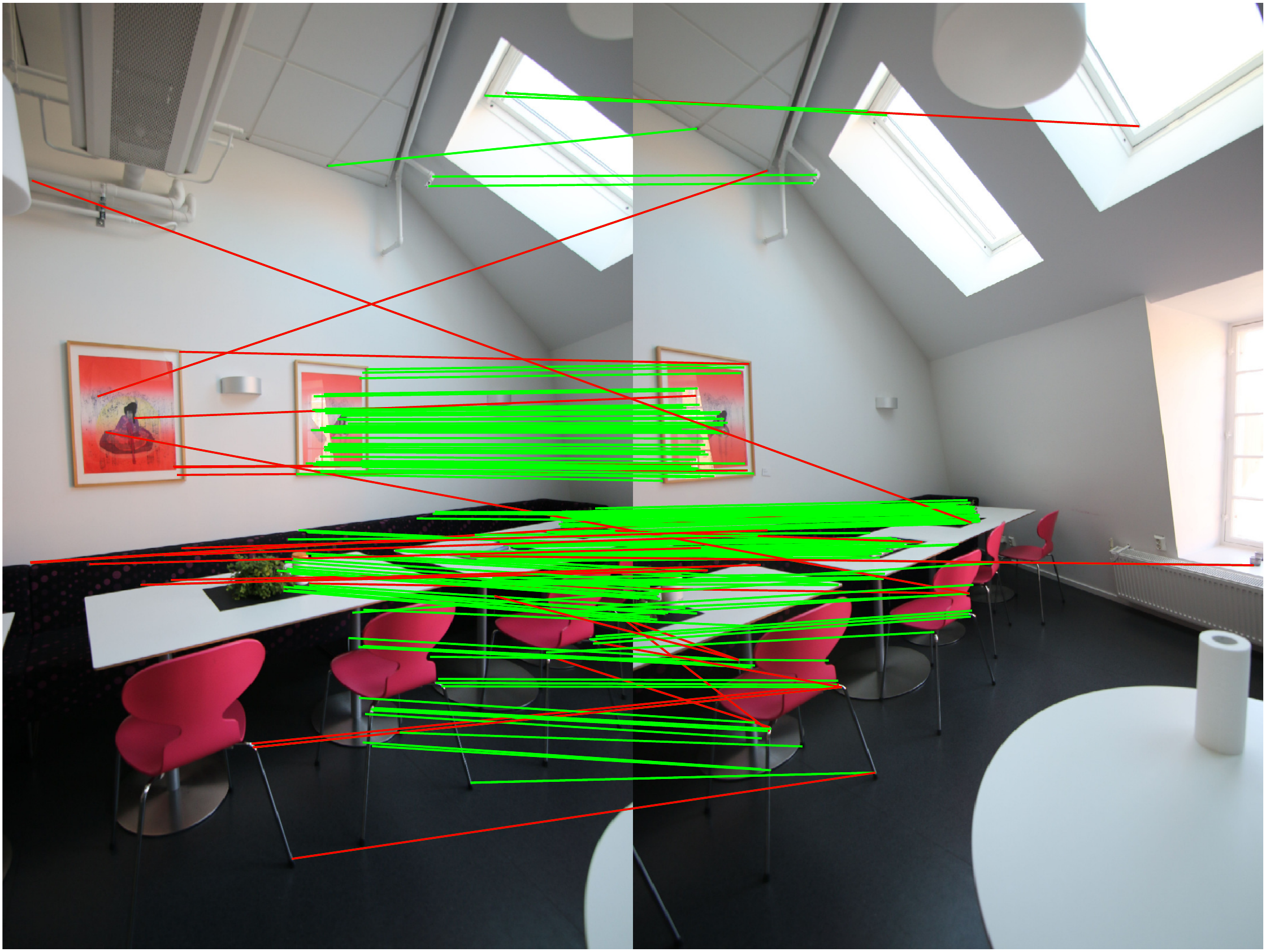}
&\includegraphics[width=0.096\linewidth]{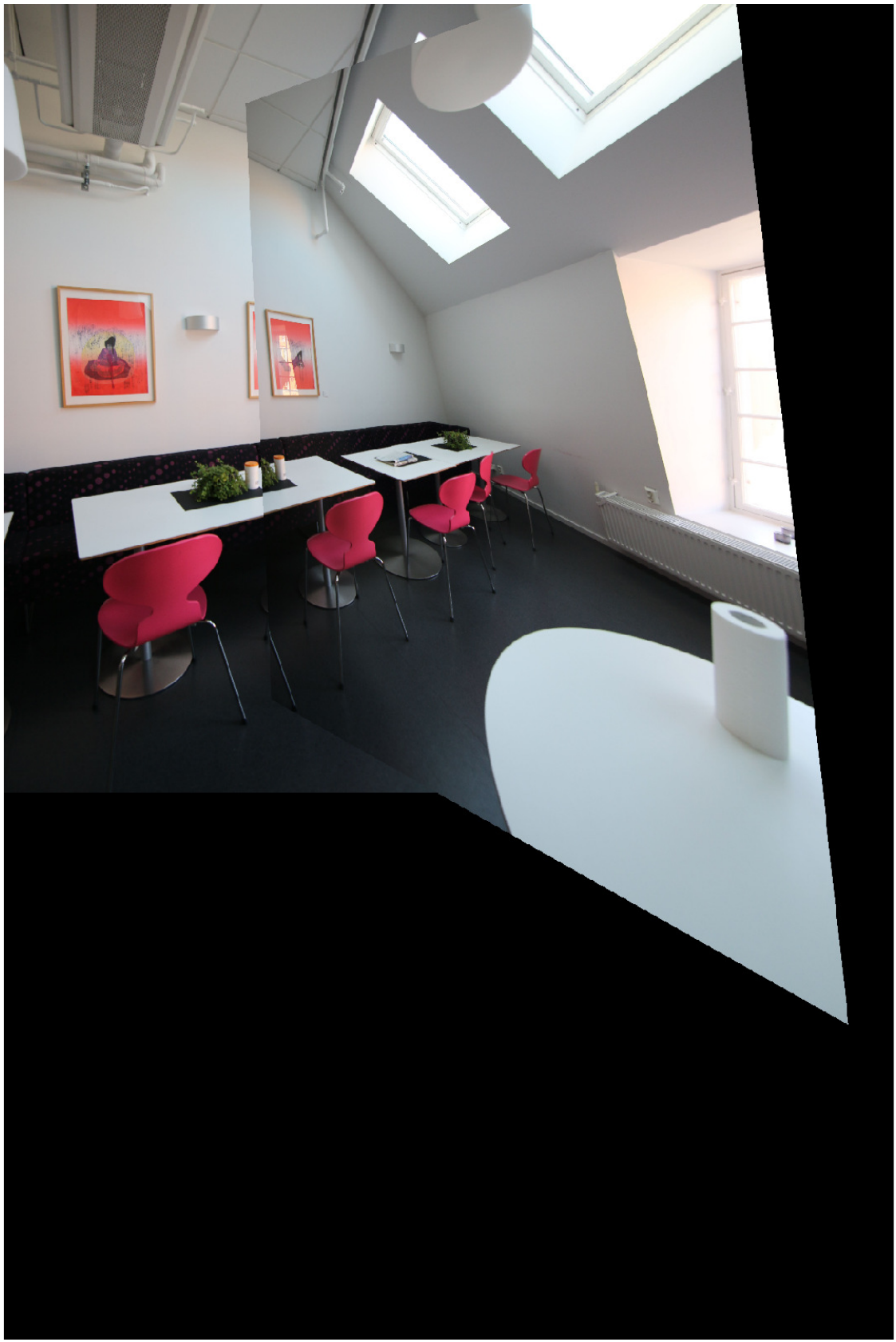}
&\includegraphics[width=0.19\linewidth]{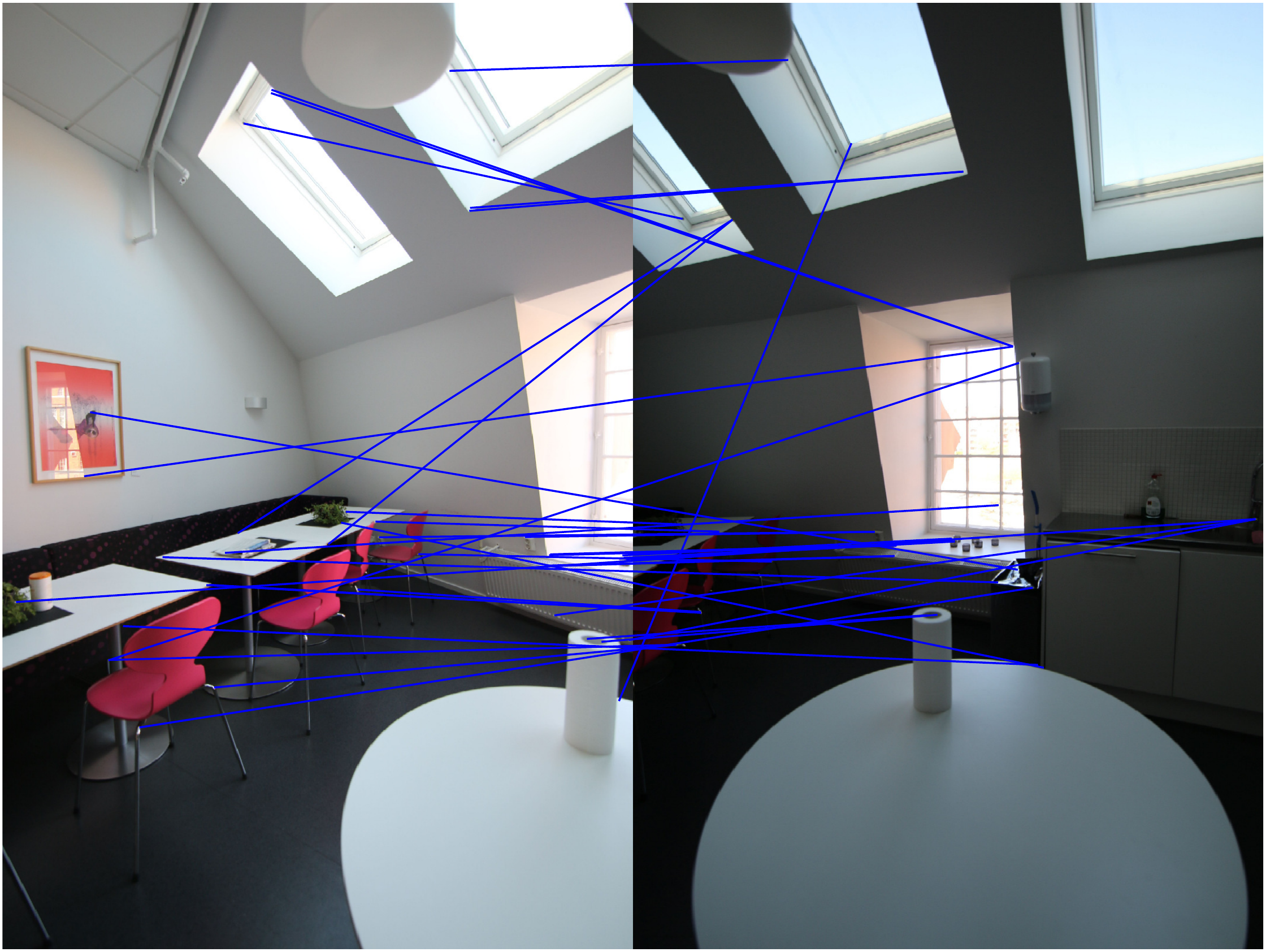}
&\includegraphics[width=0.19\linewidth]{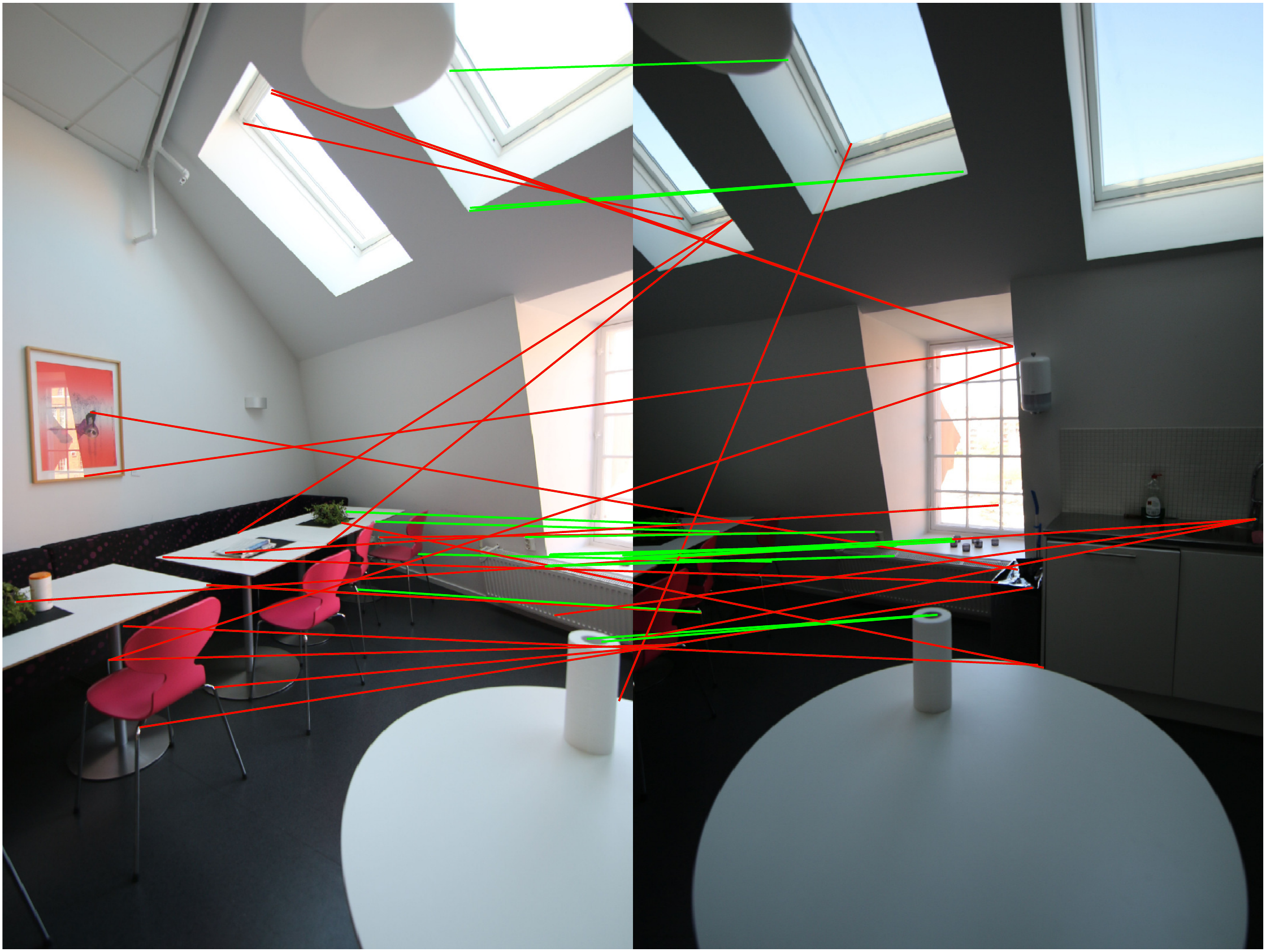}
&\includegraphics[width=0.096\linewidth]{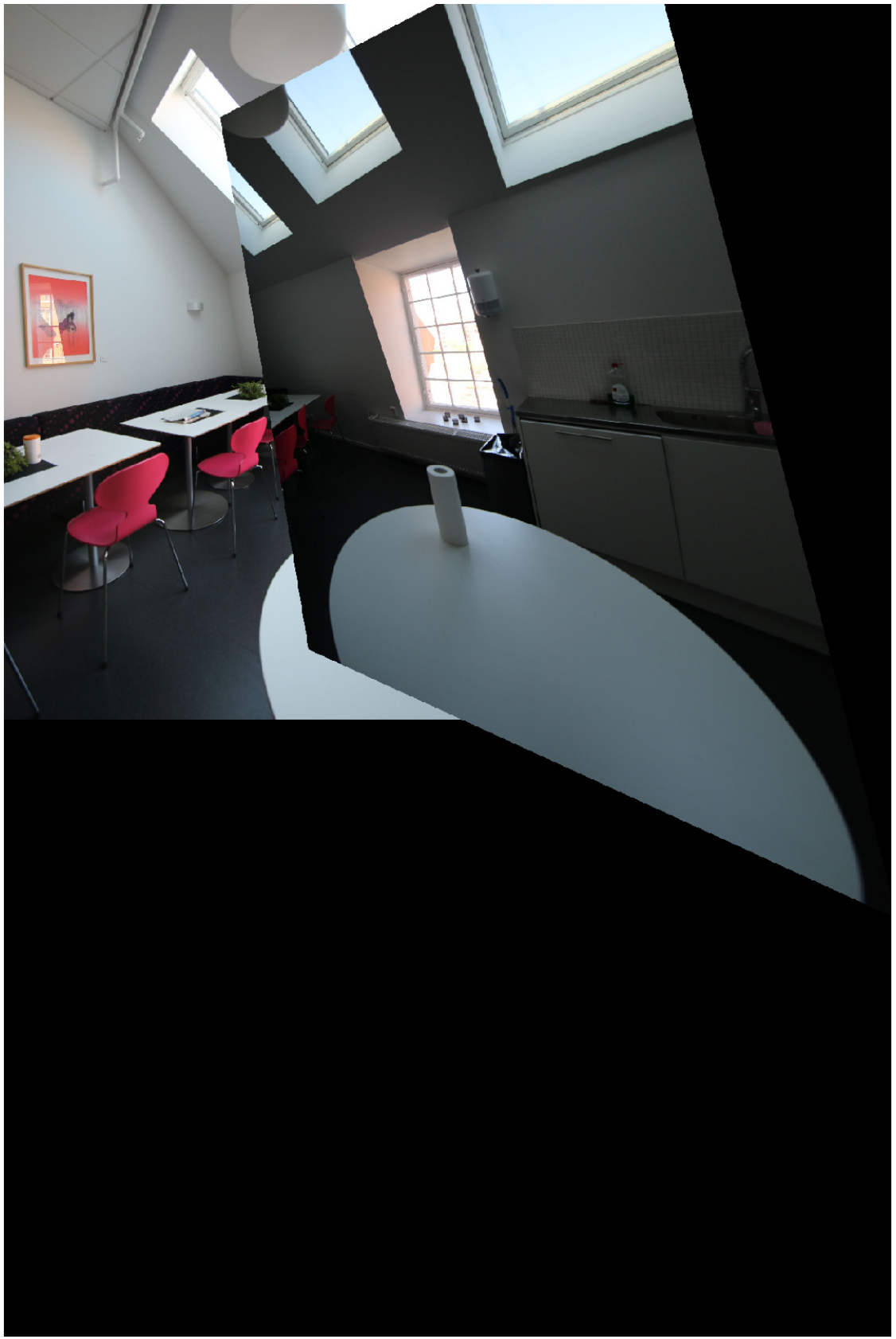} \\

\textit{Image 14\&21}, 19.00\% & Inliers Found by ICOS & Stitching&
\textit{Image 21\&30}, 13.68\% & Inliers Found by ICOS & Stitching \\

\includegraphics[width=0.19\linewidth]{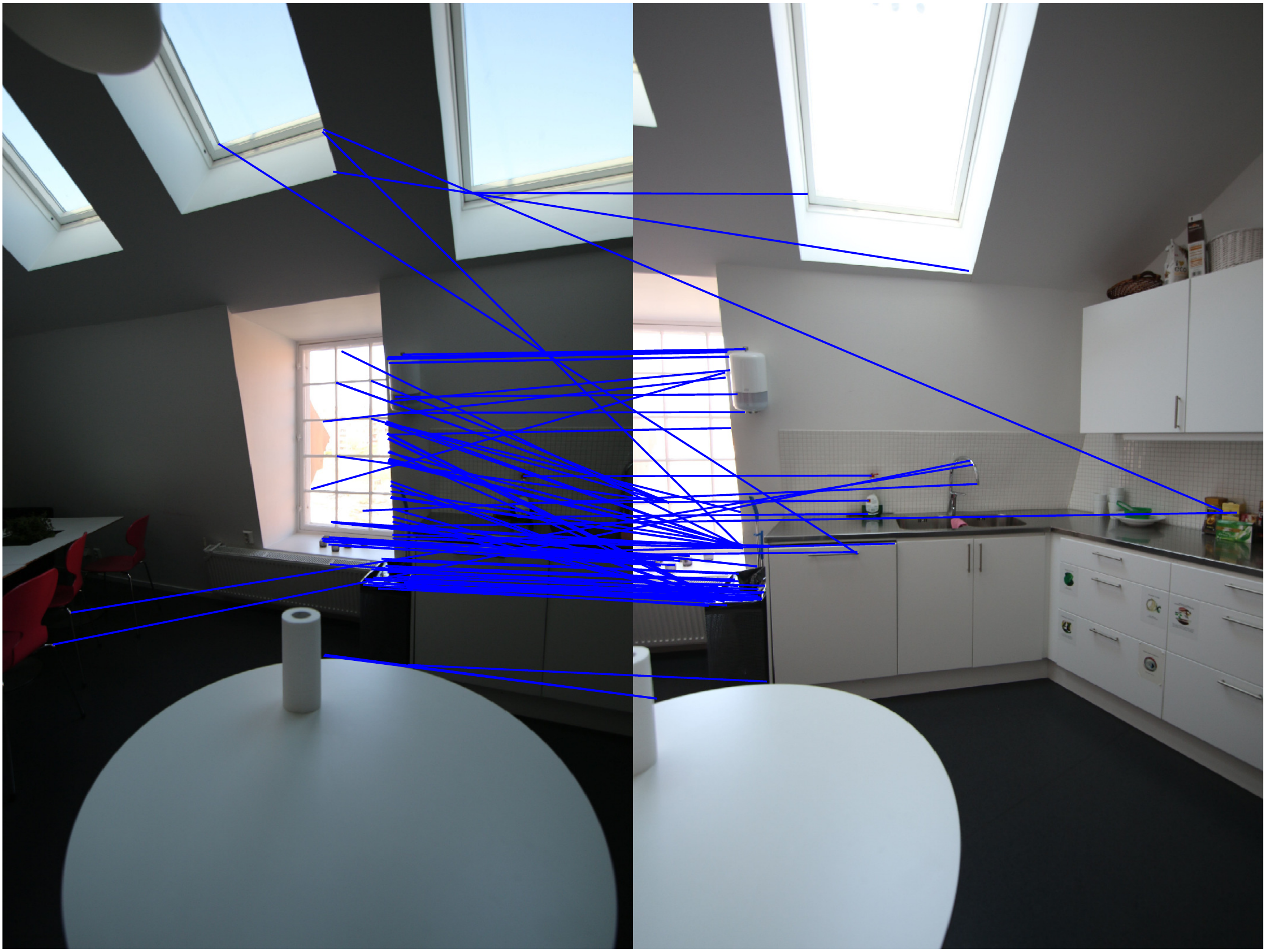}
&\includegraphics[width=0.19\linewidth]{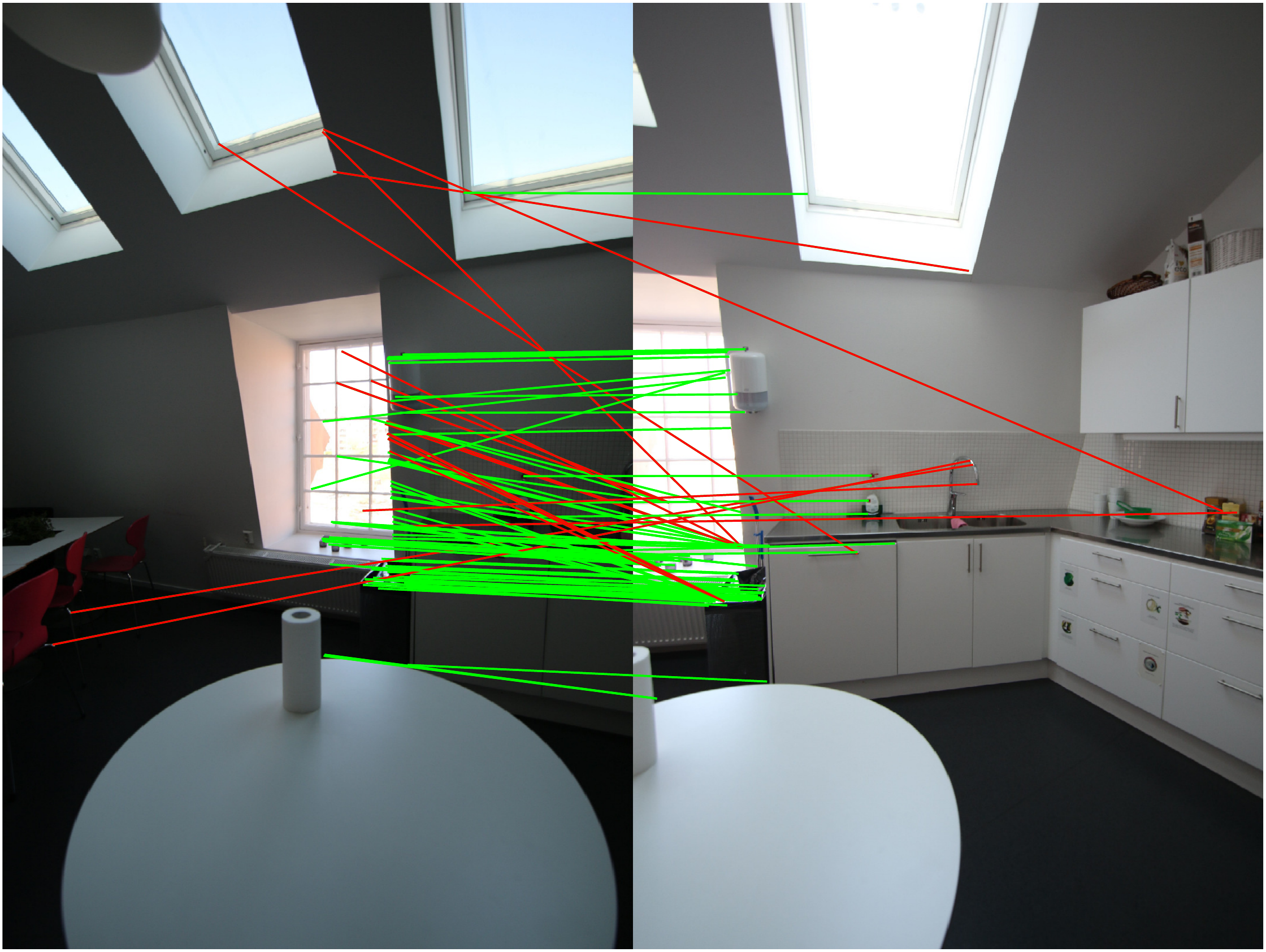}
&\includegraphics[width=0.096\linewidth]{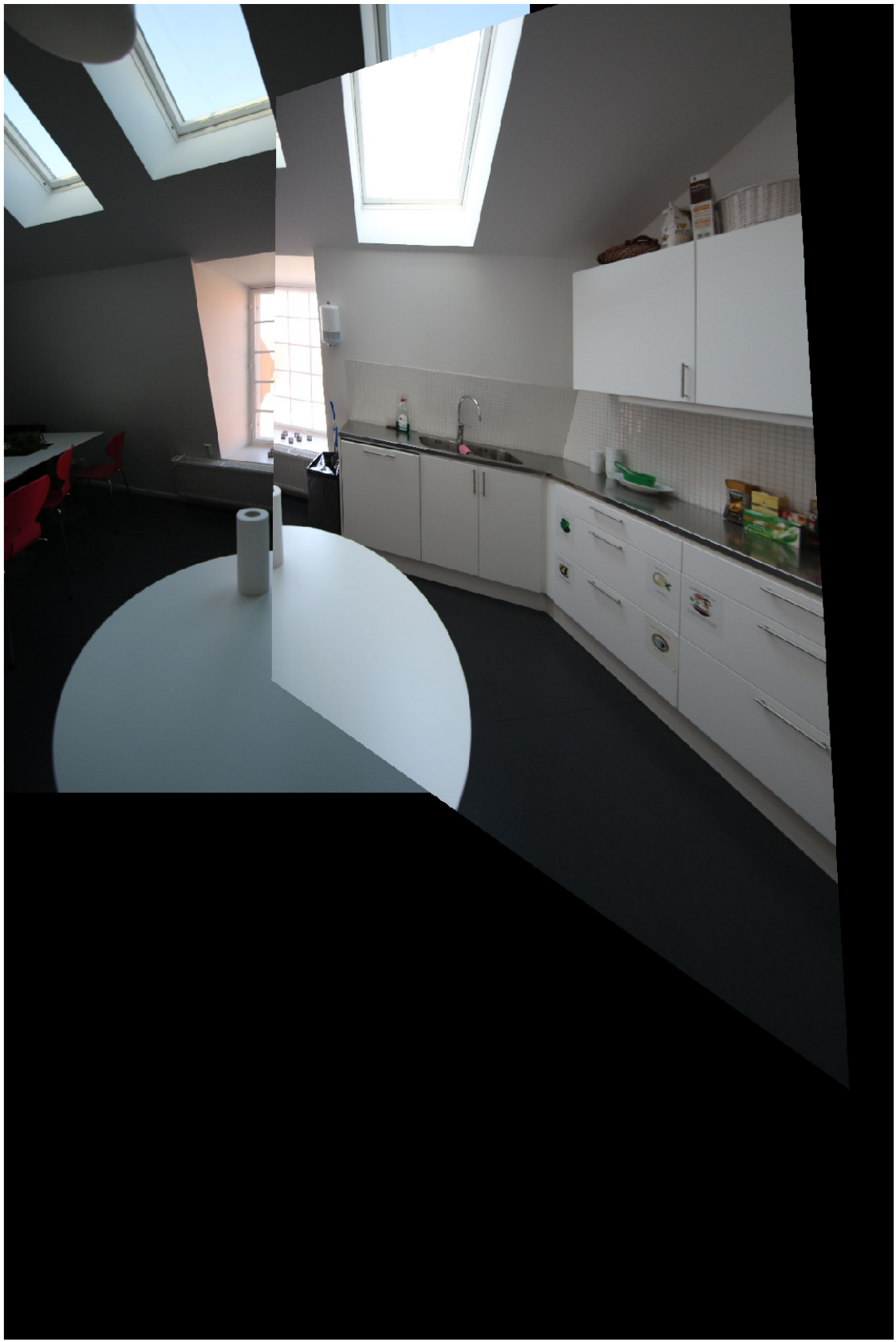}
&\includegraphics[width=0.19\linewidth]{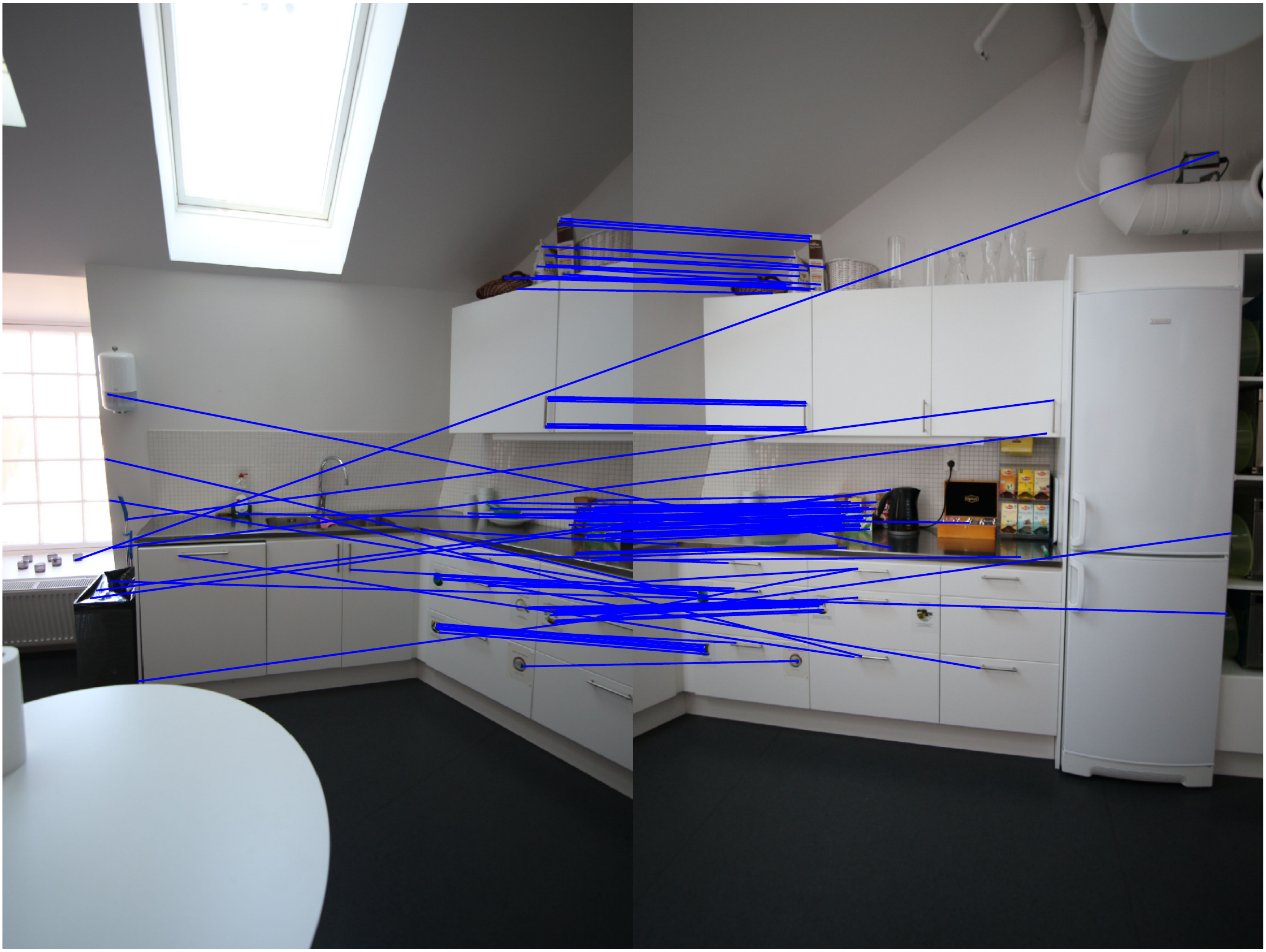}
&\includegraphics[width=0.19\linewidth]{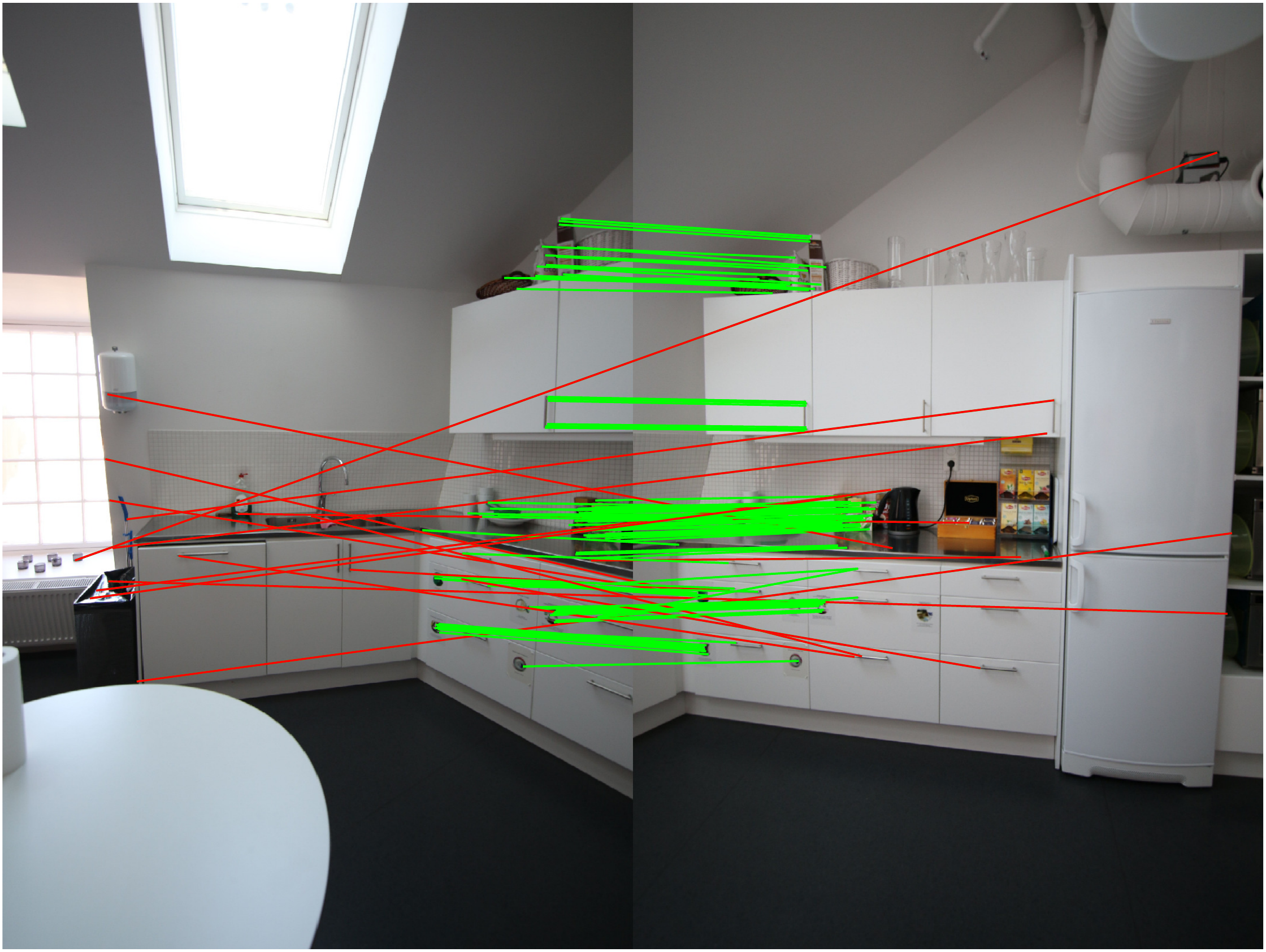}
&\includegraphics[width=0.096\linewidth]{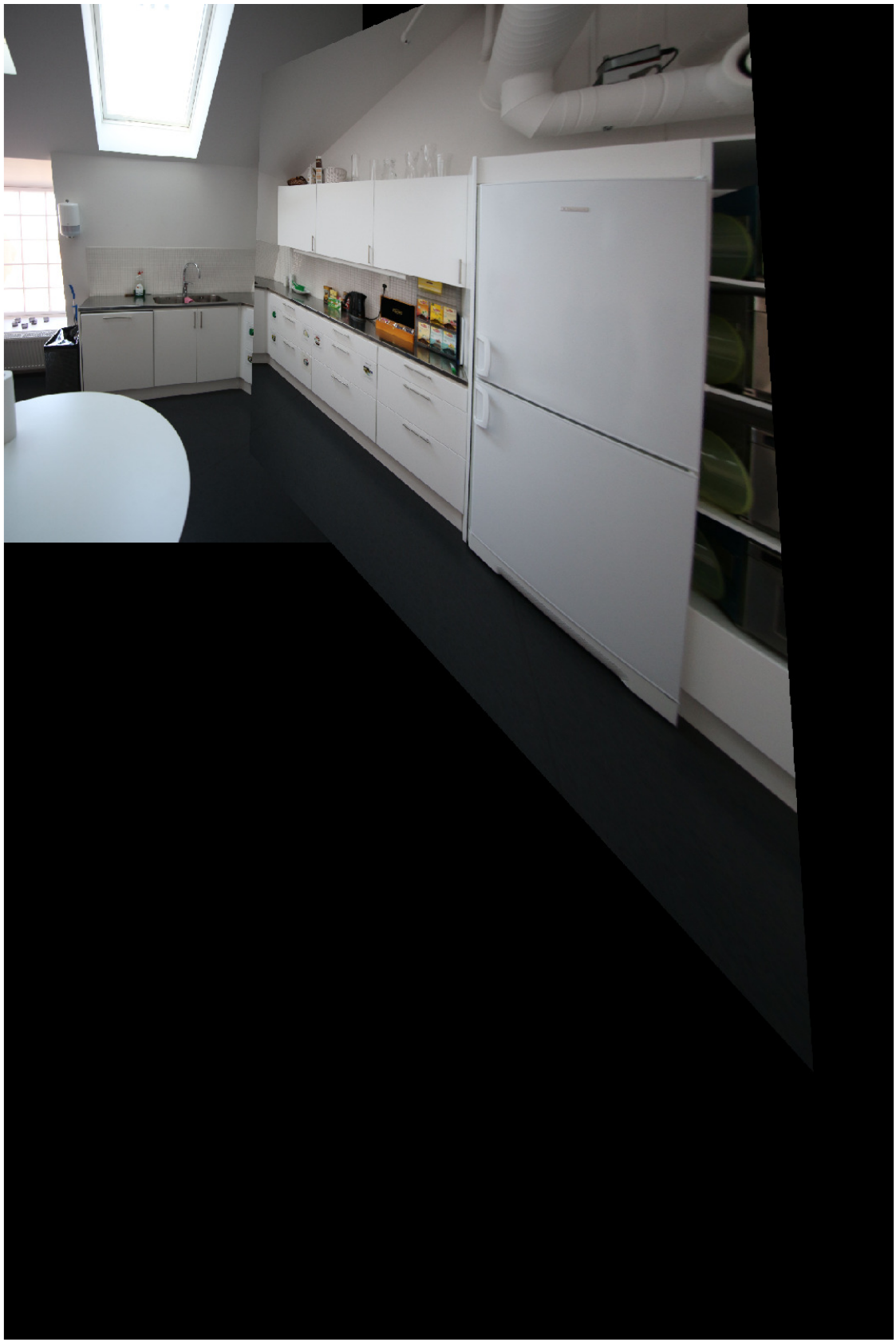} \\

\textit{Image 30\&40}, 56.56\% & Inliers Found by ICOS & Stitching&
\textit{Image 40\&48}, 43.66\% & Inliers Found by ICOS & Stitching \\

\includegraphics[width=0.19\linewidth]{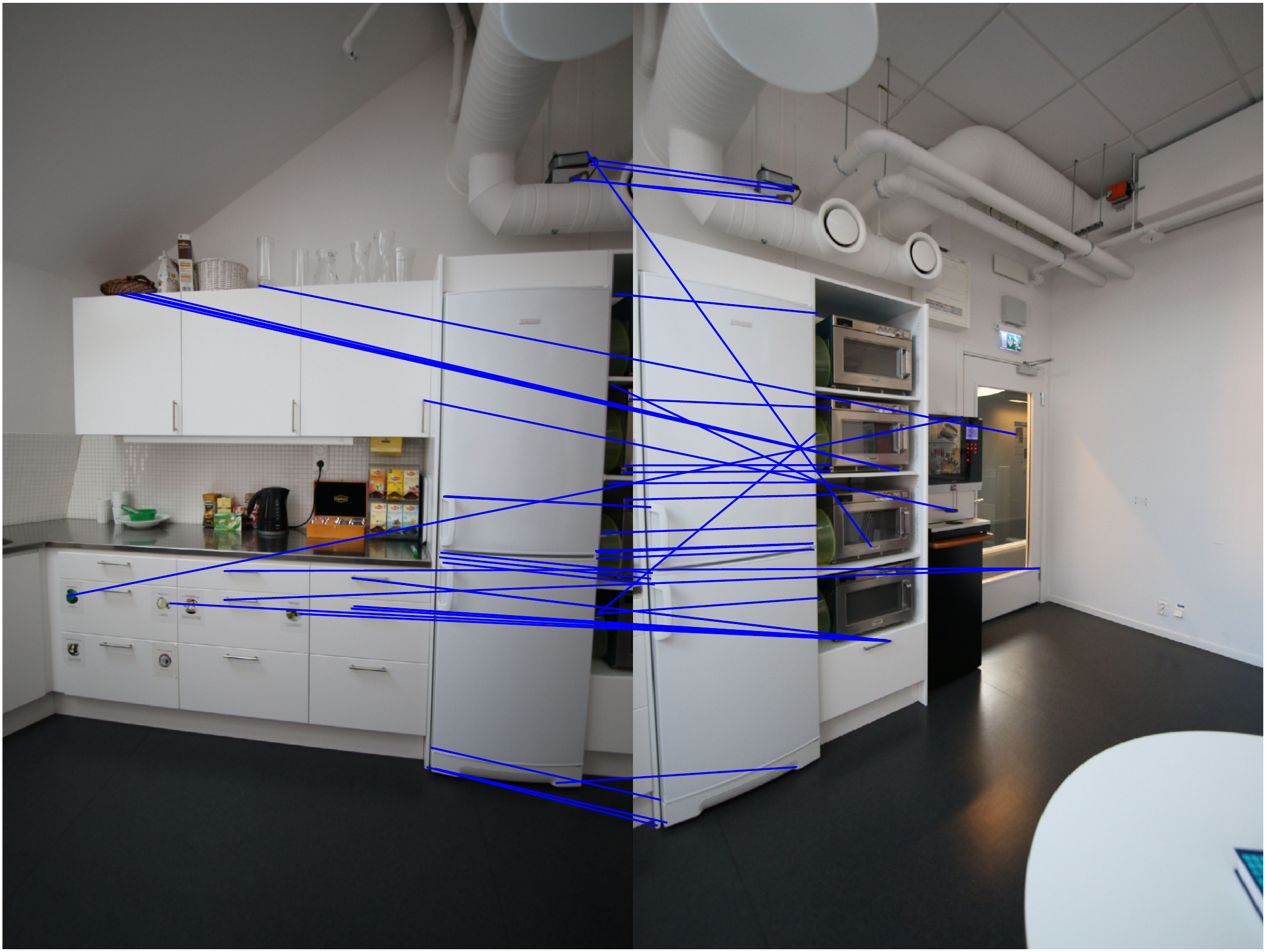}
&\includegraphics[width=0.19\linewidth]{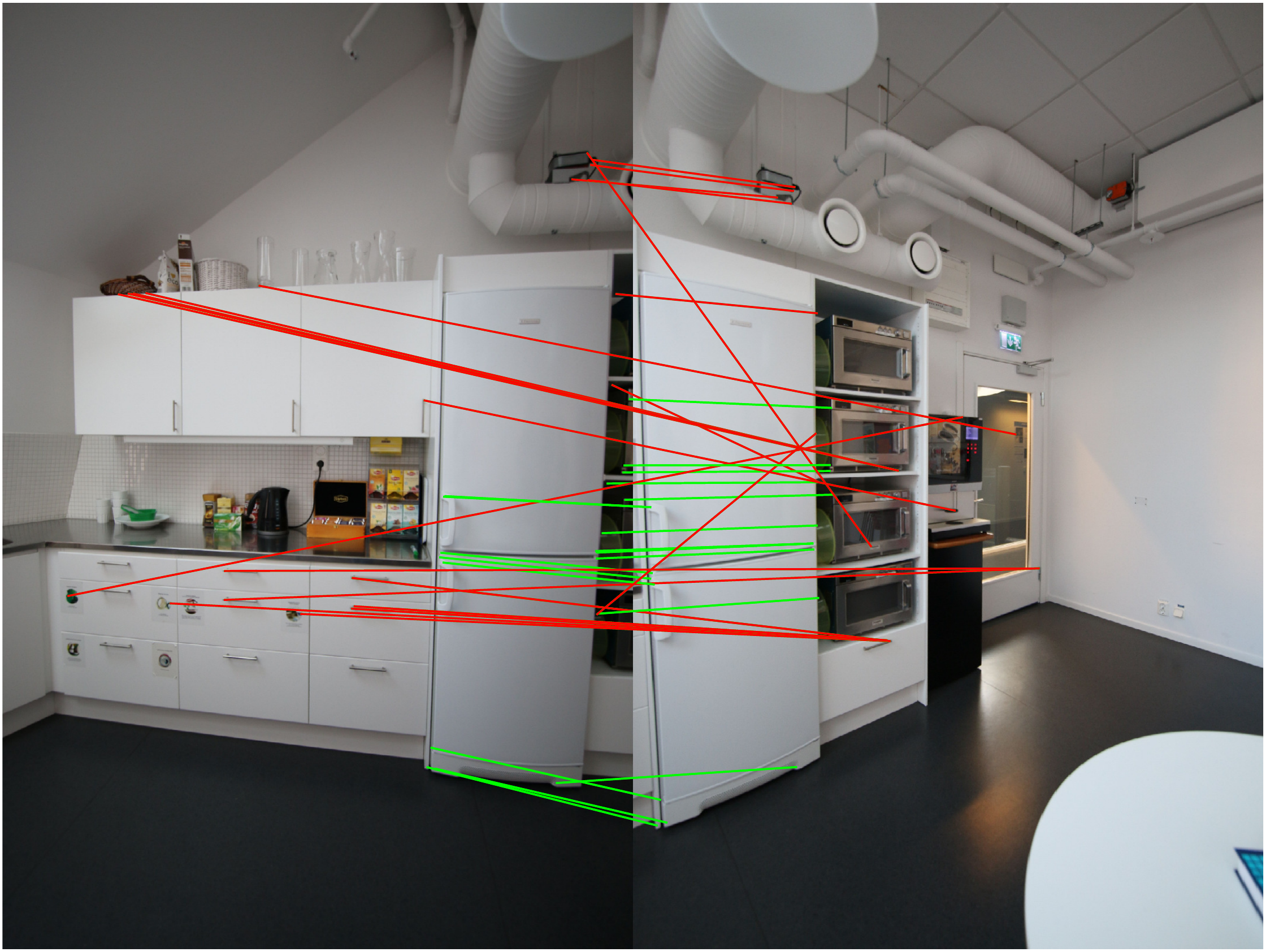}
&\includegraphics[width=0.096\linewidth]{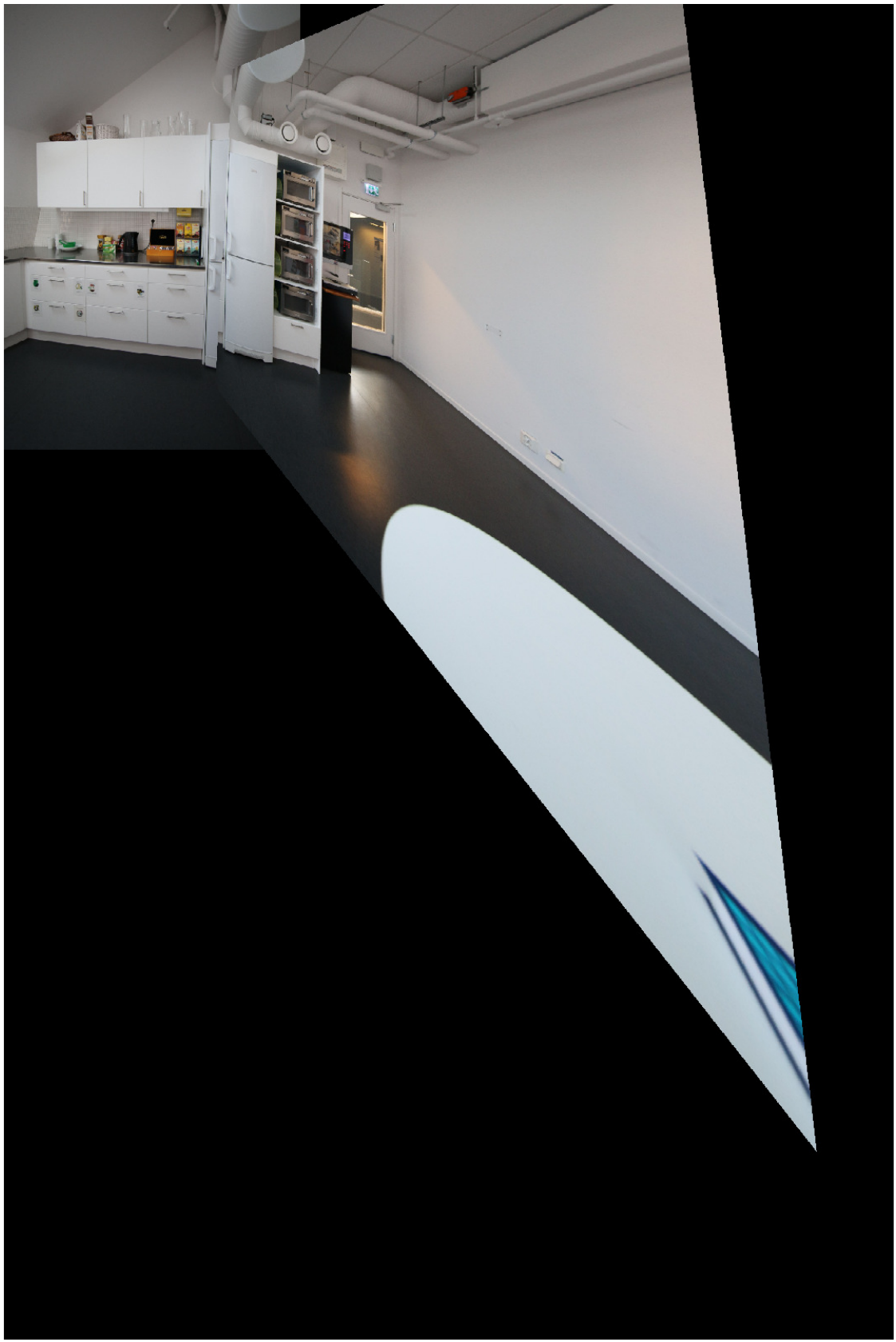}
&\includegraphics[width=0.19\linewidth]{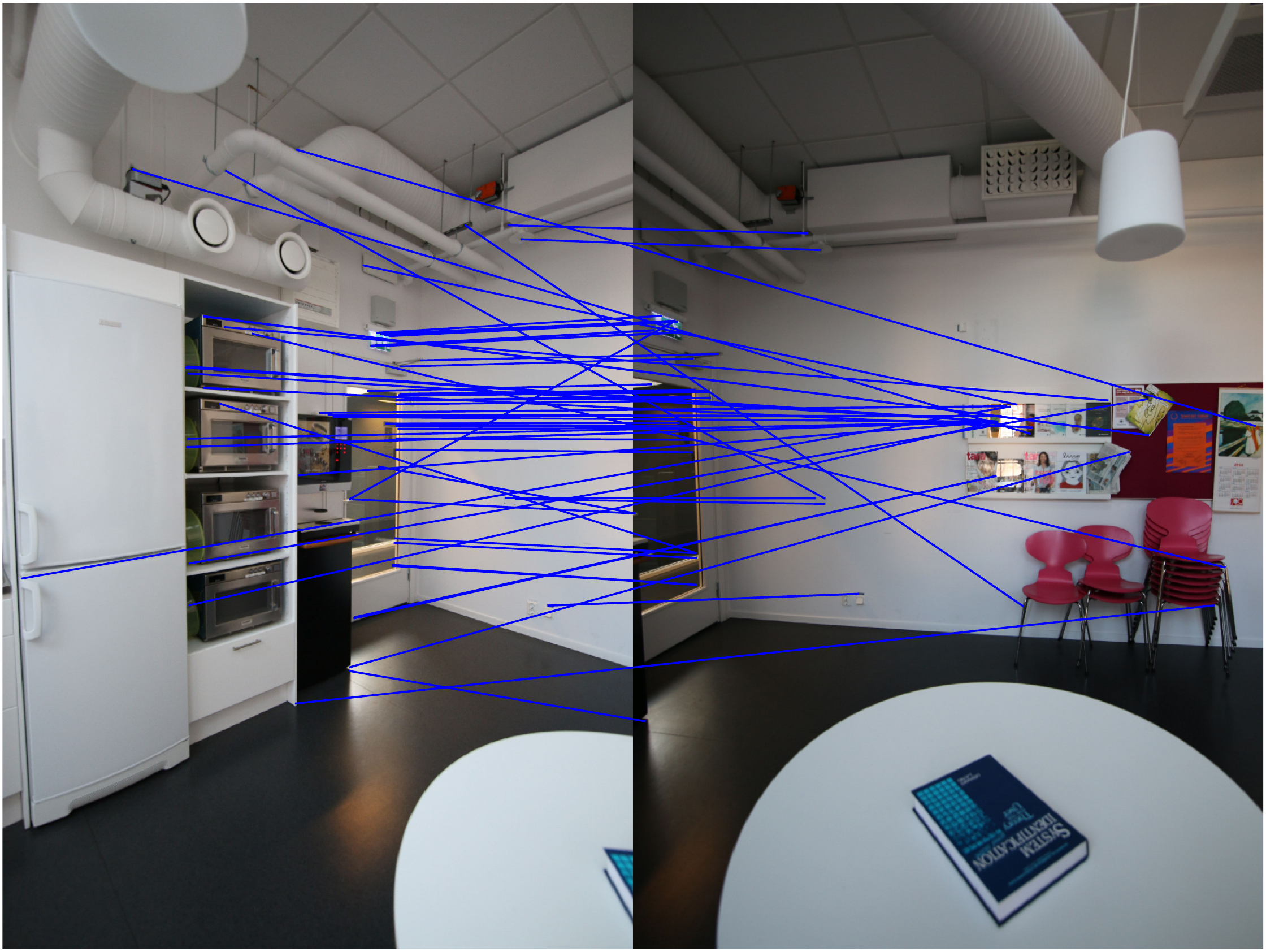}
&\includegraphics[width=0.19\linewidth]{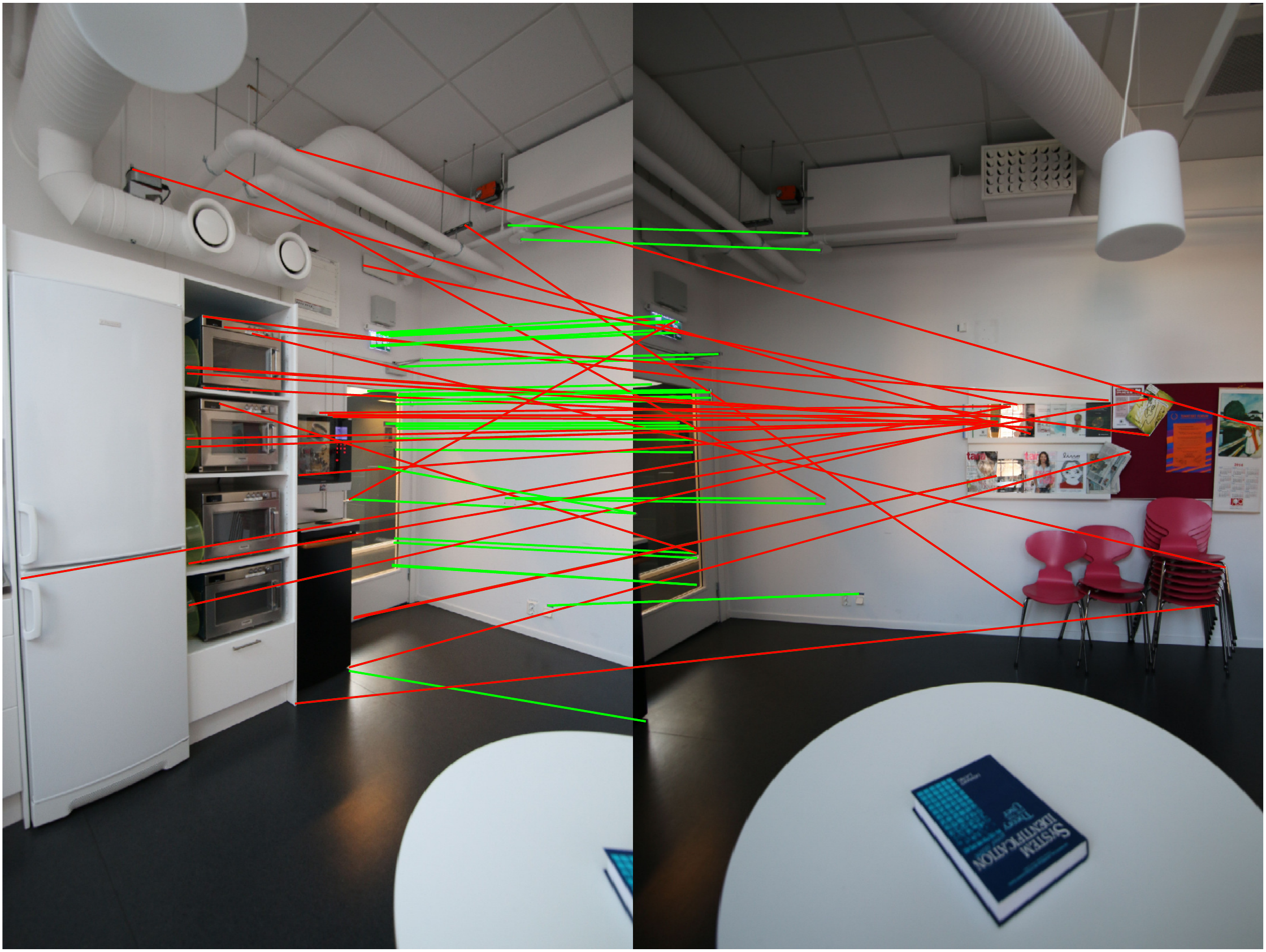}
&\includegraphics[width=0.096\linewidth]{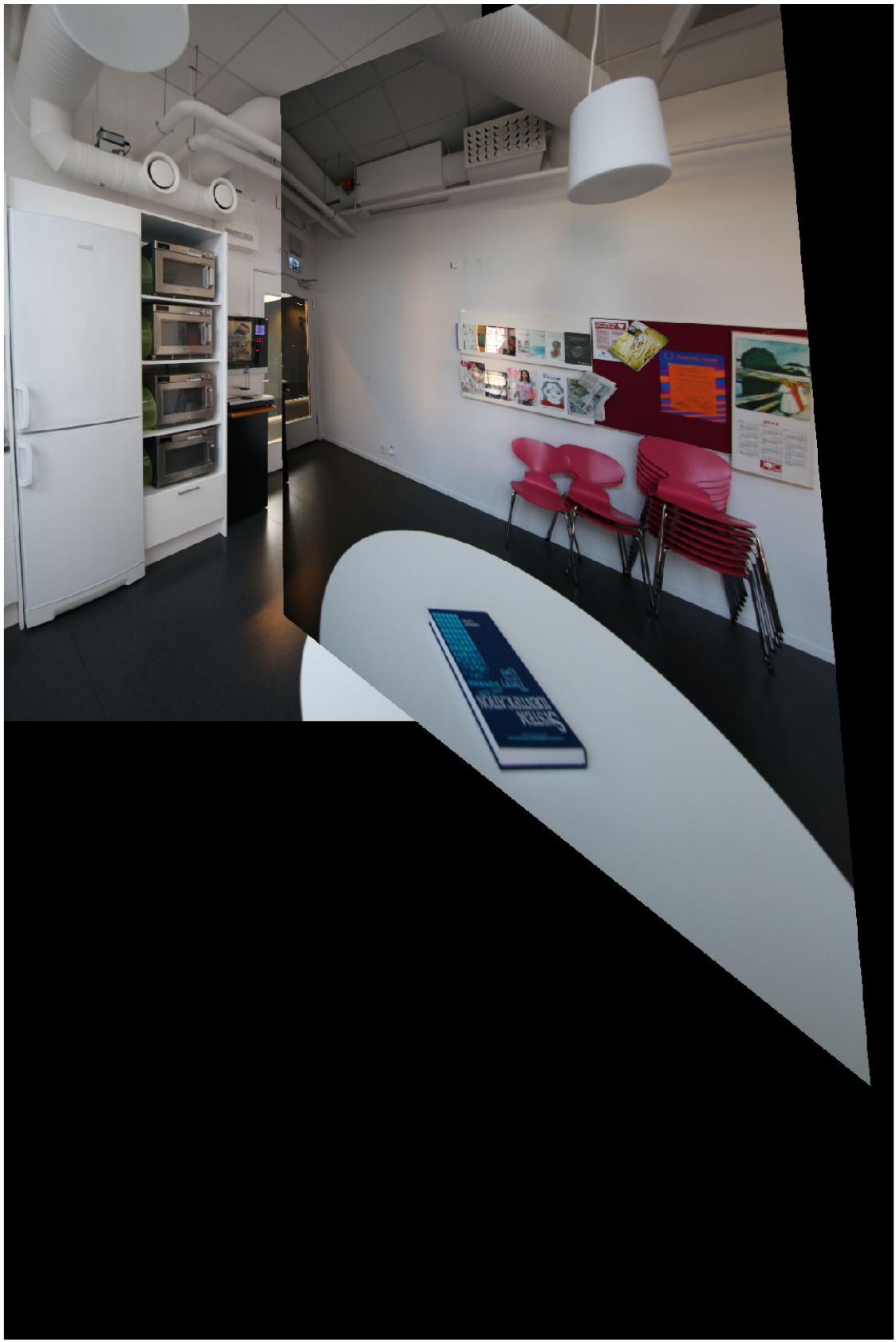} \\

\textit{Image 48\&60}, 17.68\% & Inliers Found by ICOS & Stitching&
\textit{Image 60\&69}, 48.84\% & Inliers Found by ICOS & Stitching \\

\includegraphics[width=0.19\linewidth]{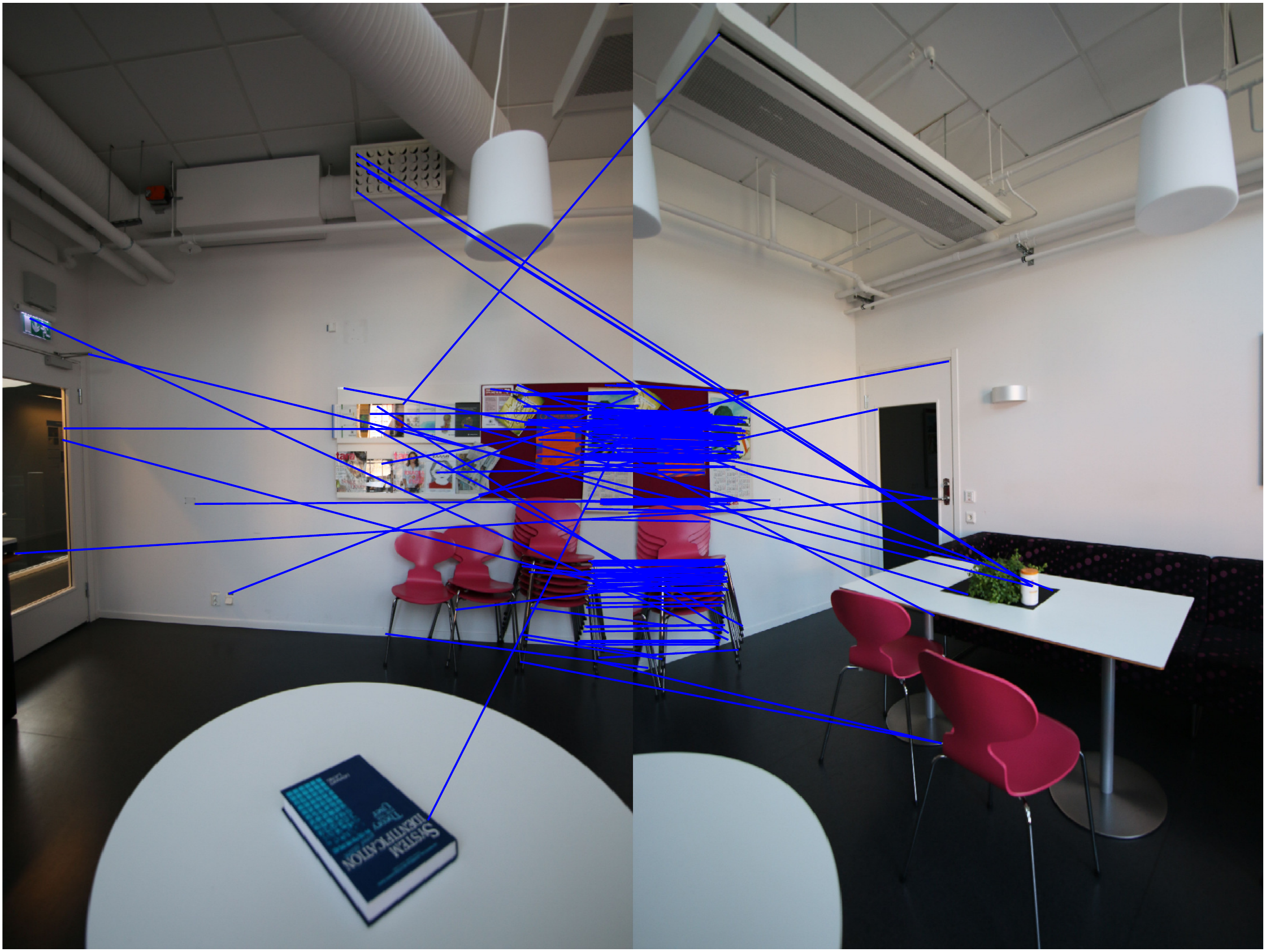}
&\includegraphics[width=0.19\linewidth]{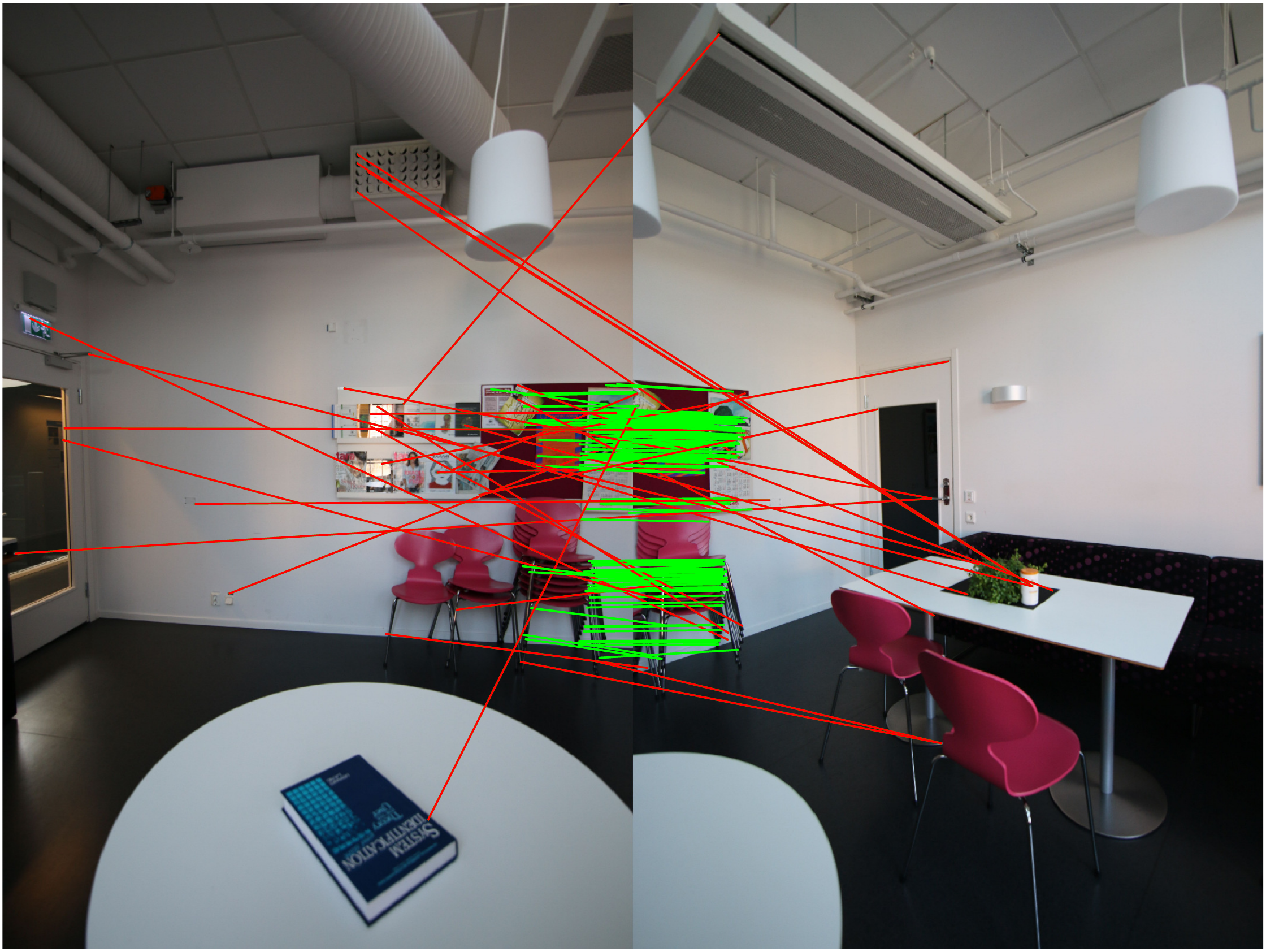}
&\includegraphics[width=0.096\linewidth]{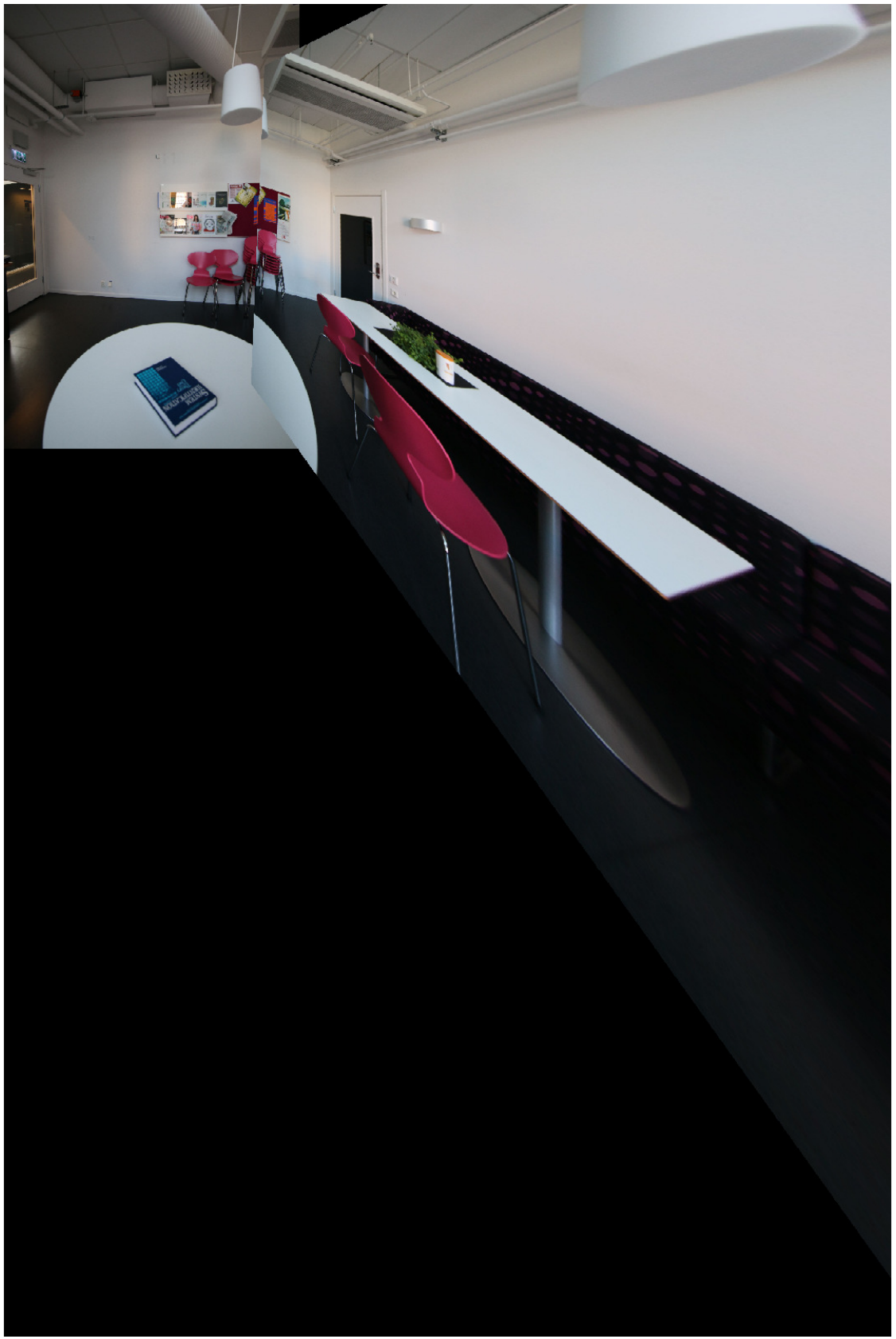}
&\includegraphics[width=0.19\linewidth]{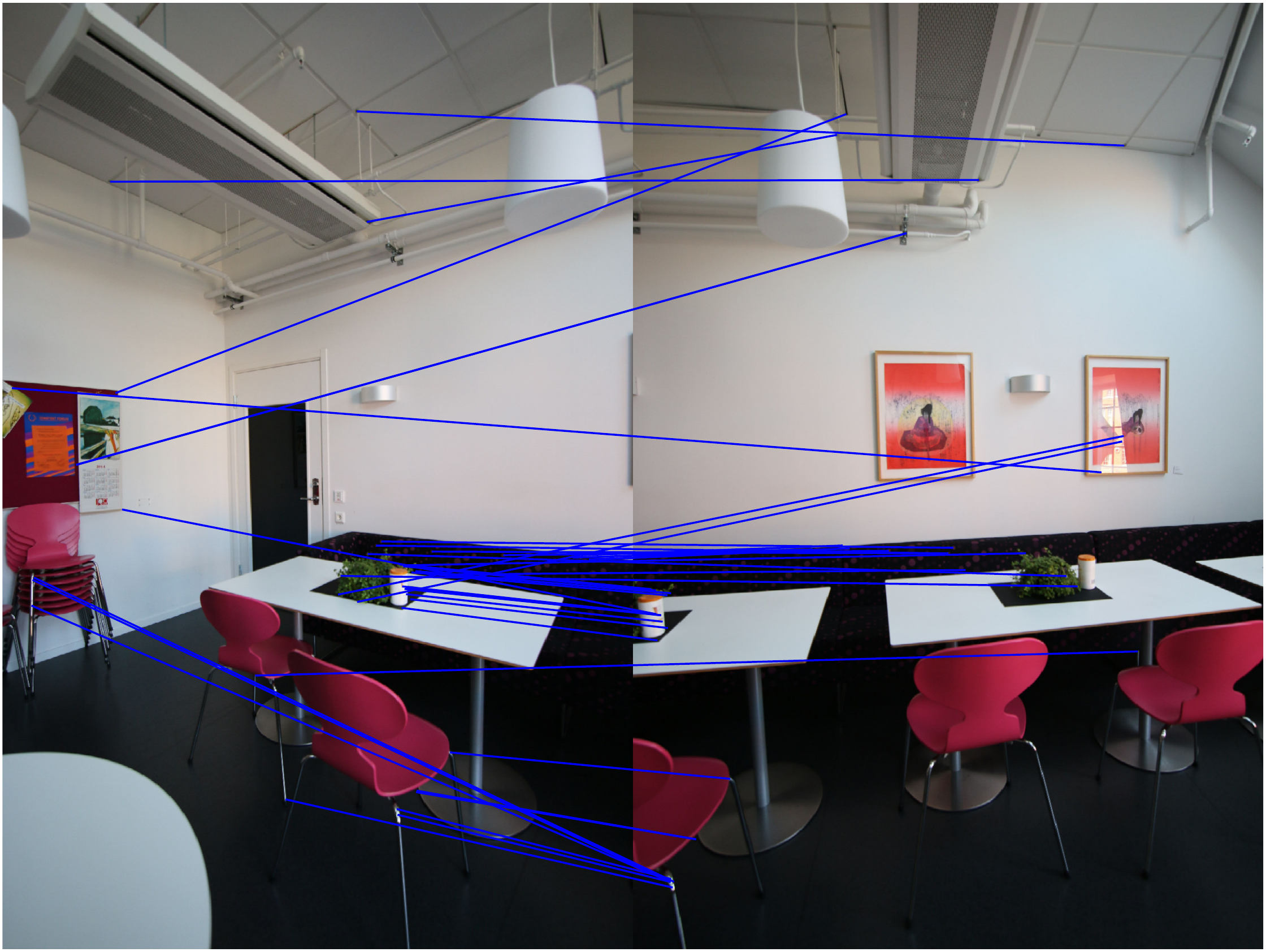}
&\includegraphics[width=0.19\linewidth]{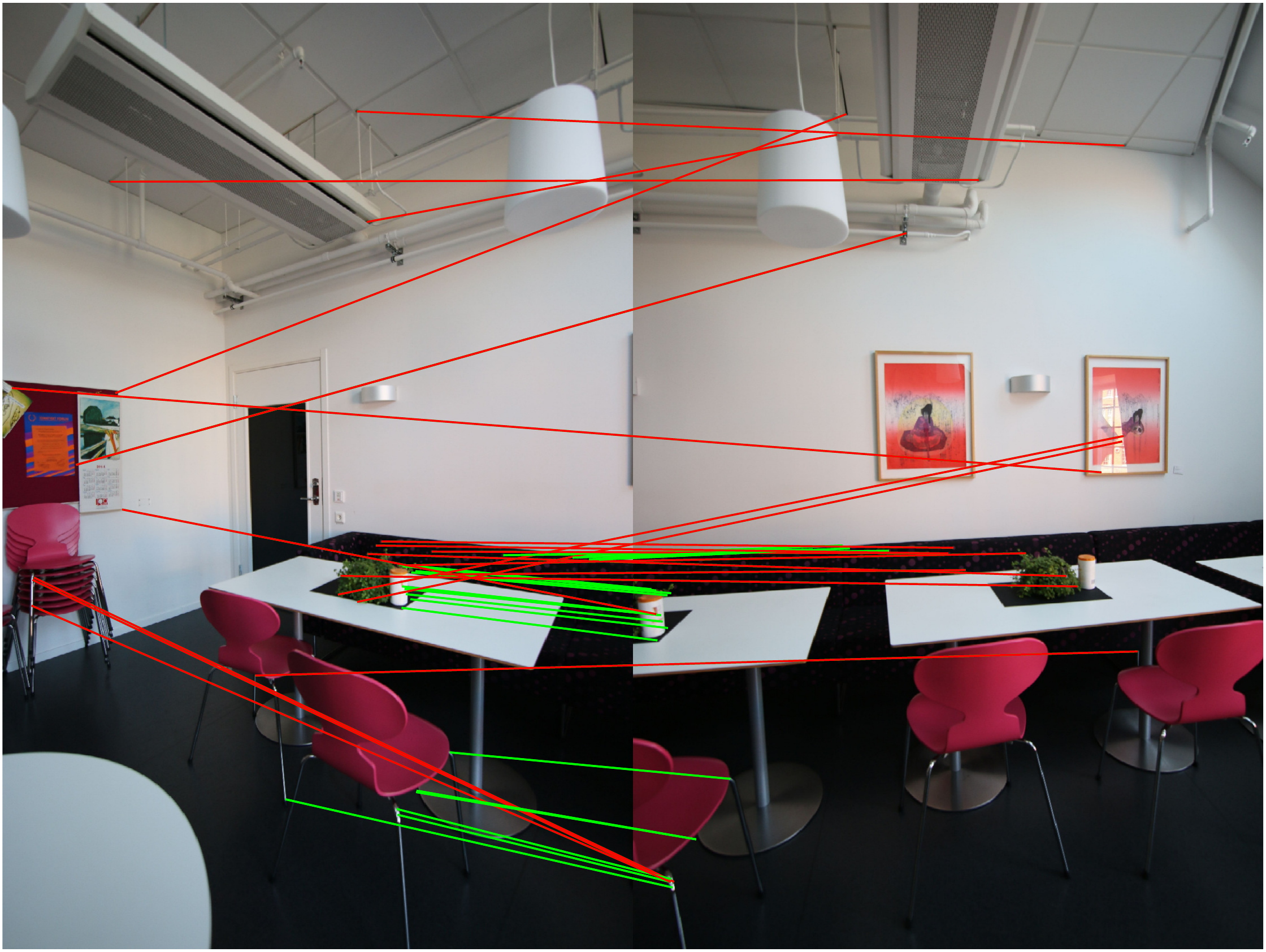}
&\includegraphics[width=0.096\linewidth]{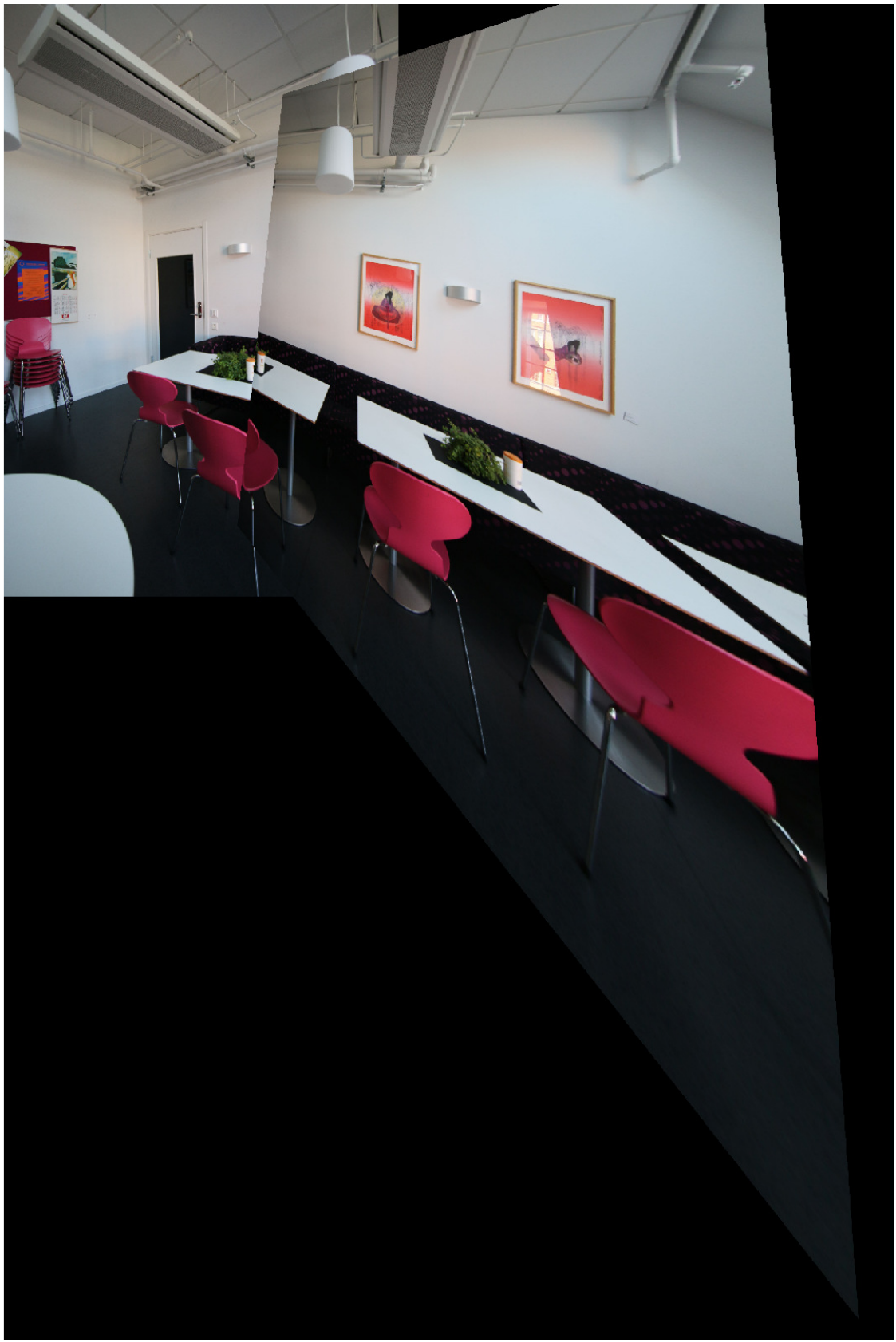} \\

\textit{Image 05\&12}, 19.77\% & Inliers Found by ICOS & Stitching&
\textit{Image 12\&19}, 3.64\% & Inliers Found by ICOS & Stitching \\

\includegraphics[width=0.19\linewidth]{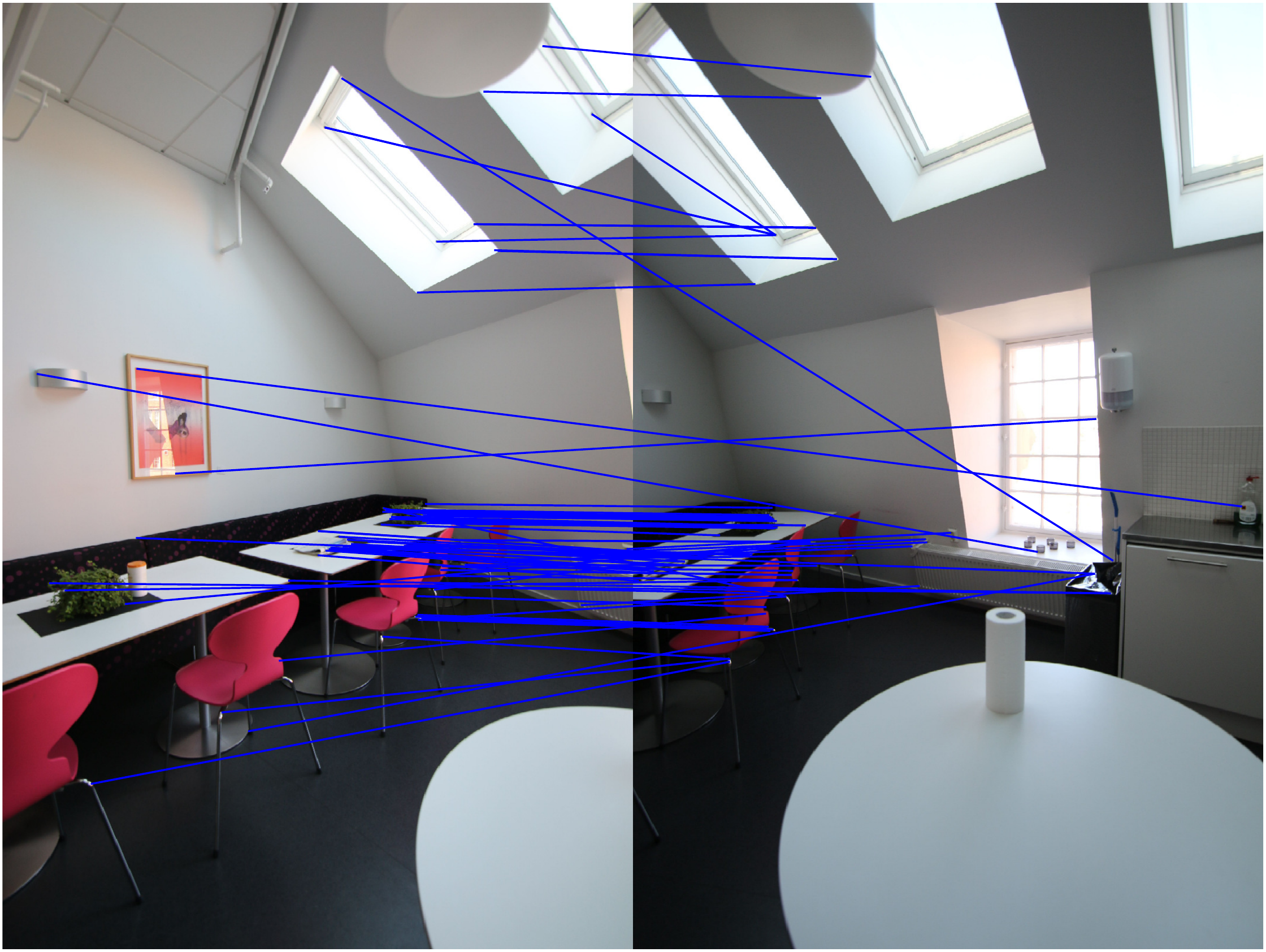}
&\includegraphics[width=0.19\linewidth]{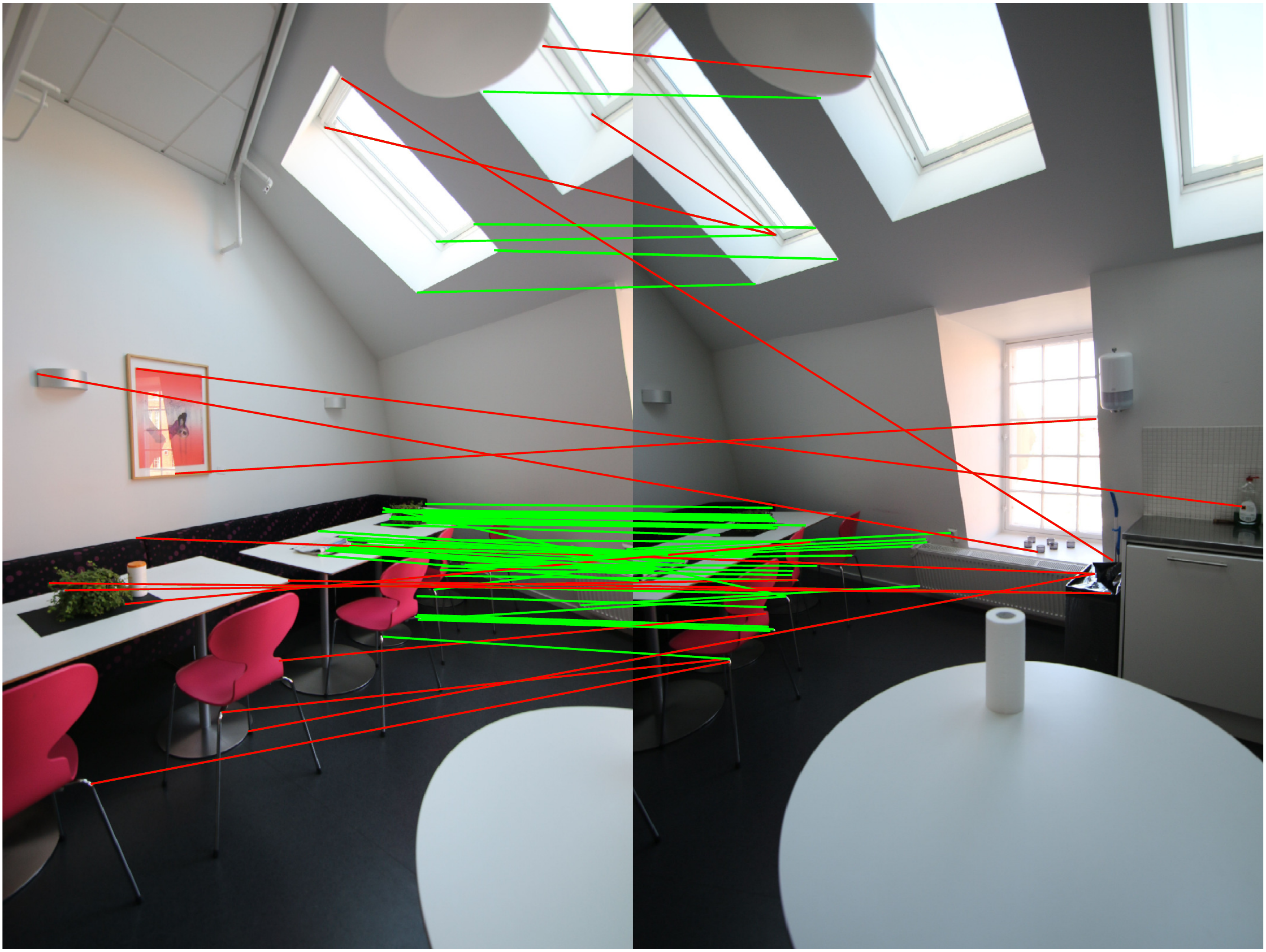}
&\includegraphics[width=0.096\linewidth]{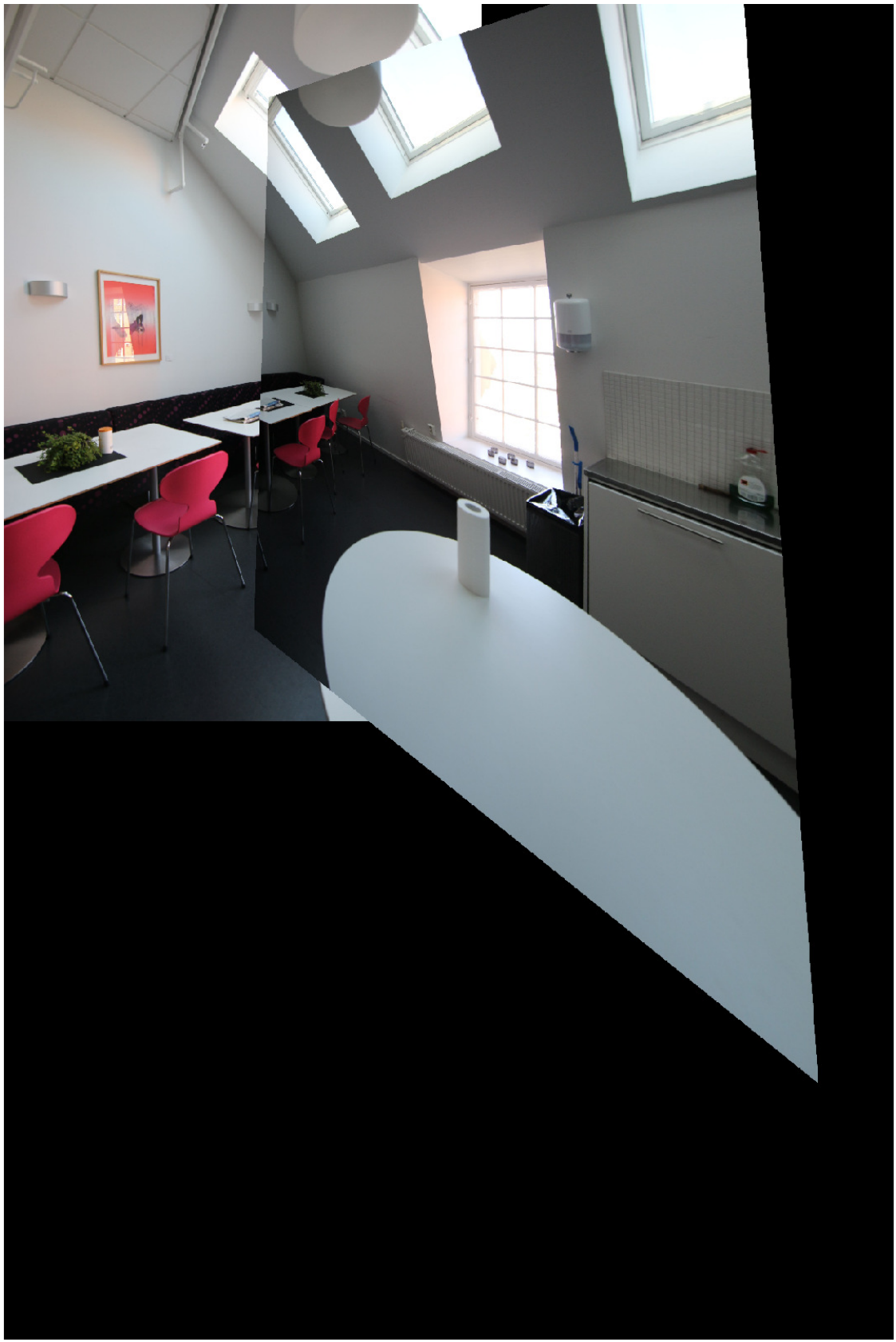}
&\includegraphics[width=0.19\linewidth]{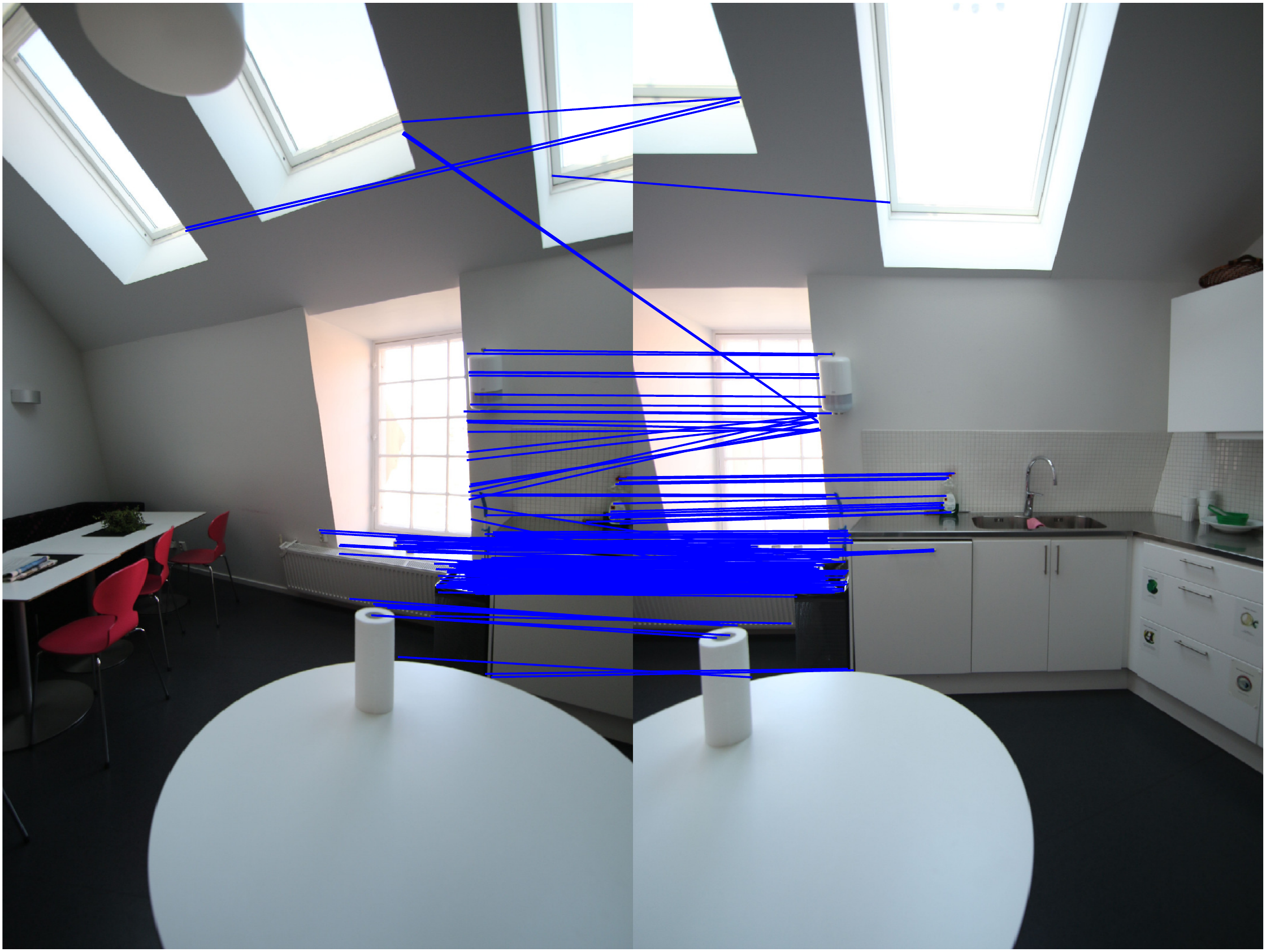}
&\includegraphics[width=0.19\linewidth]{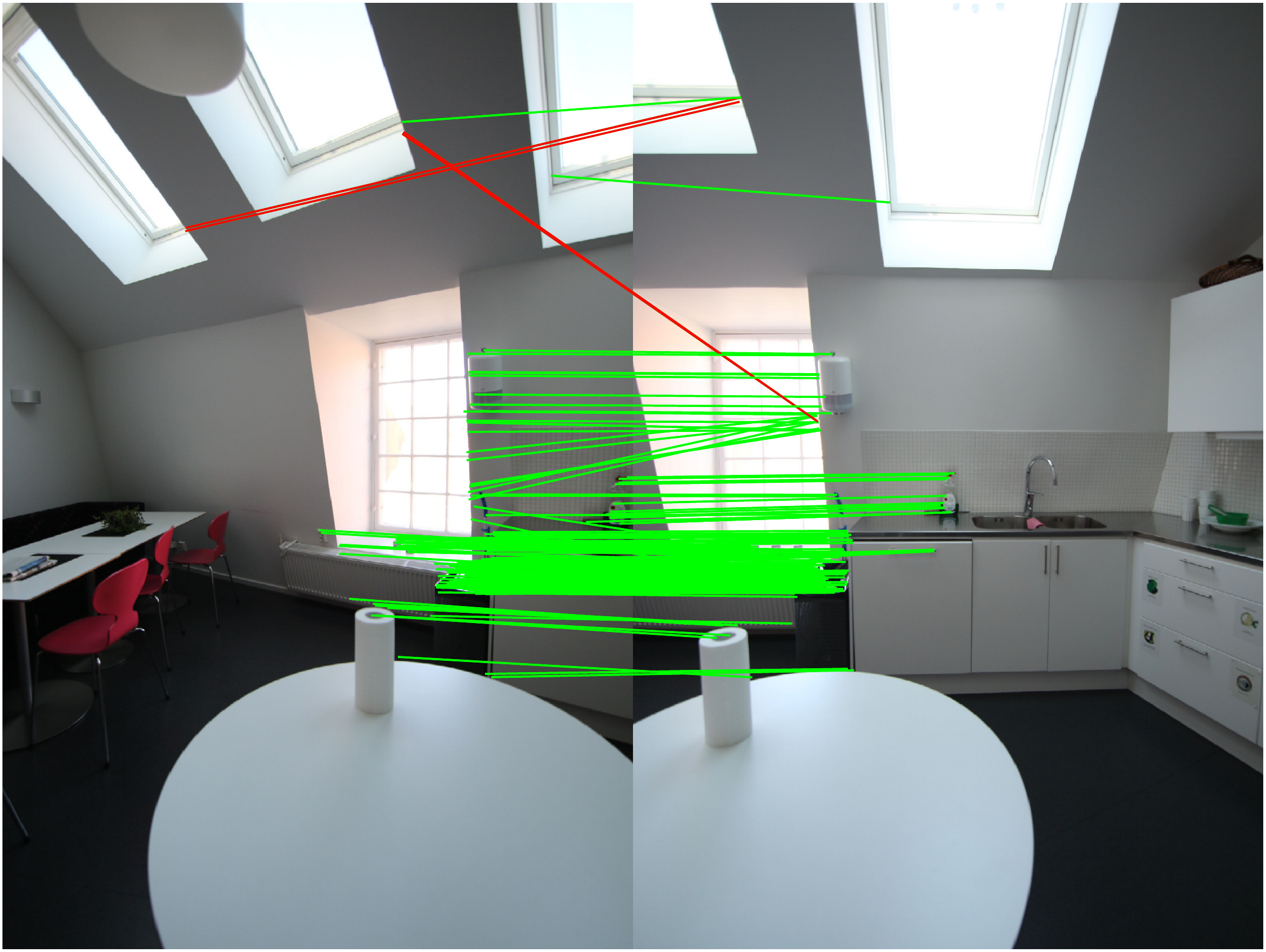}
&\includegraphics[width=0.096\linewidth]{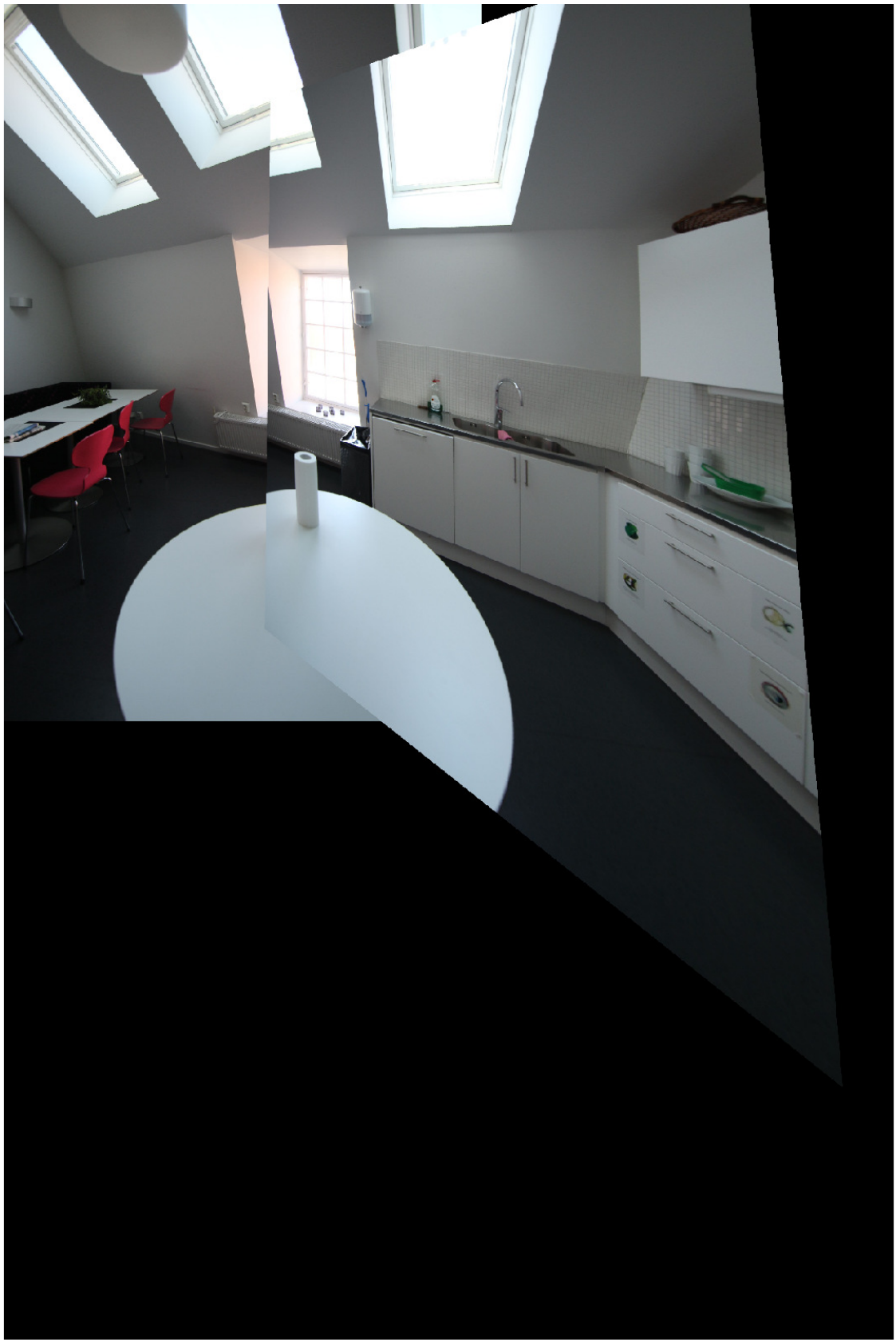} \\

\textit{Image 19\&27}, 33.82\% & Inliers Found by ICOS & Stitching&
\textit{Image 27\&36}, 5.18\% & Inliers Found by ICOS & Stitching \\

\includegraphics[width=0.19\linewidth]{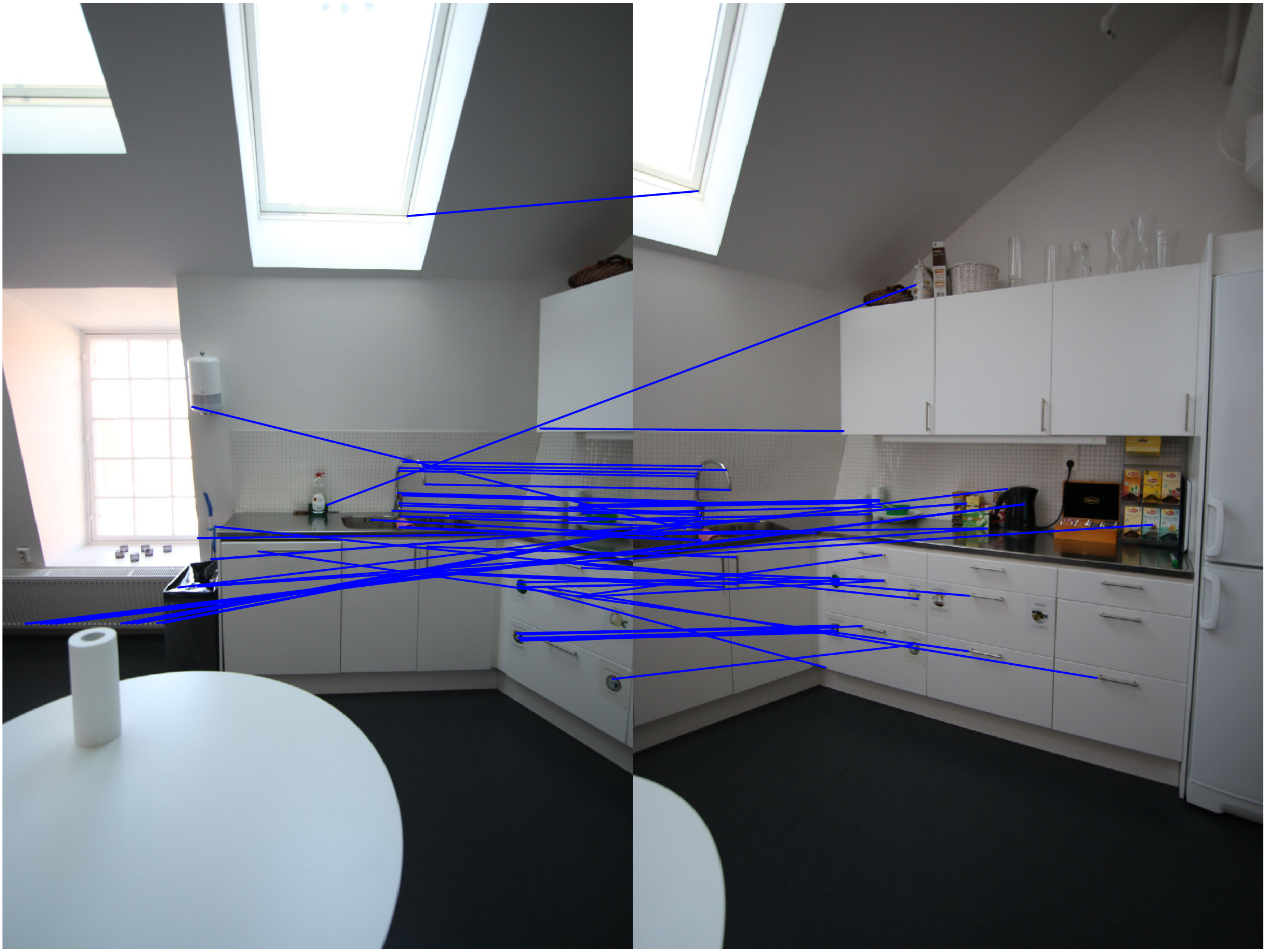}
&\includegraphics[width=0.19\linewidth]{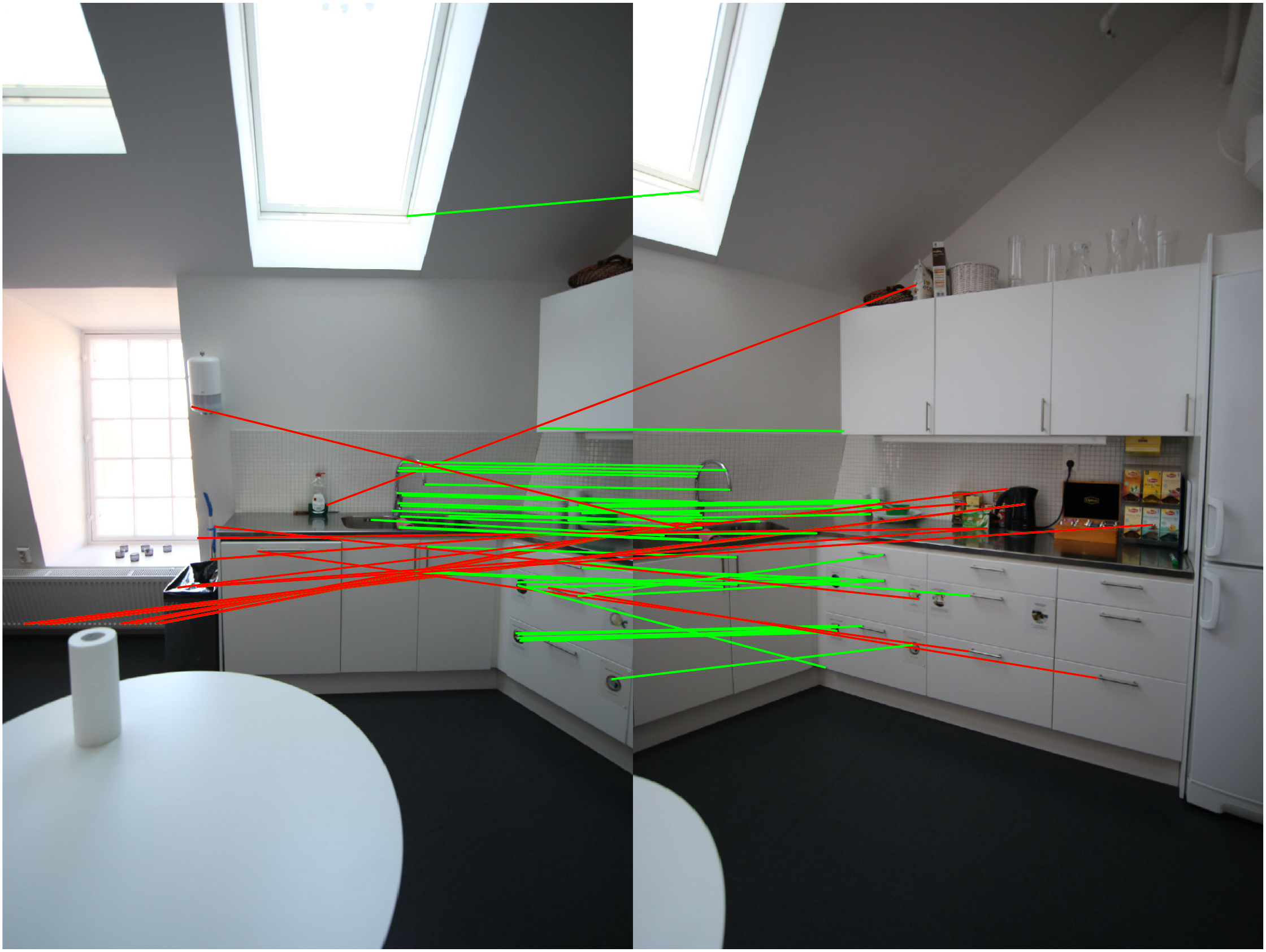}
&\includegraphics[width=0.096\linewidth]{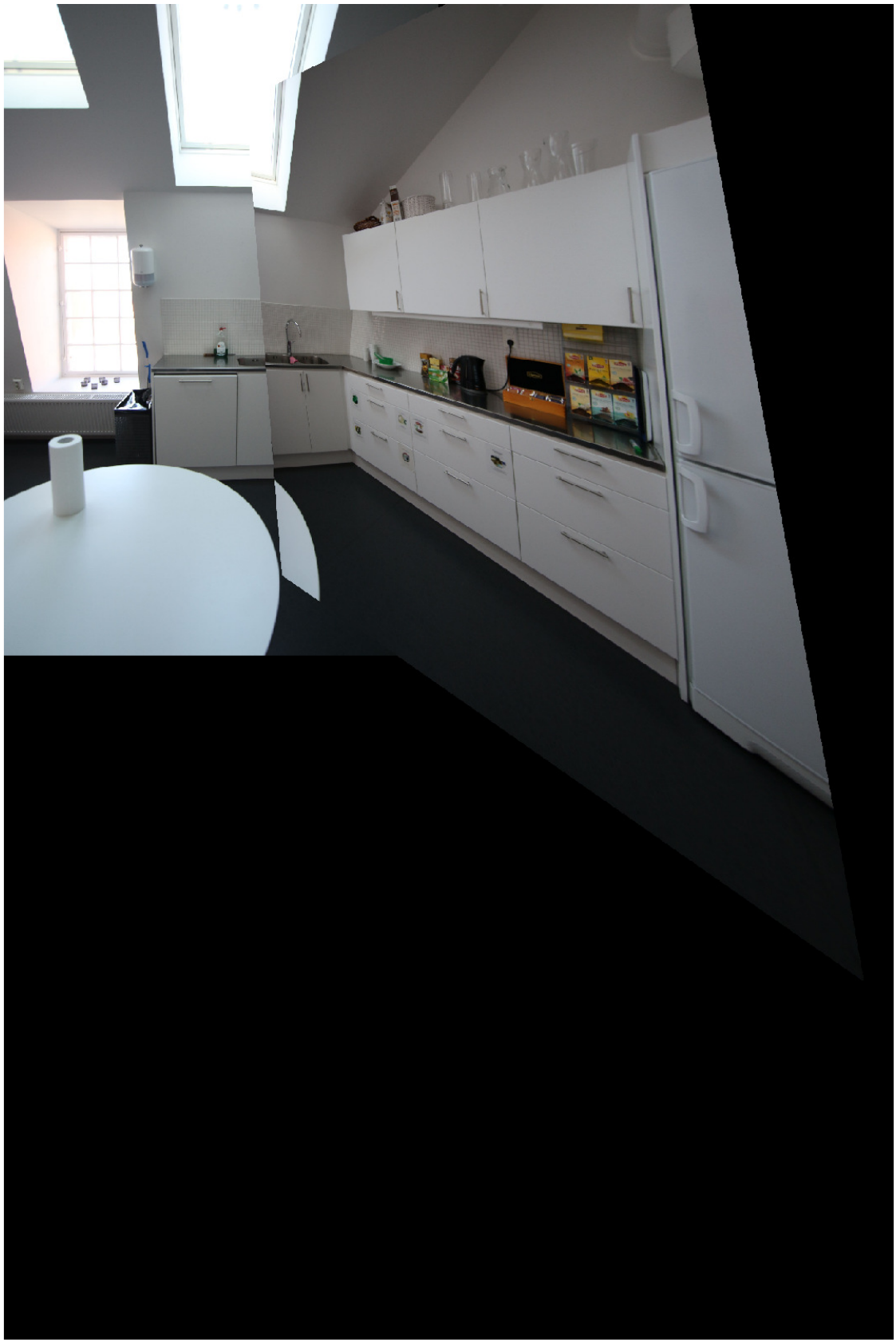}
&\includegraphics[width=0.19\linewidth]{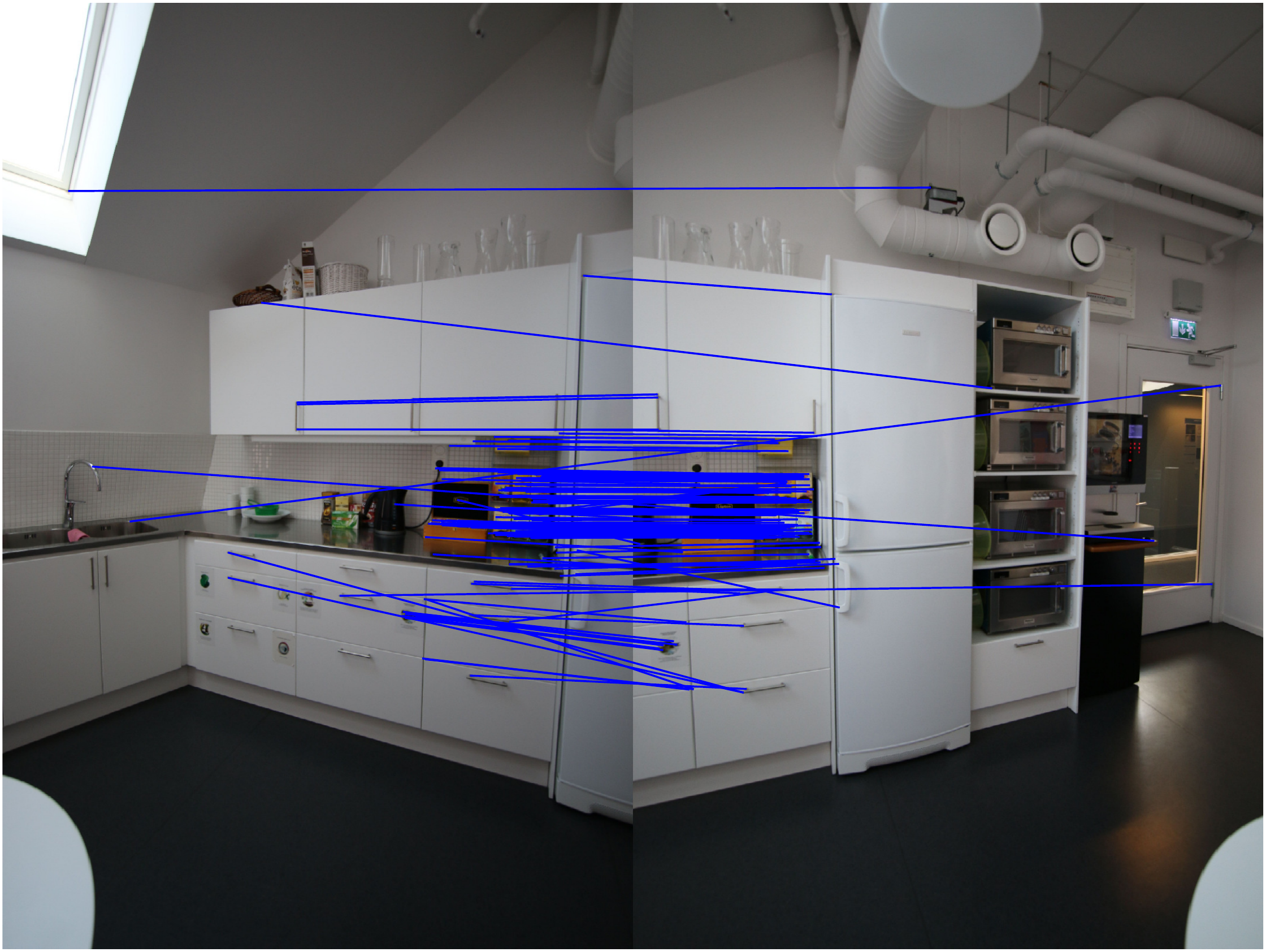}
&\includegraphics[width=0.19\linewidth]{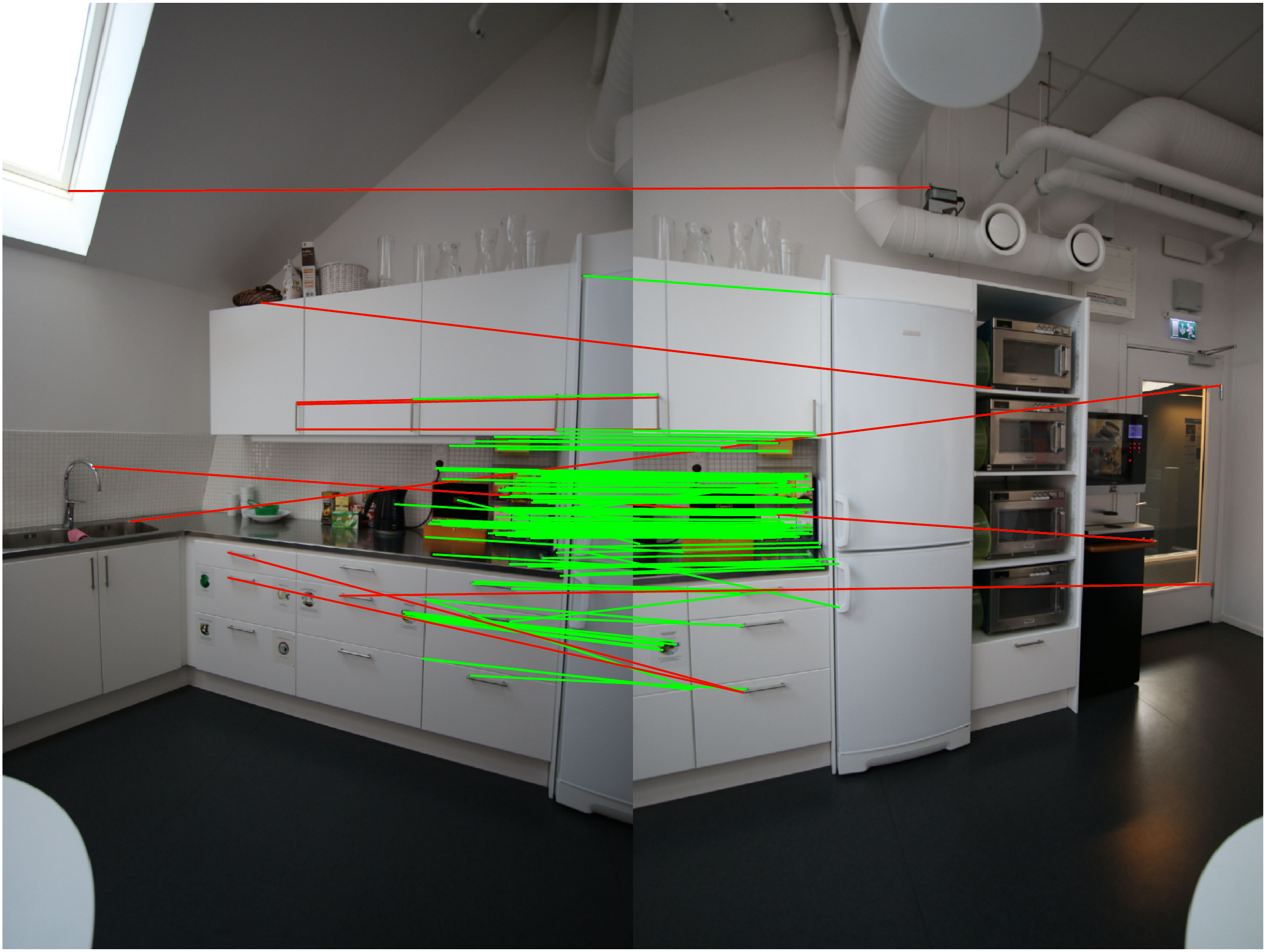}
&\includegraphics[width=0.096\linewidth]{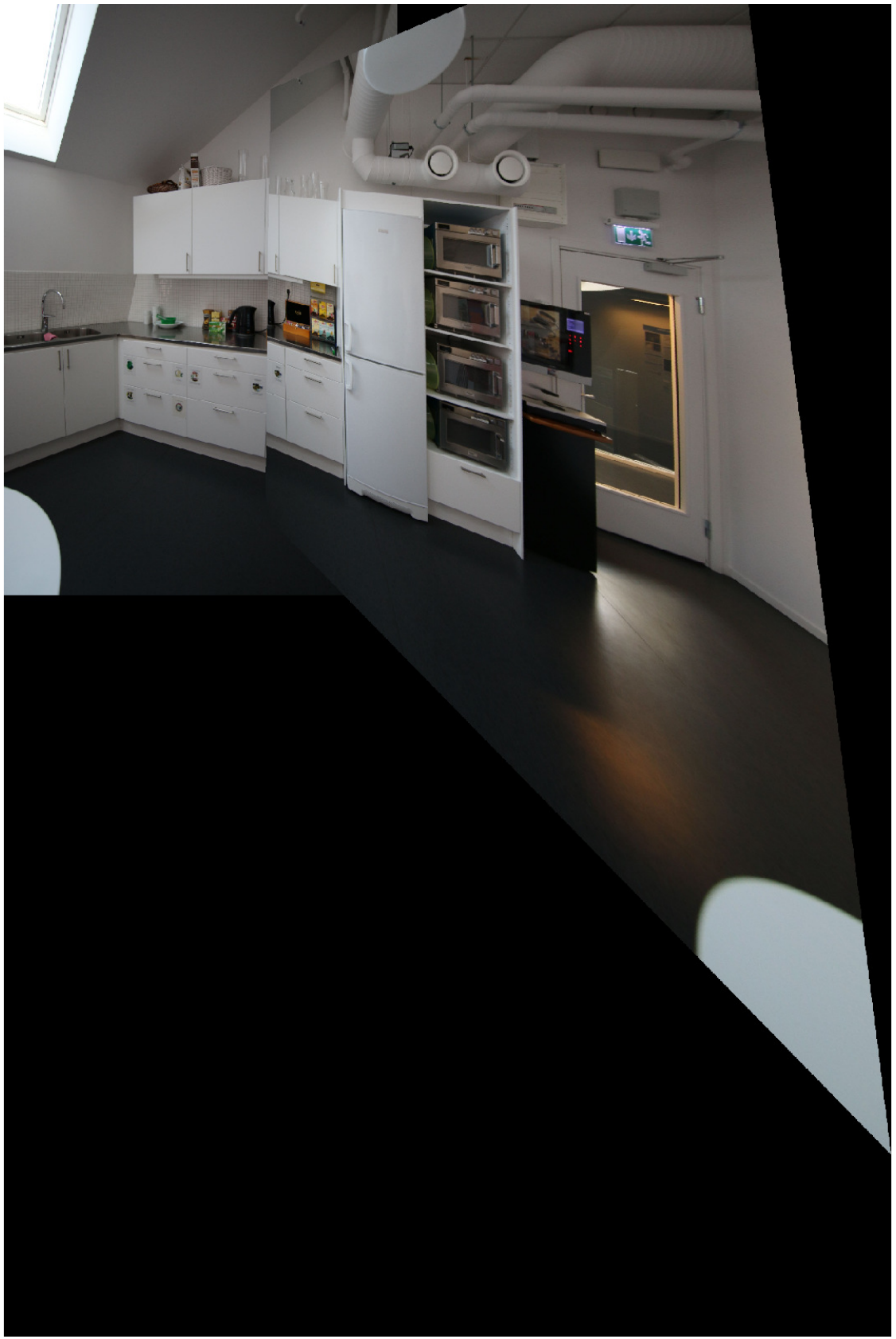} \\

\textit{Image 57\&66}, 13.29\% & Inliers Found by ICOS & Stitching&
\textit{Image 66\&72}, 13.04\% & Inliers Found by ICOS & Stitching \\

\includegraphics[width=0.19\linewidth]{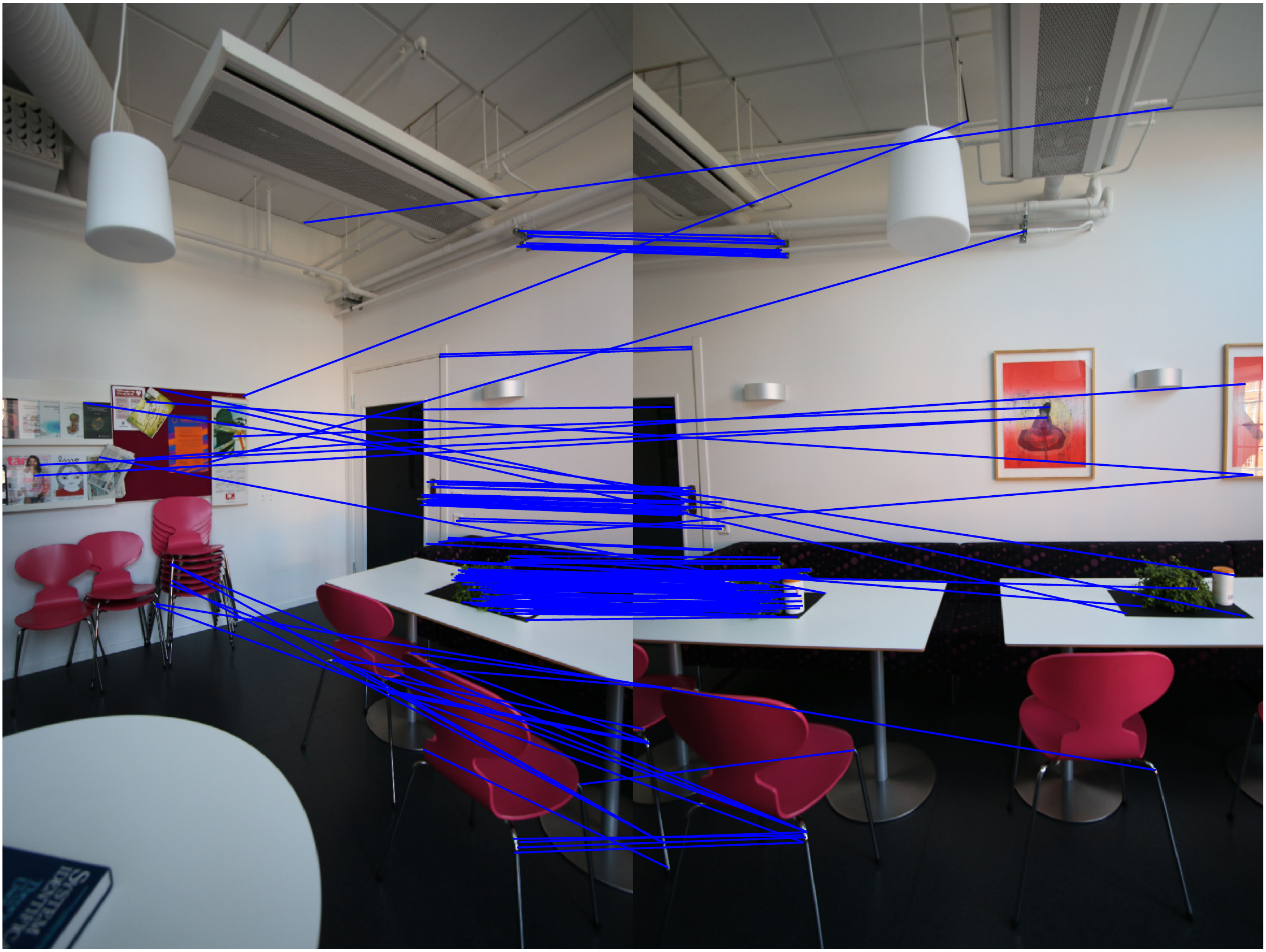}
&\includegraphics[width=0.19\linewidth]{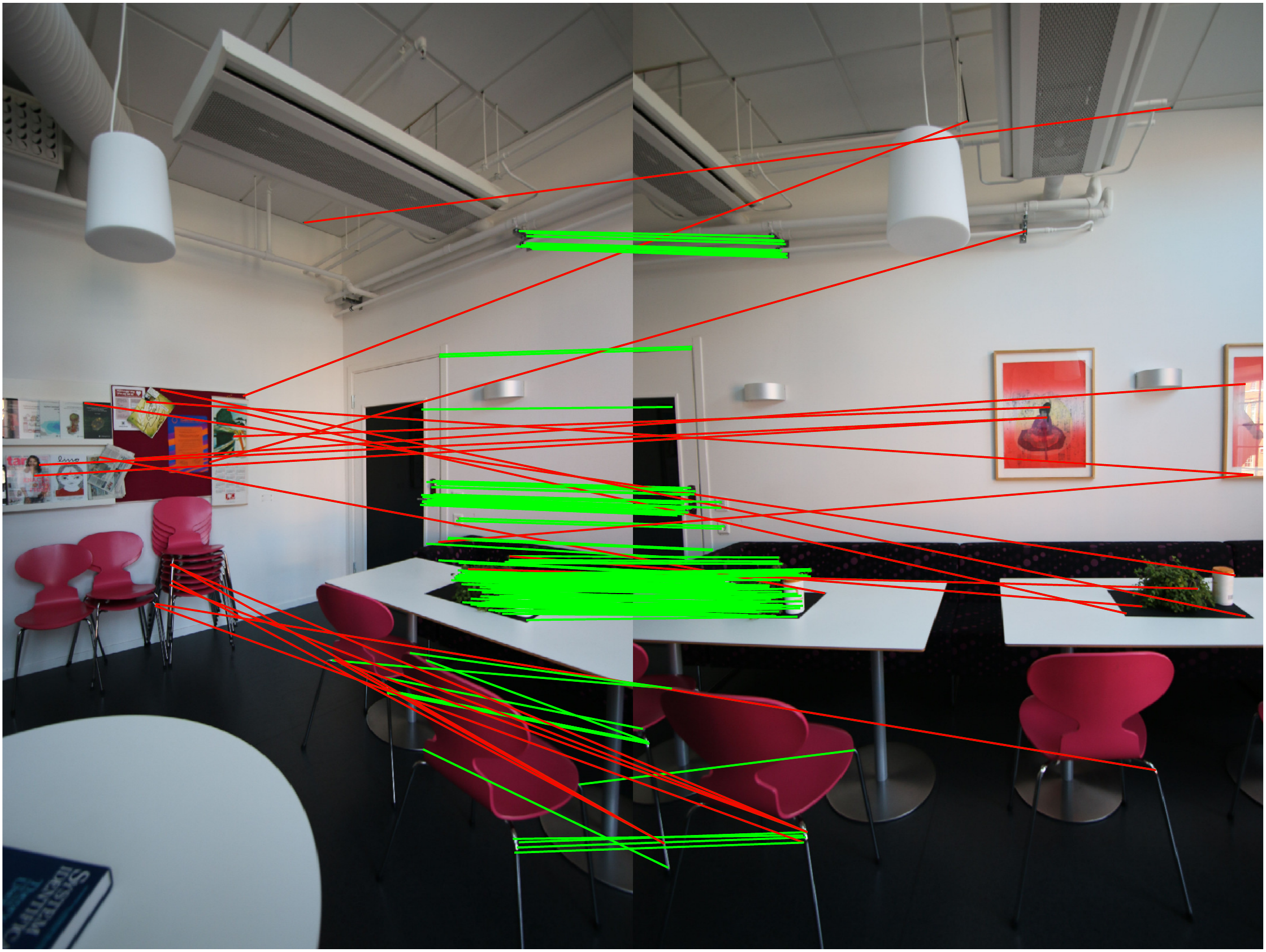}
&\includegraphics[width=0.096\linewidth]{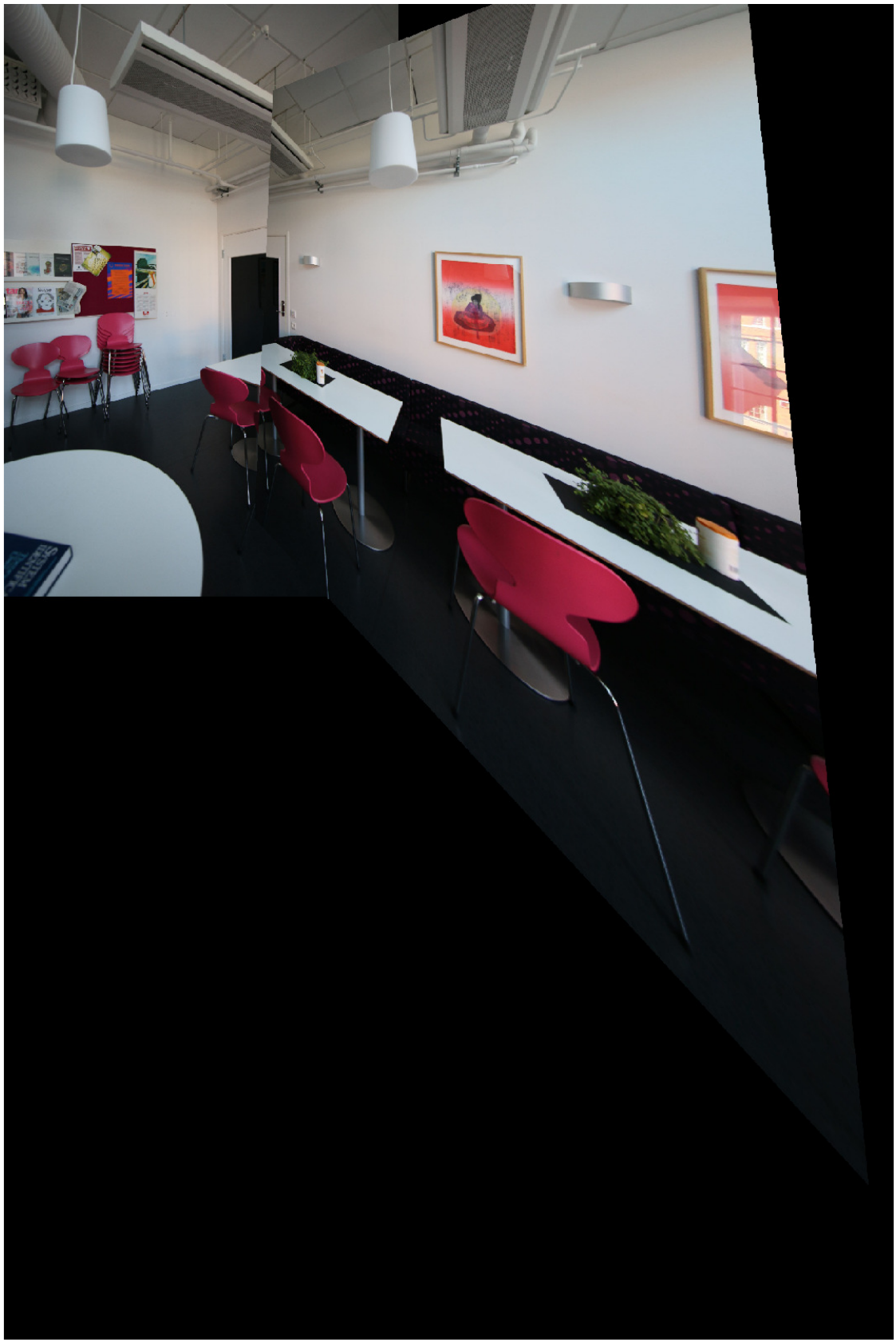}
&\includegraphics[width=0.19\linewidth]{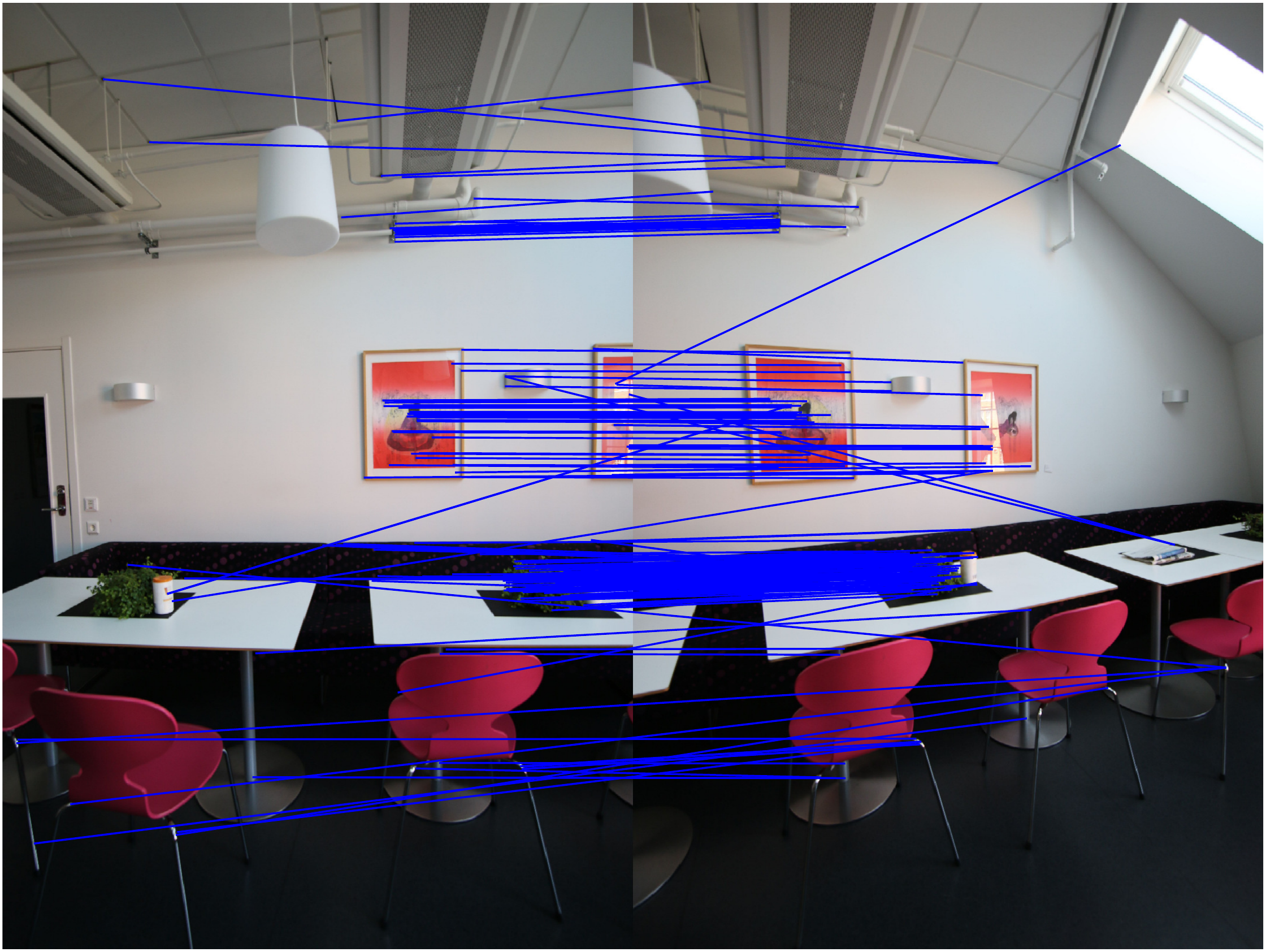}
&\includegraphics[width=0.19\linewidth]{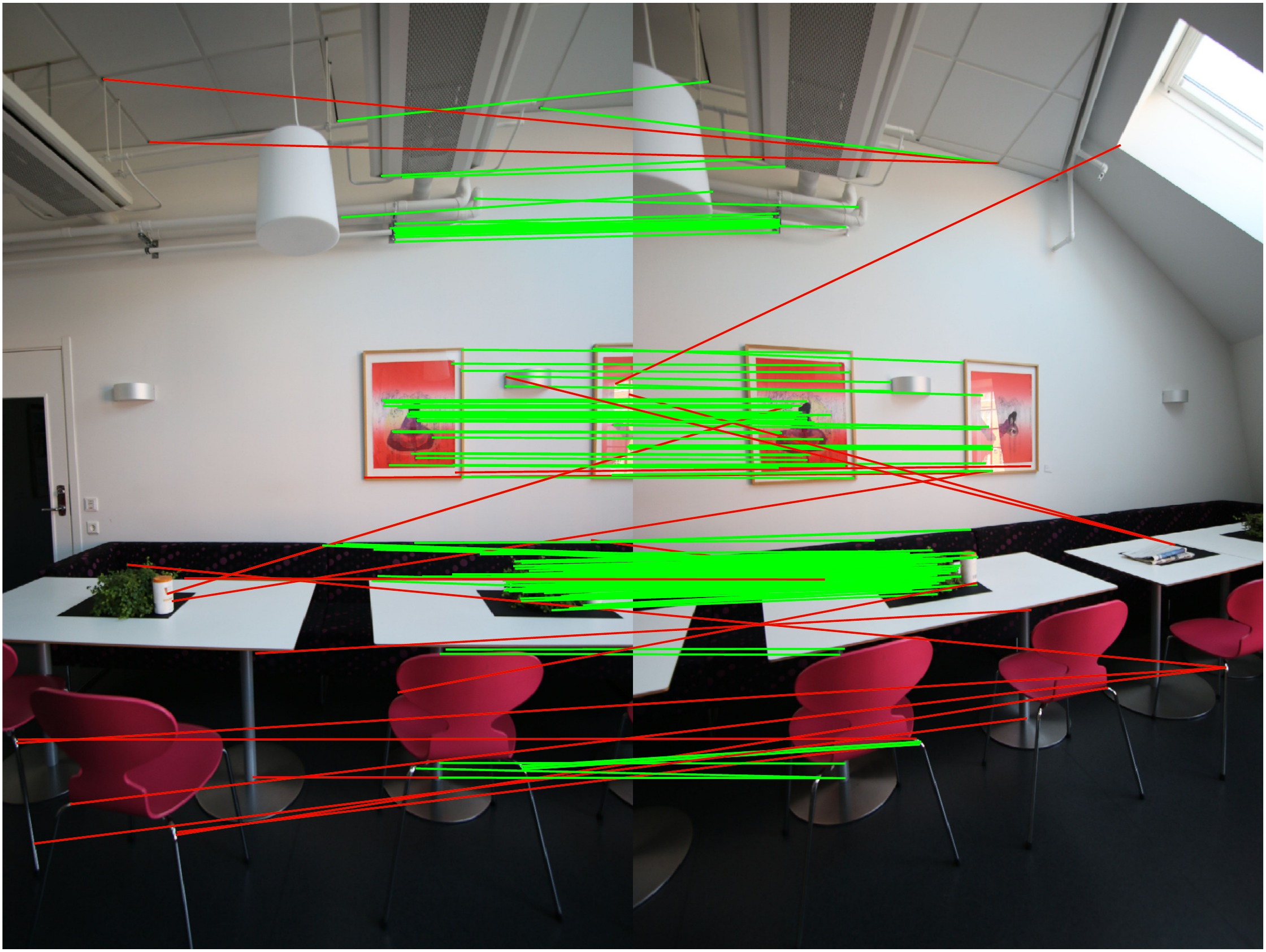}
&\includegraphics[width=0.096\linewidth]{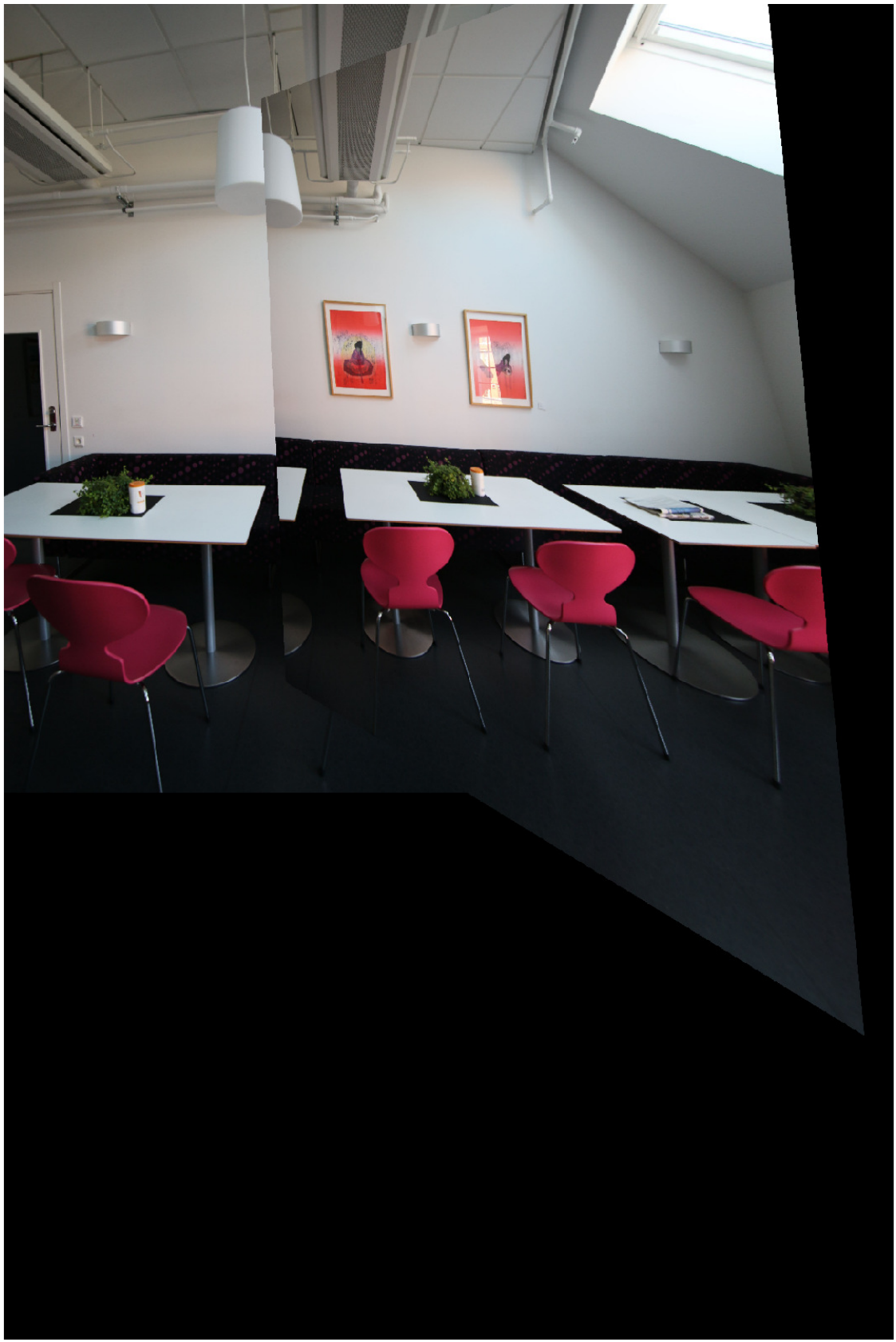}

\end{tabular}

\vspace{-1mm}

\centering
\caption{Image stitching results over the \textit{LunchRoom} dataset~\cite{meneghetti2015image}. Column 1 to column 3 display the raw correspondences matched by SURF~\cite{bay2008speeded} with the image indices and outlier ratios given on top, the inliers found by ICOS (the inliers found are labeled in green, the others are outliers, labeled in red), and the image stitching results using the rotation estimated by ICOS, respectively.}
\label{ImgStit}
\end{figure*}

\begin{figure*}[t]
\centering

\begin{tabular}{ccccc}

\textit{Scene-02}, $\mathit{s}=1$ & ICOS  & RANSAC(1000) & RANSAC(1min) \\

\includegraphics[width=0.20\linewidth]{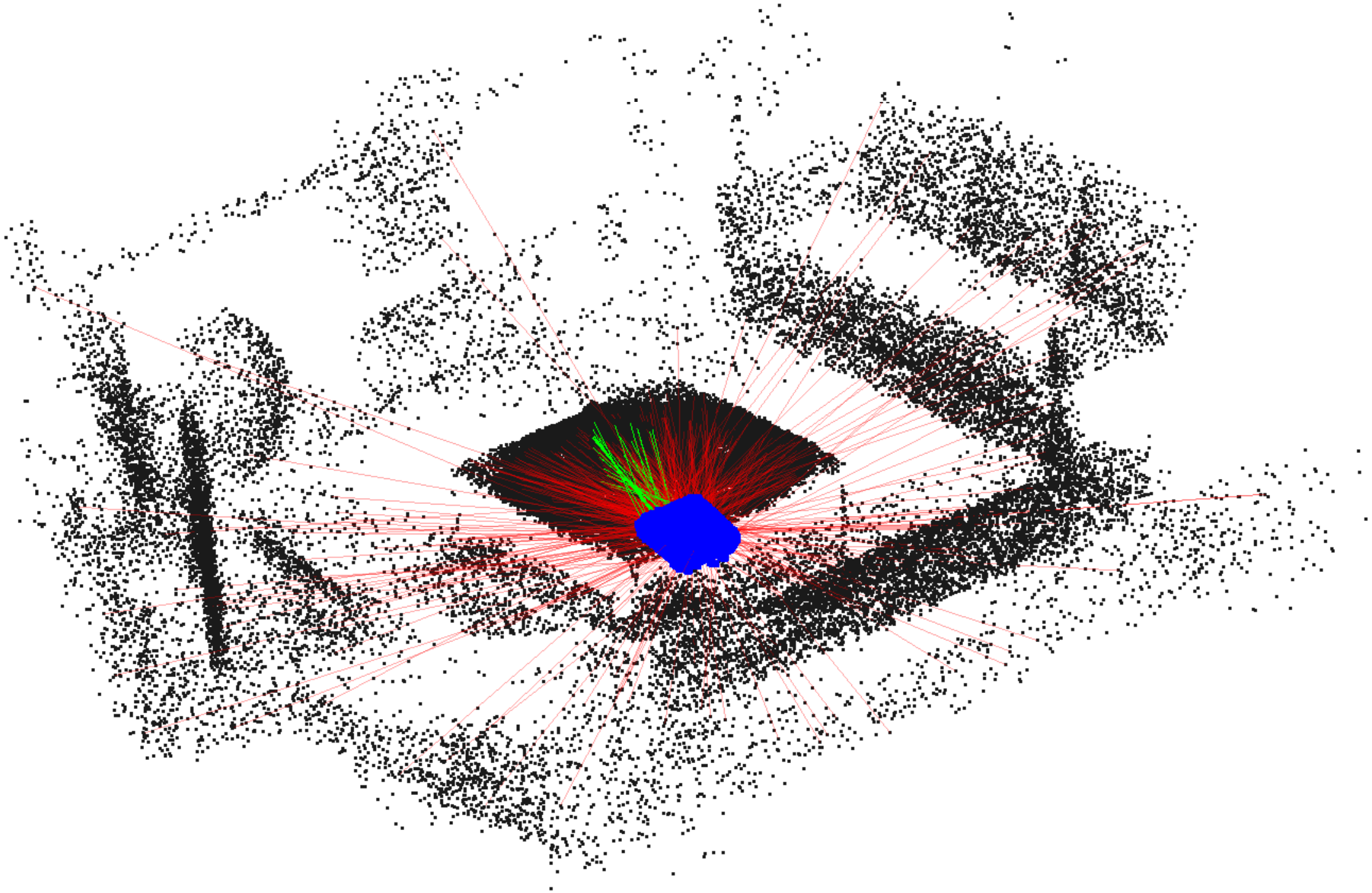}\,
&\includegraphics[width=0.20\linewidth]{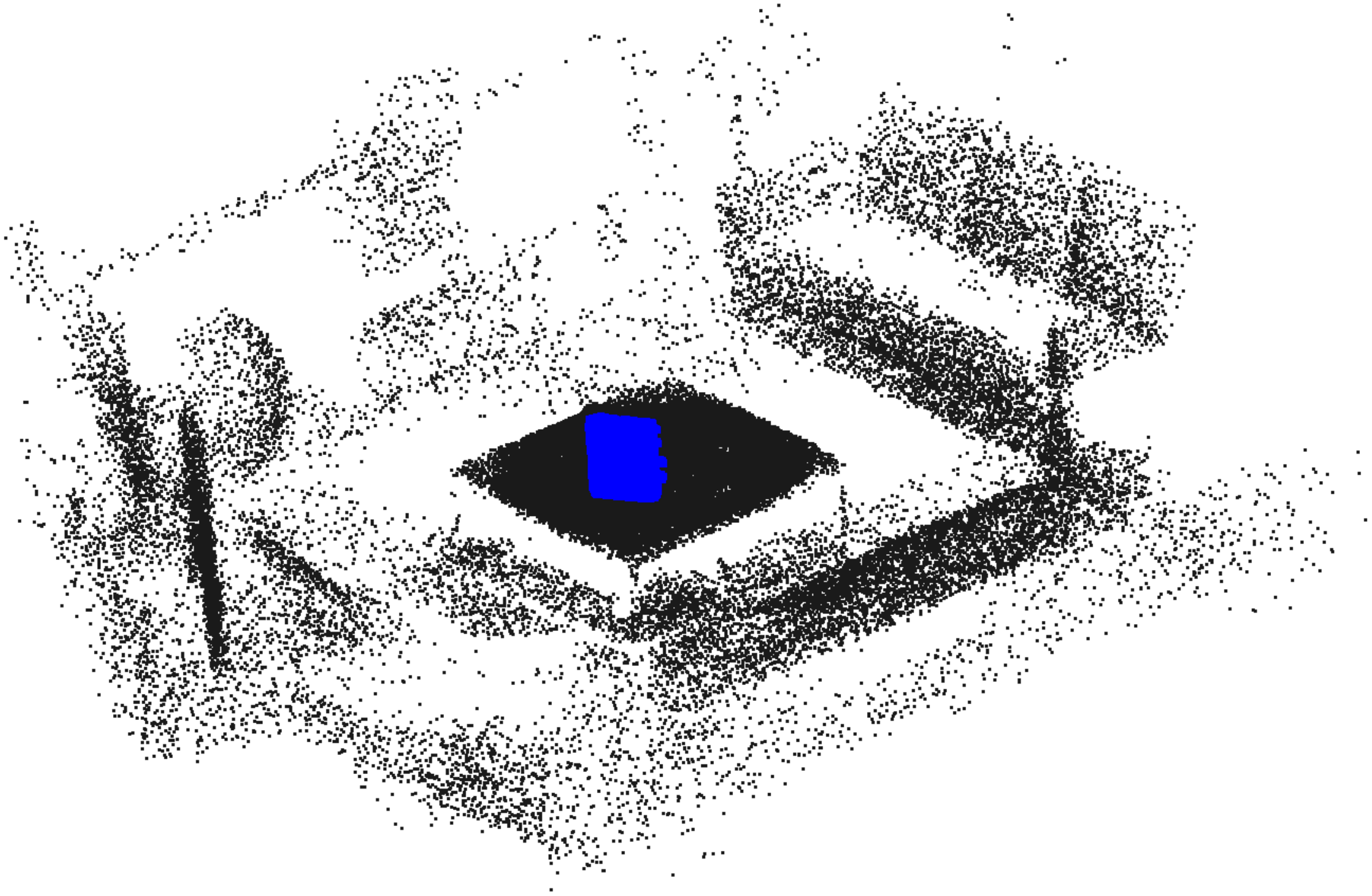}
&\includegraphics[width=0.20\linewidth]{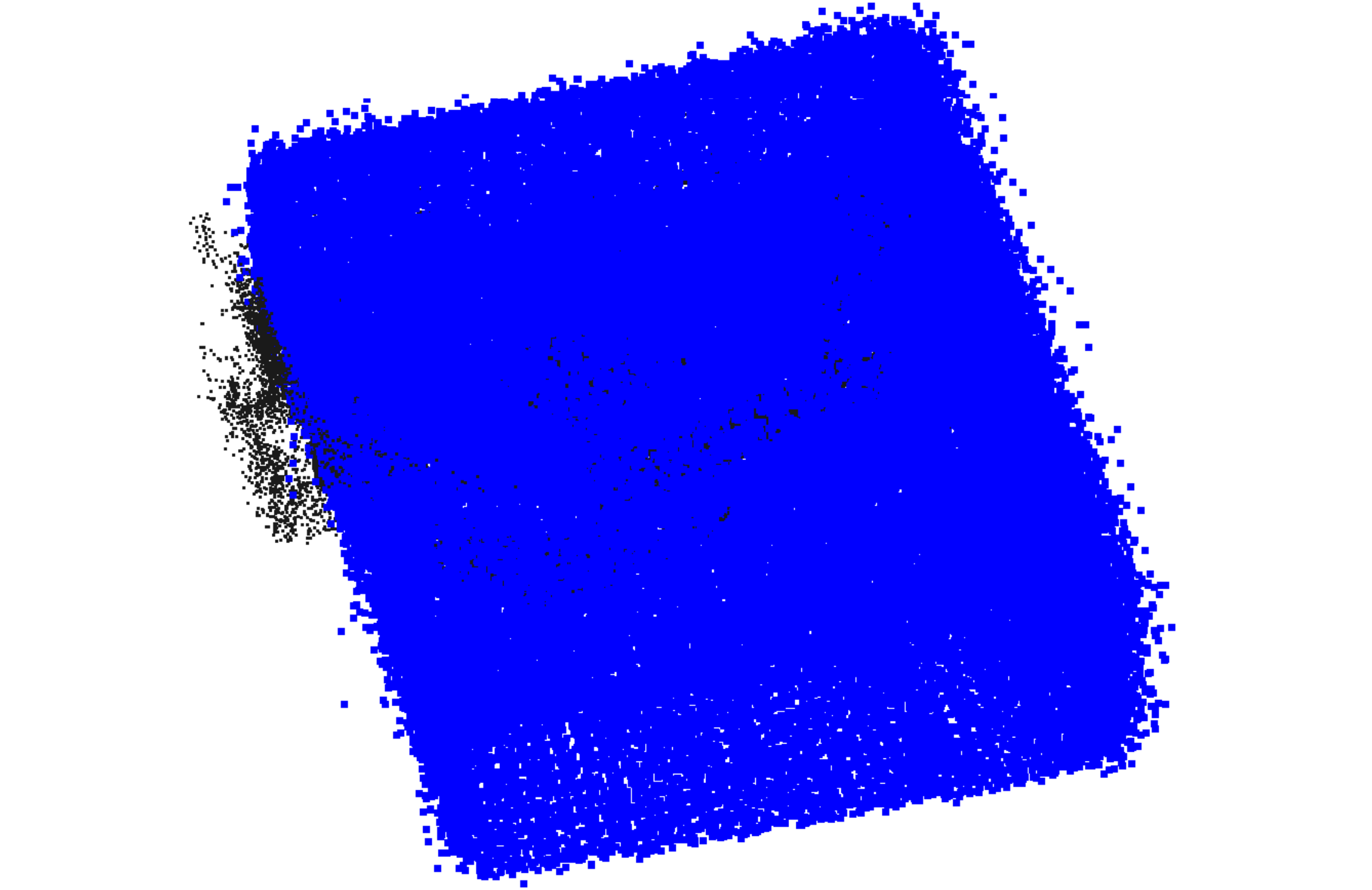}
&\includegraphics[width=0.20\linewidth]{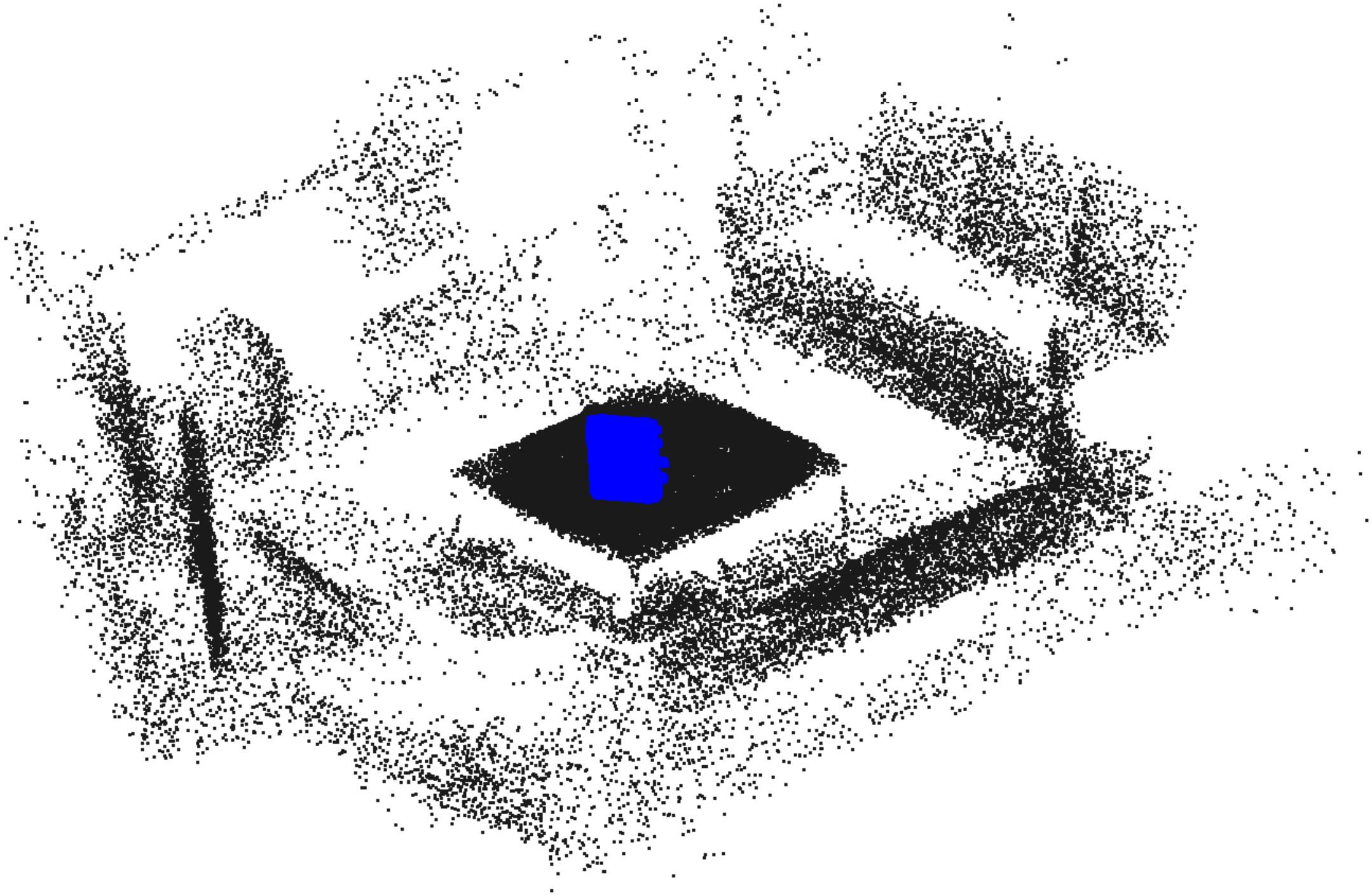} \\

\textit{Scene-05}, $\mathit{s}=1$ & ICOS  & RANSAC(1000) & RANSAC(1min) \\

\includegraphics[width=0.20\linewidth]{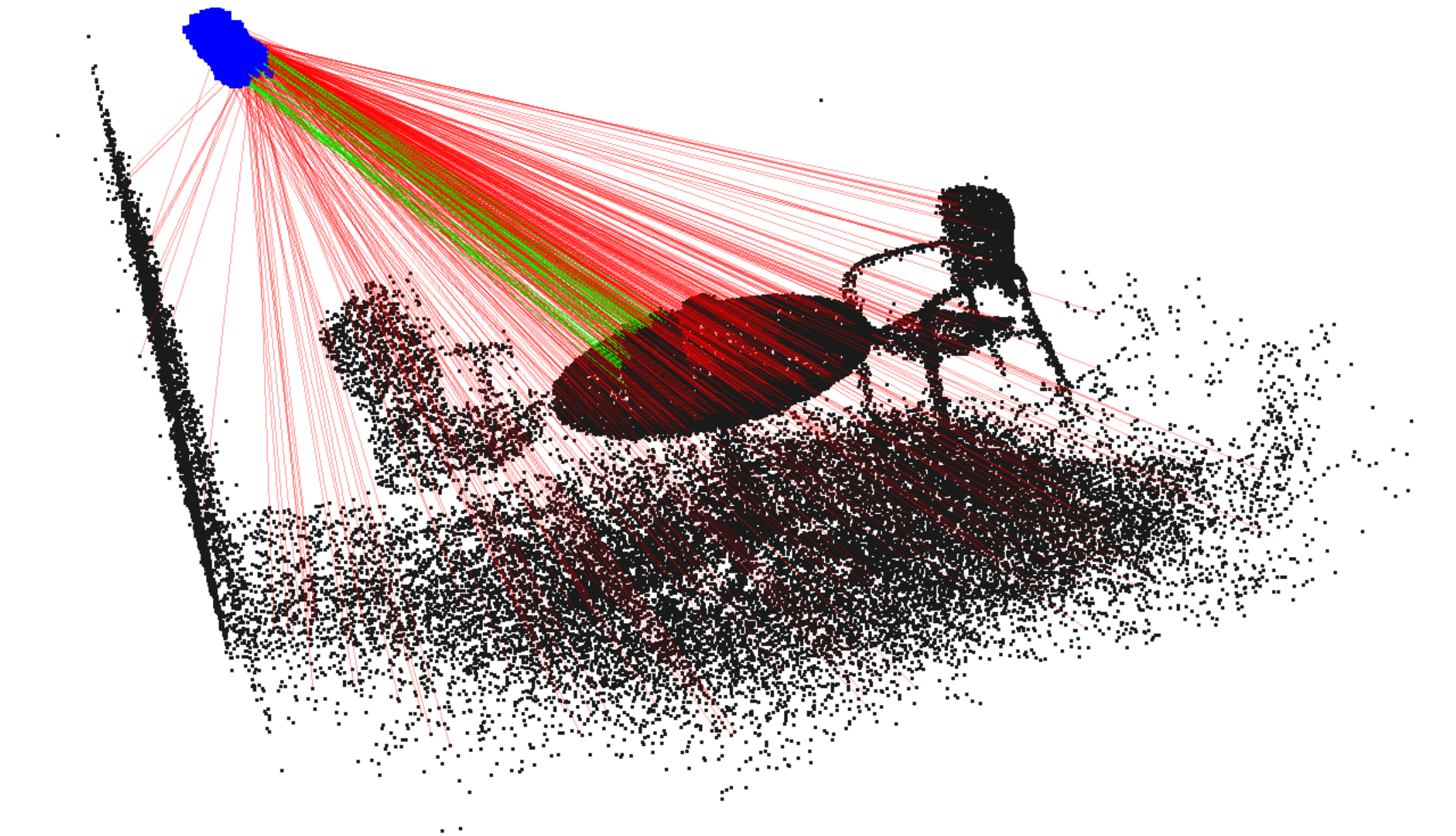}\,
&\includegraphics[width=0.20\linewidth]{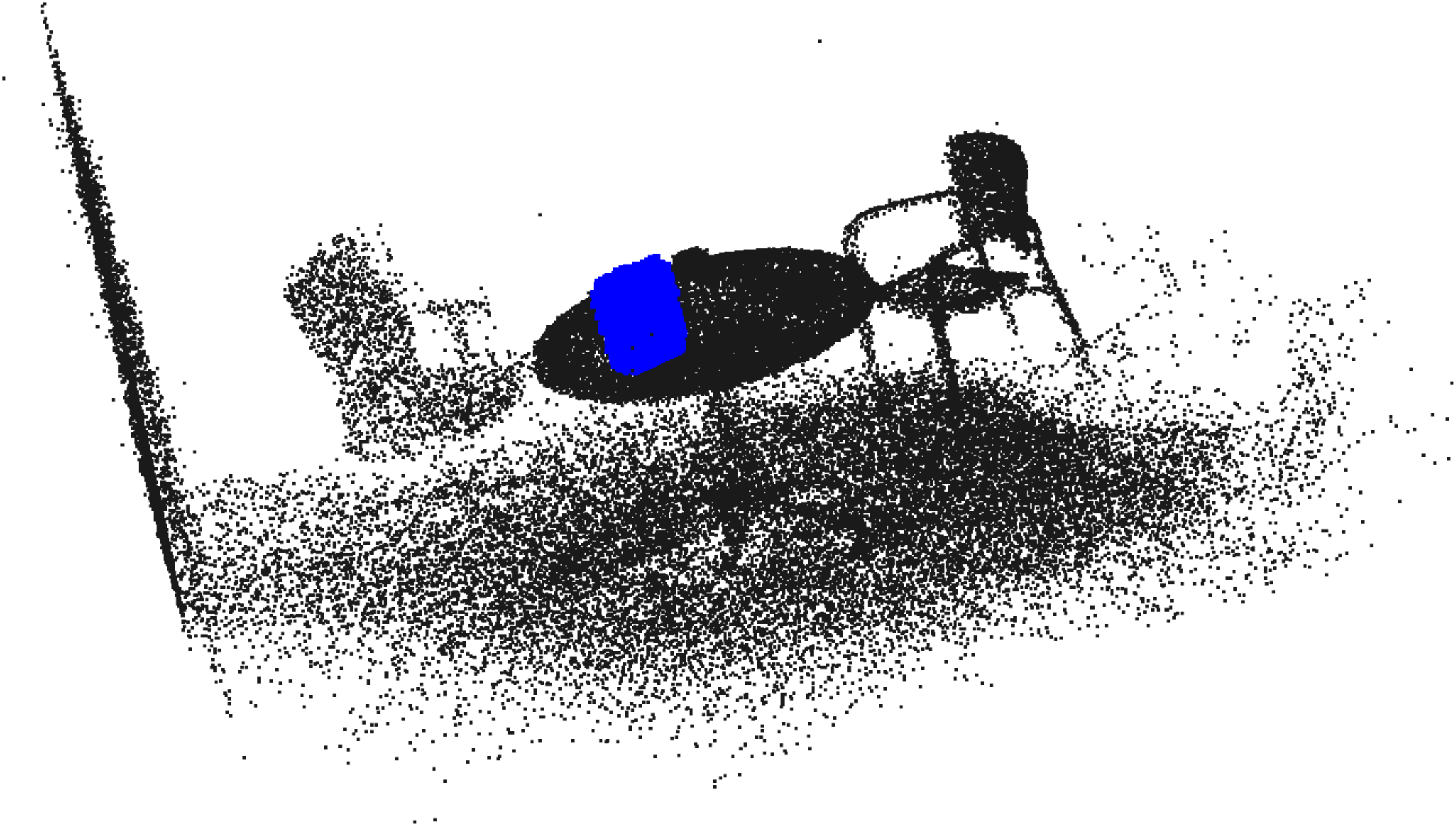}
&\includegraphics[width=0.20\linewidth]{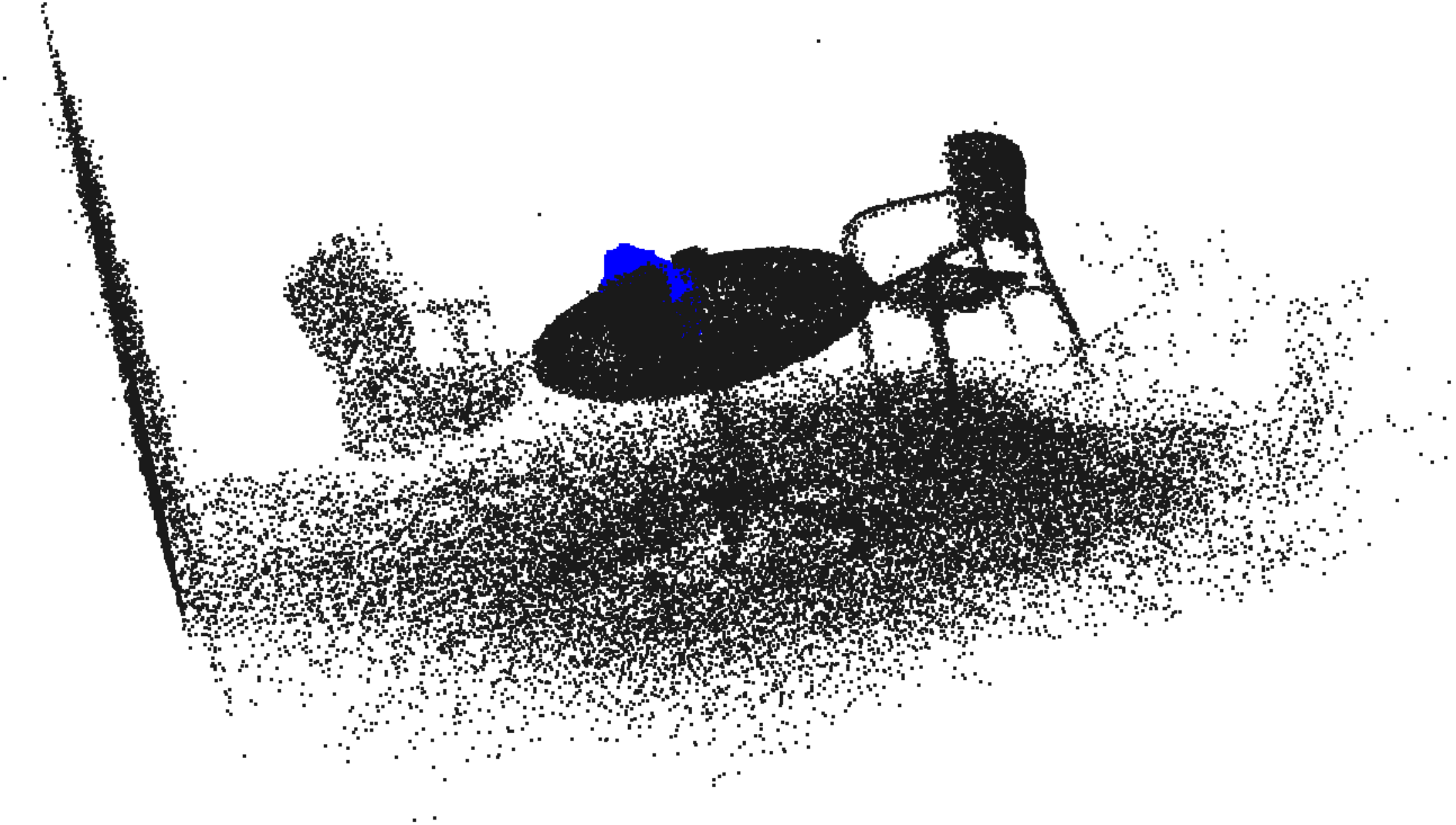}
&\includegraphics[width=0.20\linewidth]{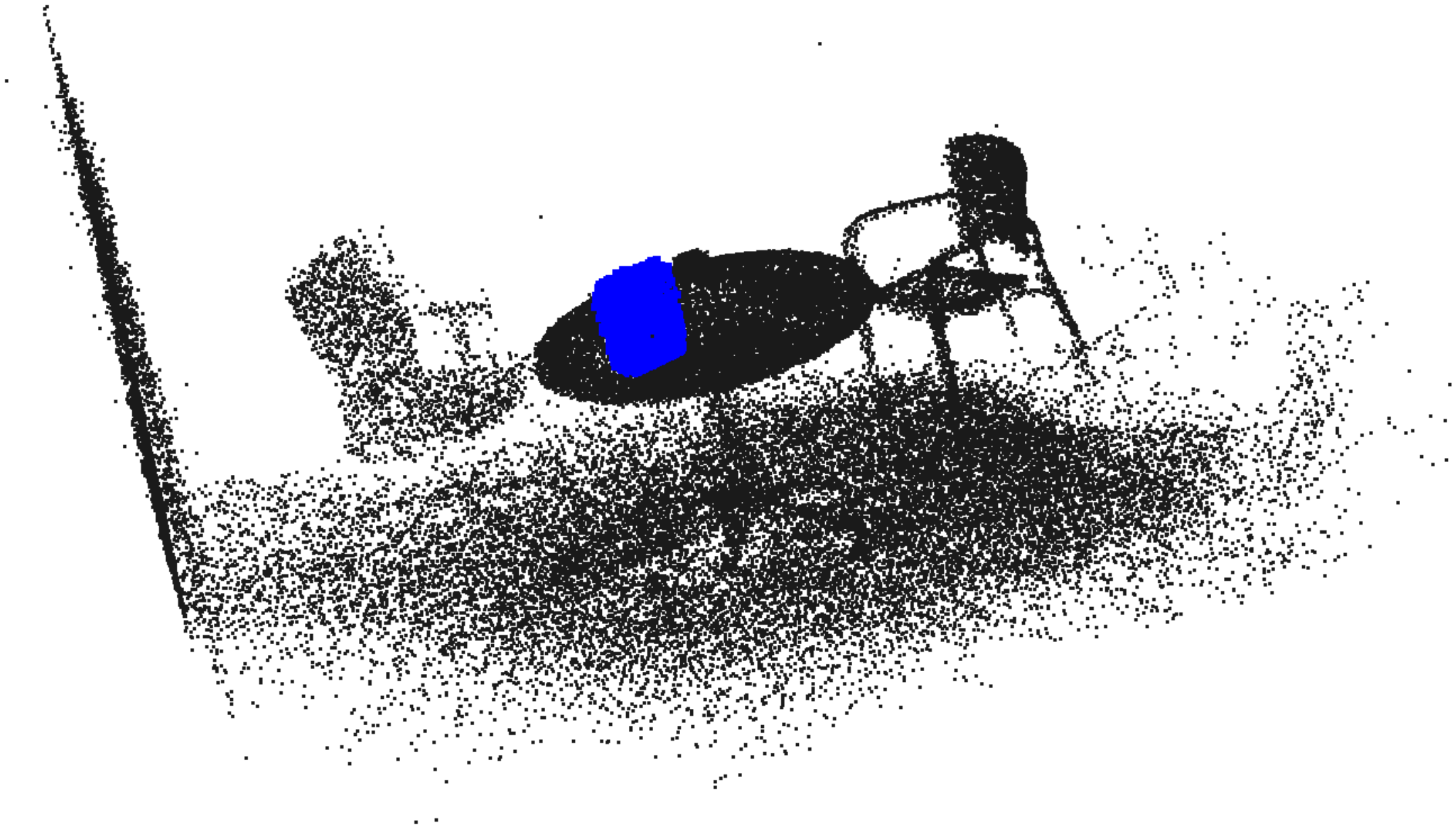} \\

\textit{Scene-07}, $\mathit{s}=1$ & ICOS  & RANSAC(1000) & RANSAC(1min) \\

\includegraphics[width=0.20\linewidth]{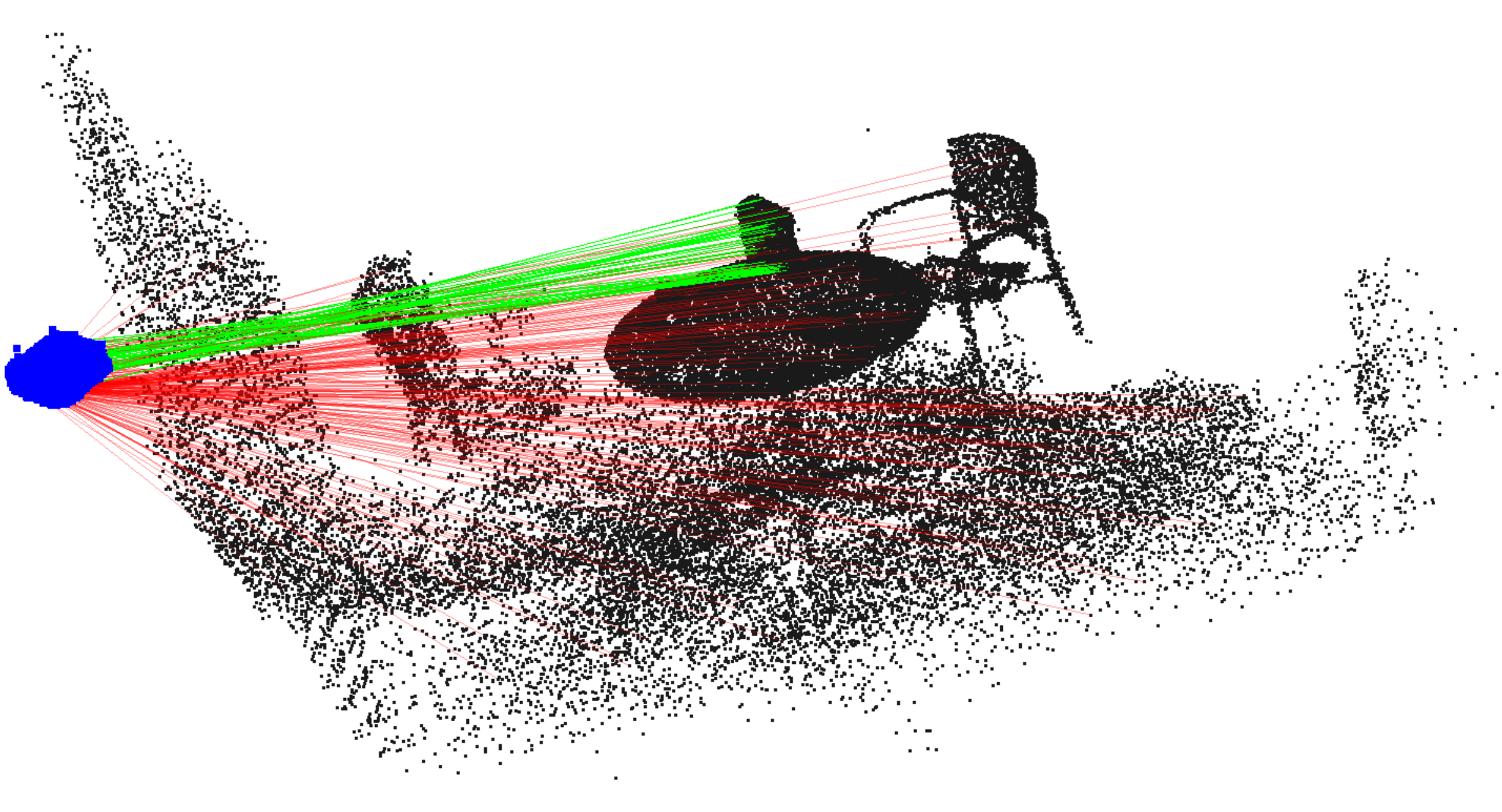}\,
&\includegraphics[width=0.20\linewidth]{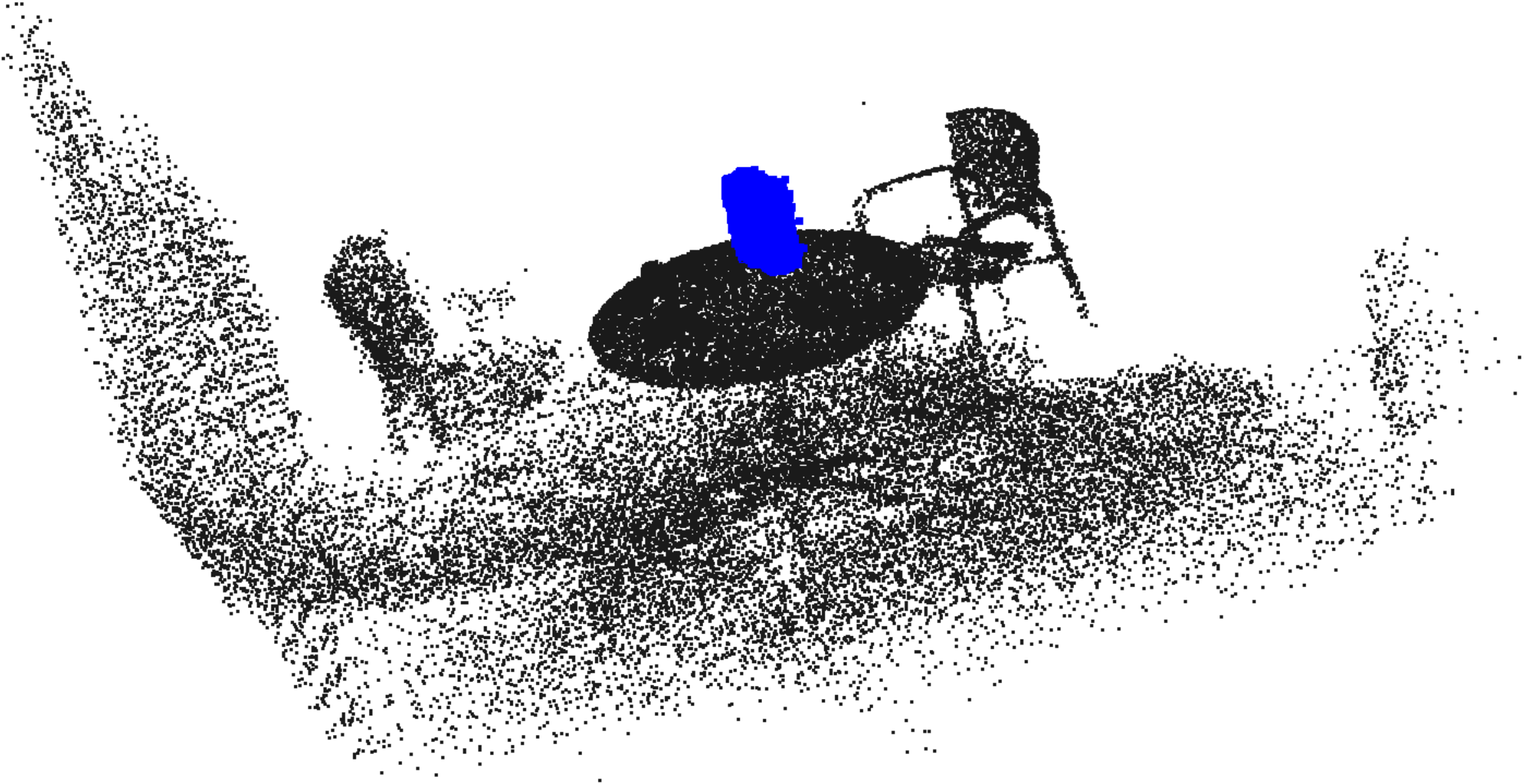}
&\includegraphics[width=0.20\linewidth]{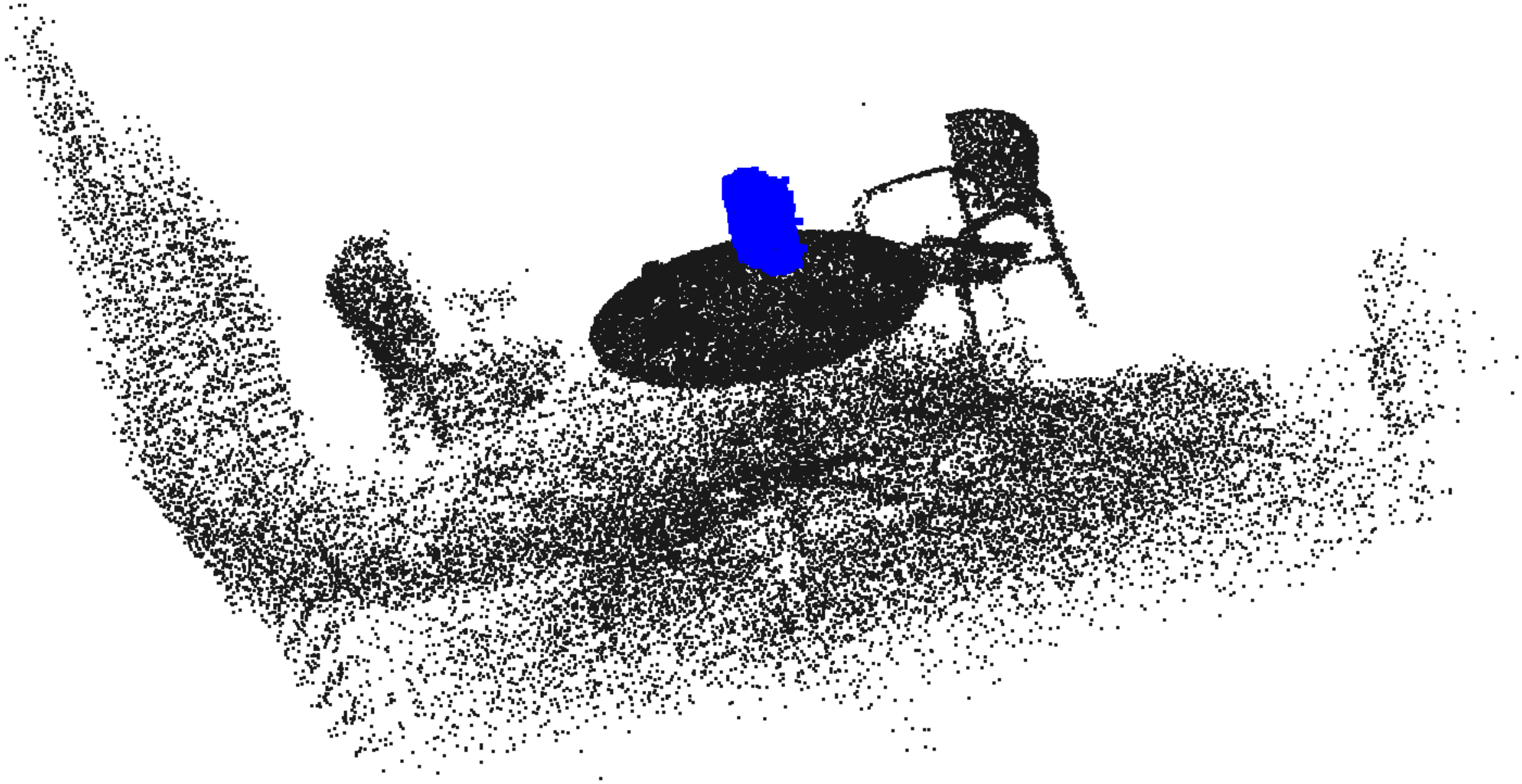}
&\includegraphics[width=0.20\linewidth]{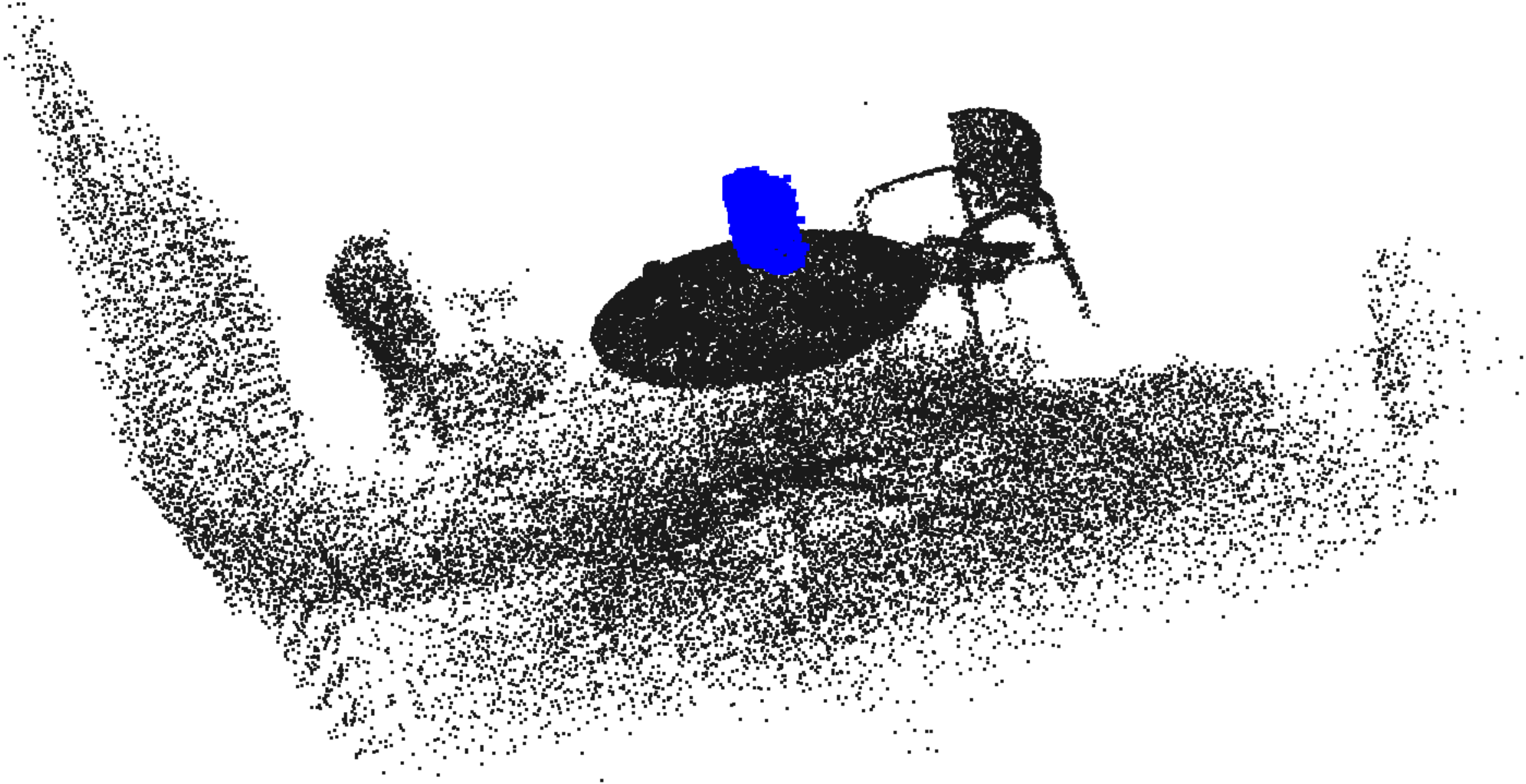} \\

\textit{Scene-09}, $\mathit{s}=1$ & ICOS  & RANSAC(1000) & RANSAC(1min) \\

\includegraphics[width=0.20\linewidth]{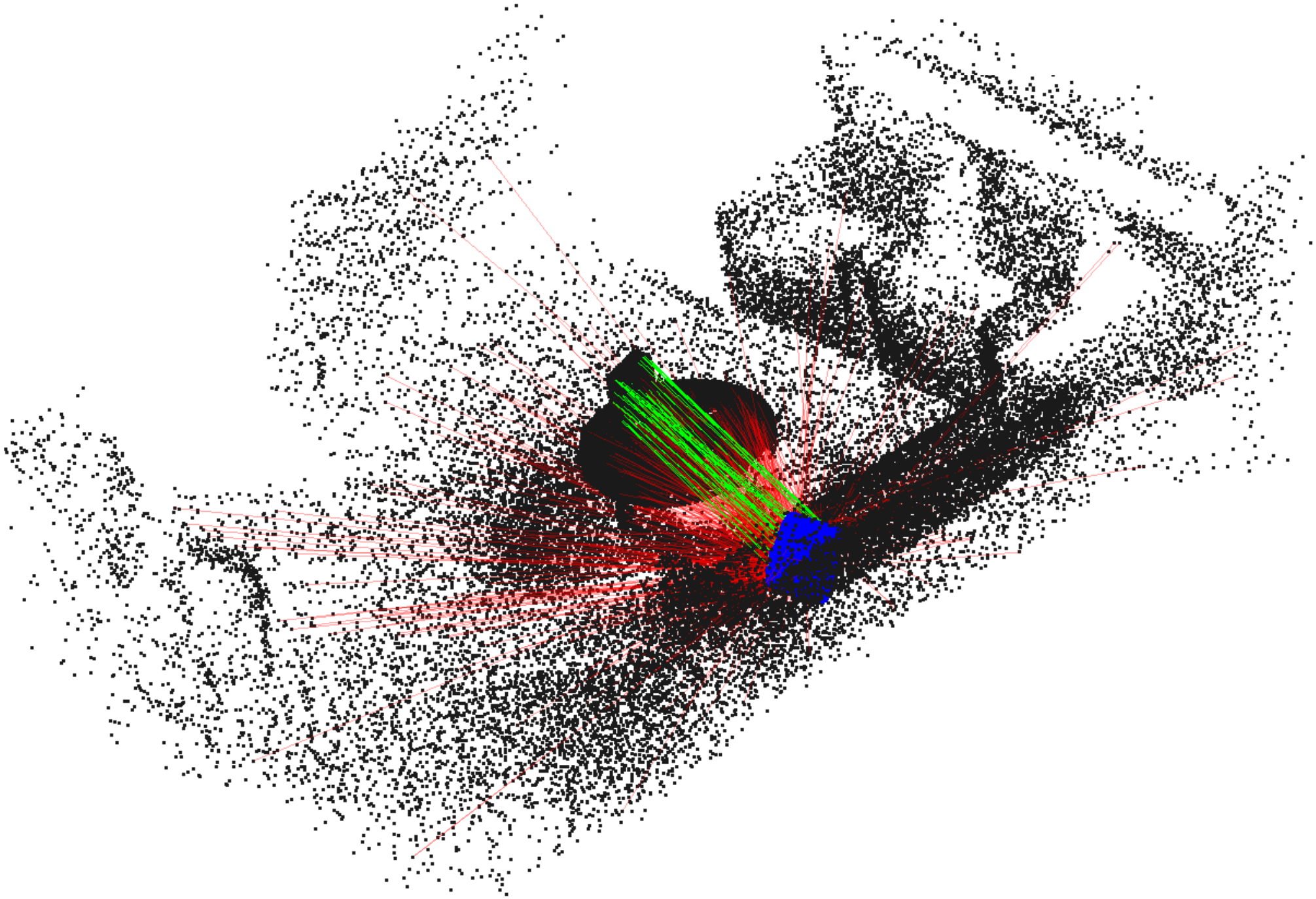}\,
&\includegraphics[width=0.20\linewidth]{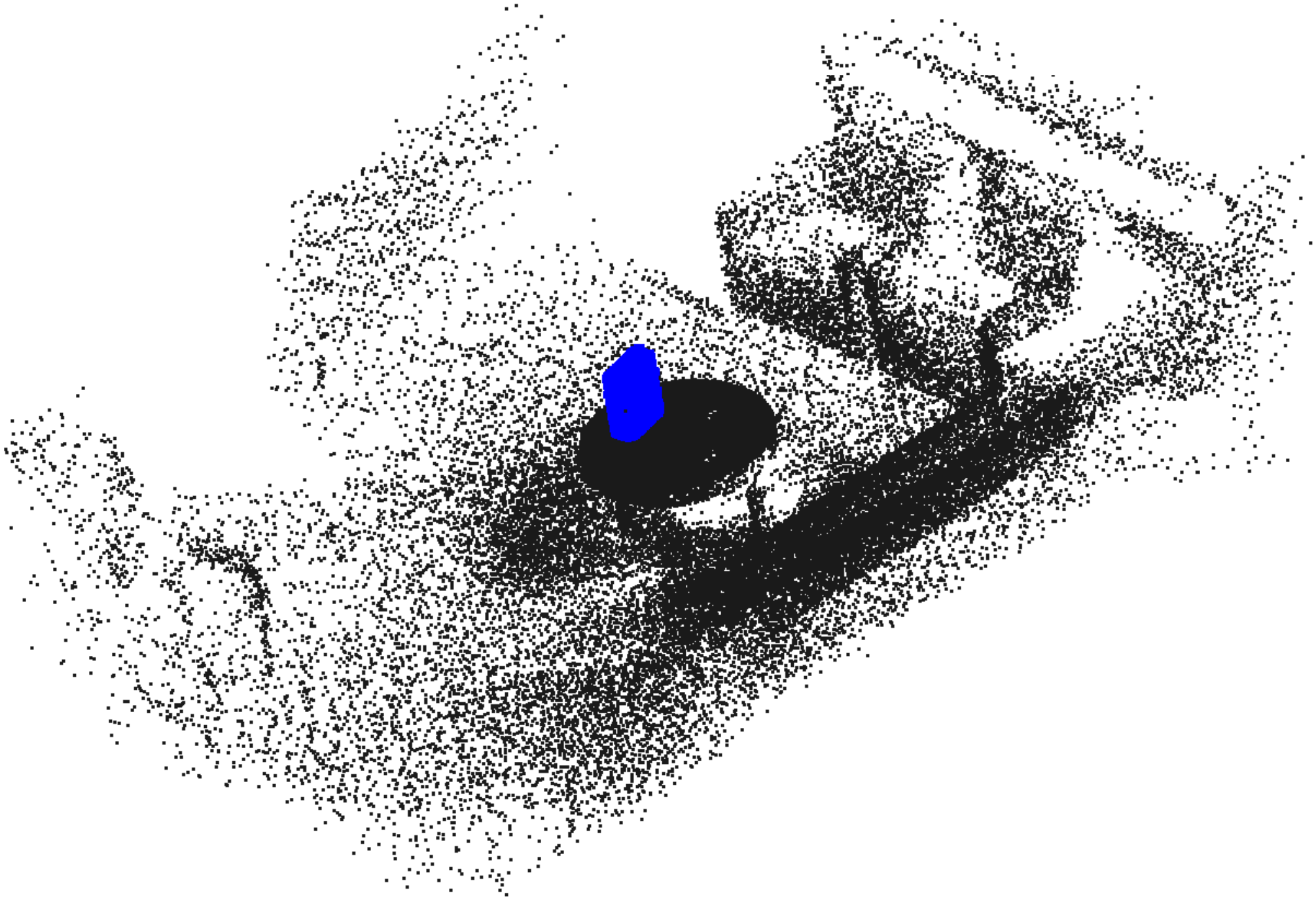}
&\includegraphics[width=0.20\linewidth]{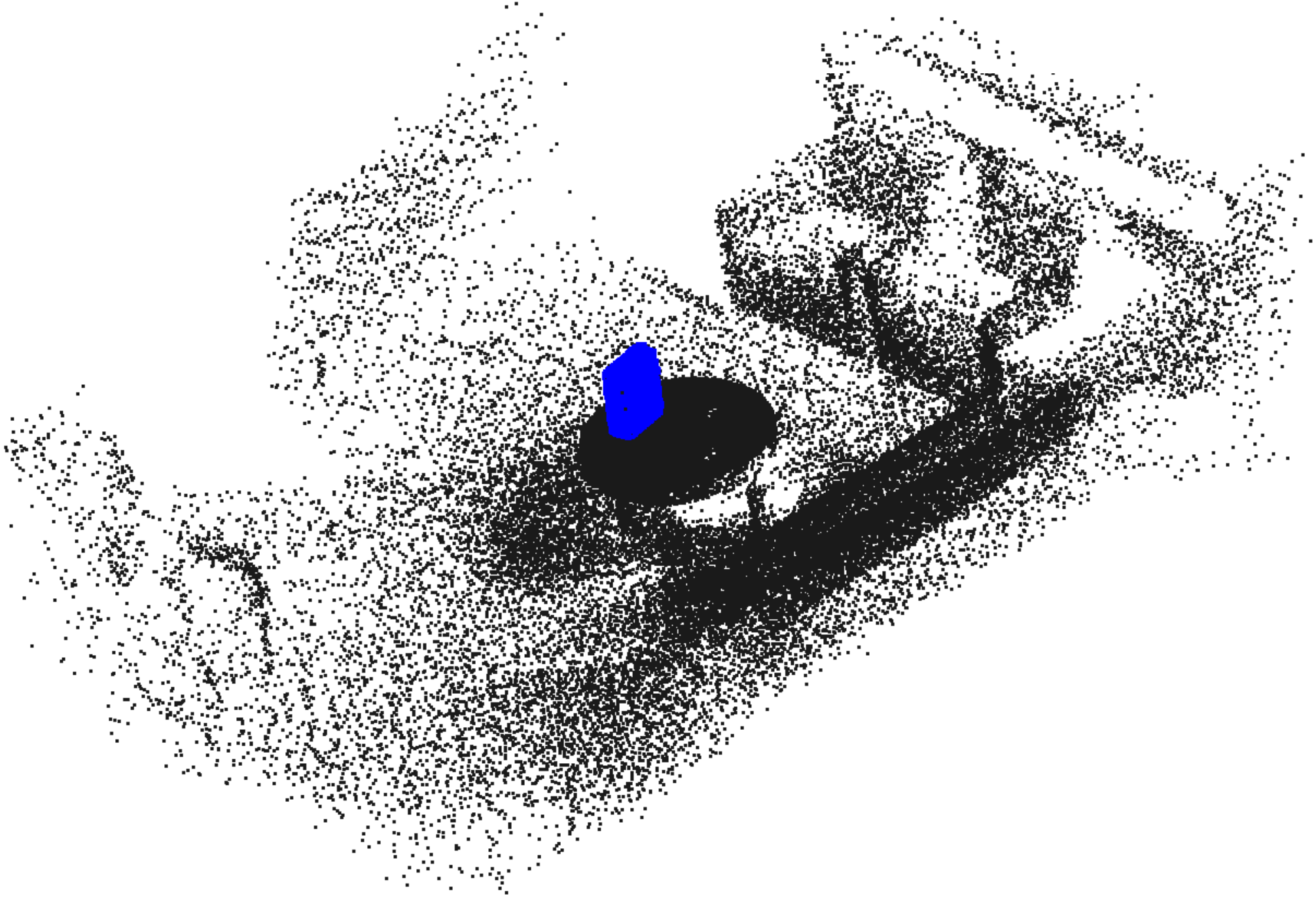}
&\includegraphics[width=0.20\linewidth]{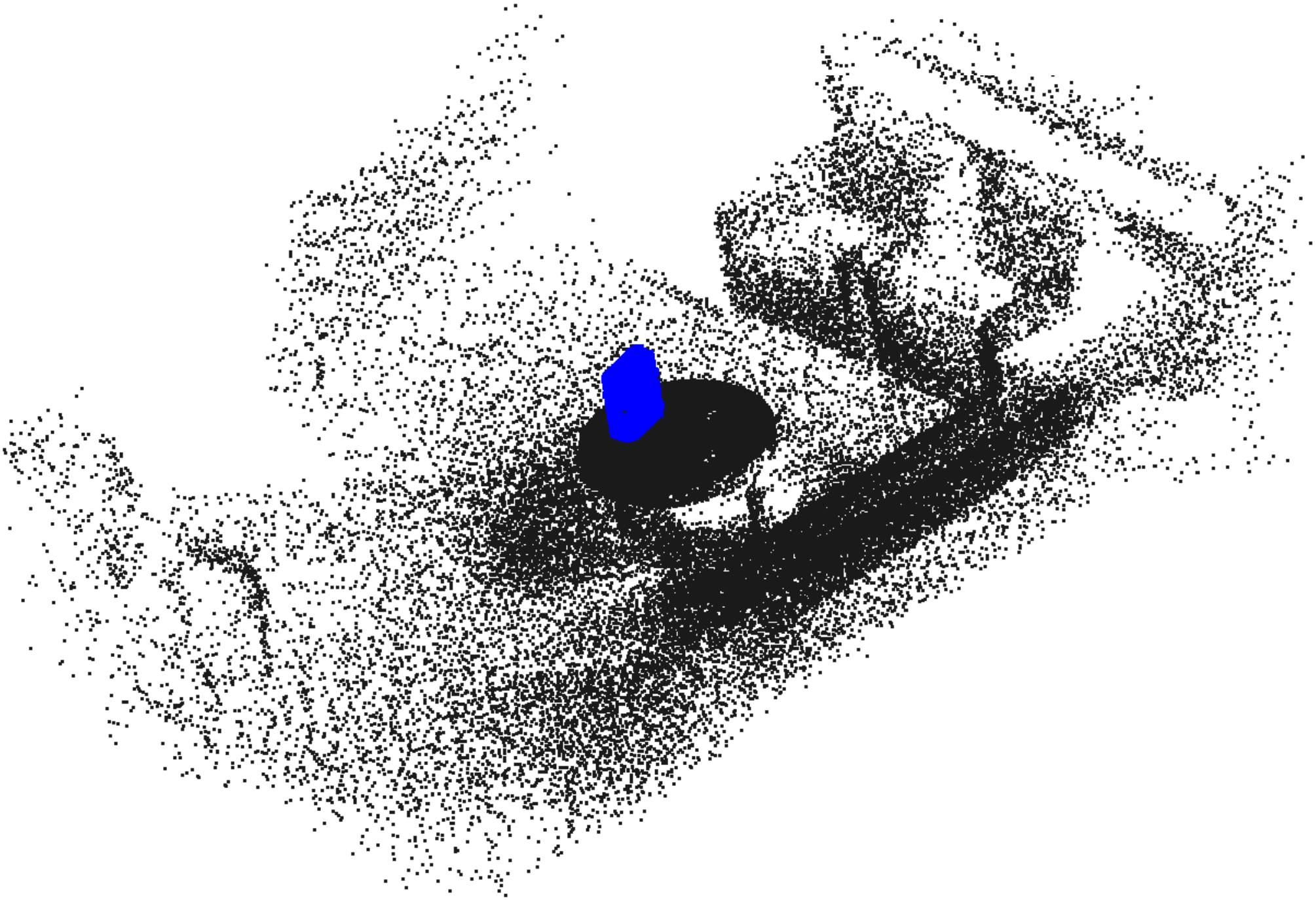} \\

\textit{Scene-02}, $\mathit{s}\in(1,5)$ & ICOS  & RANSAC(1000) & RANSAC(1min) \\

\includegraphics[width=0.20\linewidth]{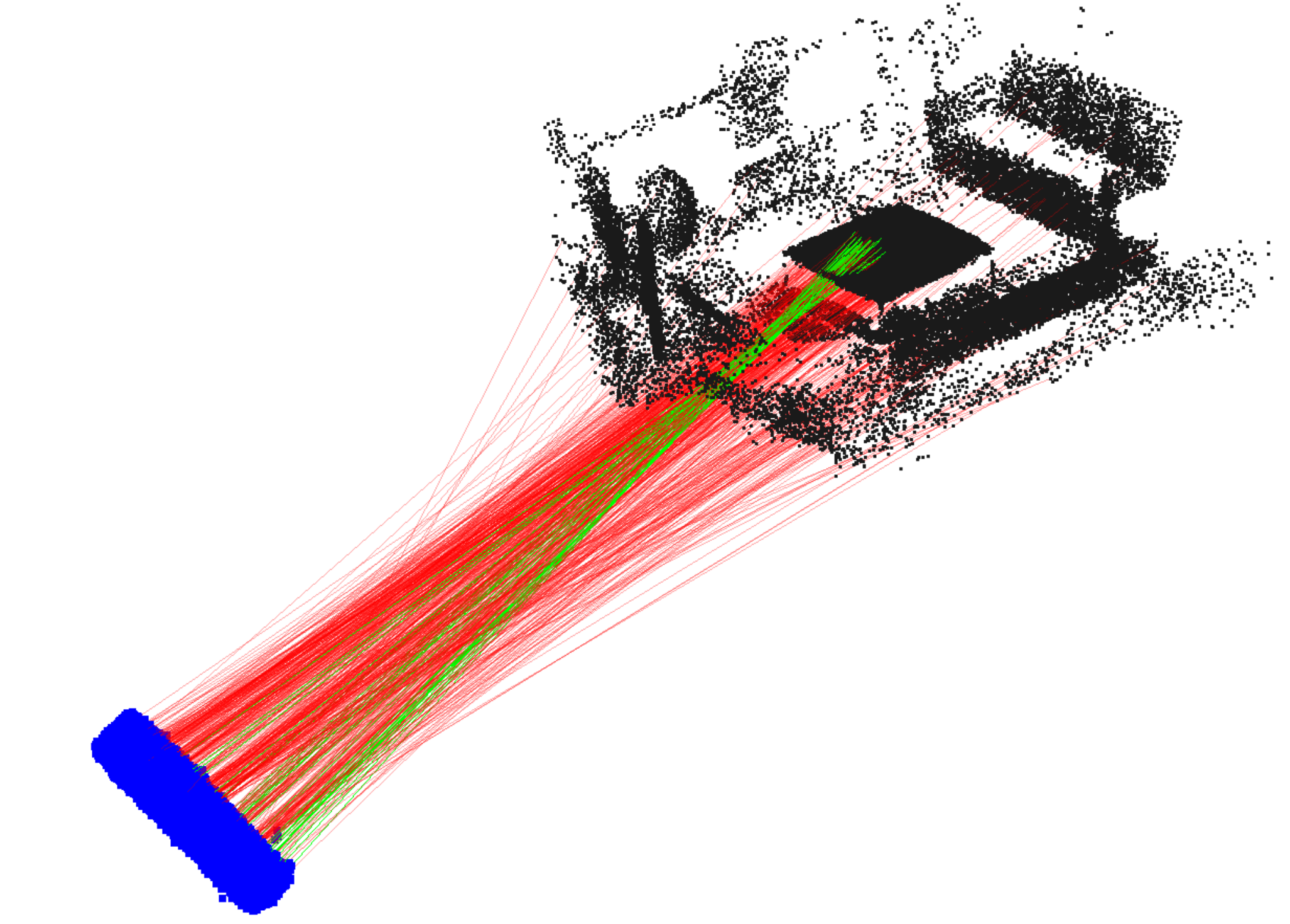}\,
&\includegraphics[width=0.20\linewidth]{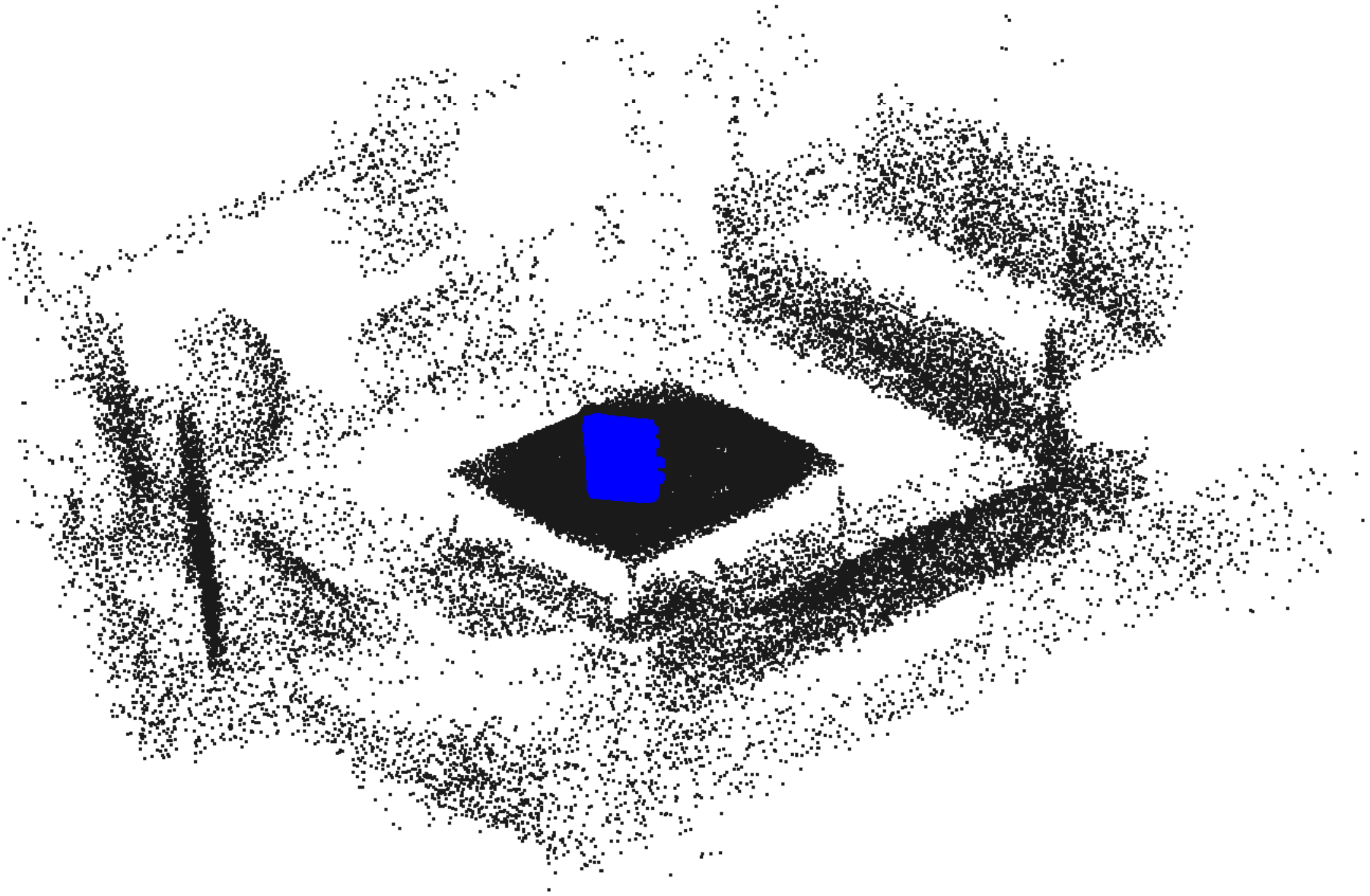}
&\includegraphics[width=0.20\linewidth]{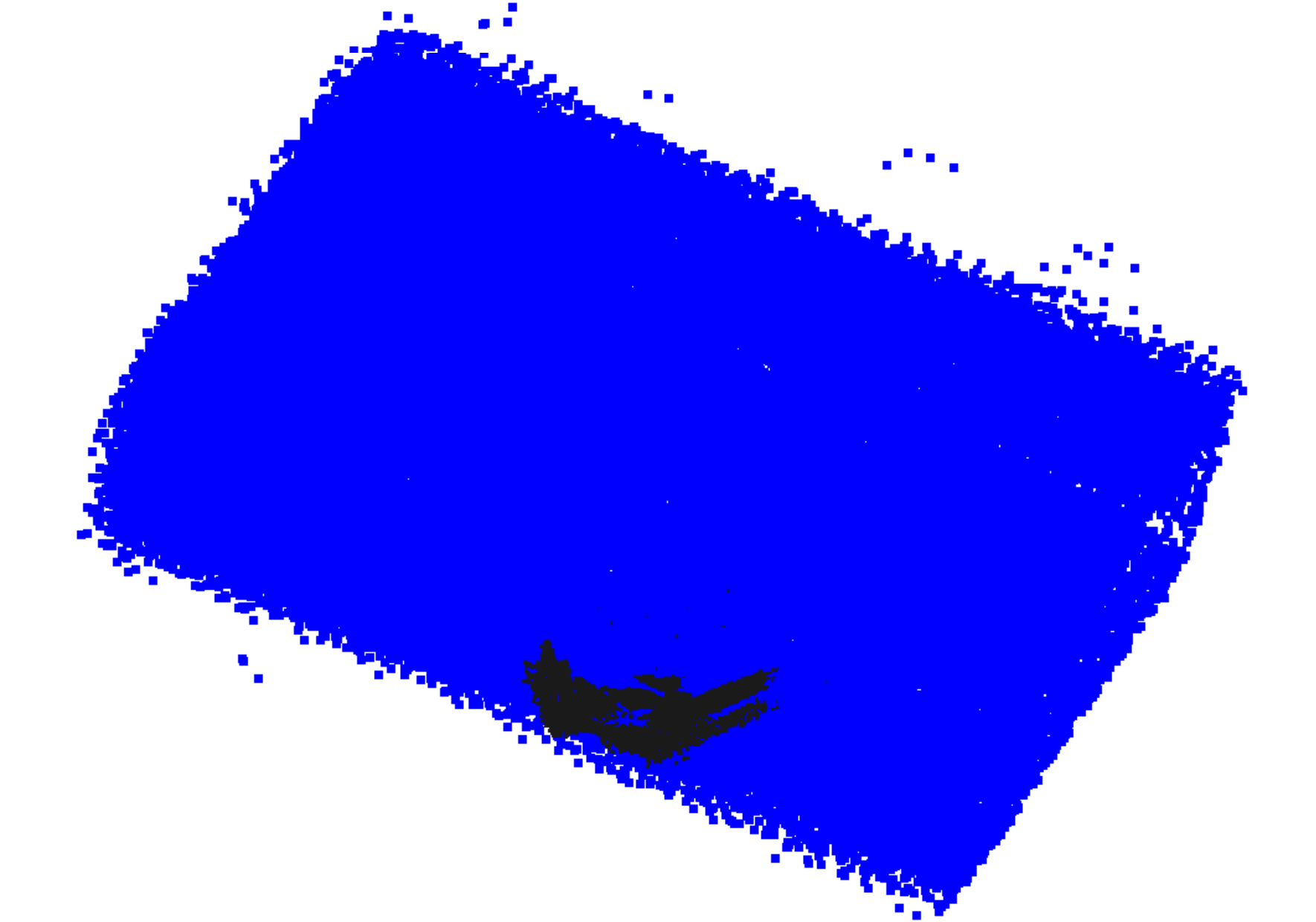}
&\includegraphics[width=0.20\linewidth]{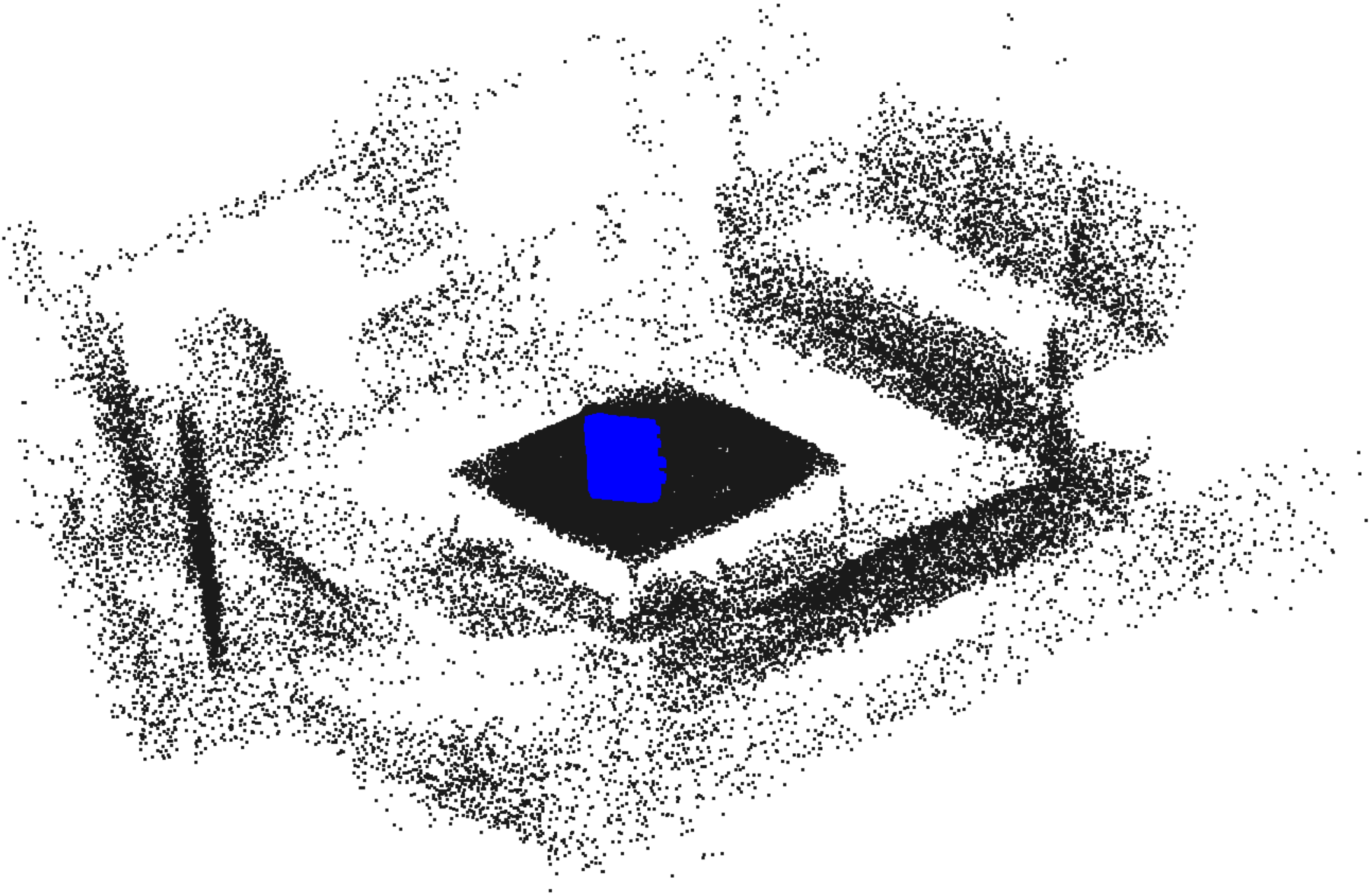} \\

\textit{Scene-05}, $\mathit{s}\in(1,5)$ & ICOS  & RANSAC(1000) & RANSAC(1min) \\

\includegraphics[width=0.20\linewidth]{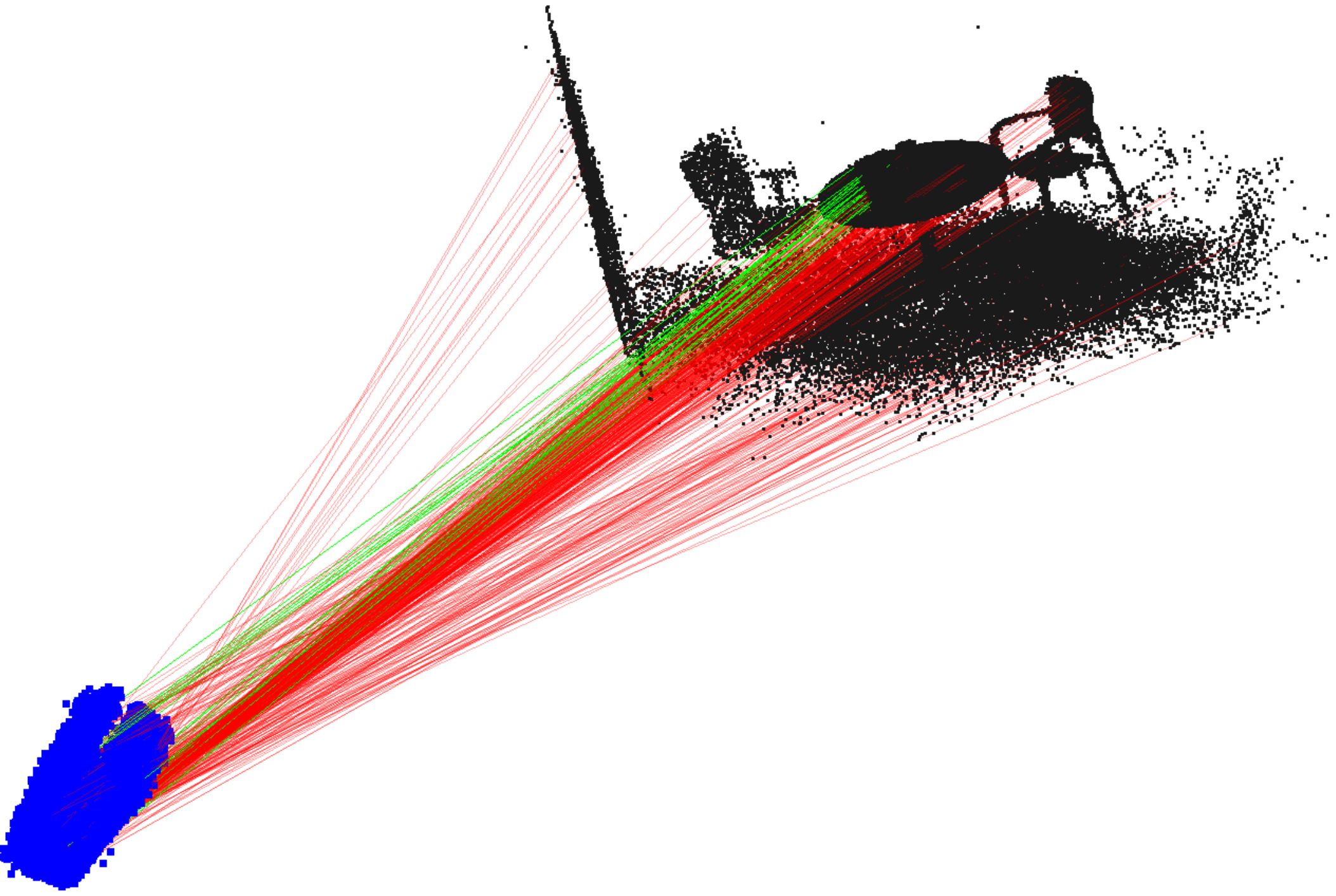}\,
&\includegraphics[width=0.20\linewidth]{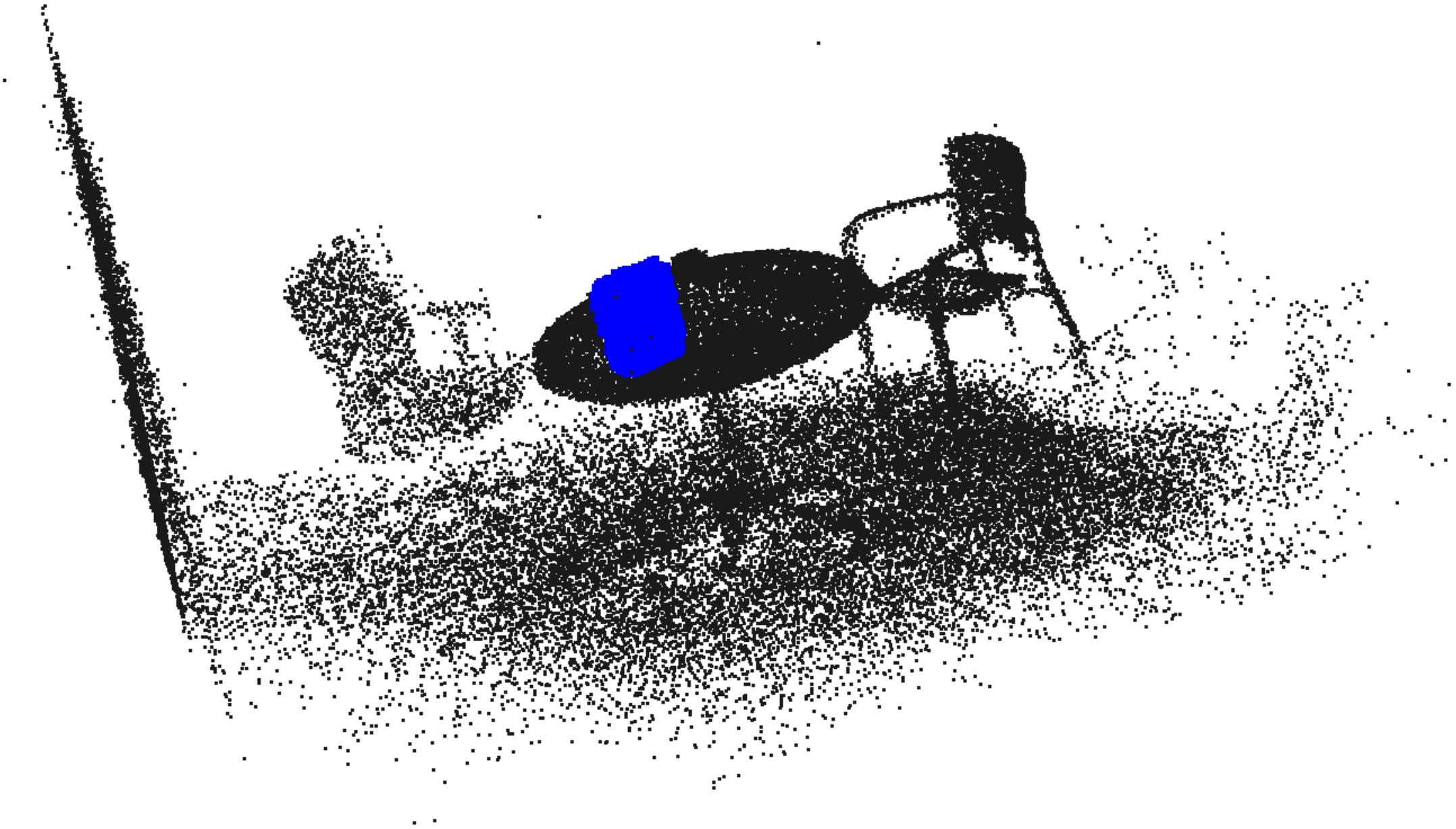}
&\includegraphics[width=0.20\linewidth]{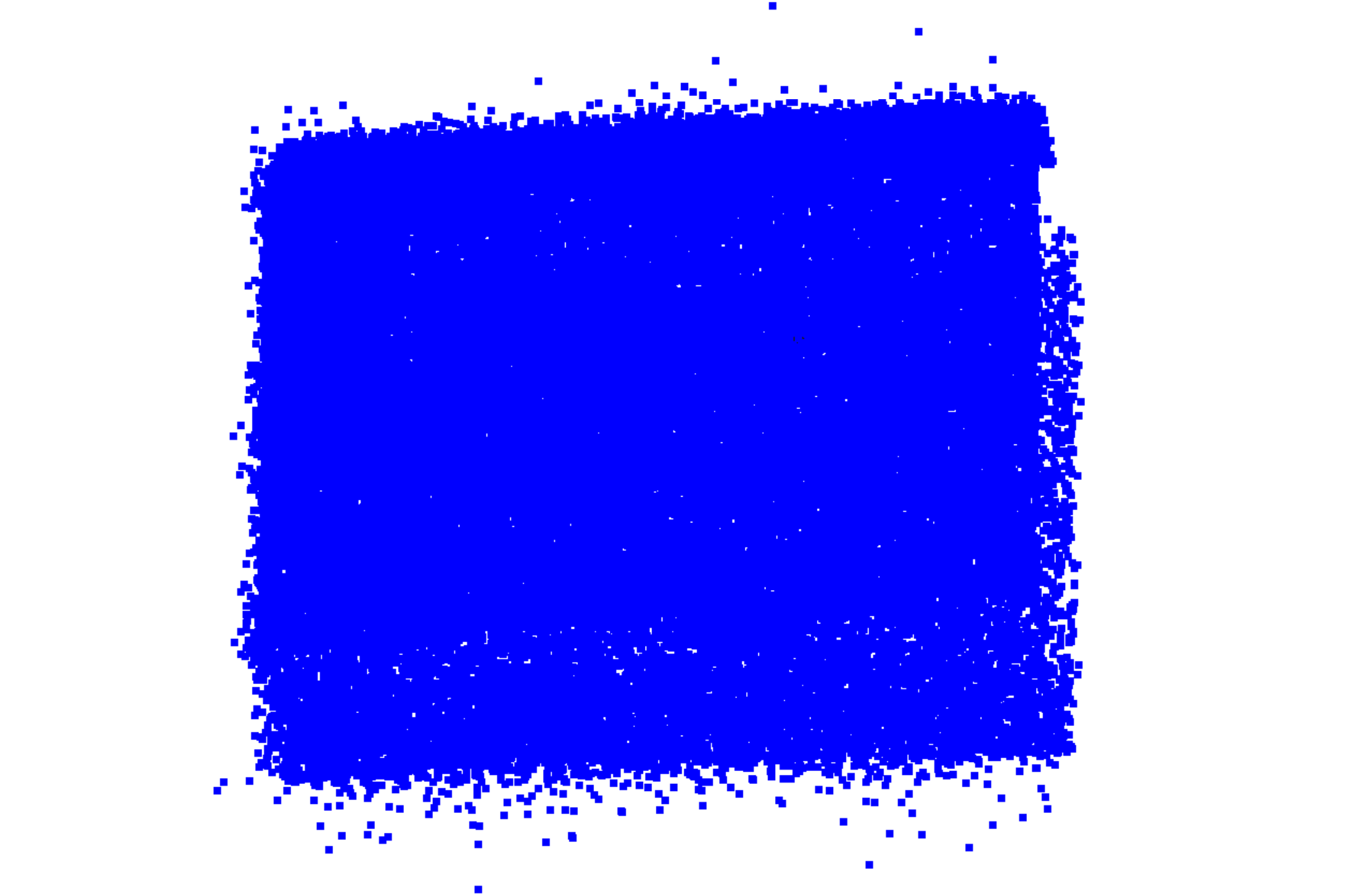}
&\includegraphics[width=0.20\linewidth]{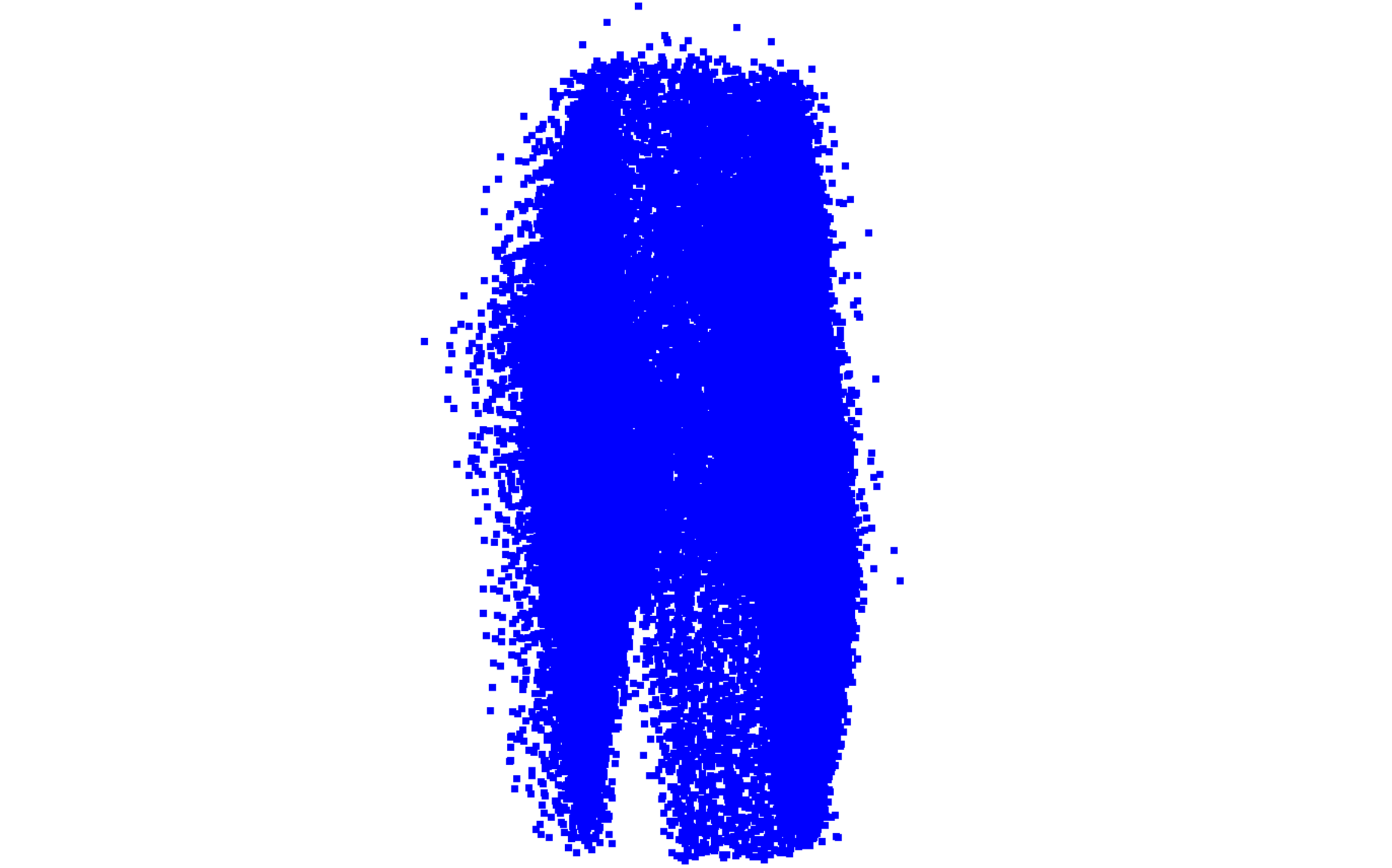} \\

\textit{Scene-07}, $\mathit{s}\in(1,5)$ & ICOS  & RANSAC(1000) & RANSAC(1min) \\

\includegraphics[width=0.20\linewidth]{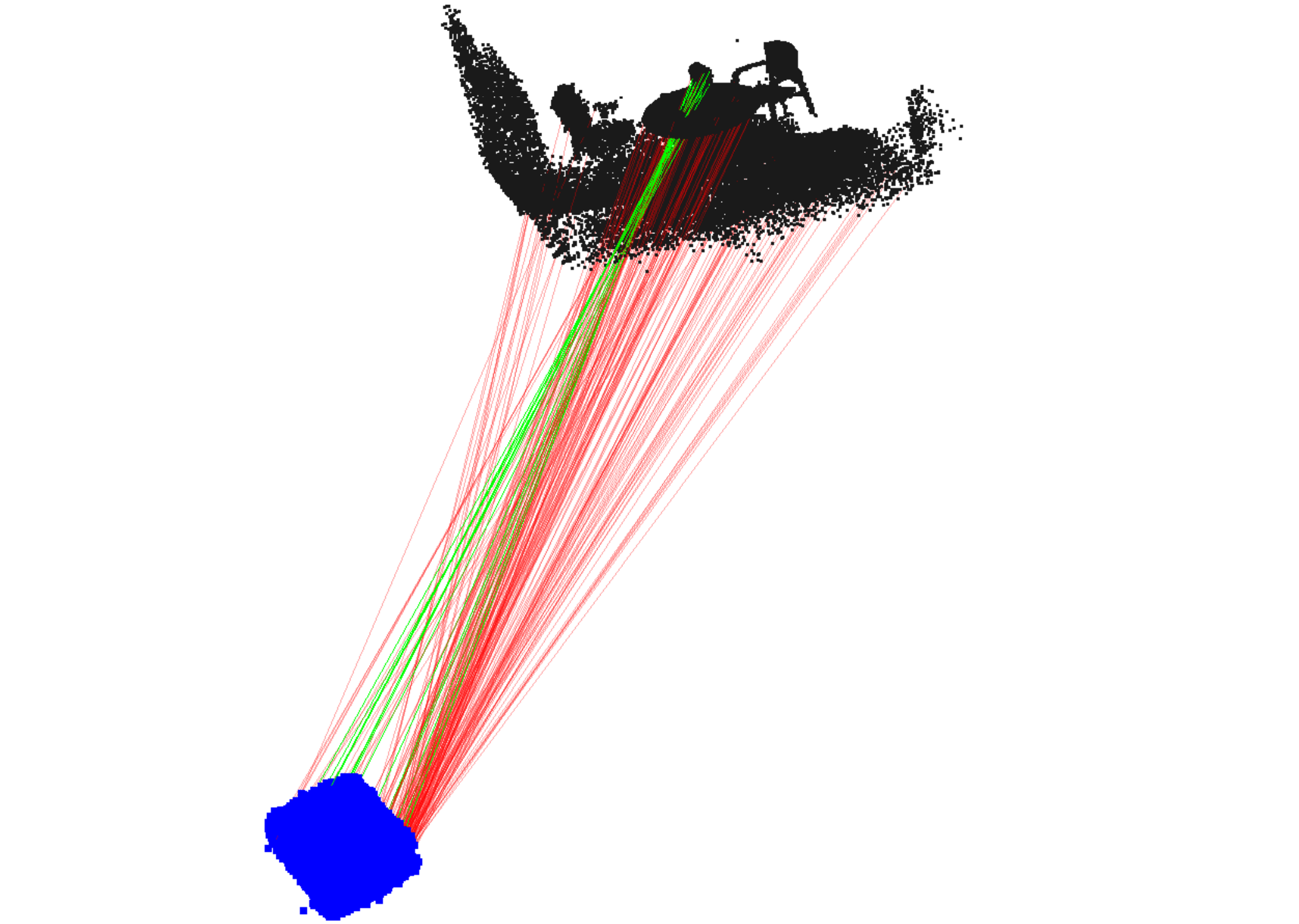}\,
&\includegraphics[width=0.20\linewidth]{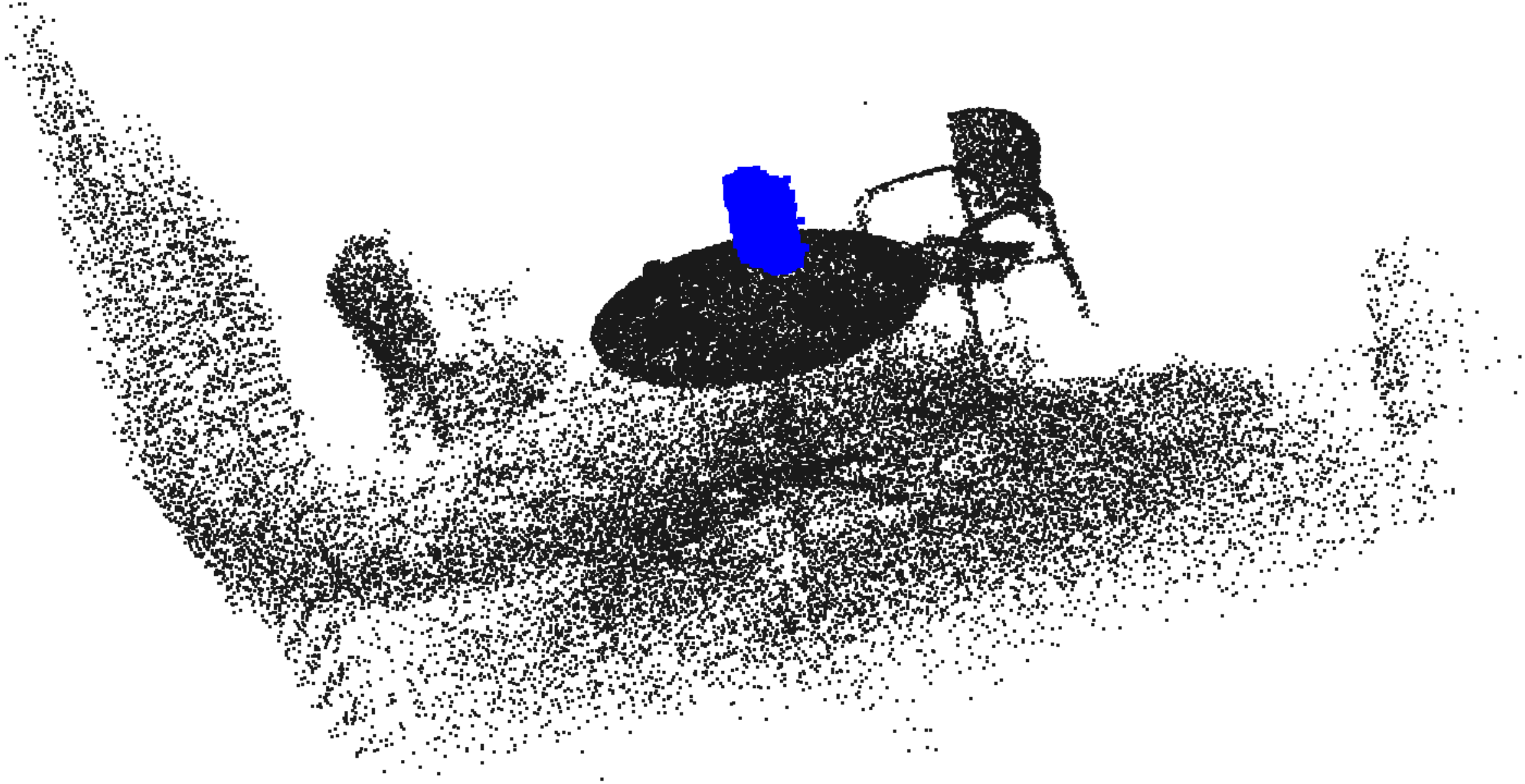}
&\includegraphics[width=0.20\linewidth]{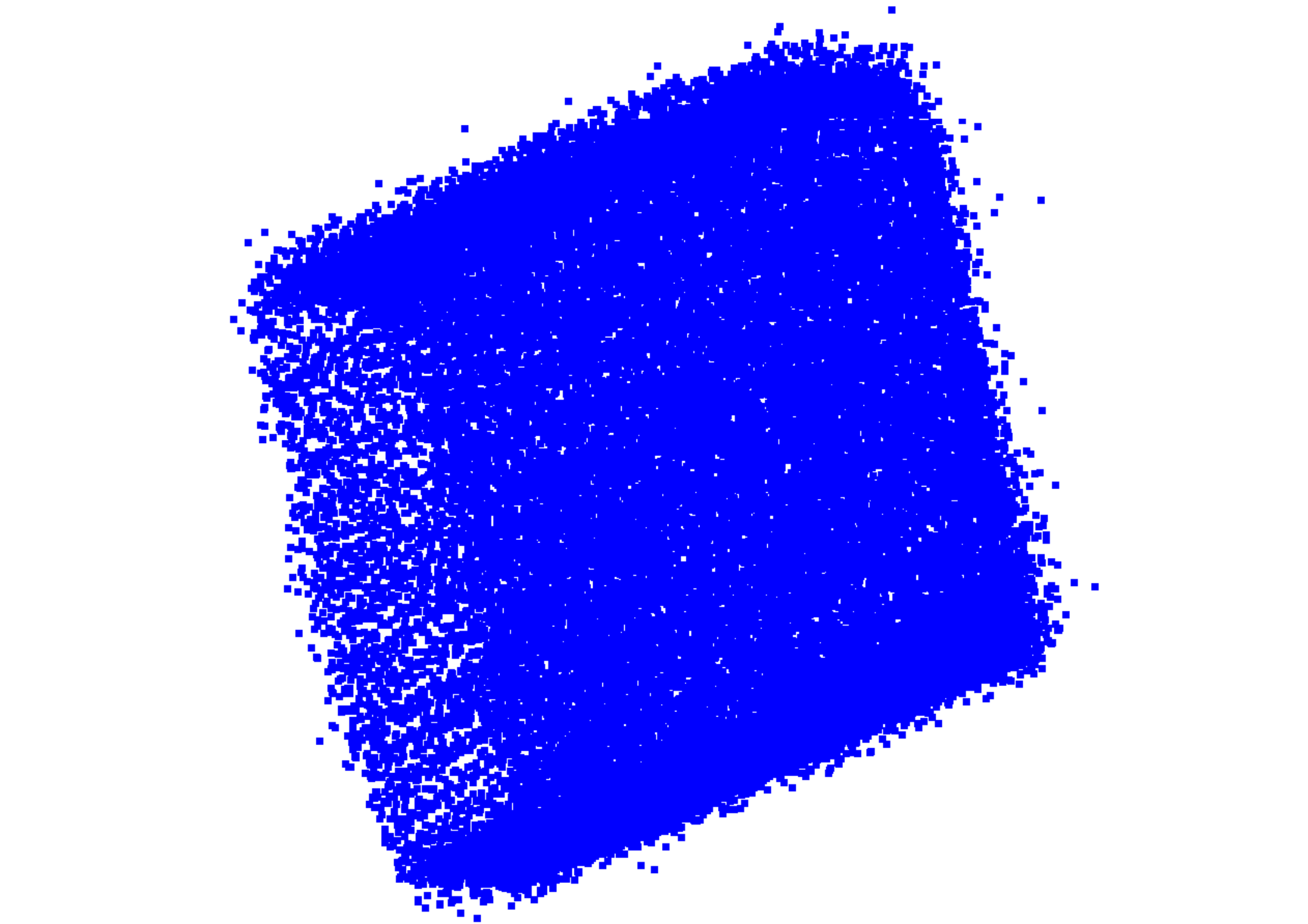}
&\includegraphics[width=0.20\linewidth]{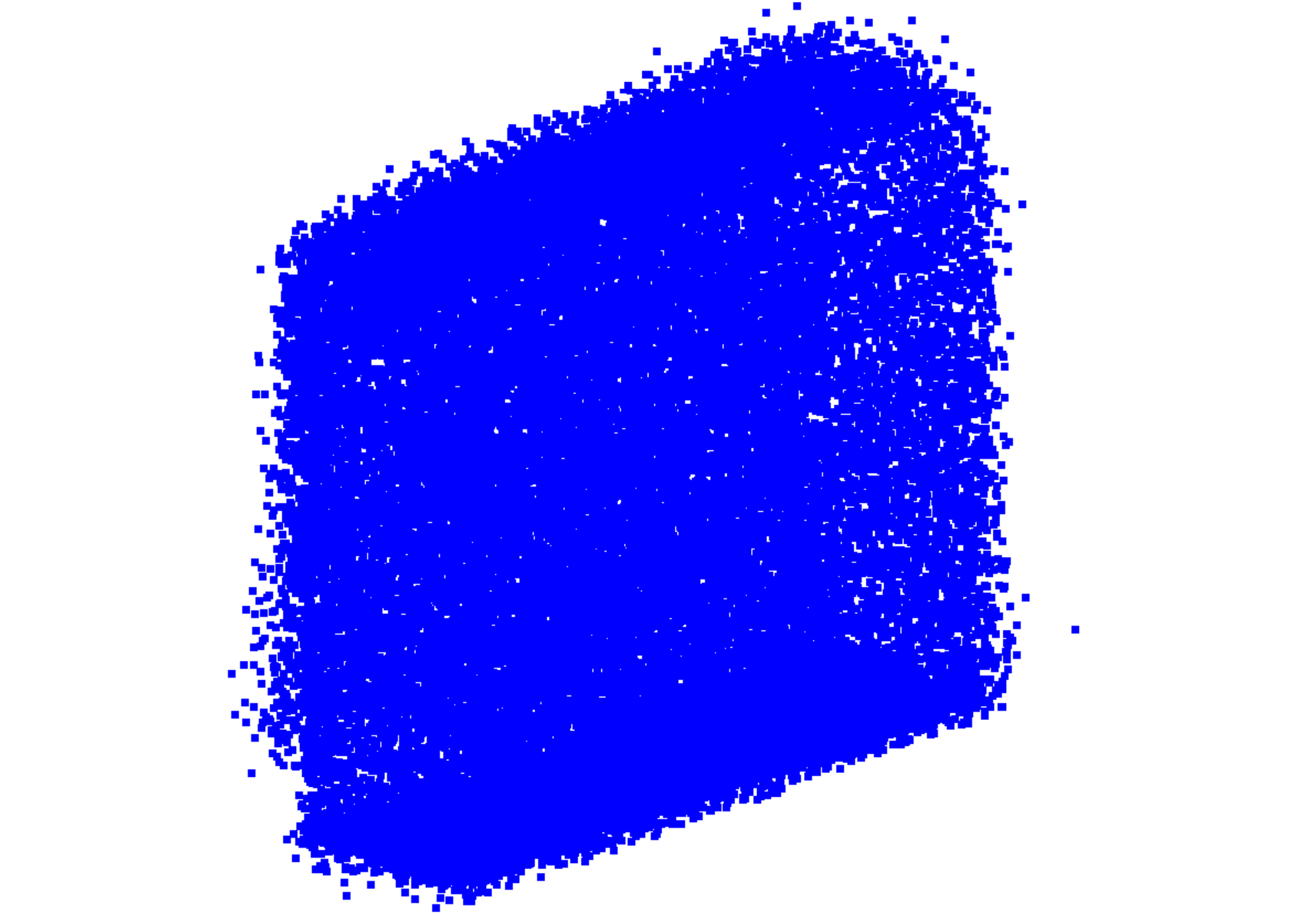} \\

\textit{Scene-09}, $\mathit{s}\in(1,5)$ & ICOS  & RANSAC(1000) & RANSAC(1min) \\

\includegraphics[width=0.20\linewidth]{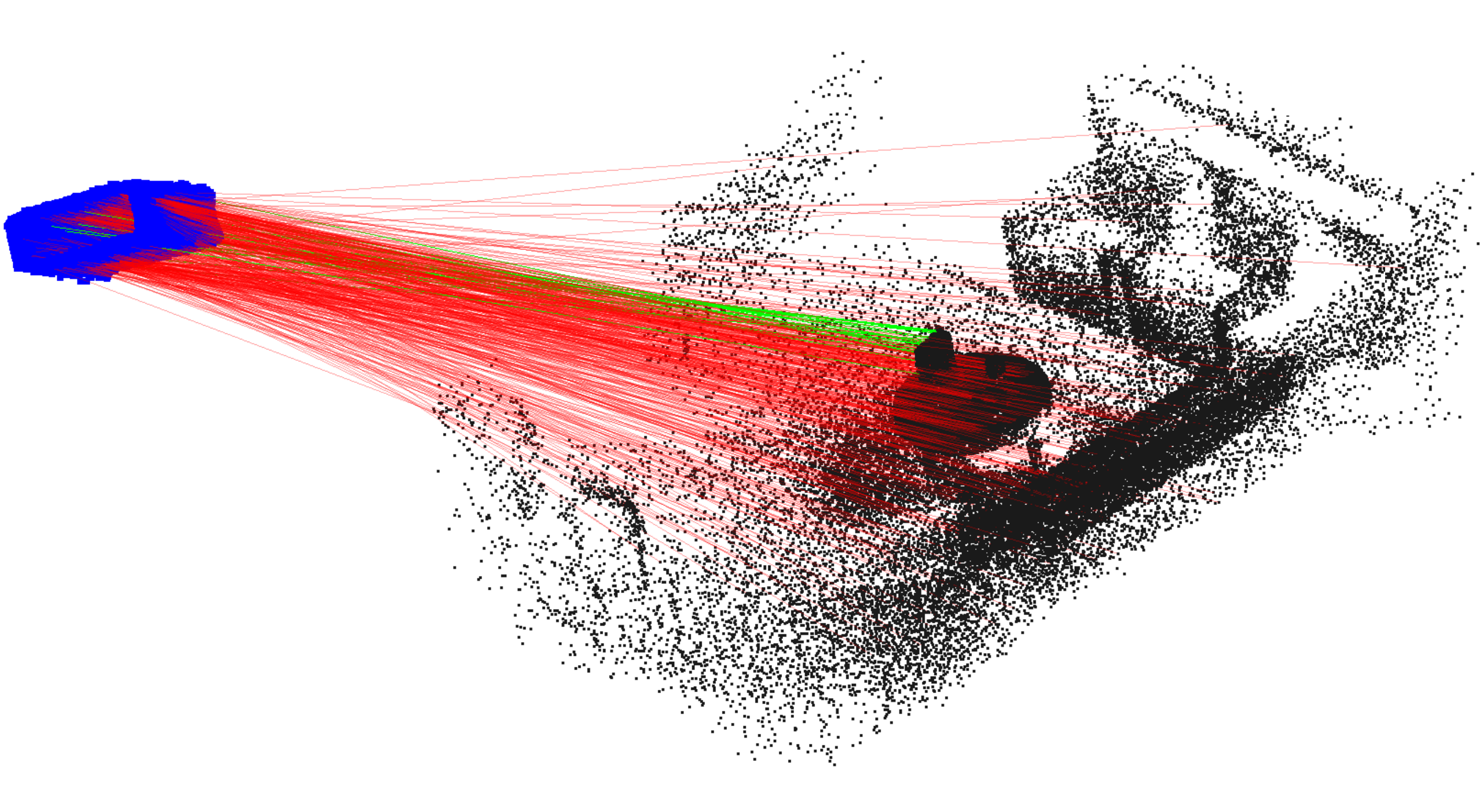}\,
&\includegraphics[width=0.20\linewidth]{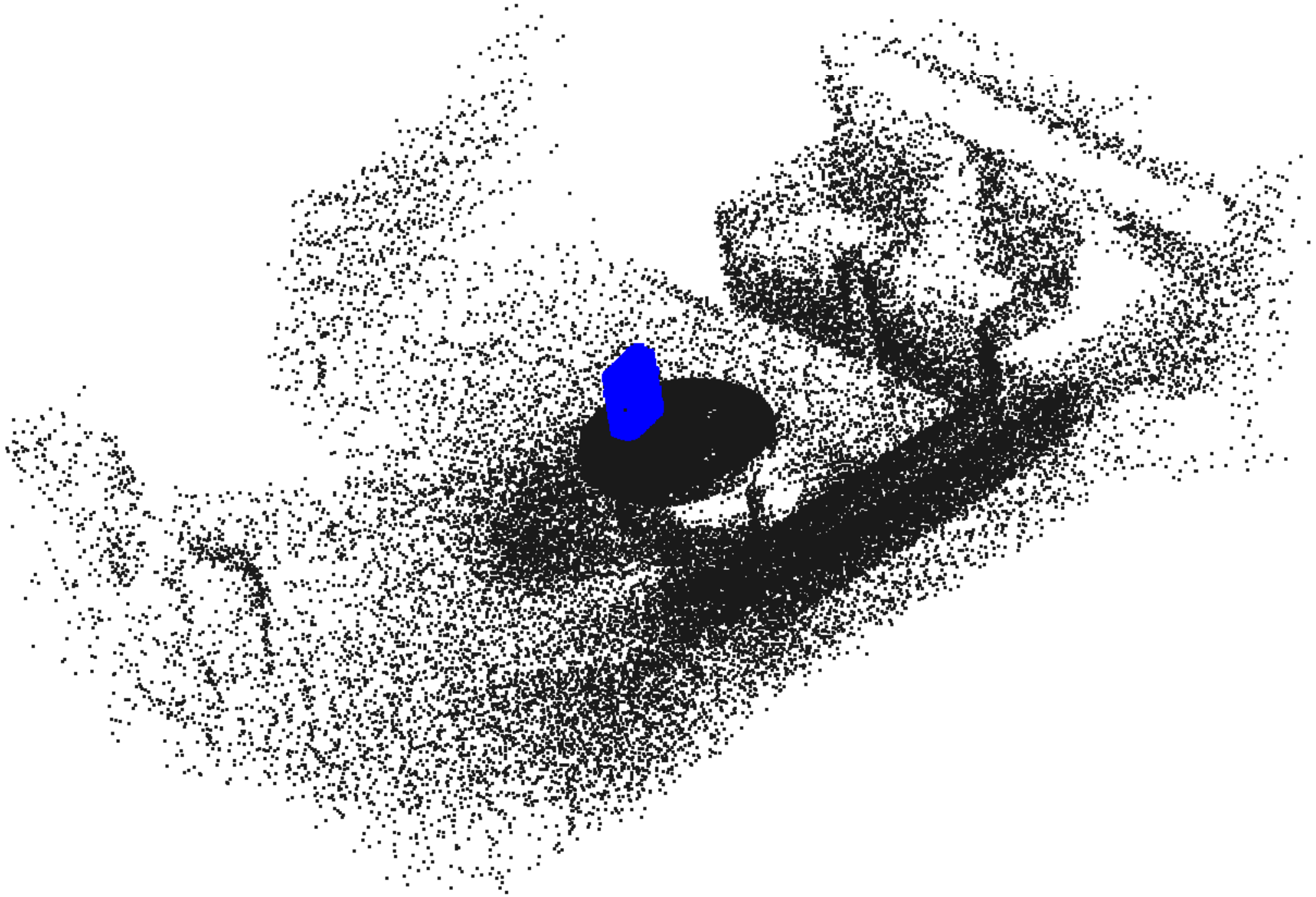}
&\includegraphics[width=0.20\linewidth]{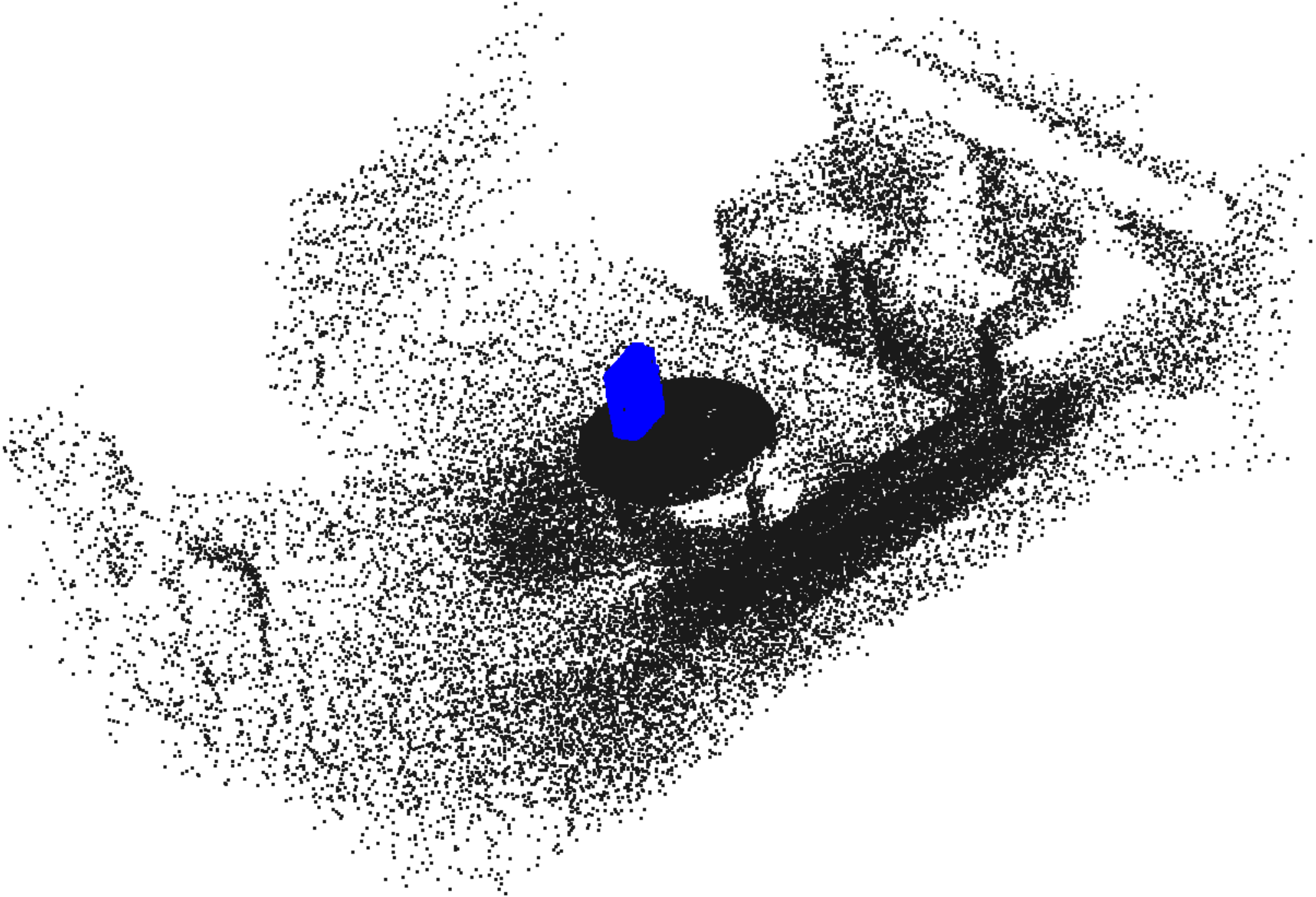}
&\includegraphics[width=0.20\linewidth]{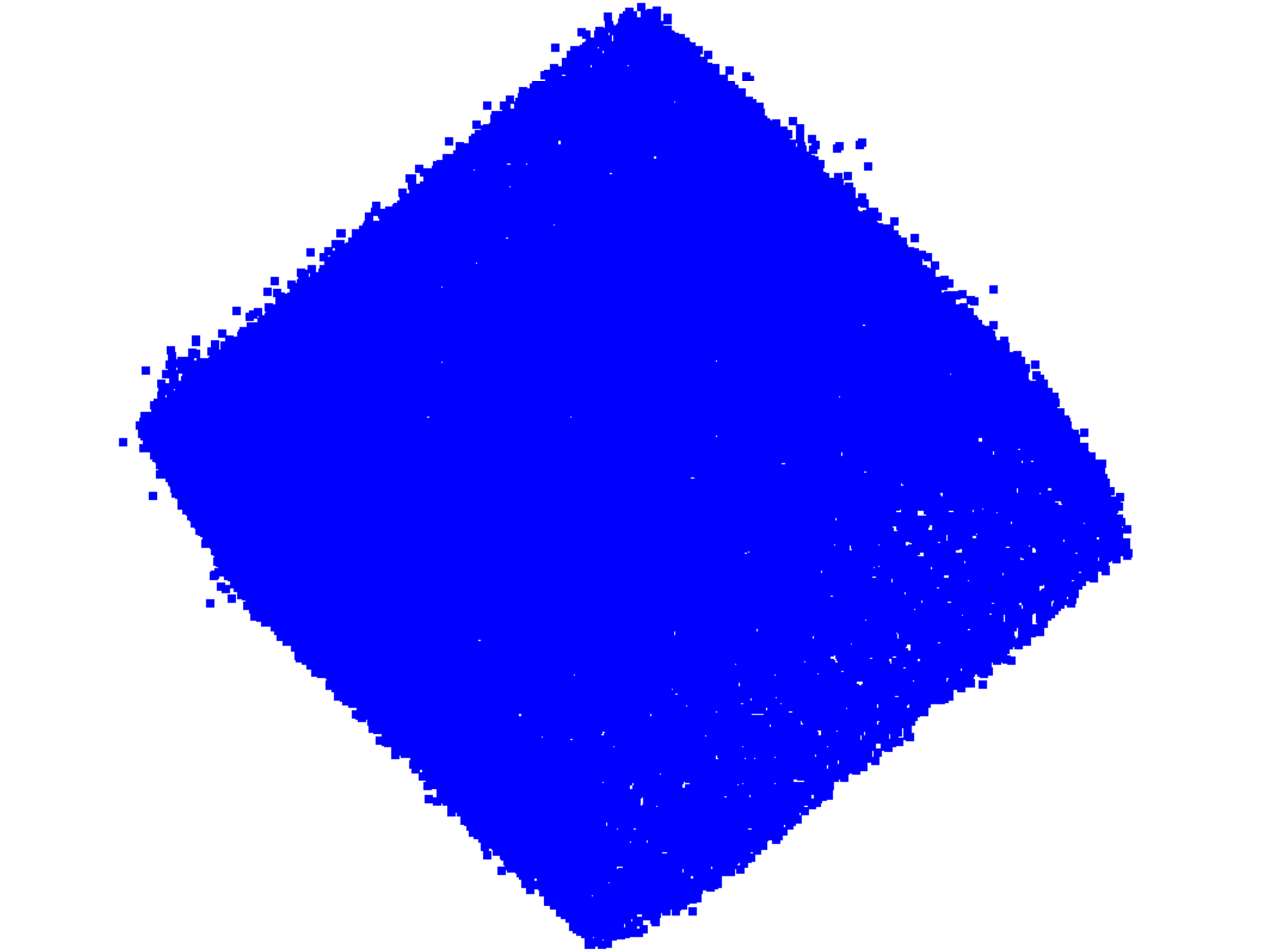}

\end{tabular}

\vspace{-1mm}

\centering
\caption{Qualitative results of 3D object localization (\textit{cereal box}) using ICOS, RANSAC(1000) and RANSAC(1min). From left to right, we show the correspondences matched by FPFH (inliers are in green lines and outliers are in red lines), and the registration (reprojection) results with ICOS, RANSAC(1000), and RANSAC(1min), respectively. Note that in all scenes the outlier ratios are over 80\%, and our ICOS can always yield very promising registration results.}
\label{Object-Lo}
\end{figure*}

\begin{figure*}[t]
\centering

\begin{tabular}{ccccc}

\textit{Scene-01}, $\mathit{s}=1$ & ICOS  & RANSAC(1000) & RANSAC(1min) \\

\includegraphics[width=0.20\linewidth]{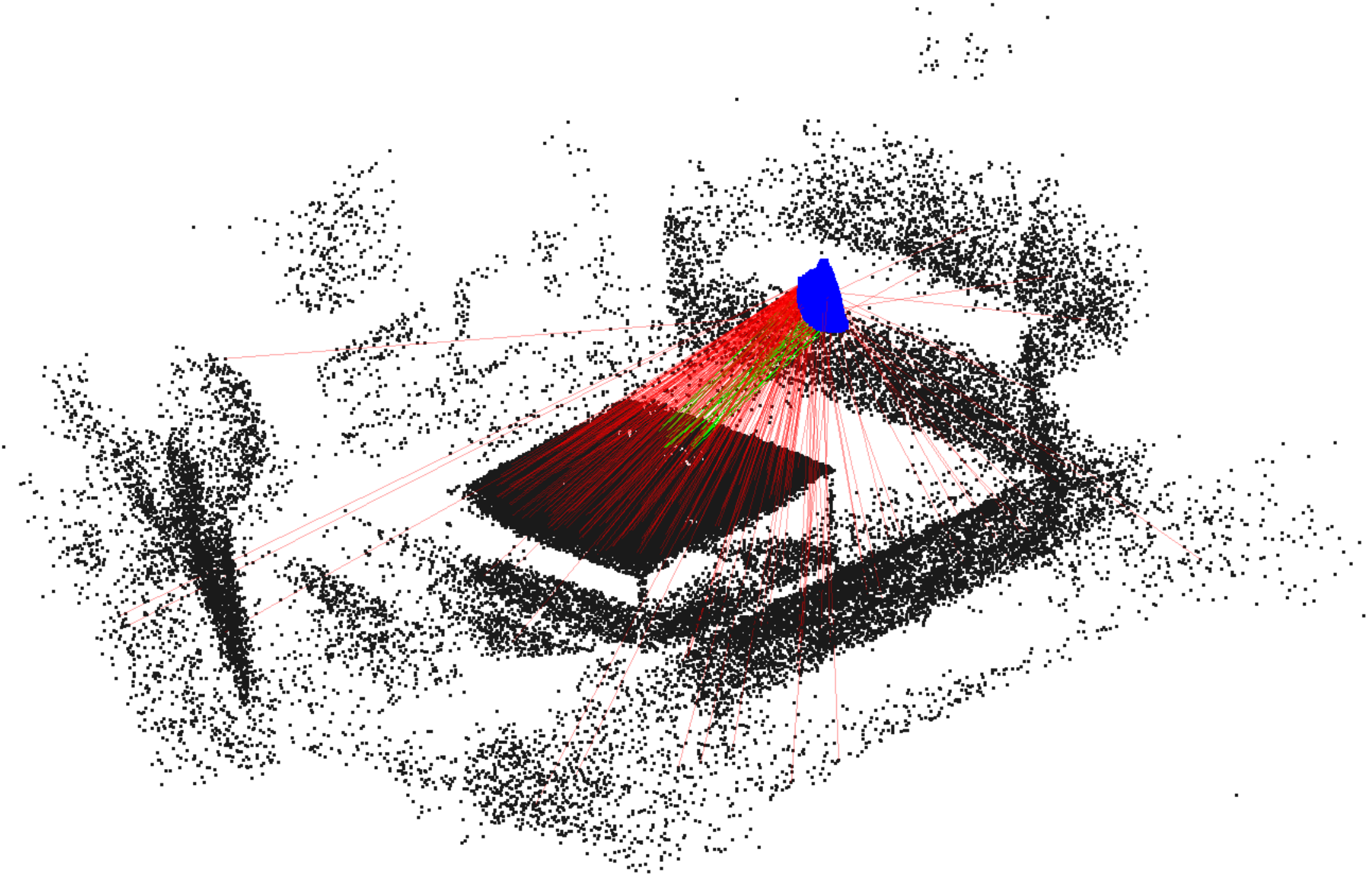}\,
&\includegraphics[width=0.20\linewidth]{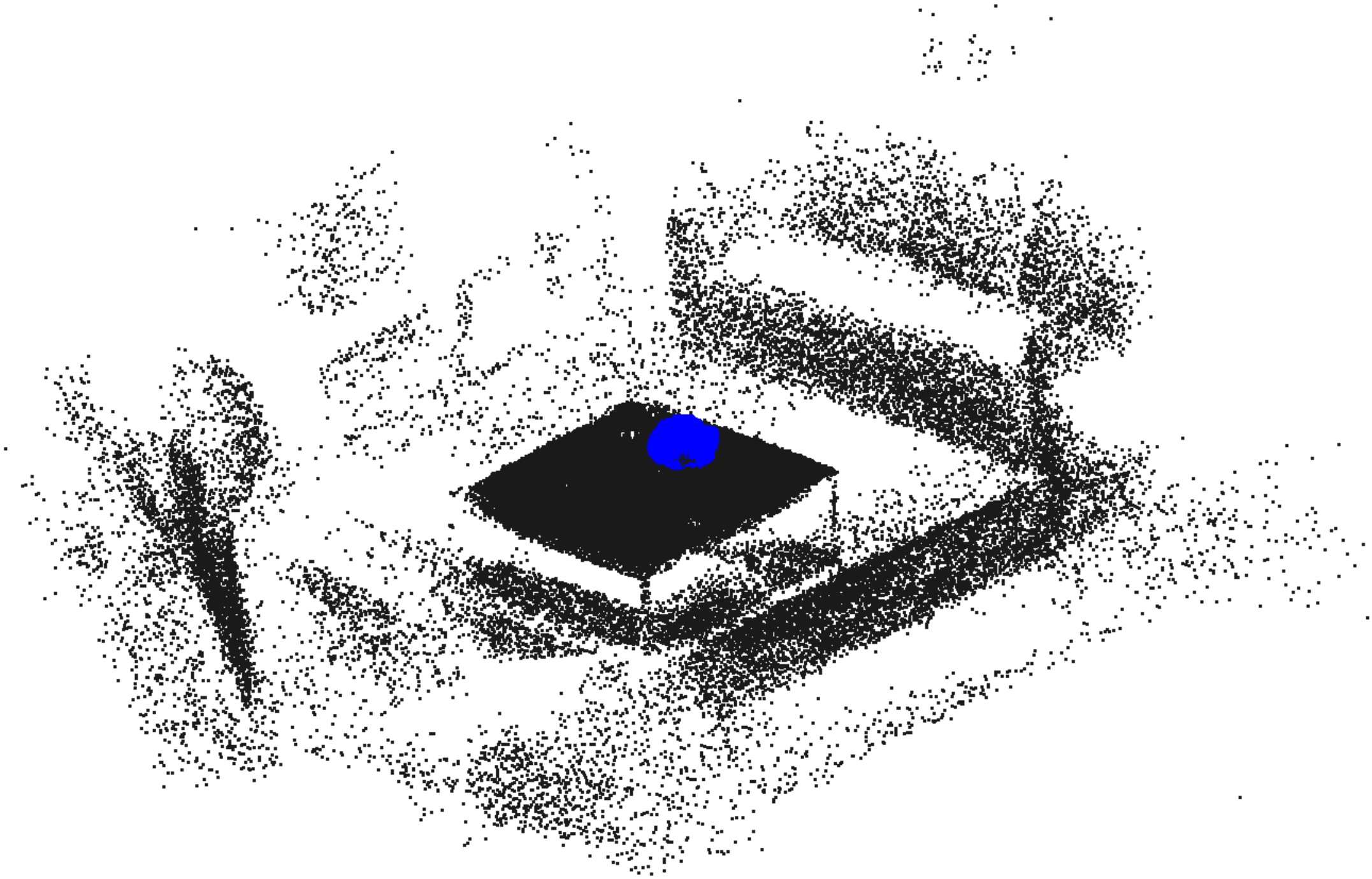}
&\includegraphics[width=0.20\linewidth]{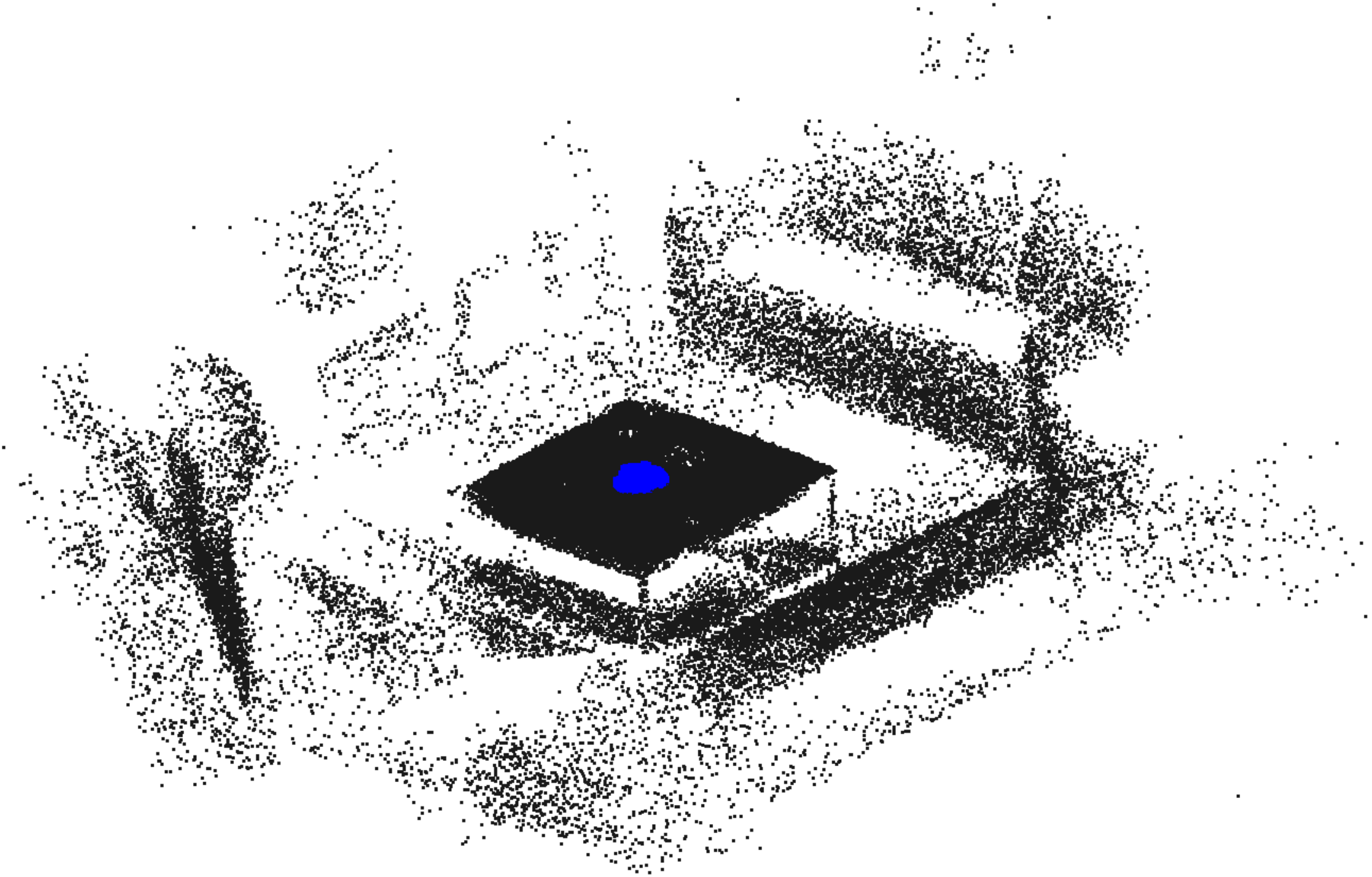}
&\includegraphics[width=0.20\linewidth]{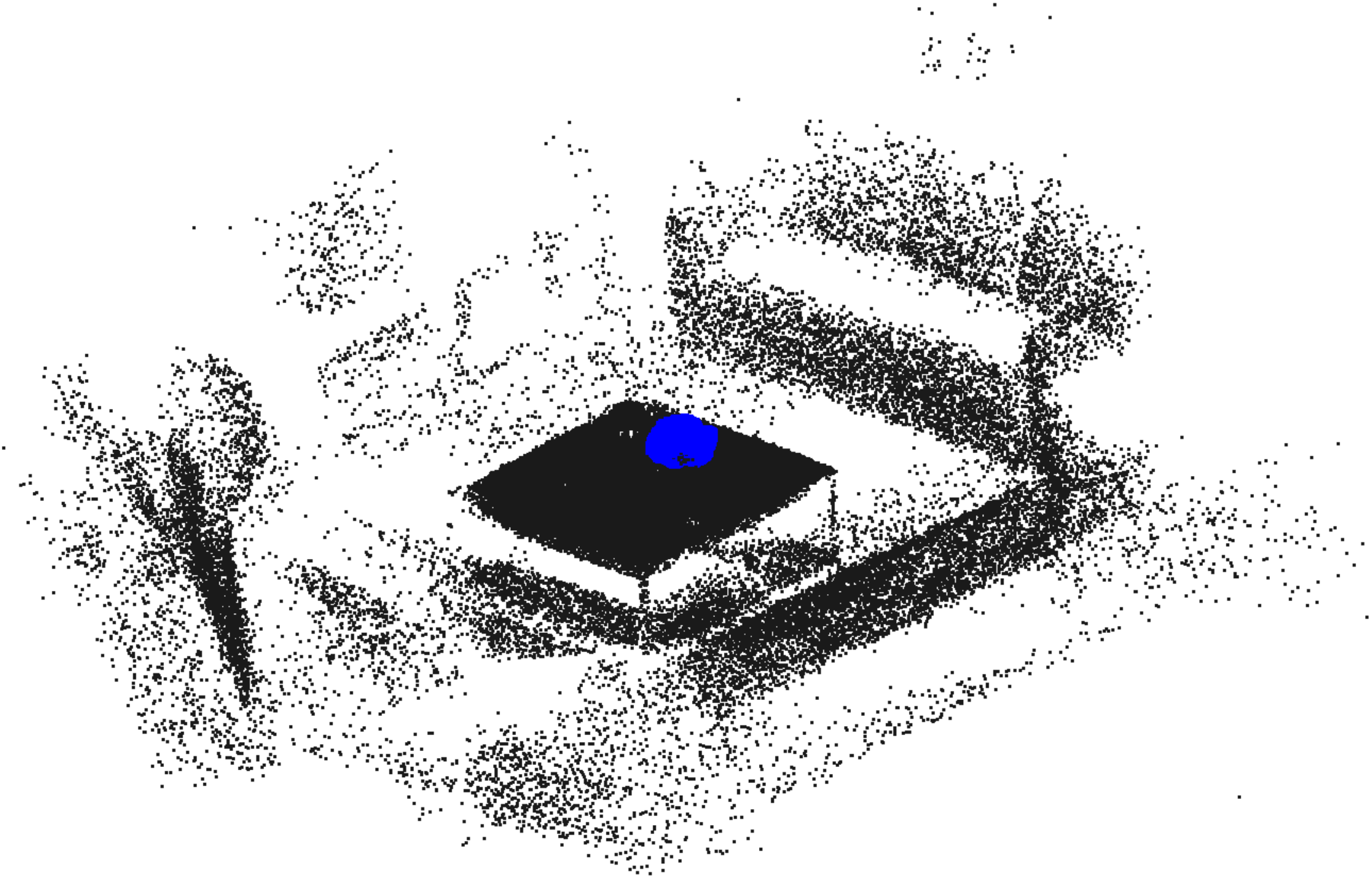} \\

\textit{Scene-03}, $\mathit{s}=1$ & ICOS  & RANSAC(1000) & RANSAC(1min) \\

\includegraphics[width=0.20\linewidth]{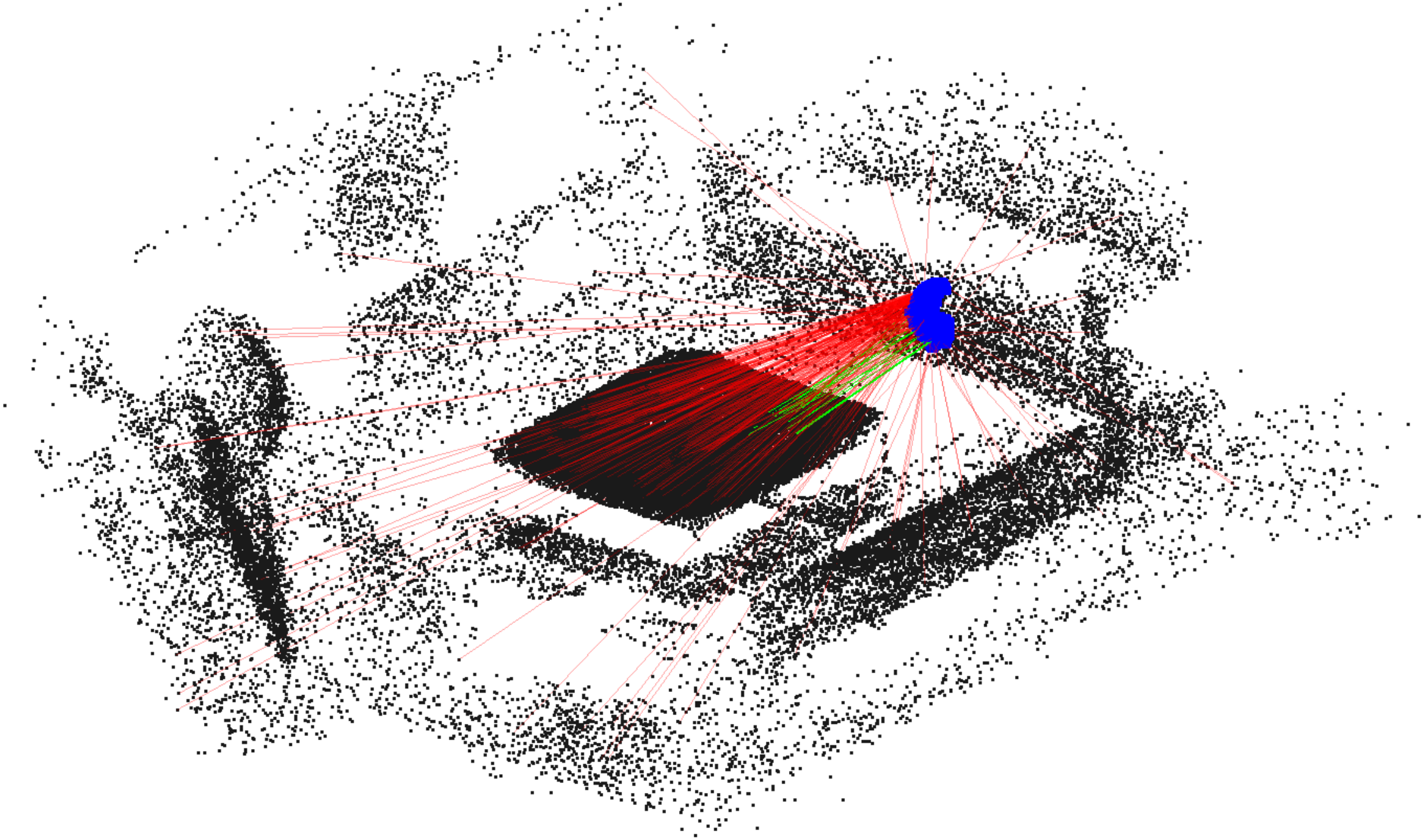}\,
&\includegraphics[width=0.20\linewidth]{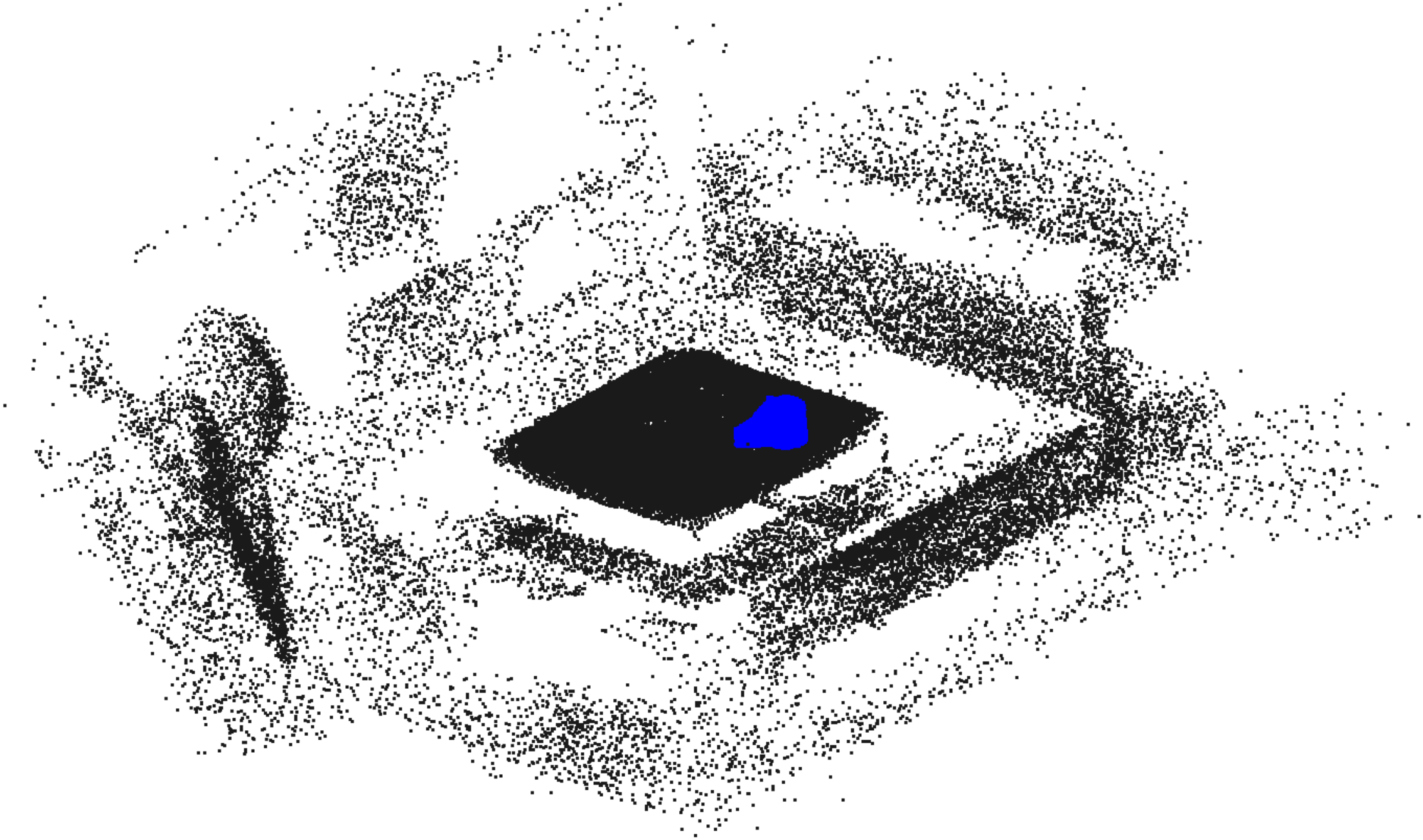}
&\includegraphics[width=0.20\linewidth]{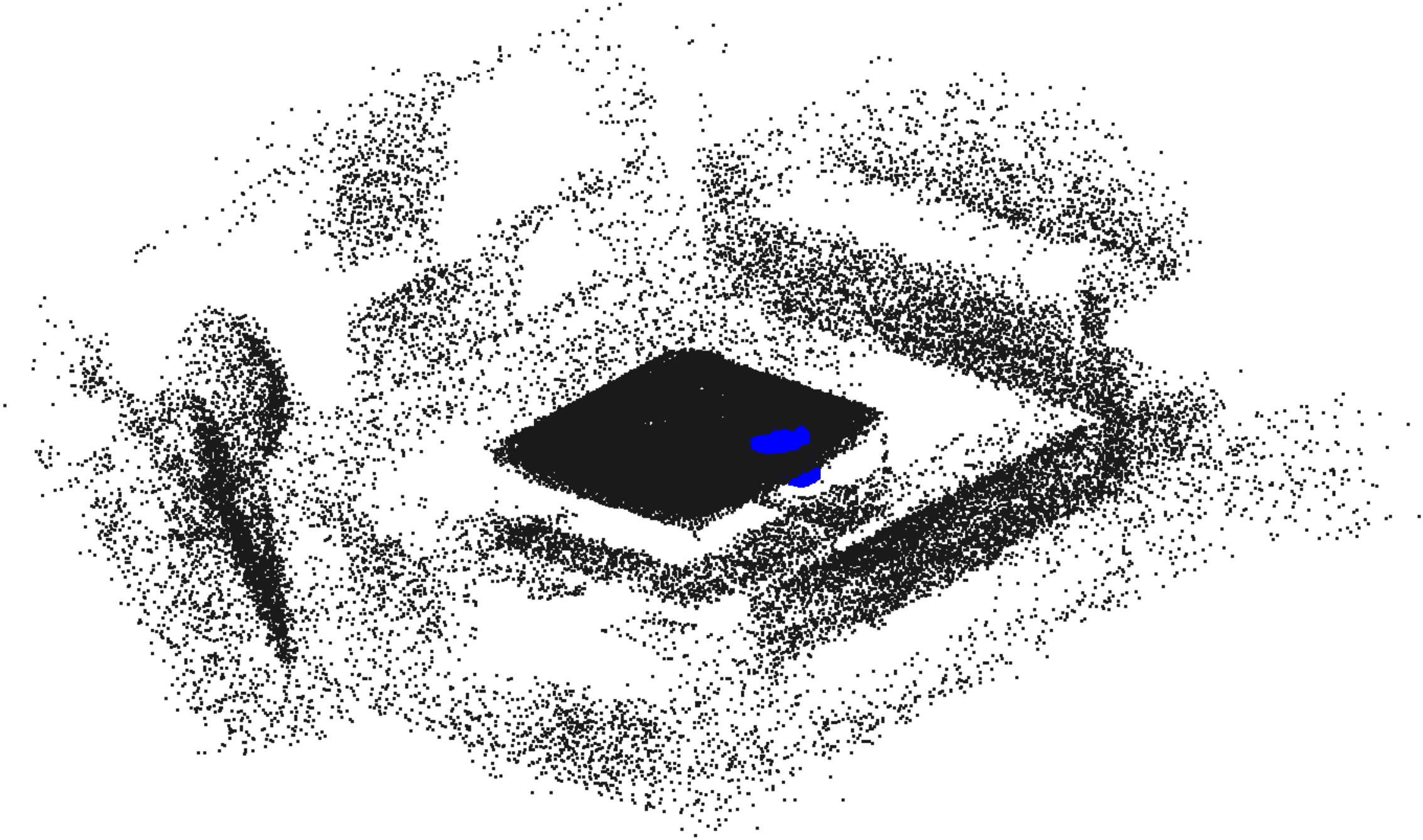}
&\includegraphics[width=0.20\linewidth]{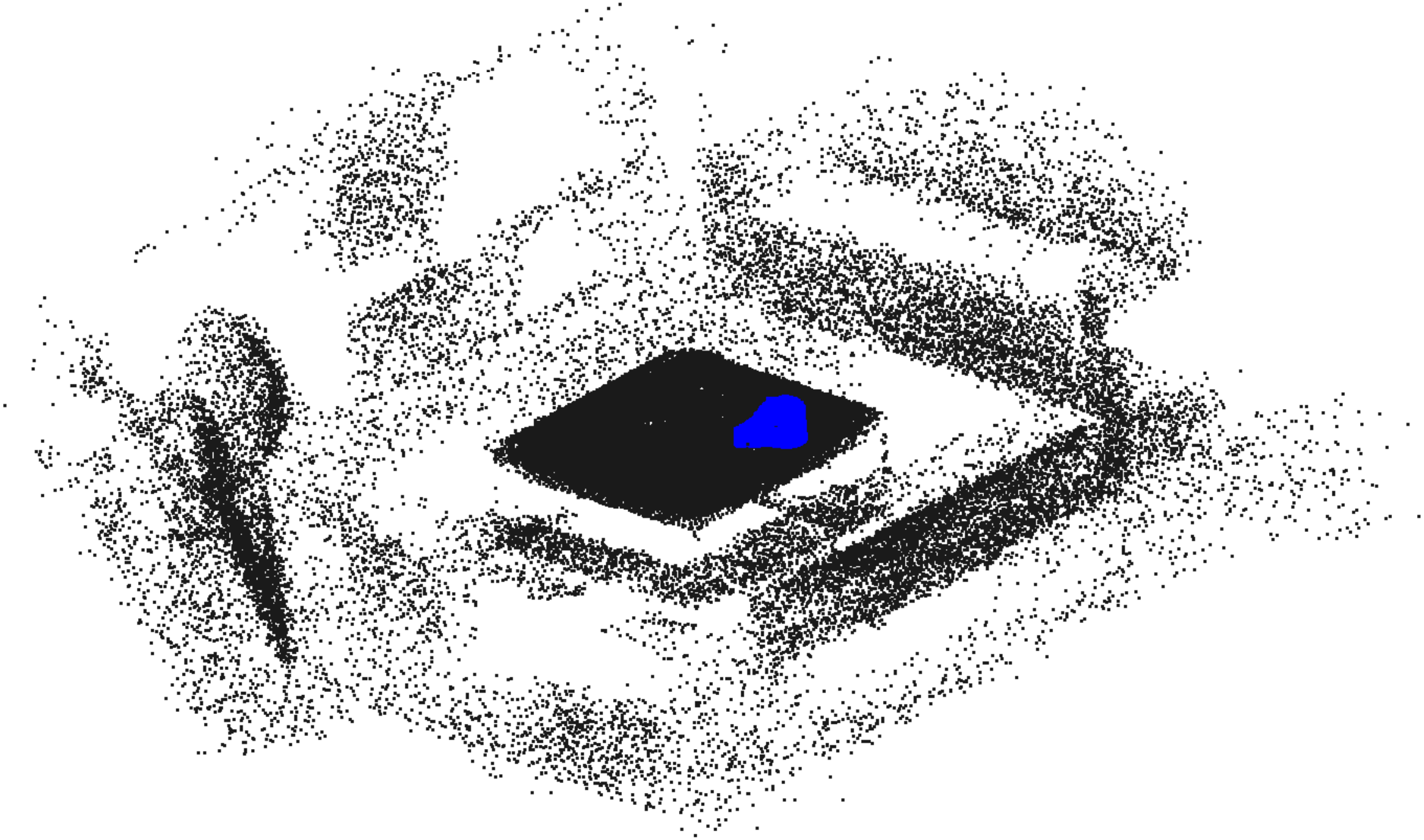} \\

\textit{Scene-09}, $\mathit{s}=1$ & ICOS  & RANSAC(1000) & RANSAC(1min) \\

\includegraphics[width=0.20\linewidth]{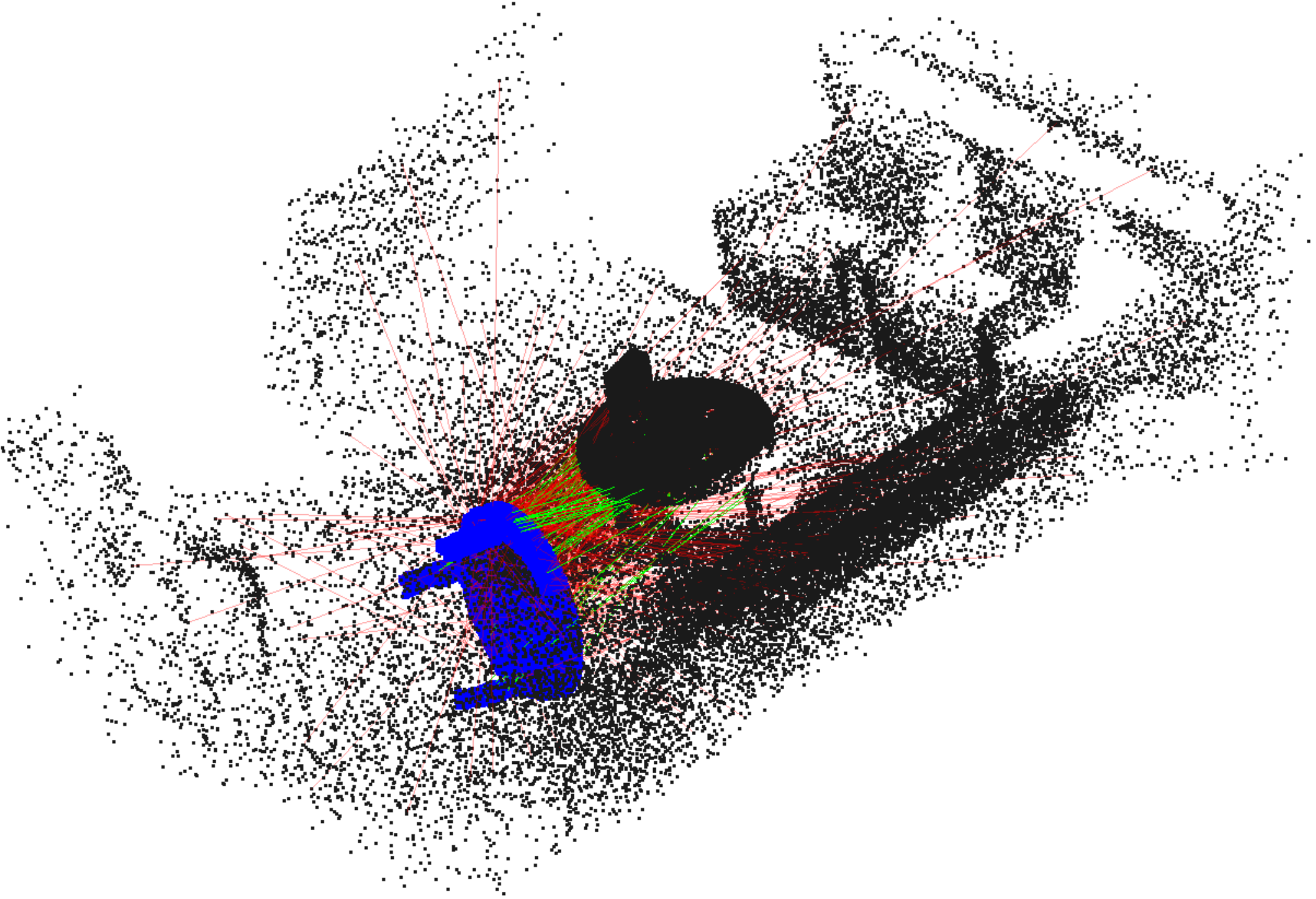}\,
&\includegraphics[width=0.20\linewidth]{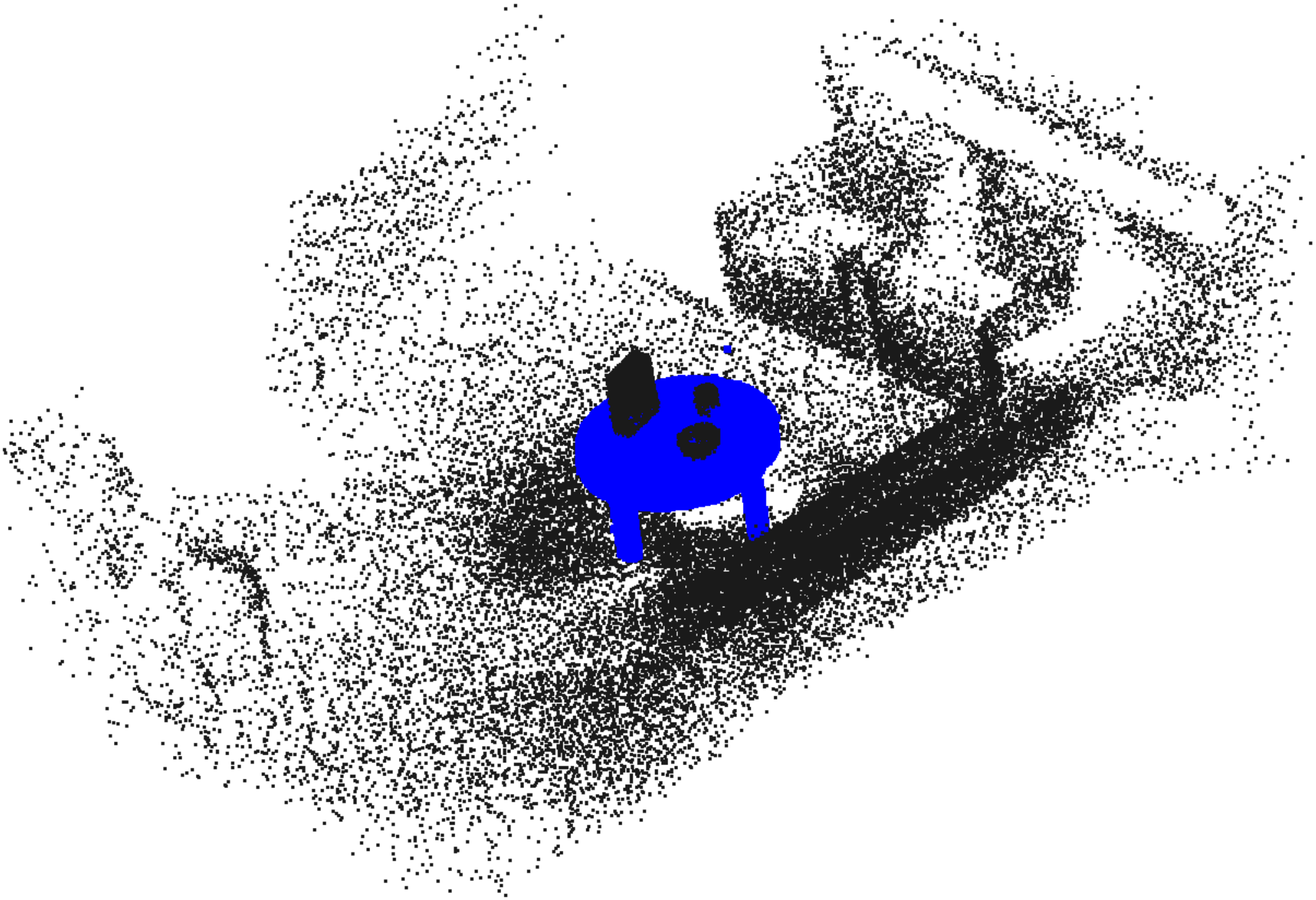}
&\includegraphics[width=0.20\linewidth]{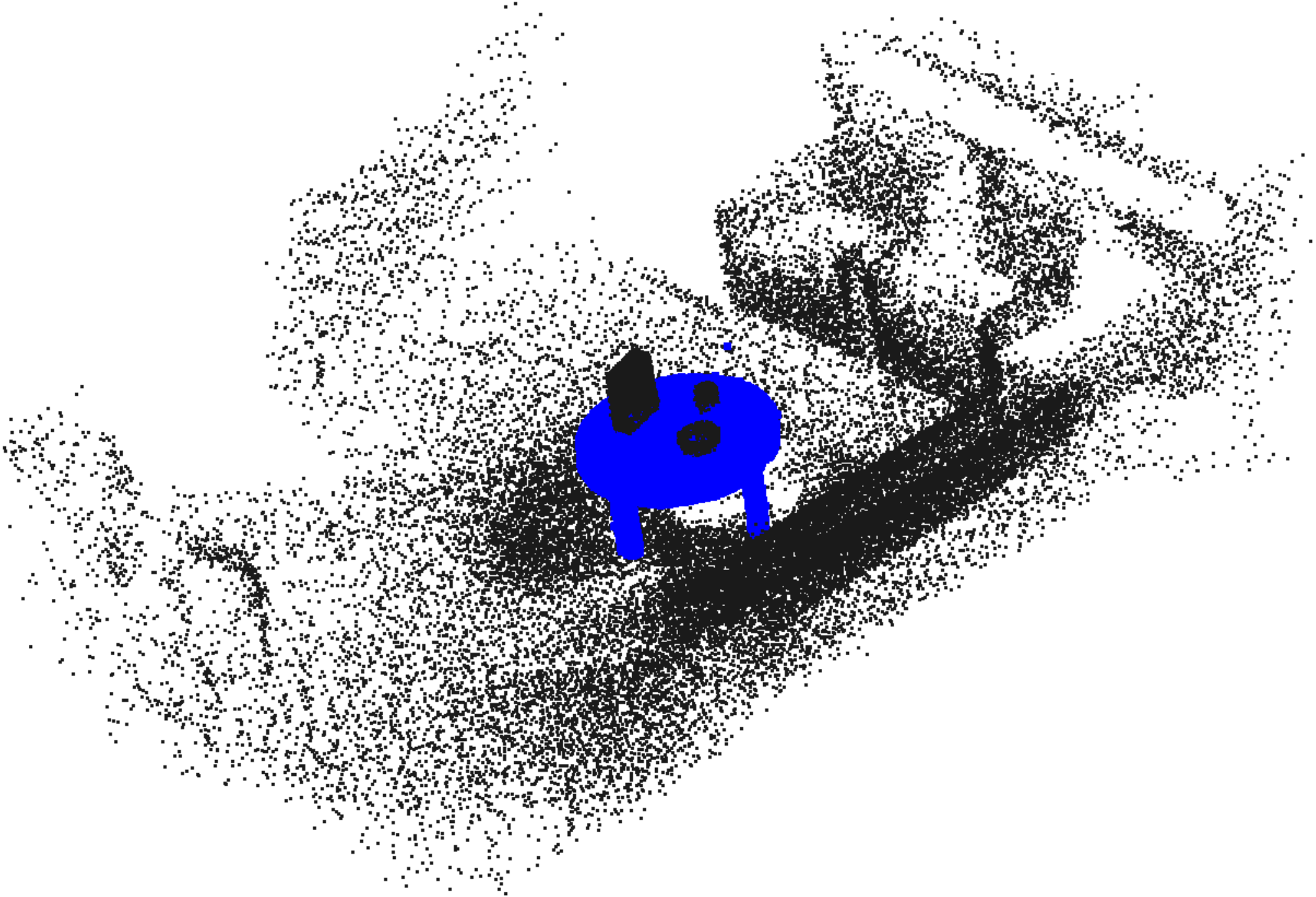}
&\includegraphics[width=0.20\linewidth]{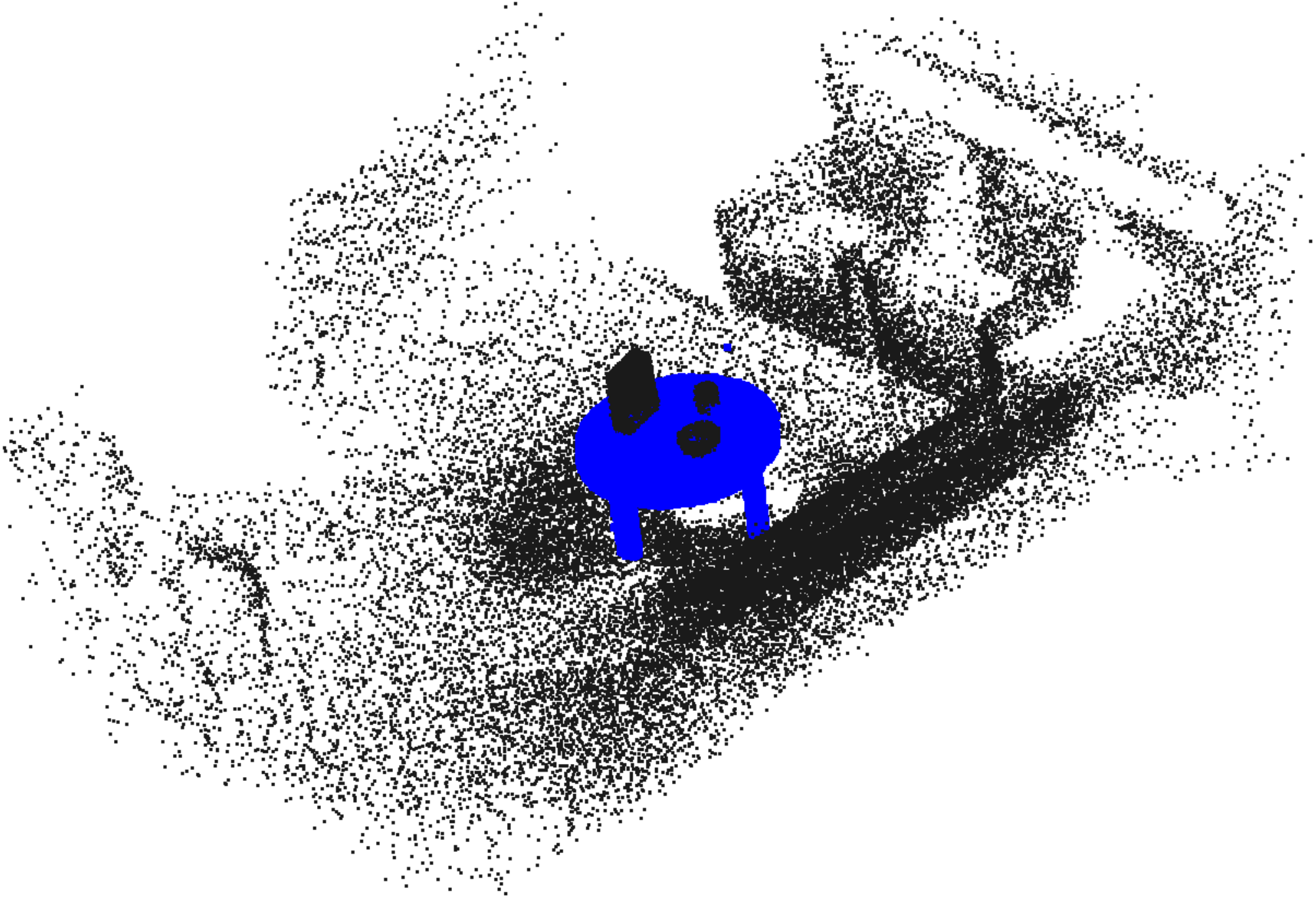} \\

\textit{Scene-12}, $\mathit{s}=1$ & ICOS  & RANSAC(1000) & RANSAC(1min) \\

\includegraphics[width=0.20\linewidth]{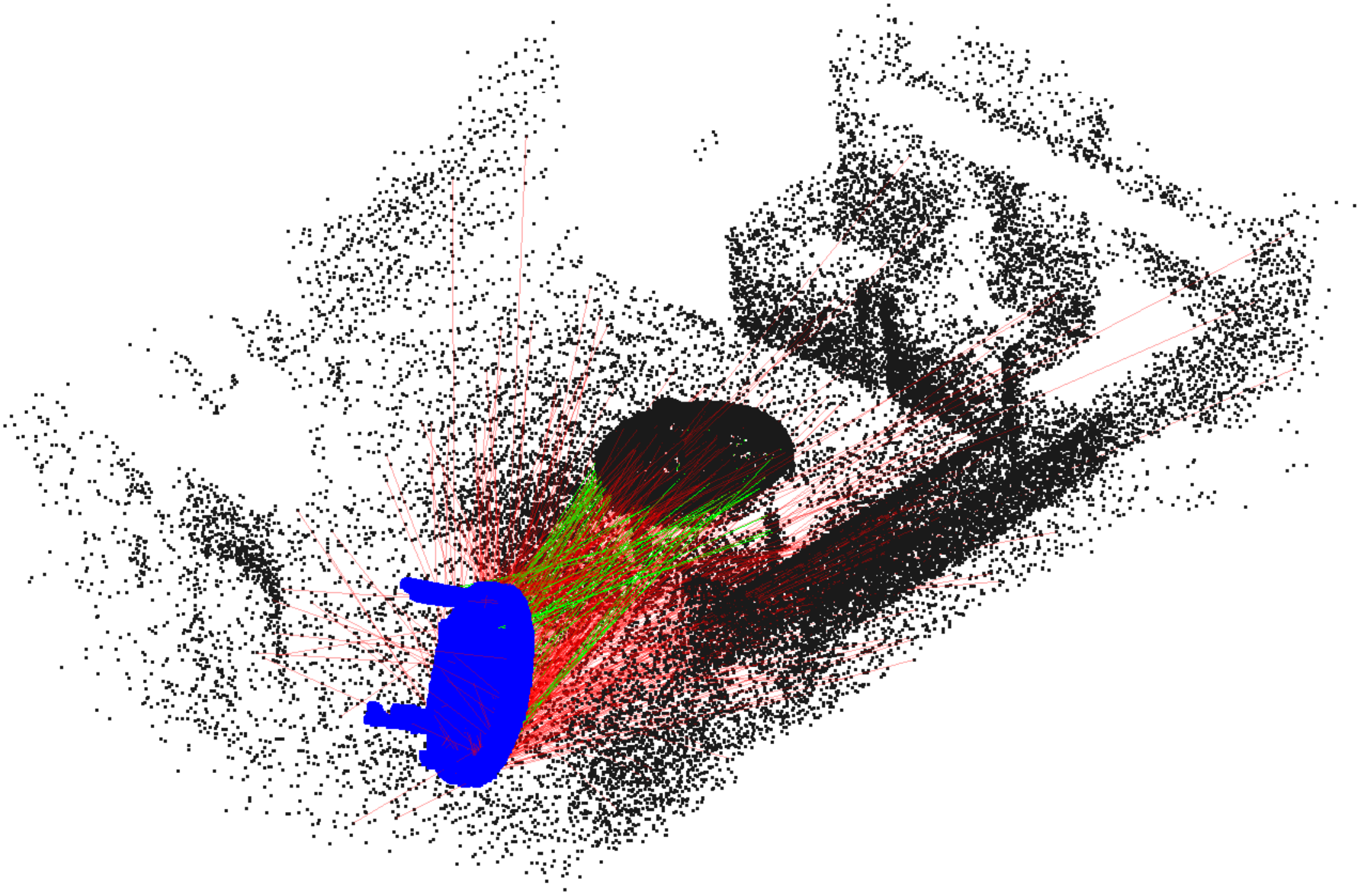}\,
&\includegraphics[width=0.20\linewidth]{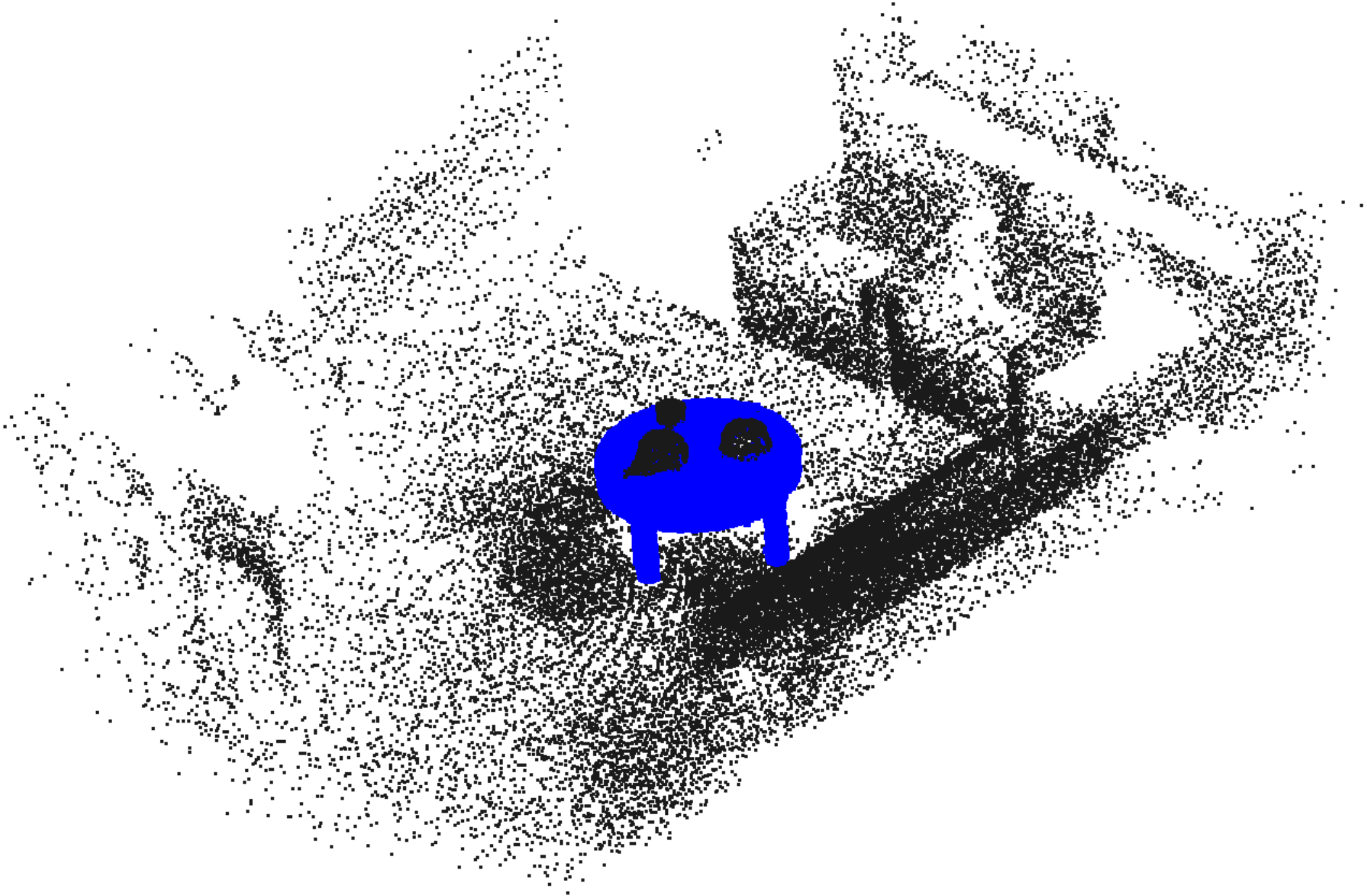}
&\includegraphics[width=0.20\linewidth]{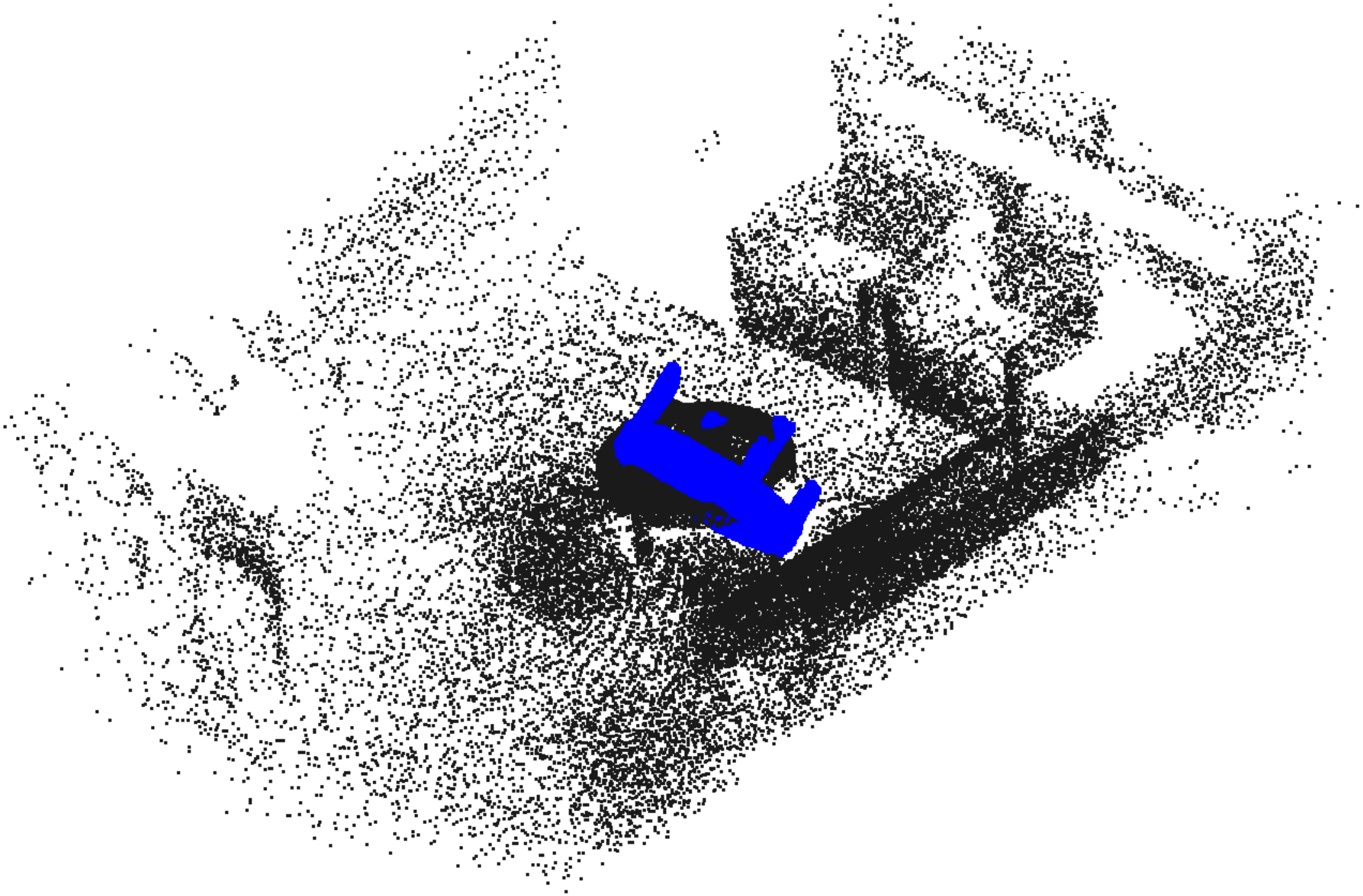}
&\includegraphics[width=0.20\linewidth]{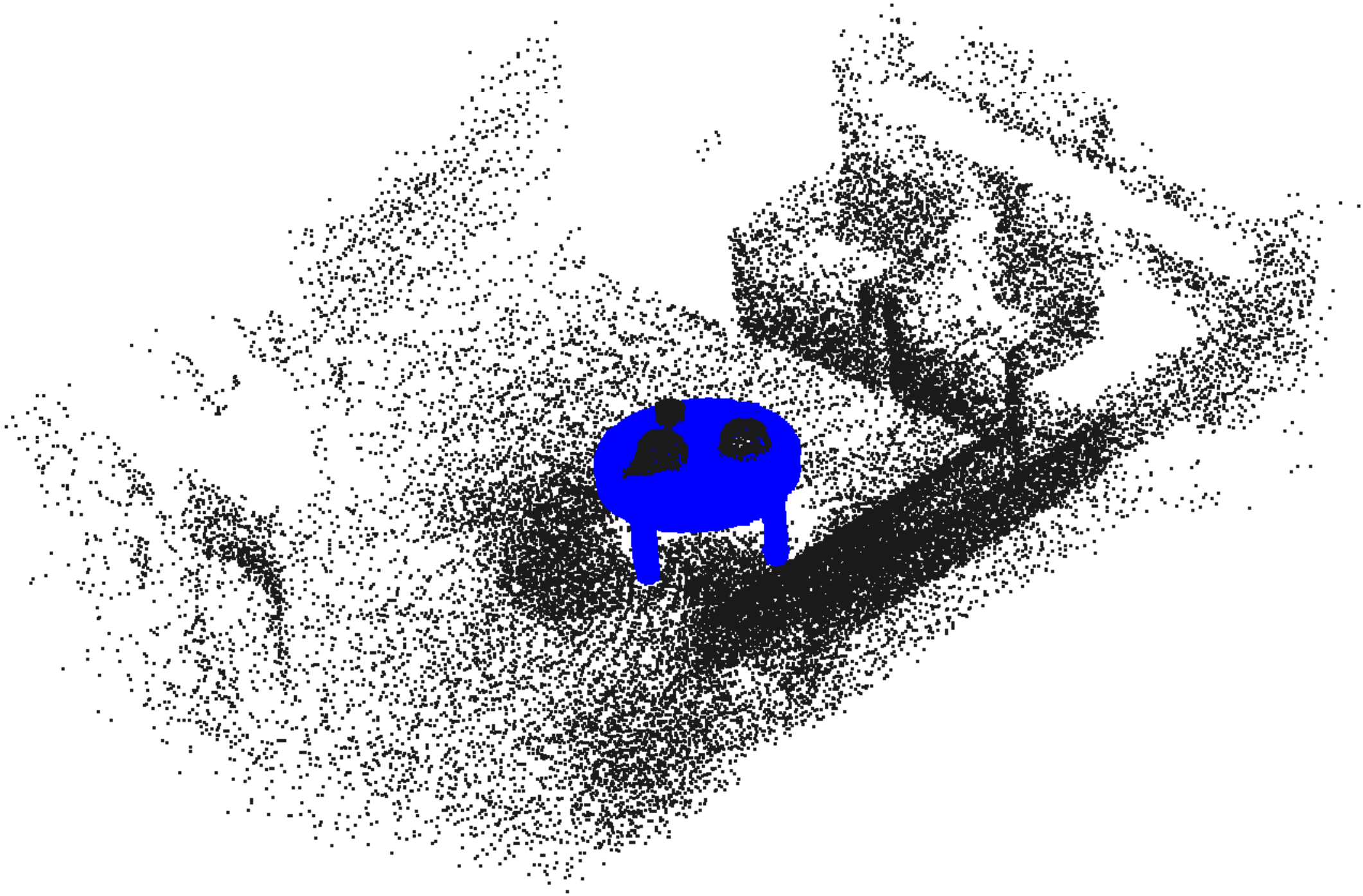} \\

\textit{Scene-01}, $\mathit{s}\in(1,5)$ & ICOS  & RANSAC(1000) & RANSAC(1min) \\

\includegraphics[width=0.20\linewidth]{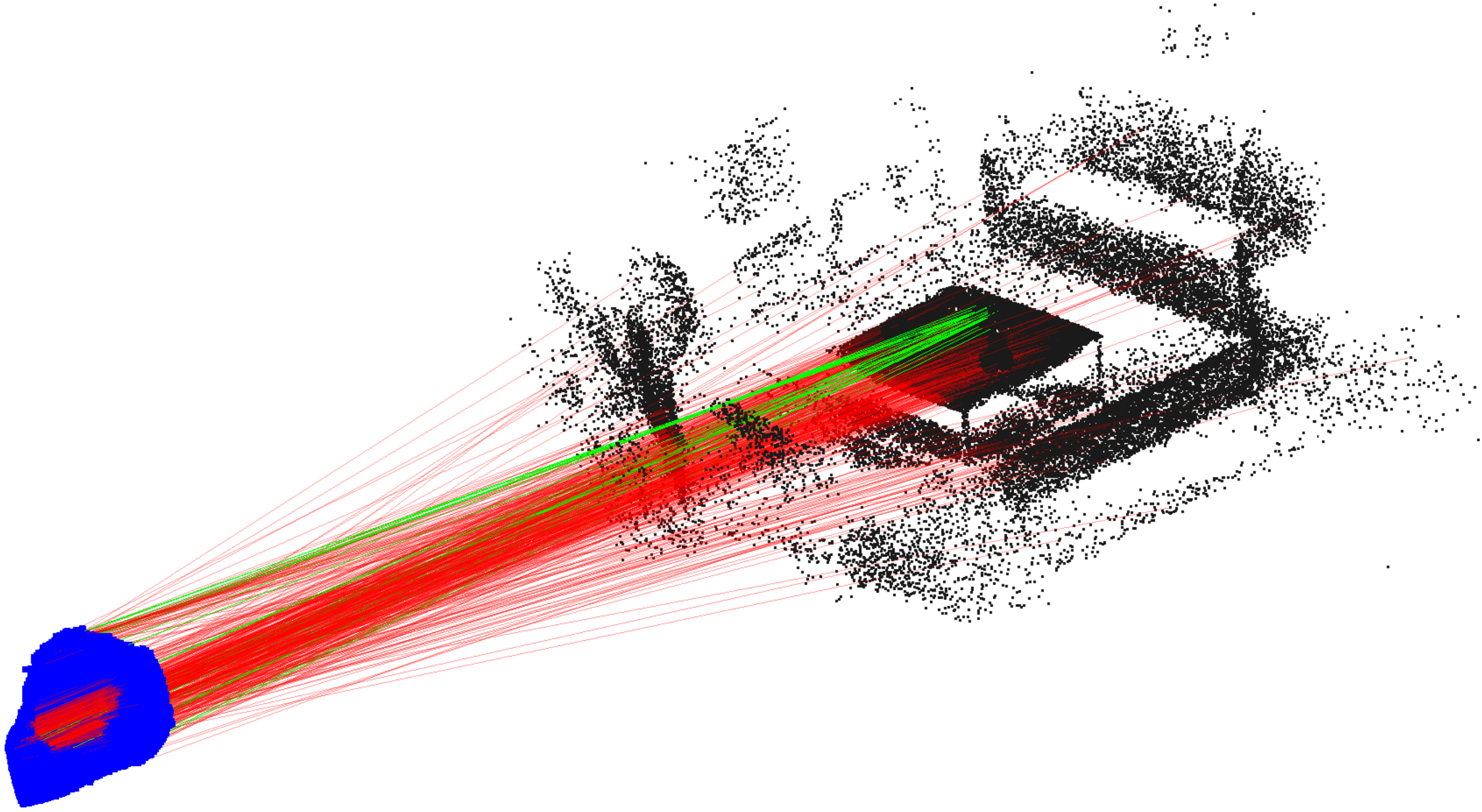}\,
&\includegraphics[width=0.20\linewidth]{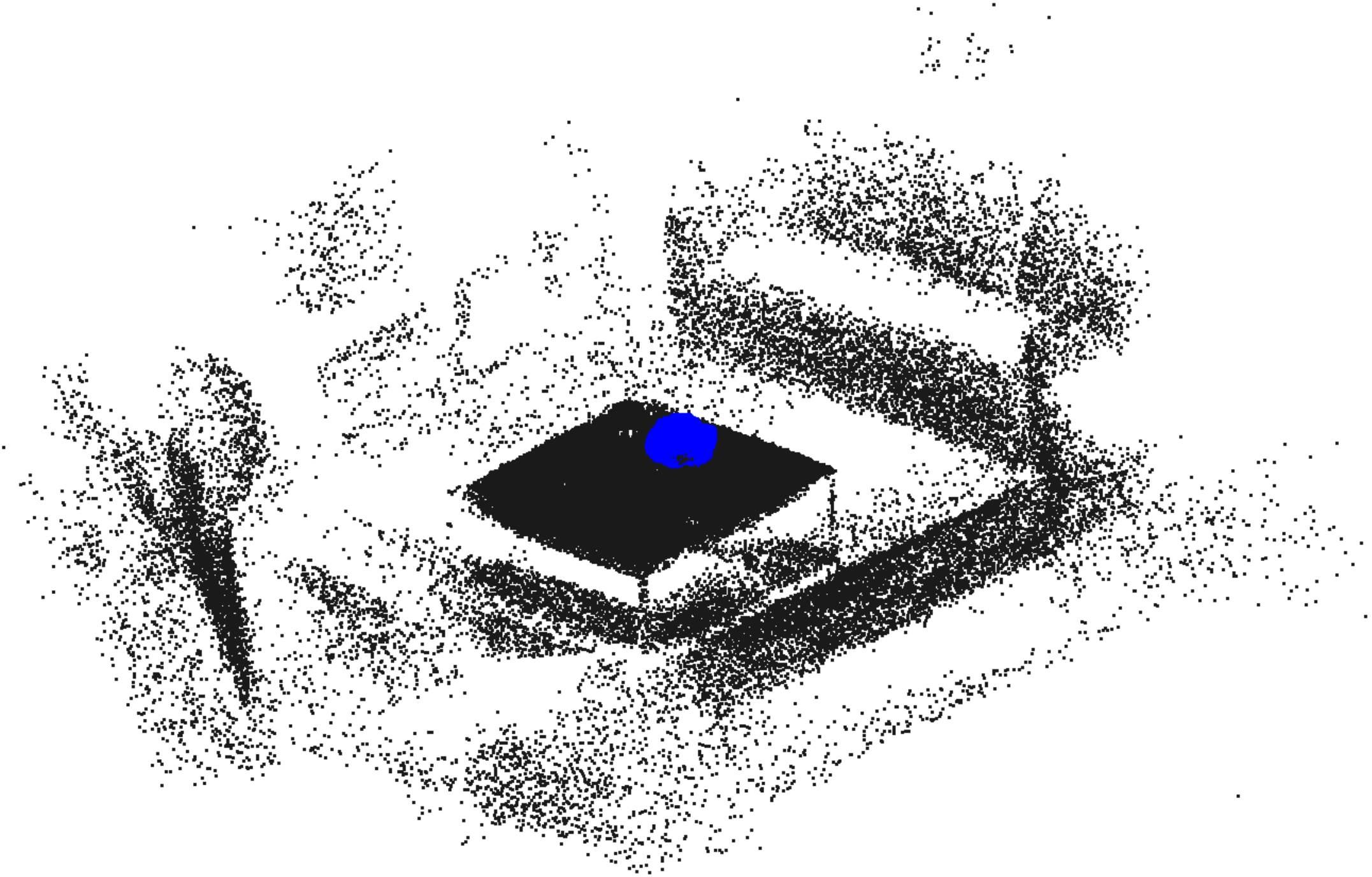}
&\includegraphics[width=0.20\linewidth]{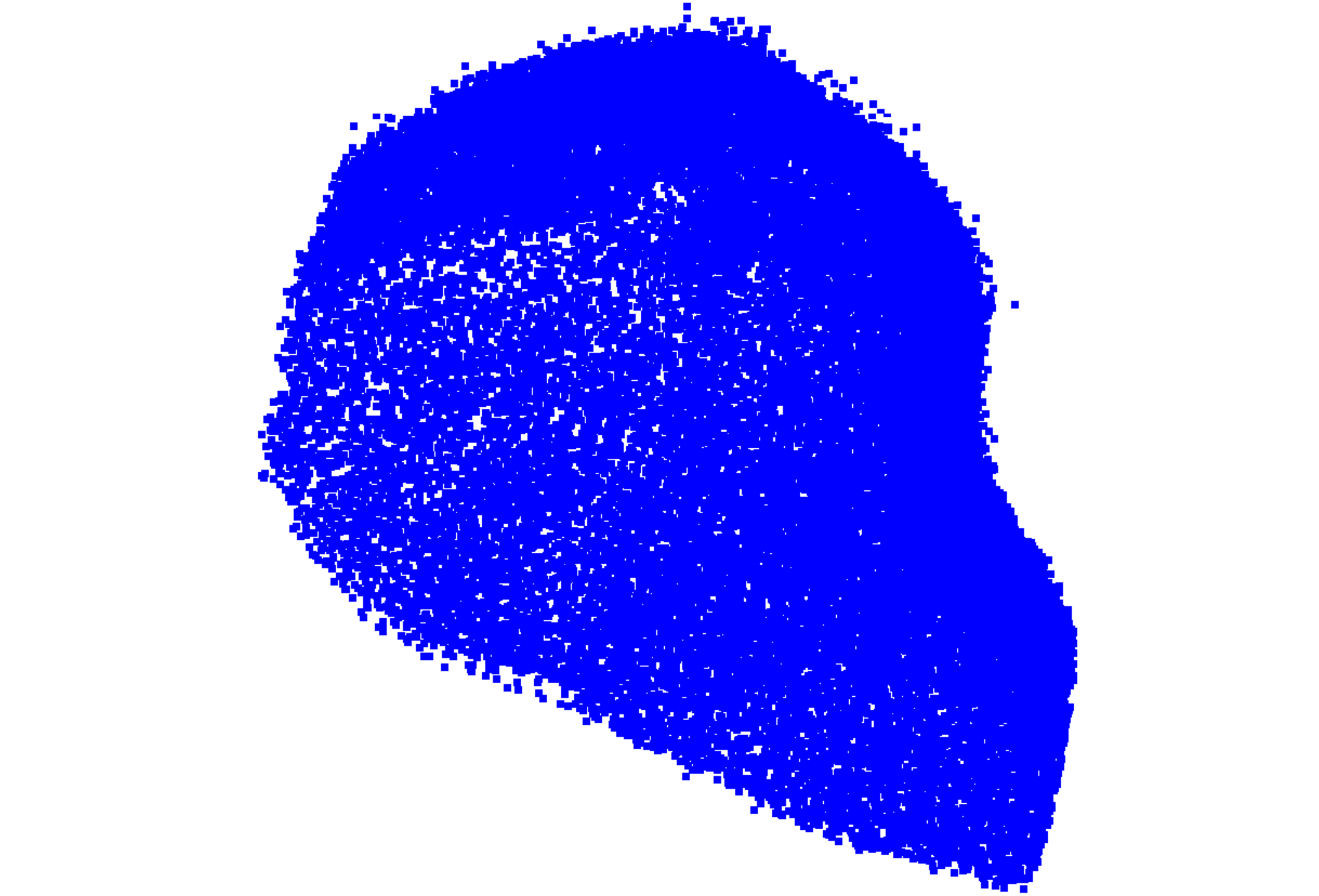}
&\includegraphics[width=0.20\linewidth]{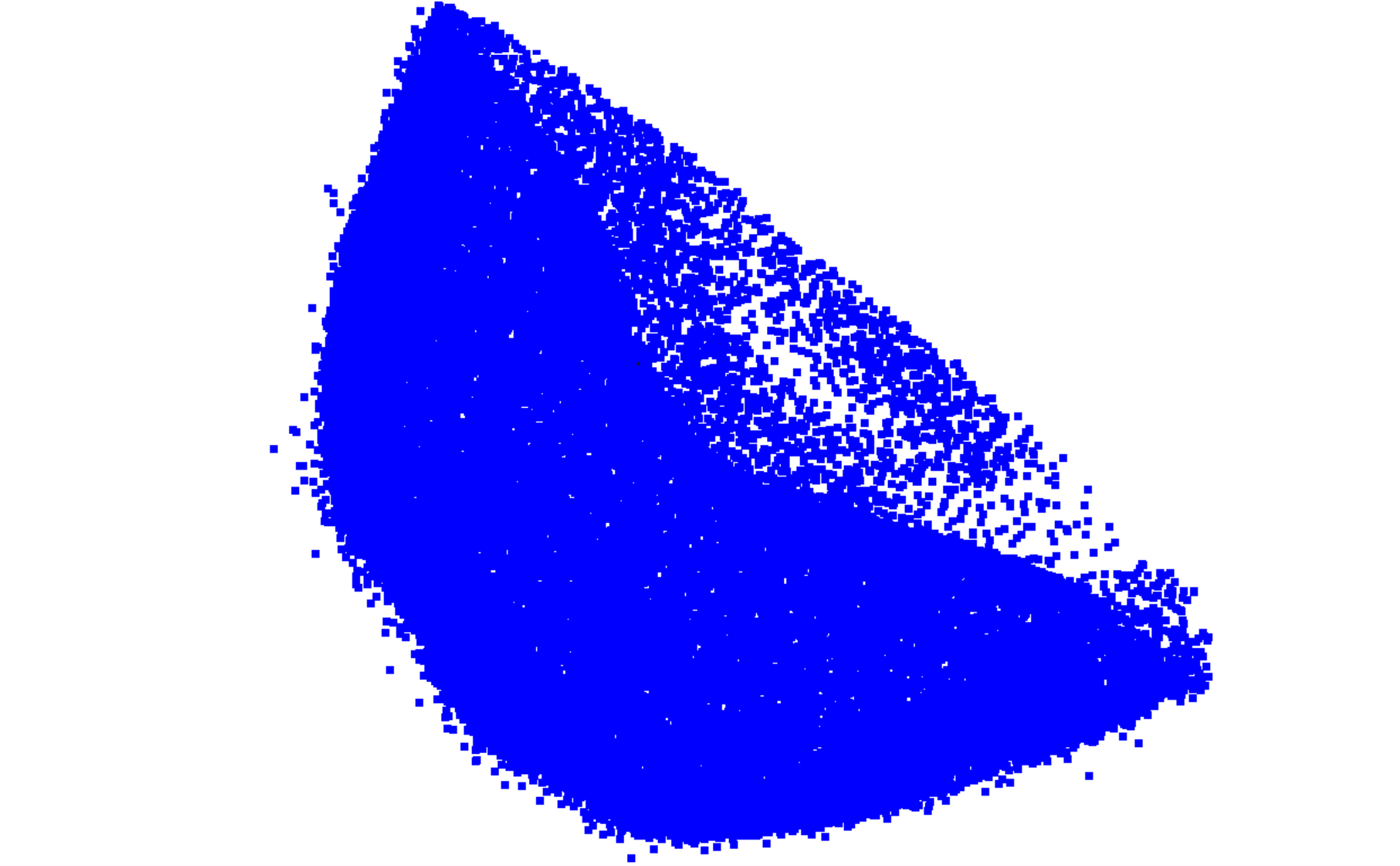} \\

\textit{Scene-03}, $\mathit{s}\in(1,5)$ & ICOS  & RANSAC(1000) & RANSAC(1min) \\

\includegraphics[width=0.20\linewidth]{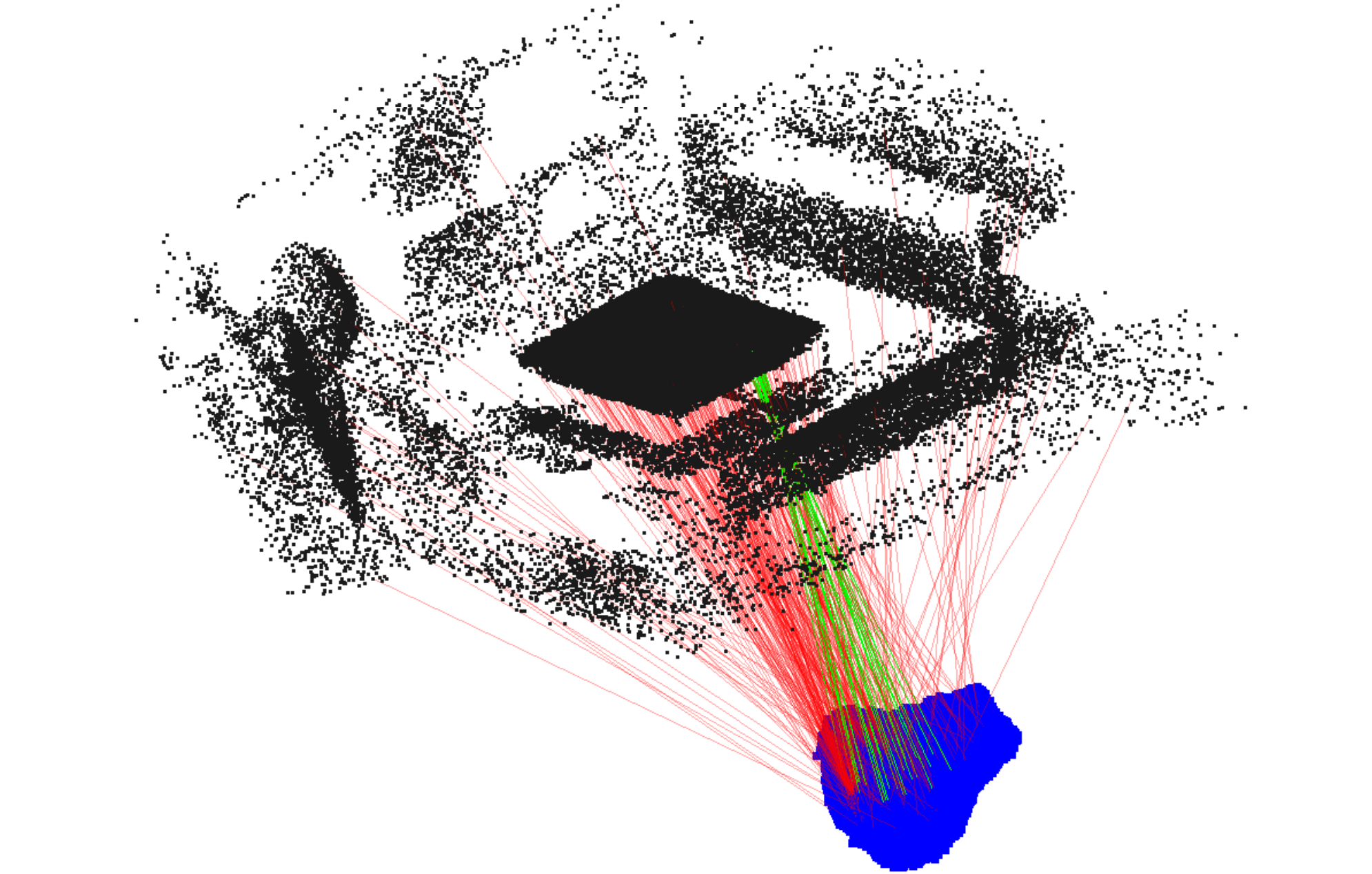}\,
&\includegraphics[width=0.20\linewidth]{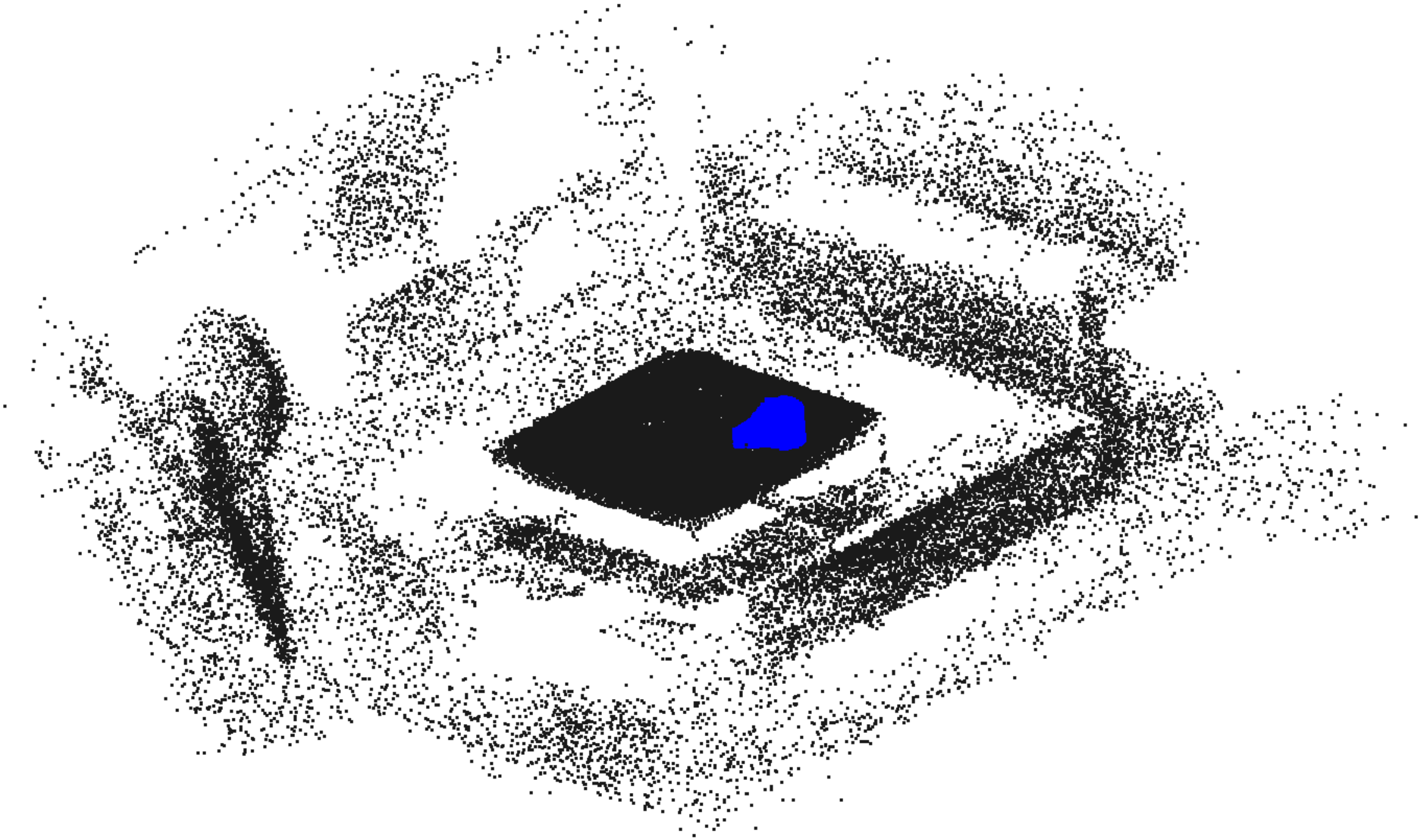}
&\includegraphics[width=0.20\linewidth]{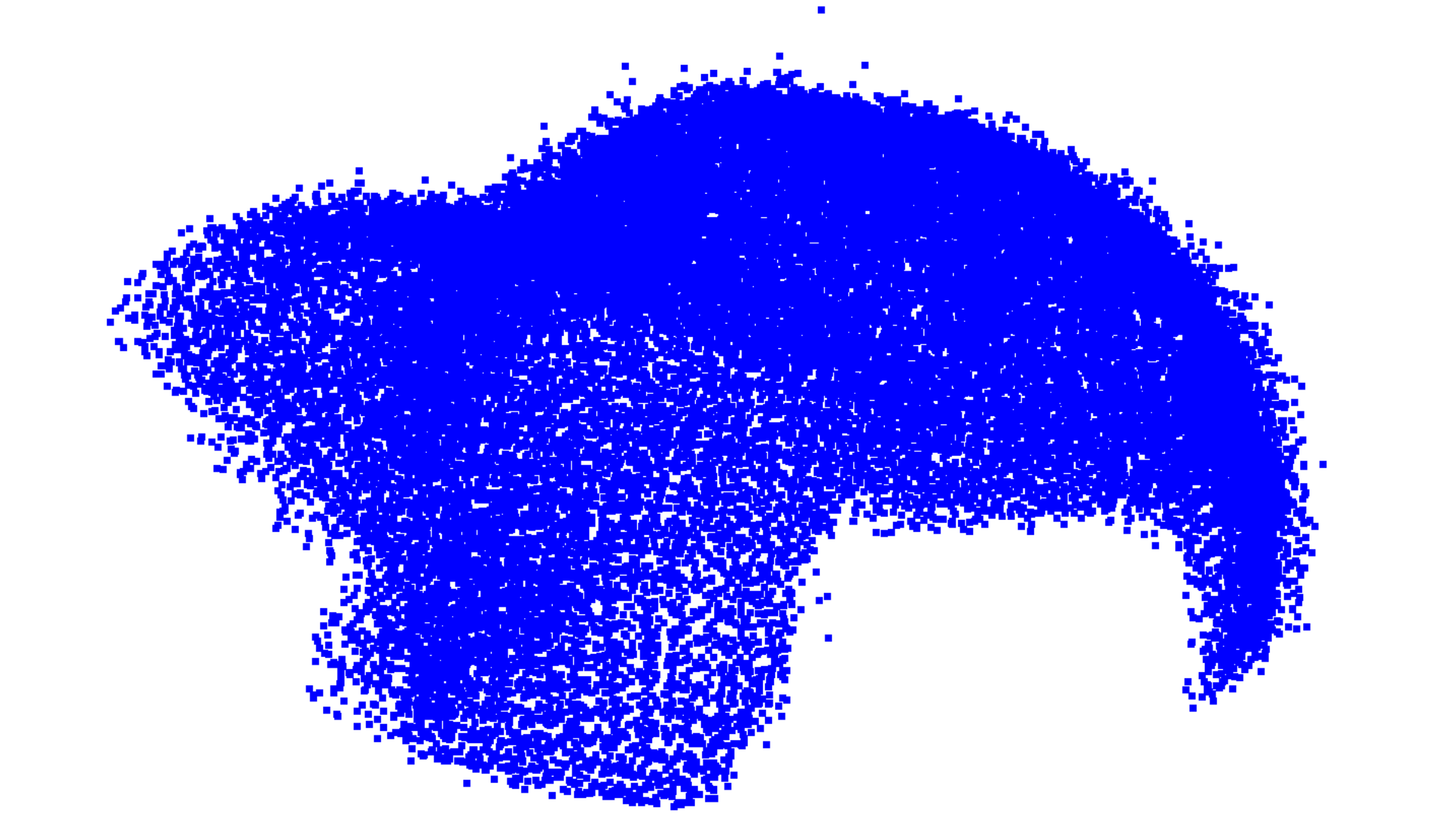}
&\includegraphics[width=0.20\linewidth]{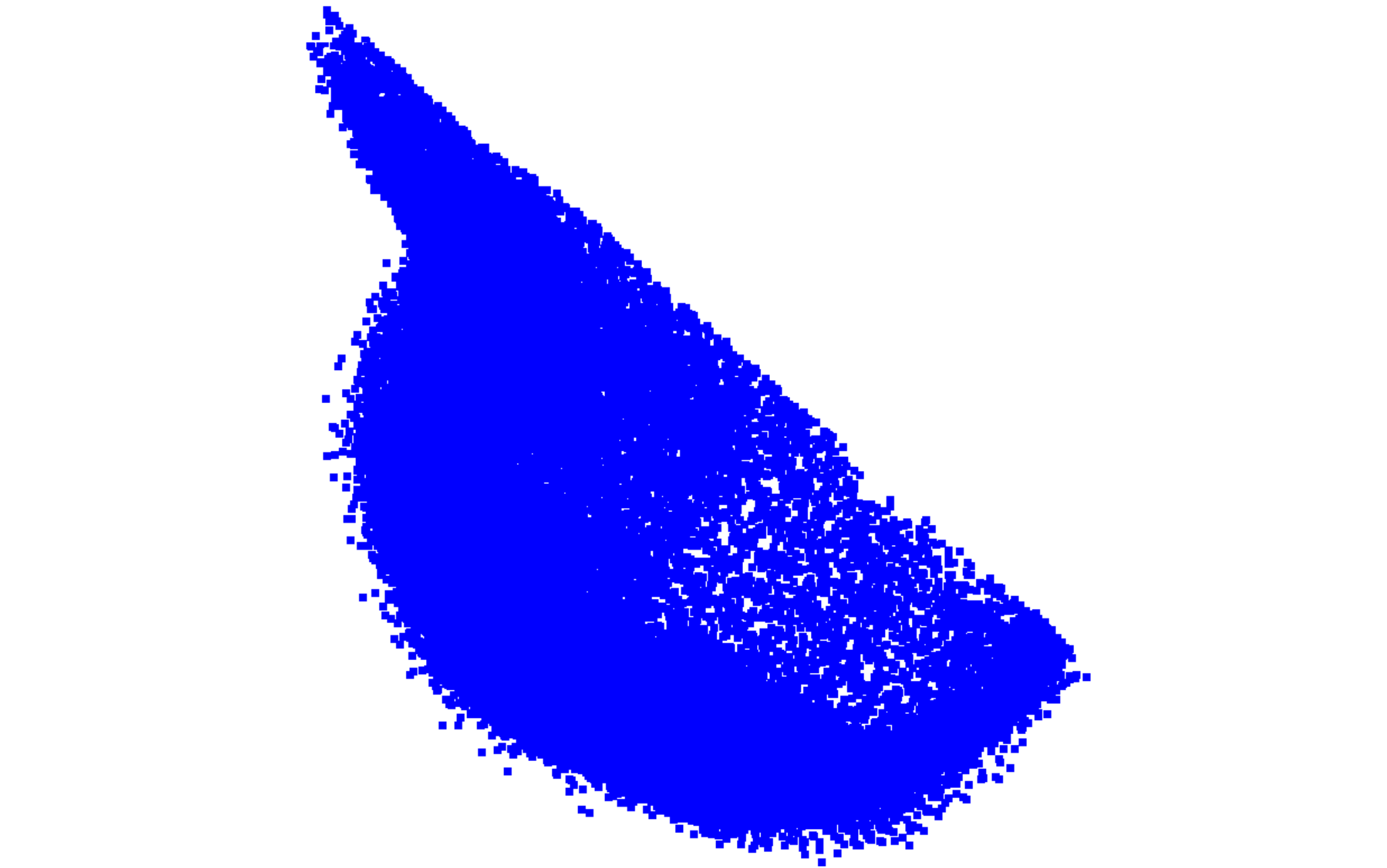} \\

\textit{Scene-09}, $\mathit{s}\in(1,5)$ & ICOS  & RANSAC(1000) & RANSAC(1min) \\

\includegraphics[width=0.20\linewidth]{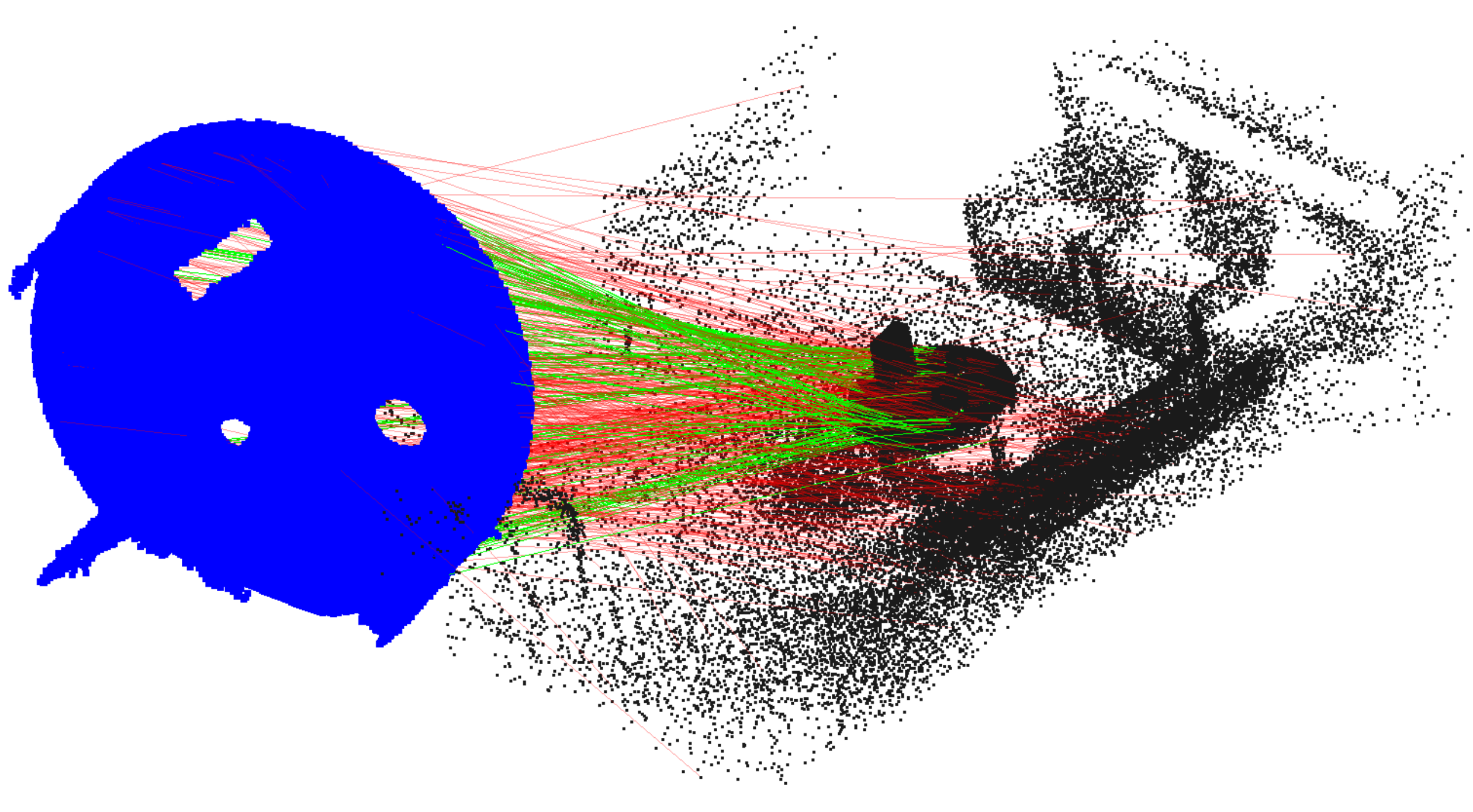}\,
&\includegraphics[width=0.20\linewidth]{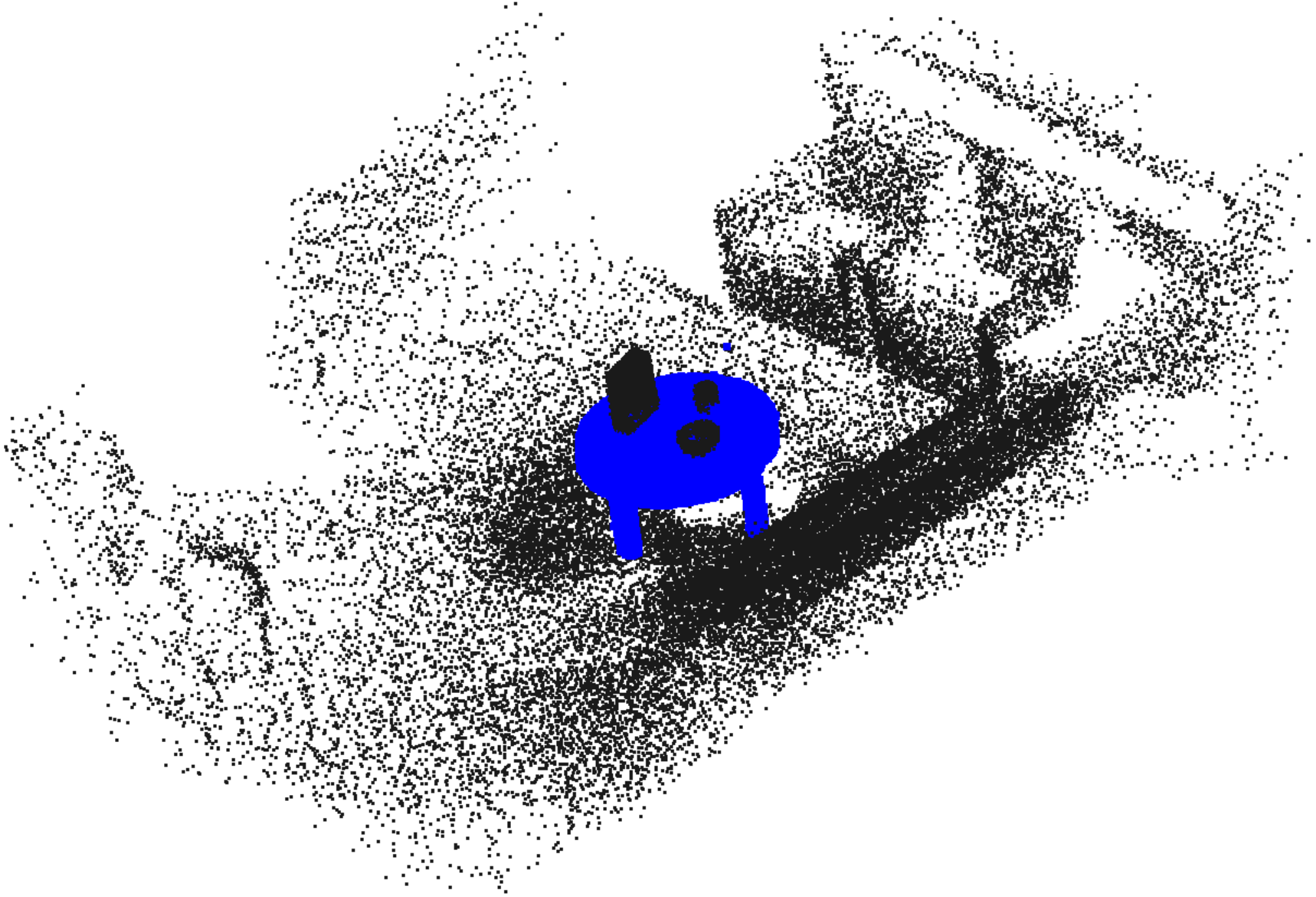}
&\includegraphics[width=0.20\linewidth]{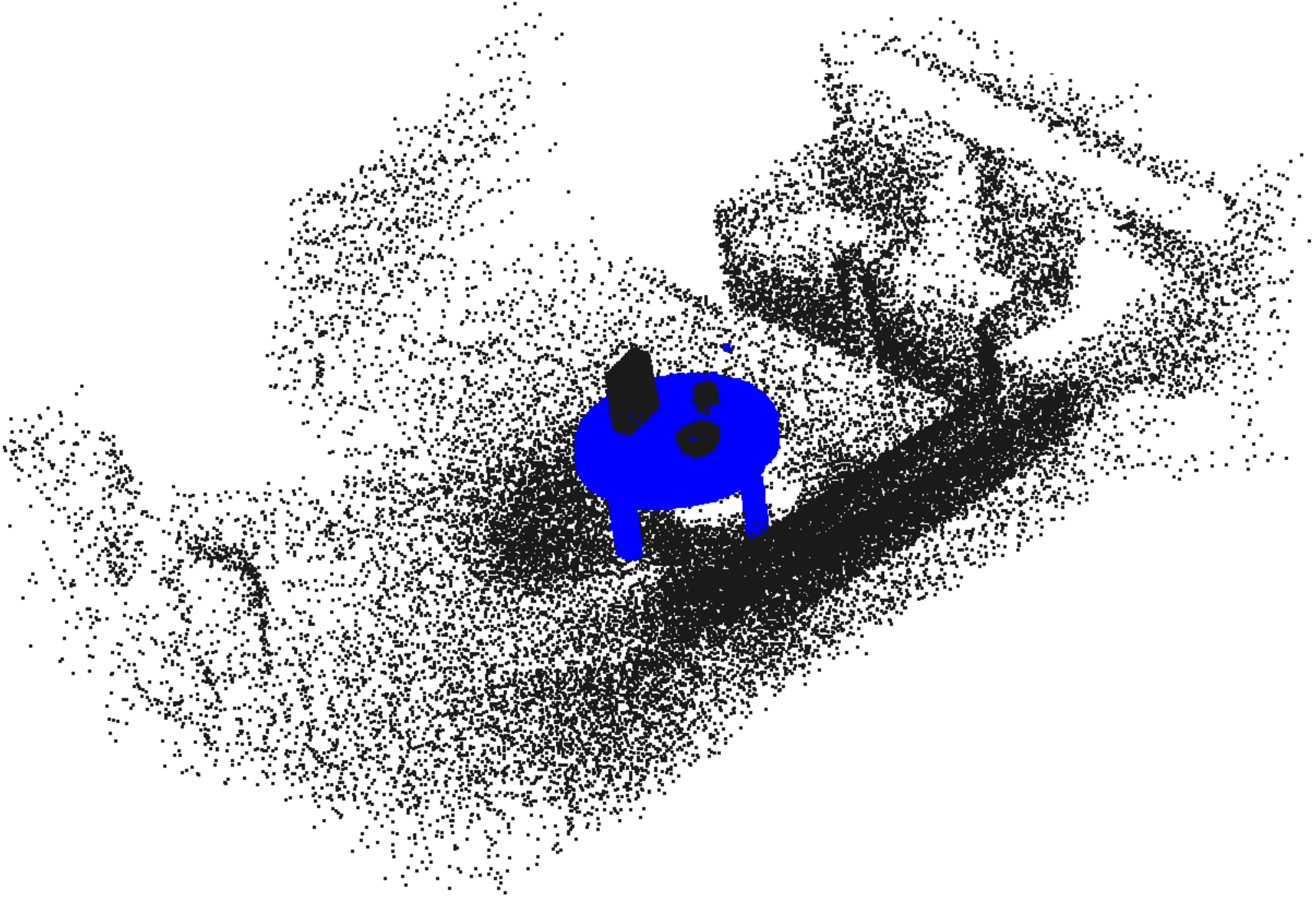}
&\includegraphics[width=0.20\linewidth]{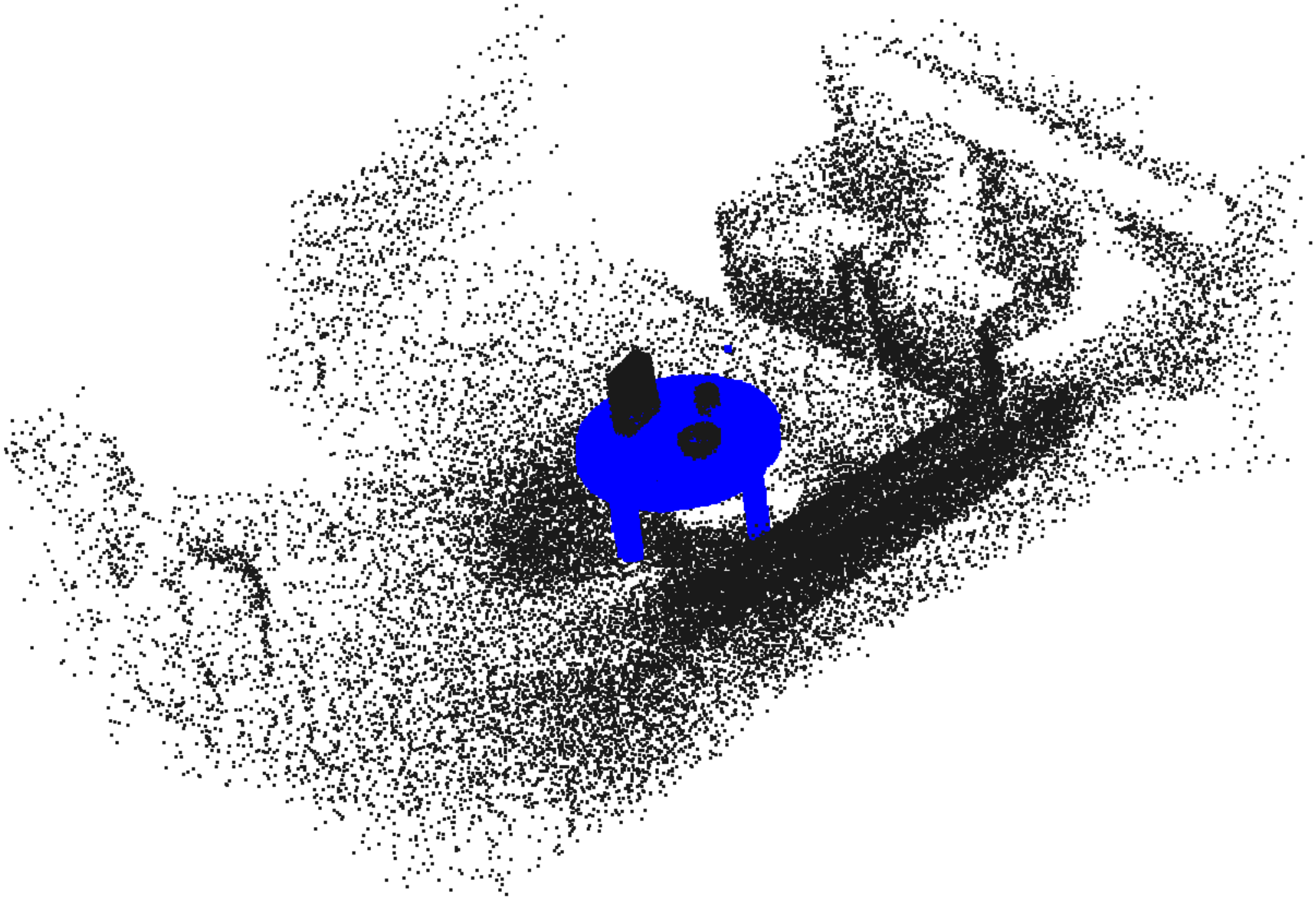} \\

\textit{Scene-12}, $\mathit{s}\in(1,5)$ & ICOS  & RANSAC(1000) & RANSAC(1min) \\

\includegraphics[width=0.20\linewidth]{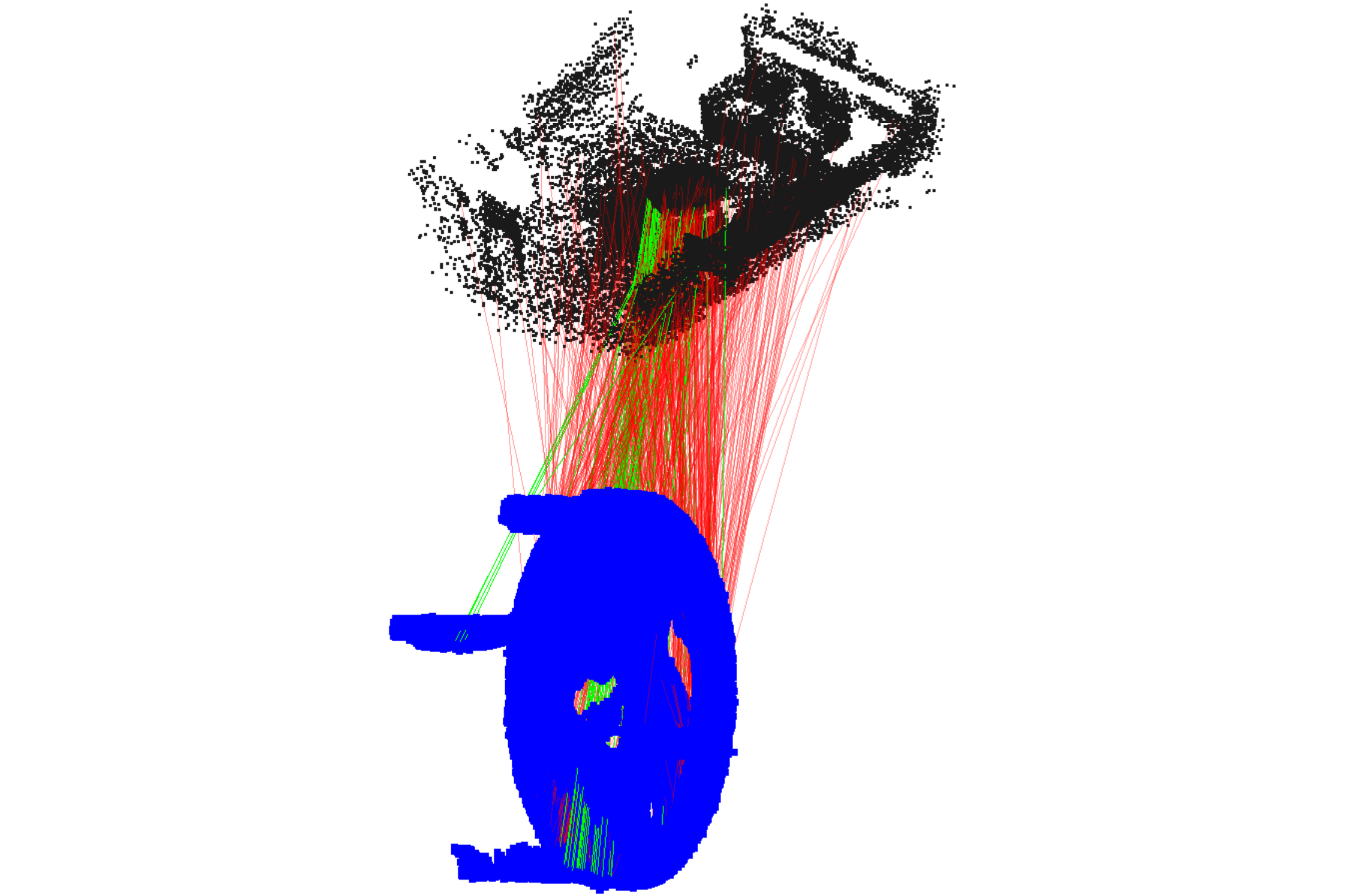}\,
&\includegraphics[width=0.20\linewidth]{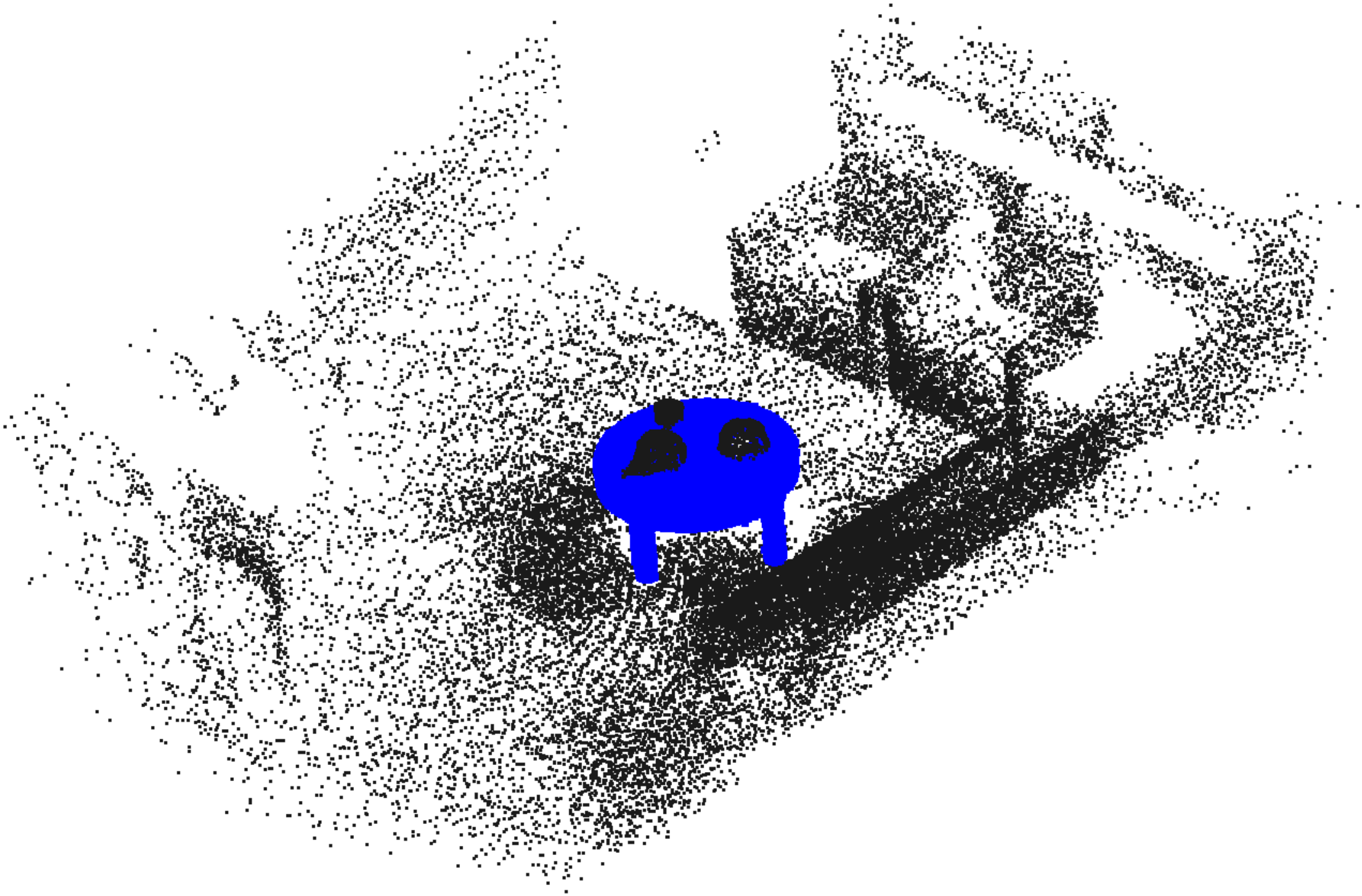}
&\includegraphics[width=0.20\linewidth]{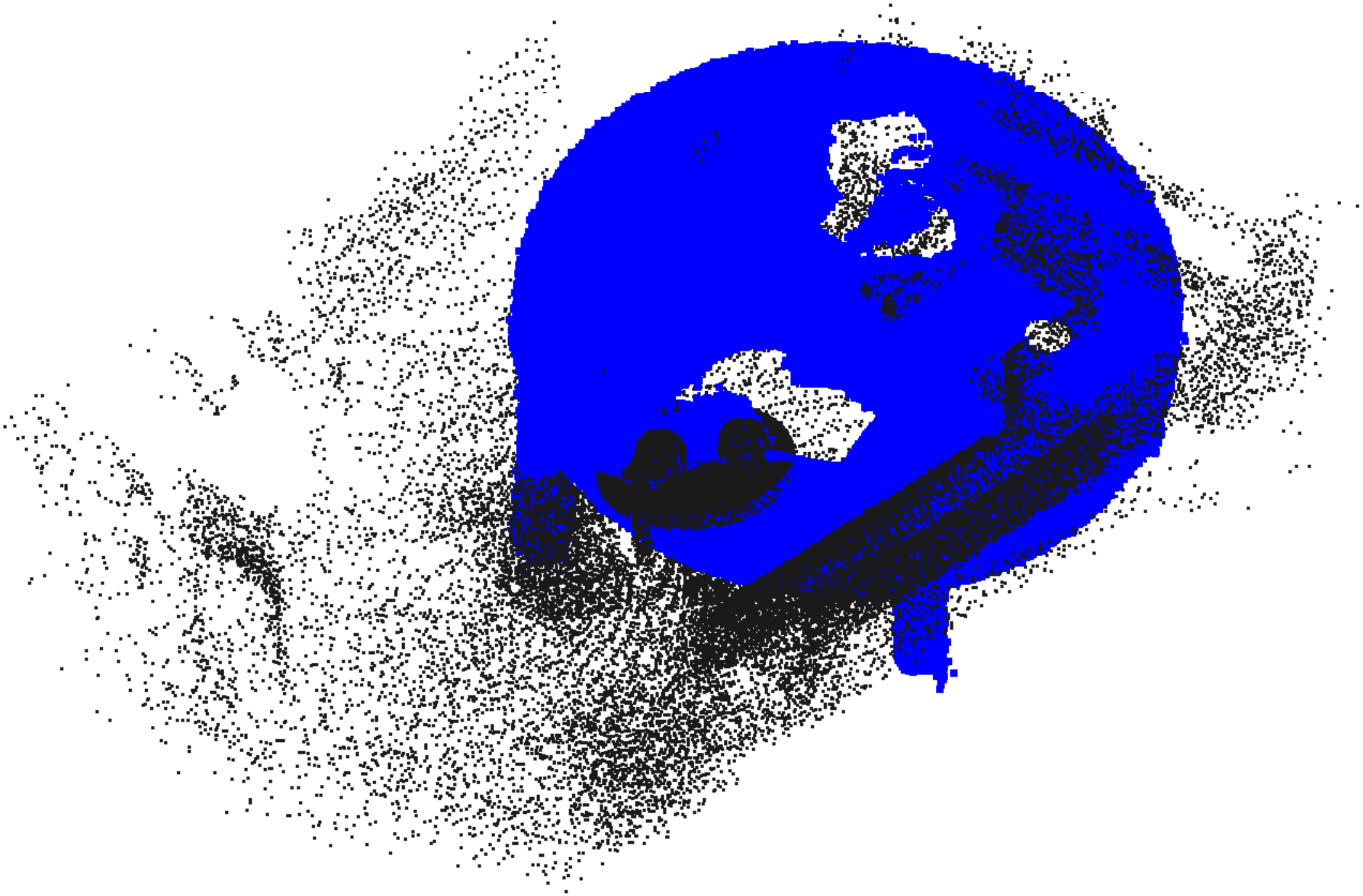}
&\includegraphics[width=0.20\linewidth]{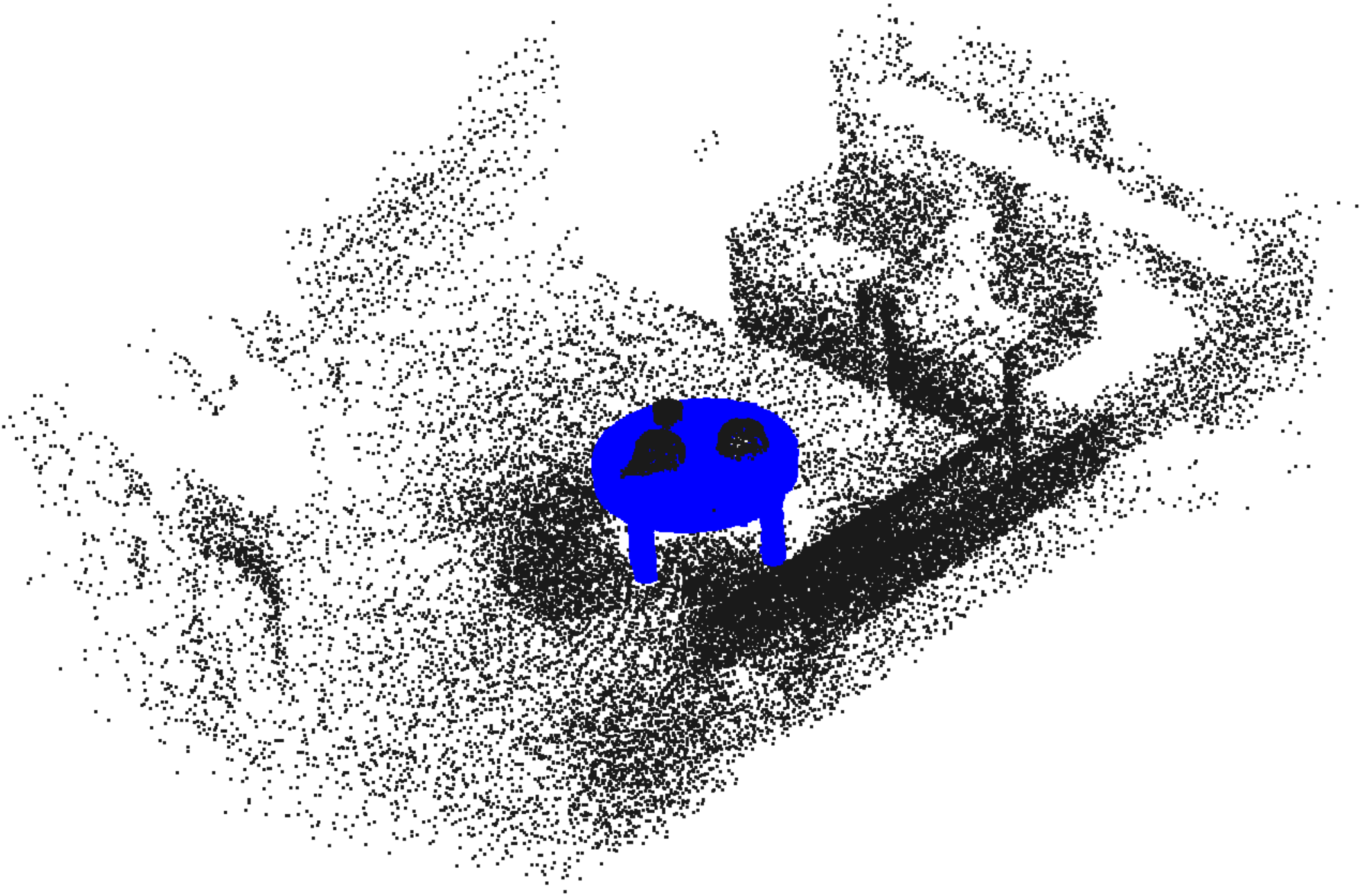} \\

\end{tabular}

\vspace{-1mm}

\centering
\caption{Qualitative results of 3D object localization (\textit{cap, table}) using ICOS, RANSAC(1000) and RANSAC(1min). From left to right, we show the correspondences matched by FPFH (inliers are in green lines and outliers are in red lines), and the registration (reprojection) results with ICOS, RANSAC(1000), and RANSAC(1min), respectively. Note that in all scenes the outlier ratios are over 80\%, and our ICOS can always yield very promising registration results.}
\label{Object-Lo2}
\end{figure*}

\clearpage

\begin{table}[h]
\setlength{\tabcolsep}{1mm}
\caption{Quantitative Object Localization Results on Cereal Box}\label{quantitative}
\centering
\vspace{-2mm}
\begin{tabular}{|c|ccc|}

\hline
Scene No.  & ICOS & RANSAC(1000) & RANSAC(1min) \\
\hline 

\textit{Scene-02}, $\mathit{s}=1$
&  $E_{\boldsymbol{R}}=\mathbf{0.196}$ 
& $E_{\boldsymbol{R}}=154.820$ 
&  $E_{\boldsymbol{R}}=0.727$\\
$N=597$, 93.80\%  
& $E_{\boldsymbol{t}}=\mathbf{0.004}$ 
& $E_{\boldsymbol{t}}=1.120$
& $E_{\boldsymbol{t}}=0.013$ \\
\quad  
& $t=\mathbf{0.071}$ 
& $t=1.559$
&  $t=34.308$\\
\hline

\textit{Scene-05}, $\mathit{s}=1$
&  $E_{\boldsymbol{R}}=\mathbf{0.202}$ 
&  $E_{\boldsymbol{R}}=61.278$
& $E_{\boldsymbol{R}}=0.991$ \\
$N=606$, 90.92\%  
& $E_{\boldsymbol{t}}=\mathbf{0.003}$ 
&  $E_{\boldsymbol{t}}=1.092$
& $E_{\boldsymbol{t}}=0.024$\\
\quad  
& $t=\mathbf{0.057}$ 
& $t=1.578$
&  $t=11.110$\\
\hline

\textit{Scene-07}, $\mathit{s}=1$
&  $E_{\boldsymbol{R}}=\mathbf{0.188}$ 
&  $E_{\boldsymbol{R}}=0.449$
& $E_{\boldsymbol{R}}=1.736$ \\
$N=249$, 81.53\%  
& $E_{\boldsymbol{t}}=\mathbf{0.005}$ 
&  $E_{\boldsymbol{t}}=0.013$
& $E_{\boldsymbol{t}}=0.056$\\
\quad  
& $t=\mathbf{0.055}$ 
& $t=0.579$
&  $t=0.582$\\
\hline

\textit{Scene-09}, $\mathit{s}=1$
&  $E_{\boldsymbol{R}}=\mathbf{0.144}$ 
&  $E_{\boldsymbol{R}}=3.319$
& $E_{\boldsymbol{R}}=0.546$ \\
$N=284$, 90.16\%  
& $E_{\boldsymbol{t}}=\mathbf{0.003}$ 
&  $E_{\boldsymbol{t}}=0.048$
& $E_{\boldsymbol{t}}=0.012$\\
\quad  
& $t=\mathbf{0.066}$ 
& $t={1.663}$
&  $t=21.465$\\
\hline

\quad
&  $E_{\boldsymbol{R}}=\mathbf{0.046}$ 
& $E_{\boldsymbol{R}}=63.040$ 
& $E_{\boldsymbol{R}}=0.257$ \\
\textit{Scene-02}, $\mathit{s}\in(1,5)$  
& $E_{\boldsymbol{t}}=\mathbf{0.004}$ 
& $E_{\boldsymbol{t}}=5.389$
&  $E_{\boldsymbol{t}}=0.233$\\
$N=618$, 94.82\%  
& $E_{\mathit{s}}=\mathbf{0.001}$ 
& $E_{\mathit{s}}=4.754$
&  $E_{\mathit{s}}=0.004$\\
\quad  
& $t=\mathbf{0.039}$ 
& $t=1.621$
& $t=59.095$ \\
\hline

\quad
&  $E_{\boldsymbol{R}}=\mathbf{0.030}$ 
& $E_{\boldsymbol{R}}=152.878$ 
& $E_{\boldsymbol{R}}=48.347$ \\
\textit{Scene-05}, $\mathit{s}\in(1,5)$  
& $E_{\boldsymbol{t}}=\mathbf{0.003}$ 
& $E_{\boldsymbol{t}}=5.094$
&  $E_{\boldsymbol{t}}=5.093$\\
 $N=555$, 95.14\% 
& $E_{\mathit{s}}=\mathbf{0.001}$ 
& $E_{\mathit{s}}=3.669$
&  $E_{\mathit{s}}=3.667$\\
\quad  
& $t=\mathbf{0.185}$ 
& $t=1.453$
& $t=60.001$ \\
\hline

\quad  
&  $E_{\boldsymbol{R}}=\mathbf{0.057}$ 
& $E_{\boldsymbol{R}}=123.266$ 
& $E_{\boldsymbol{R}}=122.522$ \\
\textit{Scene-07}, $\mathit{s}\in(1,5)$  
& $E_{\boldsymbol{t}}=\mathbf{0.004}$ 
& $E_{\boldsymbol{t}}=8.723$
&  $E_{\boldsymbol{t}}=8.721$\\
$N=185$, 92.43\%
& $E_{\mathit{s}}=\mathbf{0.001}$ 
& $E_{\mathit{s}}=4.475$
&  $E_{\mathit{s}}=4.475$\\
\quad  
& $t=\mathbf{0.066}$ 
& $t=0.535$
& $t=6.393$ \\
\hline

\quad  
&  $E_{\boldsymbol{R}}=\mathbf{0.067}$ 
& $E_{\boldsymbol{R}}=3.633$ 
& $E_{\boldsymbol{R}}=143.444$ \\
\textit{Scene-09}, $\mathit{s}\in(1,5)$  
& $E_{\boldsymbol{t}}=\mathbf{0.005}$ 
& $E_{\boldsymbol{t}}=0.241$
&  $E_{\boldsymbol{t}}=4.081$\\
$N=548$, 95.26\%
& $E_{\mathit{s}}=\mathbf{0.002}$ 
& $E_{\mathit{s}}=0.004$
&  $E_{\mathit{s}}=3.247$\\
\quad  
& $t=\mathbf{0.142}$ 
& $t=1.429$
& $t=60.001$ \\
\hline

\end{tabular}

\centering

\end{table}
\vspace{0mm}

\begin{table}[h]
\setlength{\tabcolsep}{1mm}
\caption{Quantitative Object Localization Results on Cap and Table}\label{quantitative2}
\centering
\vspace{-1mm}
\begin{tabular}{|c|ccc|}

\hline
Scene No.  & ICOS & RANSAC(1000) & RANSAC(1min) \\
\hline 

\textit{Scene-01}, $\mathit{s}=1$
&  $E_{\boldsymbol{R}}=\mathbf{0.151}$ 
& $E_{\boldsymbol{R}}=65.867$ 
&  $E_{\boldsymbol{R}}=0.662$\\
$N=341$, 90.91\%  
& $E_{\boldsymbol{t}}=\mathbf{0.005}$ 
& $E_{\boldsymbol{t}}=1.309$
& $E_{\boldsymbol{t}}=0.014$ \\
\quad  
& $t=\mathbf{0.060}$ 
& $t=0.903$
&  $t=6.245$\\
\hline

\textit{Scene-03}, $\mathit{s}=1$
&  $E_{\boldsymbol{R}}=\mathbf{0.337}$ 
&  $E_{\boldsymbol{R}}=128.182$
& $E_{\boldsymbol{R}}=0.924$ \\
$N=316$, 94.62\%  
& $E_{\boldsymbol{t}}=\mathbf{0.008}$ 
&  $E_{\boldsymbol{t}}={1.278}$
& $E_{\boldsymbol{t}}=0.2018$\\
\quad  
& $t=\mathbf{0.198}$ 
& $t=0.851$
&  $t=27.688$\\
\hline

\textit{Scene-09}, $\mathit{s}=1$
&  $E_{\boldsymbol{R}}=\mathbf{0.058}$ 
&  $E_{\boldsymbol{R}}=0.910$
& $E_{\boldsymbol{R}}=0.397$ \\
$N=579$, 86.18\%  
& $E_{\boldsymbol{t}}=\mathbf{0.001}$ 
&  $E_{\boldsymbol{t}}=0.022$
& $E_{\boldsymbol{t}}=0.010$\\
\quad  
& $t=\mathbf{0.043}$ 
& $t=1.514$
&  $t=3.013$\\
\hline

\textit{Scene-12}, $\mathit{s}=1$
&  $E_{\boldsymbol{R}}=\mathbf{0.068}$ 
&  $E_{\boldsymbol{R}}=165.928$
& $E_{\boldsymbol{R}}=0.191$ \\
$N=639$, 89.98\%  
& $E_{\boldsymbol{t}}=\mathbf{0.001}$ 
&  $E_{\boldsymbol{t}}=1.747$
& $E_{\boldsymbol{t}}=0.004$\\
\quad  
& $t=\mathbf{0.047}$ 
& $t=1.645$
&  $t=8.537$\\
\hline

\quad
&  $E_{\boldsymbol{R}}=\mathbf{0.037}$ 
& $E_{\boldsymbol{R}}=101.892$ 
& $E_{\boldsymbol{R}}=151.913$ \\
\textit{Scene-01}, $\mathit{s}\in(1,5)$  
& $E_{\boldsymbol{t}}=\mathbf{0.005}$ 
& $E_{\boldsymbol{t}}=5.779$
&  $E_{\boldsymbol{t}}=5.782$\\
$N=579$, 94.99\%  
& $E_{\mathit{s}}=\mathbf{0.002}$ 
& $E_{\mathit{s}}=3.428$
&  $E_{\mathit{s}}=3.429$\\
\quad  
& $t=\mathbf{0.069}$ 
& $t={1.508}$
& $t=60.001$ \\
\hline

\quad
&  $E_{\boldsymbol{R}}=\mathbf{0.139}$ 
& $E_{\boldsymbol{R}}=60.219$ 
& $E_{\boldsymbol{R}}=176.511$ \\
\textit{Scene-03}, $\mathit{s}\in(1,5)$  
& $E_{\boldsymbol{t}}=\mathbf{0.010}$ 
& $E_{\boldsymbol{t}}=4.718$
&  $E_{\boldsymbol{t}}=4.720$\\
 $N=345$, 91.88\% 
& $E_{\mathit{s}}=\mathbf{0.009}$ 
& $E_{\mathit{s}}=3.552$
&  $E_{\mathit{s}}=3.553$\\
\quad  
& $t=\mathbf{0.060}$ 
& $t=0.936$
& $t=9.052$ \\
\hline

\quad  
&  $E_{\boldsymbol{R}}=\mathbf{0.045}$ 
& $E_{\boldsymbol{R}}=0.065$ 
& $E_{\boldsymbol{R}}=0.096$ \\
\textit{Scene-09}, $\mathit{s}\in(1,5)$  
& $E_{\boldsymbol{t}}=\mathbf{0.001}$ 
& $E_{\boldsymbol{t}}=0.004$
&  $E_{\boldsymbol{t}}=0.007$\\
$N=619$, 80.94\%
& $E_{\mathit{s}}=\mathbf{0.001}$ 
& $E_{\mathit{s}}=\mathbf{0.001}$
&  $E_{\mathit{s}}=0.002$\\
\quad  
& $t=\mathbf{0.041}$
& $t=1.264$
& $t=1.237$ \\
\hline

\quad  
&  $E_{\boldsymbol{R}}=\mathbf{0.008}$ 
& $E_{\boldsymbol{R}}=128.166$ 
& $E_{\boldsymbol{R}}=0.079$ \\
\textit{Scene-12}, $\mathit{s}\in(1,5)$  
& $E_{\boldsymbol{t}}=\mathbf{0.001}$ 
& $E_{\boldsymbol{t}}=7.151$
&  $E_{\boldsymbol{t}}=0.005$\\
$N=631$, 89.54\%
& $E_{\mathit{s}}=\mathbf{0.001}$ 
& $E_{\mathit{s}}=3.417$
&  $E_{\mathit{s}}=\mathbf{0.001}$\\
\quad  
& $t=\mathbf{0.052}$ 
& $t=1.613$
& $t=7.560$ \\
\hline

\end{tabular}

\centering

\end{table}

According to the results, ICOS is capable of estimating the object poses robustly and accurately even with more than 95\% outliers among the 619 correspondences (e.g. \textit{Scene-09}, $\mathit{s}\in(1,5)$); however, the two RANSAC solvers may fail in nearly a half of the cases. Hence, in addition to its superior performance on point cloud registration problems as shown in standard benchmarking, ICOS also demonstrates strong practicality for tackling real-world problems.

\section{Conclusion}

In this paper, we present a novel, fast and highly robust paradigm, ICOS, applicable to the rotation search and both the known-scale and unknown-scale point cloud registration problems with correspondences. We define special compatible structures for the two problems above to generate invariants with mathematical constraints. Then, we are able to render a sequence of well-designed invariant-constrained random sampling and inlier seeking frameworks for these problems, which can eliminate outliers by means of fast Boolean conditions and search for qualified $a$-COS as well as their corresponding $b$-COS as the inliers, giving birth to our solver ICOS. Generally, ICOS can automatically stop sampling, and eventually return nearly 100\% of the inliers out of the putative correspondences.

We benchmark ICOS against other state-of-the-art rotation search and point cloud registration solvers in multiple experiments. We show that: (i) ICOS can be the most state-of-the-art, or at least one of the most stat-of-the-art, solver for the two problems, since it is robust against over 95\% outliers (at most up to 99\%), has the best estimation accuracy and runs fast in most cases, and (ii) ICOS demonstrates the similar robustness (robust against 99\% outliers) and accuracy regardless of the scale situation, suitable for generalized point cloud registration, and (iii) ICOS has proved to be practical and efficient for real-world applications problems including image stitching and 3D object localization.

Future works may include extending the framework of ICOS or its inspirations to other geometric vision or robotics problems, as a general-purpose solver.

\ifCLASSOPTIONcaptionsoff
  \newpage
\fi

{\small
\bibliographystyle{ieee_fullname}
\bibliography{egbib}
}





\end{document}